\documentclass{article}

\usepackage{xcolor} 
\usepackage[preprint]{neurips_2026}

\usepackage[utf8]{inputenc} %
\usepackage[T1]{fontenc}    %
\usepackage{hyperref}       %
\usepackage{url}            %
\usepackage{booktabs}       %
\usepackage{amsfonts}       %
\usepackage{nicefrac}       %
\usepackage{microtype}      %
\usepackage{amsmath}
\usepackage{amssymb}
\usepackage{mathtools}
\usepackage{amsthm}
\usepackage{thmtools}
\usepackage[capitalize,noabbrev]{cleveref}
\usepackage{xargs}
\usepackage{natbib} 
\usepackage{datatool}
\usepackage{tikz}
\usepackage{xparse}
\usepackage{dsfont}
\usepackage{mathrsfs}
\usepackage{algorithm}
\usepackage{algorithmic}
\usepackage{wrapfig}
\usepackage{placeins}
\usepackage{subcaption}
\usepackage{multirow}
\usepackage[dvipsnames]{xcolor}

\DeclareMathOperator*{\argmax}{arg\,max}
\DeclareMathOperator*{\argmin}{arg\,min}

\newcommand{\E}{\mathbb{E}}
\newcommand{\prob}{\mathbb{P}}

\newcommand{\R}{\mathbb{R}}
\newcommand{\spaceg}{G^*}
\newcommand{\simplexg}{\ensuremath{\bar{G^*}}}
\newcommand{\localrad}[2]{\ensuremath{R_\ndata(#1, #2)}}
\def\state{\overrightarrow{X}} %

\def\ausiliar{\rho}
\def\barg{\bar{g}}

\def\rmd{\mathrm{d}} %

\def\dbrown{\rmd\!\brown}
\def\brown{\operatorname{B}}
\def\rset{\mathbb{R}}
\def\rsetpos{\mathbb{R}_{\geq 0}}
\def\statedim{d}

\def\eqsp{\;}

\def\realint{\int_{\R^d}}

\newcommand{\nofrac}[2]{#1/#2}
\newcommand{\kl}[2]{\operatorname{D}_{\operatorname{KL}}\left(#1||#2\right)}

\newcommand{\dens}[2]{#1\left(#2\right)}

\newcommand{\fwmarg}[1]{\vec{p}_{#1}} %
\newcommand{\fwmargd}[2]{\dens{\vec{p}_{#1}}{#2}} %

\newcommand{\mudata}{\mu_{\text{data}}}
\newcommand{\hypref}[1]{\textbf{H}\ref{#1}}
\newcommand{\PE}[2]{\mathbb{E}_{#1}\left[#2\right]} %
\newcommand{\PP}[1]{\mathbb{P}\left(#1\right)} %

\newcommand{\gausspdf}{\varphi}

\newcommand{\eqdef}{:=}
\newcommand{\param}{\theta}
\newcommand{\paramdim}{d_\theta}
\def\paramsp{\Theta}

\newcommand{\pforward}[1]{\vec{
p}_{#1}}

\newcommand{\trainedsolution}[1]{X^{\param}_{#1}}
\newcommand{\trainedsolutioninit}[2]{X^{\theta,#2}_{#1}}
\newcommand{\trainedscorex}[1]{s_{\param}(T- #1, \back{#1})}
\newcommand{\trainedscoreforward}[1]{s_{\param}(t, \state_{#1})}

\newcommand{\pforwardx}[1]{\vec{p}_{#1}}

\newcommand{\trainedscore}[1]{s_{\param}(T- #1, \trainedsolution{#1})}

\newcommand{\back}[1]{\overleftarrow{X}_{#1}}
\newcommand{\backscoretilde}[1]{\nabla \log \tilde{p}_{T - #1} (\back{#1})}

\newcommand{\backscore}[1]{\nabla \log \pforward{T - #1} (\back{#1})}

\newcommand{\scorep}[1]{\nabla \log \pforward{#1}}

\newcommand{\discsolution}[2][\param]{p_{#2}^{#1}}
\newcommand{\schedulerVE}{g}
\newcommand{\information}{\mathcal{I}}
\newcommand{\pinf}{\pi_{\infty}}

\newcommand{\ind}{\mathds{1}}
\newcommand{\Id}{\mathrm{I}}
\newcommand{\scfc}{\mathsf{s}}

\newcommand{\initerror}{\mathrm{E}_{\operatorname{init}}}
\newcommand{\trainerror}{\mathrm{E}_{\operatorname{train}}}
\newcommand{\discerror}{\mathrm{E}_{\operatorname{disc}}}

\newcommand{\calF}{\mathcal{F}}

\def\ndata{n}

\makeatletter
\NewDocumentCommand{\distn}{g g}{%
  d_{\ndata}\IfValueTF{#1}{\left( #1 , #2 \right)}{}%
}
\def\maxray{r_{\ndata}}
\NewDocumentCommand{\distnparam}{g g}{%
  \bar{d}_{\ndata}\IfValueTF{#1}{\left( #1 , #2 \right)}{}%
}
\newcommand{\nballs}[3]{N(#1, #2, #3)}

\newcommand{\emin}[1]{\widehat{#1}^{\ndata}}
\newcommand{\tmin}[1]{#1^{\star}}
\NewDocumentCommand{\mfun}{o o}{\operatorname{m}\IfValueTF{#1}{(#2; #1)}{}}
\NewDocumentCommand{\zeromfun}{o o}{\operatorname{g}_{\IfValueTF{#1}{#1}{}}\IfValueTF{#2}{\left(#2\right)}{}}
\NewDocumentCommand{\eintmfun}{o o}{\operatorname{M}_{\ndata}\IfValueTF{#2}{\IfValueTF{#1}{(#2; #1)}{(#2)}}{}}
\NewDocumentCommand{\tintmfun}{o}{\operatorname{M}_{\star}\IfValueTF{#1}{(#1)}{}}
\NewDocumentCommand{\empmeas}{o o}{\mu_{\operatorname{emp}}\IfValueTF{#1}{\IfValueTF{#2}{\left(#1; #2\right)}{\left(#1\right)}}{}}
\NewDocumentCommand{\lpnorm}{m o o}{\left\|#1\right\|_{\operatorname{L}^{#2}(#3)}}
\NewDocumentCommand{\onorm}{m o}{\left\|#1\right\|_{\psi_#2}}
\NewDocumentCommand{\flow}{o o}{\operatorname{T}\IfValueTF{#1}{_{#1}}{}\IfValueTF{#2}{(#2)}{}}
\NewDocumentCommand{\jflow}{o o}{\operatorname{J}_{\operatorname{T}\IfValueTF{#1}{_{#1}}{}}\IfValueTF{#2}{(#2)}{}}
\NewDocumentCommand{\invflow}{o o}{\operatorname{T}^{-1}\IfValueTF{#1}{_{#1}}{}\IfValueTF{#2}{(#2)}{}}
\NewDocumentCommand{\jinvflow}{o o}{\operatorname{J}_{\operatorname{T}^{-1}\IfValueTF{#1}{_{#1}}{}}\IfValueTF{#2}{(#2)}{}}

\NewDocumentCommand{\mlip}{o}{\operatorname{b}\IfValueTF{#1}{\left(#1\right)}{}}

\def\refnfmeas{u}
\newcommand{\pushforward}[2]{{#1}_{\# #2}}
\newcommand{\worstconcentrationvar}[2]{\operatorname{Z}_{\ndata}(#1, #2)}

\newcounter{hypH}
\newenvironment{hypH}{
    \refstepcounter{hypH}
    \begin{itemize}
    \item[{\bf H\arabic{hypH}}]
    }
{\end{itemize}}
\def\densratioub{\overline{r}}
\newcommand{\indi}[1]{\ind_{#1}}

\title{Generalizing Score-based generative models for Heavy-tailed Distributions}

\author{%
  Tiziano Fassina \\
  STIM, Mines Paris \\
  Paris, France \\
  \texttt{tiziano.fassina@minesparis.psl.eu} \\
  \And
  Gabriel Cardoso \\
  STIM, Mines Paris \\
  Paris, France \\
  \texttt{gabriel.victorino\_cardoso@minesparis.psl.eu} \\
  \AND 
  Sylvain Le Corff \\
  LPSM, Sorbonne Université \\
  Paris, France \\
  \texttt{sylvain.le\_corff@sorbonne-universite.fr} \\
  \And
  Thomas Romary \\
  STIM, Mines Paris \\
  Paris, France \\
  \texttt{thomas.romary@minesparis.psl.eu} \\
}

\begin{document}
\theoremstyle{plain}
\newtheorem{theorem}{Theorem}[section]
\newtheorem{proposition}[theorem]{Proposition}
\newtheorem{lemma}[theorem]{Lemma}
\newtheorem{corollary}[theorem]{Corollary}
\theoremstyle{definition}
\newtheorem{definition}[theorem]{Definition}
\newtheorem{assumption}[theorem]{Assumption}
\theoremstyle{remark}
\newtheorem{remark}[theorem]{Remark}
\crefname{assumption}{assumption}{assumptions}
\maketitle

\begin{abstract}
Score-based generative models (SGMs) have achieved remarkable empirical success, motivating their application to a broad range of data distributions. 
However, extending them to heavy-tailed targets remains a largely open problem. 
Although dedicated models for heavy-tailed distributions have been proposed, their generative fidelity remains unclear and they lack solid theoretical foundations, leaving important questions open in this regime.
In this paper, we address this gap through two theoretical contributions. First, we show that 
combining early stopping with a suitable initialization is sufficient to extend the diffusion 
framework to any target distribution; in particular, we establish the well-posedness of 
the backward process and prove convergence of the approximated diffusion in KL divergence. 
Second, we derive novel theoretical guarantees for generation with normalizing flows, obtaining 
convergence results that hold under mild conditions on the flow family and without any assumption 
on the tail behavior of the target distribution. Building on these results, we propose a unified 
generative framework for heavy-tailed distributions: a normalizing flow is first trained to capture 
the tail behavior and is then used as an initialization prior for an SGM, which refines the samples 
by recovering fine-grained structural details. This design leverages the complementary strengths of 
the two model classes within a theoretically principled pipeline, overcoming the limitations of existing approaches.
\end{abstract}

\section{Introduction}
Score-based generative models (SGMs)~\citep{ho2020denoising, song2021scorebased} have become central in machine learning for sampling high-dimensional distributions, achieving state-of-the-art results in image and video generation, outperforming GANs, VAEs, and Normalizing Flows~\citep{dhariwal2021diffusionbeatGan, muller2022medfusion}. They underpin major models such as Stable Diffusion~\citep{rombach2022high}, DALL·E~3~\citep{Betker2023dalle3}, Imagen~3~\citep{Baldridge2024Imagen3}, and FLUX~\citep{blackForestLabs_FLUX1dev_2025}, and extend effectively to audio and text generation~\citep{chen2020wavegrad, kong2021diffwave, li2022diffusionlm, zou2023survey}.

The core idea of SGMs is twofold: the data distribution is gradually corrupted by Gaussian noise, and a neural network is trained to reverse this process. New samples are then generated by starting from Gaussian noise and iteratively denoising it with the trained network. The empirical success of SGMs has motivated their application to increasingly challenging regimes. A setting of growing interest is the generation of heavy-tailed data~\citep{bingham1987regular, resnick2007heavy}. Accurately modeling extreme events and distributional tails is central in climate science, insurance, and finance, where data routinely exhibit heavy-tailed behavior~\citep{mardani2025residual, beirlant2004statistics}.

To our knowledge, the existing literature on SGMs for heavy-tailed distributions consists of the works~\citep{yoon2023levy, shariatian2025heavy, pandey2025heavytailed}, all built on the same intuition: replacing the Gaussian noising process with a heavy-tailed one. While these methods report empirical improvements in tail generation, the evidence available so far is mostly limited to specific settings, and a systematic assessment of their effectiveness across distributions is still missing. Their theoretical justification also remains partial, which is particularly relevant in the cited applications, where statistical accuracy rather than visual quality is the primary concern.

On the theoretical side, existing analyses of SGMs aim to prove convergence of the approximate backward process to the target and to quantify its rate, under assumptions on the initialization, the score, and the SDE approximation. Such guarantees typically rely on strong conditions on the data distribution: bounded support~\citep{debortoli2022convergence}, log-concavity~\citep{gao2025wasserstein}, or finite moments of high order~\citep{conforti2023kl}. Related analyses impose analogous hypotheses~\citep{bruno2024wasserstein, silveri2025beyond, bortoli2021diffusion, strasman2025analysis}. Heavy-tailed distributions violate most of these assumptions by design: log-concavity fails, high-order moments may not exist, and the support is typically unbounded. As a result, it remains unclear under which conditions the approximate backward process converges to the target when the latter is heavy-tailed. To our knowledge, the only theoretical work linking SGMs and heavy tails~\citep{yu2026diffusion} addresses the convergence of the empirical score to the true score, leaving the global convergence question open.

In this paper, building on~\citep{conforti2023kl, strasman2025analysis}, we argue that the key to extending SGMs to arbitrary, and in particular heavy-tailed, targets lies in the use of two complementary features: a flexible initialization of the backward process and early stopping.
A proper initialization for the backward denoising is necessary because the Gaussian is not always a suitable approximation of the noised distribution. Specifically, for heavy-tailed distributions, the noised distribution remains heavy-tailed at all finite times, whereas the Gaussian is inherently light-tailed, so any finite-time initialization based on it is bound to mismatch the true noised distribution in the tails. Early stopping, in turn, is beneficial because convolving any distribution with small Gaussian noise induces strong regularity properties.

These findings naturally suggest a two-stage pipeline: a generative model tailored to heavy-tailed distributions provides the initialization, while an SGM refines the fine-scale structure via early-stopped denoising. This does not merely shift the problem from learning a distribution through diffusion to learning an appropriate initialization: as shown in~\citep{silveri2025beyond}, the noising process acts as a progressive smoothing that simplifies the target distribution, eventually reaching a unimodal regime. Replacing the standard Gaussian prior with a heavy-tailed model capable of capturing simple, unimodal heavy-tailed distributions is therefore a tractable objective.

The idea of improving generation through non-standard initialization has been explored from different perspectives: \citet{zheng2023truncated} learn a prior at a truncated diffusion time via an Adversarial Autoencoder, \citet{zand2024diffusion} use a Glow Normalizing Flow to approximate the noised distribution at an arbitrary diffusion step, and \citet{DBLP:journals/corr/abs-2306-12360} combine a score-based model with an Energy-Based Model at a single noise level; further works include~\citep{lee2022priorgrad, shen2025information, scholz2026warm}. While these methods aim at improving generation efficiency and sample quality, our focus differs fundamentally: we enable SGMs for heavy-tailed distributions and extend the theoretical foundations of the framework itself.

The landscape of generative models for heavy-tailed data extends well beyond SGMs, encompassing dedicated approaches based on Normalizing Flows~\citep{mcdonald2022cometflows, hickling2025flexibletails, laszkiewicz2022marginal}, VAEs~\citep{lafon2023vae}, and GANs~\citep{allouche2022evgan, huster2021pareto}. Among these candidates for the initialization stage, Normalizing Flows stand out as a natural choice, thanks to their tractable likelihood and compositional structure. However, the theoretical understanding of Normalizing Flows in the heavy-tailed regime is still limited. Existing analyses have mainly focused on universal approximation properties of invertible neural networks and on convergence rates for standard normalizing flows~\citep{teshima2020couplinguniversal, ishikawa2023universal, huang2020augmented, koehler2021representational, bogachev2005triangular}, under assumptions that typically exclude heavy-tailed targets. Moreover, several works have documented the inability of Lipschitz normalizing flows with Gaussian priors to model heavy-tailed distributions~\citep{jaini2020, laszkiewicz2022marginal}, yet no convergence guarantees are currently available in this setting. Building on~\citep{tang2021empirical}, which establishes convergence to the class minimum for variational autoencoders via M-estimation, we show that analogous results hold for normalizing flows across a broad class of distributions, including heavy-tailed ones.

\paragraph{Contributions.} Specifically, this paper makes the following contributions.
\begin{itemize}
\item We prove that early stopping with a flexible initialization replacing the standard Gaussian prior suffices to guarantee well-posedness of the backward diffusion and convergence of the approximate process to the target in KL divergence, \emph{without any assumption on the target distribution}, generalizing~\citet[Theorem~1]{conforti2023kl}. 
\item We show that, under mild assumptions on the target and on the flow class, Normalizing Flows can provably learn heavy-tailed distributions, providing a theoretically grounded choice for the initialization stage and extending~\citet[Theorem~2]{tang2021empirical}.
\item We instantiate the resulting two-stage pipeline, Normalizing Flow initialization followed by SGM refinement, on synthetic data exhibiting both heavy tails and multimodality, providing a theoretically principled framework for heavy-tailed generative modeling.
\end{itemize}

\section{Convergence of Variance Exploding SGMs for General Distributions}
We study in this section the convergence of the approximate backward process to the target distribution in the Variance Exploding (VE) framework, without imposing any assumption on the data distribution. 
Let $(\Omega, \calF, \prob)$ be a probability space and let $\mudata$ denote the unknown data distribution on $\R^d$.
\subsection{Forward and Backward Process}
Given $\state_0 \sim \mudata$ the distribution of interest and the standard forward Variance Exploding (VE) equation \cite{song2021scorebased}
\begin{equation}\label{eq:VE_SDE}
\rmd \state_t = \schedulerVE_t \, \dbrown_t\eqsp.
\end{equation}
where $\schedulerVE_t : \R^+ \to \R^+$ is assumed to be bounded, increasing and infinitely differentiable, we can state by applying \citet[Theorem 2.5, Theorem 2.9]{karatzas1991brownian} without any assumption on the data distribution that the forward solution $(\state_t)_{t\in[0,T]}$ is defined on the whole interval for $T>0$.
Defining 
\[
\sigma_t^2 \coloneqq \int_0^t \schedulerVE_s^2 \, \rmd s\eqsp.
\]
the solution satisfies 
\[
\state_t = \state_0 + \int_{0}^{t} \schedulerVE_s \, \dbrown_s  \overset{\mathcal{L}}{=} \state_0 + \sigma_t Z \eqsp . 
\]
In particular, for $t>0$, $\state_t$ admits a density $\pforward{t}(x)$ such that $ (t,x)\mapsto\pforward{t}(x) \in C^{\infty}([\delta, T] \times \R^d)$, as shown in \cref{lem:regularity_VE}. 

To introduce the backward process, the usual references such as \cite{anderson1982reverse, song2021scorebased} requires the base distribution to be almost surely bounded (i.e., supported on a ball of finite radius), which, of course, excludes most distributions, including heavy-tailed ones.
To handle this setting, we instead rely on \cite{haussmann1985time}, which provides time-reversal results under weaker integrability conditions. We verify in \cref{lem:existence-backward} that these conditions are satisfied for any distribution $\mudata$ on any interval $[\delta, T]$ with $\delta > 0$.
Thanks to this result, we can state that there exists a Brownian motion $\tilde{\brown}_t$ in the defined probability space such that $\back{} = \{\back{t}\}_{t \in [0, T-\delta]} = \{\state_{T-t}\}_{t \in [\delta, T]}$ is a solution to
\begin{equation}\label{eq:VE_back_sde}
    \rmd \back{t} = \barg_t^2 \backscore{t} \rmd t + \barg_t \, \rmd \tilde{\brown}_t \eqsp,
\end{equation}
where $\barg_t \coloneqq \schedulerVE_{T-t}$.
\subsection{KL control on approximated process.}
In practice, $\mudata$ is unknown and only i.i.d. samples $\{\state_0^i\}_{i=1,...,\ndata} \sim \mudata$ are available, rendering analytical formulas for both the score and the terminal distribution $\pforwardx{T}$ inaccessible, and the backward SDE \cref{eq:VE_back_sde} must be solved numerically.
For the purpose of theoretical analysis, we may rely on either the Euler–Maruyama scheme \cite{gao2025wasserstein, silveri2025beyond} or the Exponential Euler–Maruyama scheme \cite{conforti2023kl, strasman2025analysis}. In the VE setting these two schemes coincide, allowing for a unified treatment of the KL bound.
\paragraph{Score approximation and discretization.}
By replacing the real score function $\nabla \log \pforward{T - t}$ using a trained neural network $s_{\param}(T- t,\cdot)$, we
introduce the discretized backward dynamics
\begin{equation}
\rmd \trainedsolution{t}=  \barg_t^2 \trainedscore{t_k} \, \rmd t + \barg_t \, \rmd\tilde{\brown}_t\eqsp,
\end{equation}
for $t \in [t_k, t_{k+1}]$, where $t_{k+1} = t_k + h_{k+1}$, with $t_0 = 0$ and $t_N = T-\delta$.
Integrating over $[t_k, t_{k+1}]$, with initialization $\trainedsolution{0} \sim \discsolution[\param]{0}$, yields
\begin{equation}\label{eq:discsolution}
\trainedsolution{t_{k+1}}= \trainedsolution{t_{k}} + (\sigma^2_{T-t_{k}}-\sigma^2_{T-t_{k+1}}) \trainedscore{t_k} + \left(\sigma^2_{T-t_{k}}-\sigma^2_{T-t_{k+1}}\right)^{\frac{1}{2}}Z \eqsp,
\end{equation}
$k = 0,...,N-1$, $Z \sim \mathcal{N}(0, \Id_d)$. 
We stress that the use of an independent Gaussian noise $Z$ in \cref{eq:discsolution} is not immediate: \cref{eq:VE_back_sde} has a specific Brownian motion adapted to a particular filtration, and the equivalence with a discretization driven by fresh Gaussian increments requires a non-trivial argument. We address this point in \cref{lem:existence-backward}.

For probability measures $\mu, \nu$, recall the KL divergence $\kl{\mu}{\nu} = \int \log \frac{\rmd \mu}{\rmd \nu} \, \rmd \mu$ and the Fisher information $\information(\nu) := \int \| \nabla \log \frac{\rmd \nu}{\rmd \lambda_{\R^d}} \|^2 \, \rmd \nu$. 
\paragraph{KL bound for Variance Exploding SGMs.}
We make the following two assumptions. The first concerns  the initialization. 
\begin{hypH}
\label{assum:kl_finite}
The initialization satisfies $\kl{\pforward{T}}{\discsolution[\param]{0}} < \infty$.
\end{hypH}
The second assumption requires the score to be well approximated uniformly over the discretization grid:
\begin{hypH}
\label{assum:train}
There exists $\varepsilon > 0$ such that for all $0 \leq k \leq N-1$,
\[
    \E\Big[\|\backscore{t_k} - s_{\param}(T-t_k, \back{t_k})\|^2\Big] \leq \varepsilon \eqsp.
\]
\end{hypH}
We also set $\scfc_t(x) \coloneqq \nabla \log \pforward{t}(x)$ and, for $0 \leq k \leq N-1$ and $t \in [0,T]$, $\gamma_k \coloneqq \int_{t_k}^{t_{k+1}} \barg_t^2 \, \rmd t$.
\begin{theorem}\label{thm:KL_bound}
For any distribution $\mudata$, for all $\delta > 0$,
\[
    \kl{\pforward{\delta}}{\discsolution{T-\delta}} \leq \initerror(\theta) + \trainerror(\theta) + \discerror \eqsp,
\]
where
\begin{align}
    \initerror(\theta) &= \kl{\pforward{T}}{\discsolution{0}} \label{eq:err:mixVE} \eqsp, \\
    \trainerror(\theta) &= \sum_{k=0}^{N-1} \gamma_k \, \E\Big[\|\scfc_{T-t_k}(\back{t_k}) - s_\theta(T-t_k,\back{t_k})\|^2\Big] \label{eq:err:approxVE} \eqsp, \\
    \discerror &= \sum_{k=0}^{N-1} \int_{t_k}^{t_{k+1}} \barg_t^2 \, \E\Big[\|\scfc_{T-t_k}(\back{t_k}) - \scfc_{T-t}(\back{t})\|^2\Big] \, \rmd t \label{eq:err:discrVE} \eqsp . 
\end{align}
Moreover, the Fisher information of $\pforward{t}$ is finite $\forall t>0$ and $\discerror$ is upper bounded by
\begin{equation}\label{eq:information}
    \discerror' = \max_{k=0,\dots,N-1} \gamma_k \left\{ \information(\pforward{\delta}) - \information(\pforward{T}) \right\} \eqsp.
\end{equation}
Thanks to \hypref{assum:train}, the bound simplifies to
\begin{equation}\label{eq:thm_simplified}
    \kl{\pforward{\delta}}{\discsolution{T-\delta}} \leq \kl{\pforward{T}}{\discsolution{0}} + \sigma_T^2 \, \varepsilon + \discerror' \eqsp.
\end{equation}
\end{theorem}
Crucially, \hypref{assum:kl_finite} and \hypref{assum:train} concern \emph{only} the score network and the initialization family. No regularity, moment, or support condition is imposed on $\mudata$. 
\paragraph{The role of initialization}
By \cref{thm:KL_bound}, the initialization error $\kl{\pforward{T}}{\discsolution{0}}$ directly controls the final approximation error. In the heavy-tailed regime,
 a Gaussian initialization is not admissible, as $\pforward{T}$ 
 remains heavy-tailed and \hypref{assum:kl_finite} fails. 
 This motivates the use of flexible initialization schemes with 
 explicit control over tail behavior; we focus on Normalizing Flows~\citep{rezende2015variational, dinh2017realnvp}, which admit dedicated constructions for heavy-tailed distributions~\citep{hickling2025flexibletails, mcdonald2022cometflows, laszkiewicz2022marginal} and whose statistical analysis as initialization estimators is the subject of the next section.

\section{Flexible Prior Initialization for SGMs through Normalizing Flows}
By \cref{thm:KL_bound}, the approximation error critically depends on the initialization $\discsolution{0}$, which we choose by minimizing $\kl{\pforward{T}}{\cdot}$ over a parametric family of Normalizing Flows given samples from $\pforward{T}$. Adapting the M-estimation arguments of~\citet{tang2021empirical} from Variational Autoencoders, we derive non-asymptotic excess-risk guarantees and characterize how the tail behavior of $\pforward{T}$ enters the statistical complexity.
\paragraph{Estimation problem and Normalizing Flows.}
Given a parametric family $\{\discsolution[\param]{0}\}_{\param \in \paramsp}$, we aim to solve
\[
\tmin{\param} \in \argmin_{\param \in \paramsp} \kl{\pforward{T}}{\discsolution[\param]{0}} \eqsp ,
\]
which is equivalent to
\[
\tmin{\param} \in \argmin_{\param \in \paramsp} \E\big[\mfun[\param][\state_T]\big],
\quad \text{where} \quad
\mfun[\param][x] = \log \pforward{T}(x) - \log \discsolution[\param]{0}(x).
\]
In practice, given i.i.d.\ samples $\{\state_T^i\}_{i=1}^{\ndata}$ from $\pforward{T}$, this reduces to maximum likelihood estimation:
\begin{equation}\label{eq:empirical_problem}
    \emin{\param} \in \argmax_{\param \in \paramsp} \ndata^{-1}\sum_{i=1}^{\ndata} \log \discsolution[\param]{0}(\state_T^i).
\end{equation}
We model $\discsolution[\param]{0}$ as a Normalizing Flow. Let $\refnfmeas$ be a reference density on $\R^d$ and let $\flow: \R^d \times \paramsp \to \R^d$ be a family of invertible transformations. The model distribution is the pushforward $\discsolution[\param]{0} = \pushforward{\flow[\param]}{\refnfmeas}$, with log-density given by the change of variables formula:
\[
\log \discsolution[\param]{0}(x) = \log \refnfmeas(\invflow[\param][x]) - \log \det \jflow[\param][x].
\]
This formulation provides explicit control over both the likelihood and the tail behavior through the choice of $\refnfmeas$ and $\flow[\param]$.
\subsection{KL non asymptotic control for Normalizing Flows}
\paragraph{Non-asymptotic guarantees via M-estimation.}
The statistical difficulty of~\cref{eq:empirical_problem} depends on how well the model class captures the tails of $\pforward{T}$. One way to quantify the tails of a distribution is through the Orlicz norm~\citep[Section 5.6]{wainwright2019highdimensional}, which is defined for a random variable $X$ as
\[
\onorm{X}[1] = \inf \left\{ K > 0 : \E \left[ \exp \left( \frac{|X|}{K} \right) \right] \leq 2 \right\}.
\]
We introduce two assumptions. The first controls the compatibility between the model class and the tails of $\pforward{T}$; the second is a standard regularity condition on the parametric family.
\begin{hypH}
\label{assump:tang2021empirical:A}
There exists $D > 0$ such that
\[
    \onorm{\sup_{\param \in \paramsp}|\mfun[\param][\state_T]|}[1] \leq D.
\]
\end{hypH}
\begin{hypH}
\label{assump:tang2021empirical:C1}
$\paramsp \subset \R^{\paramdim}$ is compact, and there exists $\mlip: \R^d \to \R_+$ such that for all $x$ and all $\param, \tilde{\param} \in \paramsp$,
\[
    |\mfun[\param][x] - \mfun[\tilde{\param}][x]| \leq \mlip[x]\|\param - \tilde{\param}\|_2,
\]
with $\onorm{\mlip[\state_T]}[1] < \infty$.
\end{hypH}
Under these assumptions, we obtain the following non-asymptotic guarantee on the excess risk.
\begin{theorem}
\label{thm:tang2021empirical:thm3main}
Under \hypref{assump:tang2021empirical:A} and \hypref{assump:tang2021empirical:C1}, there exist constants $c_0, c_1, c_2 > 0$ such that, with probability
$ 1 - c_0 \exp\Bigl(-c_1 \min \left\{ \paramdim \log(\ndata \paramdim) , \frac{\sqrt{\ndata \paramdim \log(\ndata \paramdim)}}{D \log \ndata} \right\}\Bigr)$ it holds
\begin{equation}\label{eq:nf_bound}
    \kl{\pforward{T}}{\discsolution[\emin{\theta}]{0}}
    \leq \inf_{\gamma > 0} \Big\{ 2(1+\gamma)\, \kl{\pforward{T}}{\discsolution[\tmin{\theta}]{0}} + c_2 (1+\gamma^{-1}) \frac{\paramdim D \log(\ndata) \log(\ndata \paramdim)}{\sqrt{\ndata}} \Big\}.
\end{equation}
\end{theorem}
The bound decomposes the excess risk into two terms. The first measures how well the best-in-class model $\discsolution[\tmin{\theta}]{0}$ approximates $\pforward{T}$; it vanishes whenever $\pforward{T}$ belongs to the model class, in which case~\cref{eq:nf_bound} reduces to the statistical term and the estimator is consistent. The second term scales as $\paramdim D / \sqrt{\ndata}$ up to logarithmic factors, with $\gamma > 0$ trading off the two contributions.
\paragraph{Tail complexity and the choice of the flow family.}
The constant $D$ in \hypref{assump:tang2021empirical:A} captures the \emph{tail complexity} of the estimation problem: heavier tails of $\pforward{T}$ relative to those of the model class yield larger $D$ and slow down convergence.
Conversely, $D$ remains controlled whenever the model class successfully matches the tails of $\pforward{T}$.
We arrived to derive in \cref{appendix:sec:estimating_D} detailled upper bounds for $D$ in the case of sub-exponential and sub-gaussian $\mudata$ distributions when for the case of a RealNVP flow \citep{dinh2017realnvp} respectively in \cref{prop:nf:onormratio_bound:heavy} and \cref{prop:nf:onormratiobound:gaussian}.

The bound of \cref{prop:nf:onormratiobound:gaussian}  decrease with the increase of the time-horizon $\sigma_T^2$, showing that indeed it is easier to learn a flow for $\pforward{T}$ as $T$ grows.
This is in itself expected, as the forward diffusion can be interpreted as a smoothing operator: $\pforward{T}$ is a regularized version of $\mudata$, and is therefore typically closer to a tractable parametric family than $\mudata$ itself. 
This makes the approximation problem on $\pforward{T}$ substantially more tractable than directly modeling $\mudata$, rendering the design of a sufficiently expressive flow class a realistic objective. 
In the ideal case where $\pforward{T}$ belongs to the model class, the approximation error vanishes and the estimator $\discsolution[\emin{\theta}]{0}$ is consistent.

\subsection{Choice of $T$ and training scheme for the Flow}
We propose training a flow on the noised distribution $\pforward{T}$
and recovering the data distribution by running a standard diffusion
denoiser over the final $T$ steps. The horizon $T$ should be
small enough that $\pforward{T}$ remains close to $\pforward{0}$,
yet large enough that $\pforward{T}$ is unimodal, where
flows perform reliably. 
\paragraph{Choosing $T$ via empirical unimodality.}
We exploit the fact that joint unimodality implies unimodality of
every one-dimensional projection, and that unimodality, once reached,
is preserved for all $T' > T$.
In order to find unimodality we propose the following method : given samples $\{\state_T^i\}_{i=1}^{\ndata}$, we draw a battery of
random unit directions, project the samples onto each, and use
Hartigan's dip test~\citep{hartigan1985dip} to locate, per direction,
the smallest horizon $T$ for which the projection is accepted as
unimodal. The worst case across directions yields our estimate
$\sigma_T$, validated on a held-out battery of directions.
\paragraph{Two training schemes for the flow.}
\cref{eq:empirical_problem} approximates the optimization problem of
$\kl{\pforward{T}}{\discsolution[\param]{0}}$ up to an additive
constant on a fixed noised dataset
$\{\state_0^i + \sigma_T Z^{(i)}\}_{i=1}^{\ndata}$.
Equivalently, the same KL objective can be empirically reformulated as 
\[
  \argmax_{\param \in \paramsp} \sum_{i=1}^{\ndata}
    \E_{Z}\!\left[
      \log \discsolution[\param]{0}\!\left(\state_0^i + \sigma_T Z\right)
    \right],
  \quad Z\sim \mathcal{N}(0, \mathbf{I}),
\]
which suggests resampling $Z$ at every epoch instead of
fixing the noised dataset upfront, an approach reminiscent of data
augmentation. We compare both variants, that we name fixed and dynamic
(\cref{alg:training}), empirically in \cref{sec:experiments_toy}.
\begin{algorithm}[t]
\caption{\textbf{Fixed-Noise} (left) vs.\ \textbf{Dynamic-Noise} (right) Flow Training.}
\label{alg:training}
\begin{minipage}[t]{0.48\linewidth}
\begin{algorithmic}
\STATE {\bfseries Input:} $\mathcal{D} = \{\mathbf{x}_0^{(i)}\}_{i=1}^N$, noise level $\sigma_T$
\STATE {\bfseries Output:} Optimized $\discsolution[\param]{0}$
\STATE Sample $\mathbf{z}^{(i)} \sim \mathcal{N}(0, \mathbf{I})$, $i=1,\dots,N$
\STATE Build $\mathcal{D}_T = \{\mathbf{x}_0^{(i)} + \sigma_T \mathbf{z}^{(i)}\}_{i=1}^N$
\FOR{each epoch}
    \STATE Update $\param$ on $-\log \discsolution[\param]{0}(\mathbf{x}_T)$
\ENDFOR
\end{algorithmic}
\end{minipage}%
\hfill\hfill
\begin{minipage}[t]{0.48\linewidth}
\begin{algorithmic}
\STATE {\bfseries Input:} $\mathcal{D} = \{\mathbf{x}_0^{(i)}\}_{i=1}^N$, noise level $\sigma_T$
\STATE {\bfseries Output:} Optimized $\discsolution[\param]{0}$
\FOR{each epoch}
\STATE Sample $\mathbf{z}^{(i)} \sim \mathcal{N}(0, \mathbf{I})$, $i=1,\dots,N$
\STATE Build $\mathcal{D}_T = \{\mathbf{x}_0^{(i)} + \sigma_T \mathbf{z}^{(i)}\}_{i=1}^N$
    \STATE Update $\param$ on $-\log \discsolution[\param]{0}(\mathbf{x}_T)$ 
\ENDFOR
\end{algorithmic}
\end{minipage}
\end{algorithm}

\section{Numerical Illustration}\label{sec:experiments_toy}
We present a synthetic experiment illustrating the potential of our 
approach. The score of $\pforward{t}$ is unavailable in closed form for 
any heavy-tailed targets, so we consider two distributions to isolate 
different aspects of our method. We consider a 
Gaussian mixture model (GMM), for which the score is analytically tractable: this lets us isolate and study the
initialization. We then consider 
\begin{wrapfigure}{r}{0.51\textwidth} 
    \centering
    \begin{minipage}{0.25\textwidth}
        \centering
        \includegraphics[width=\textwidth]{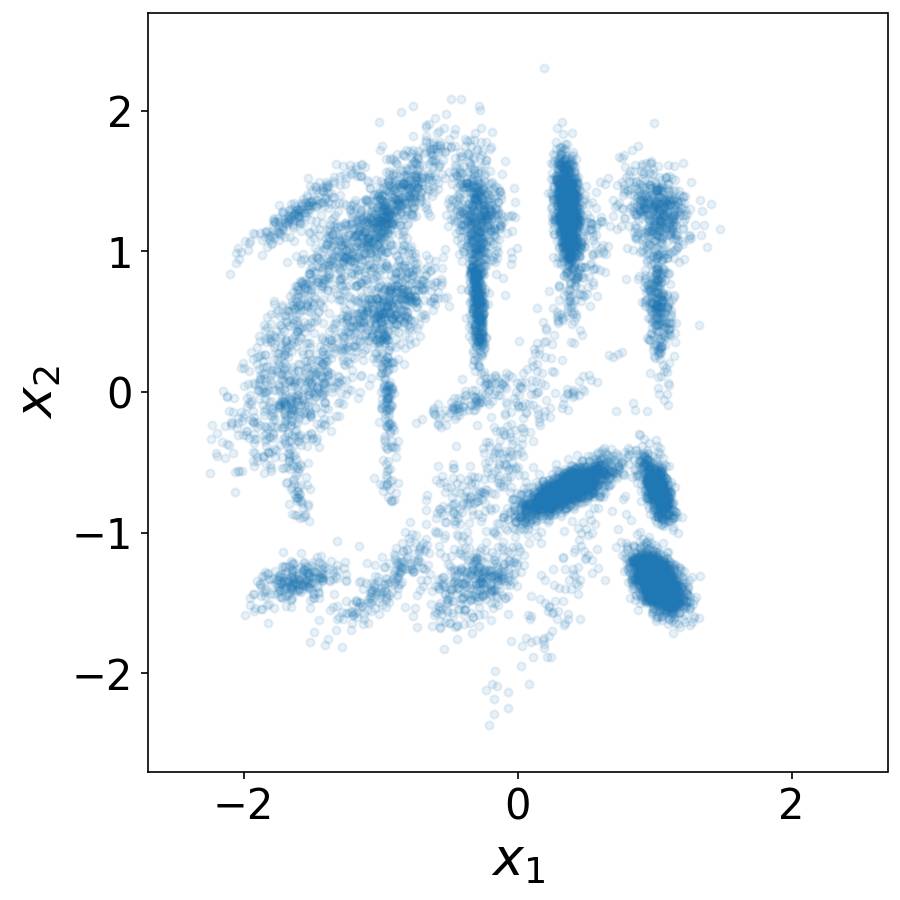}
    \end{minipage}
    \hfill %
    \begin{minipage}{0.25\textwidth}
        \centering
        \includegraphics[width=\textwidth]{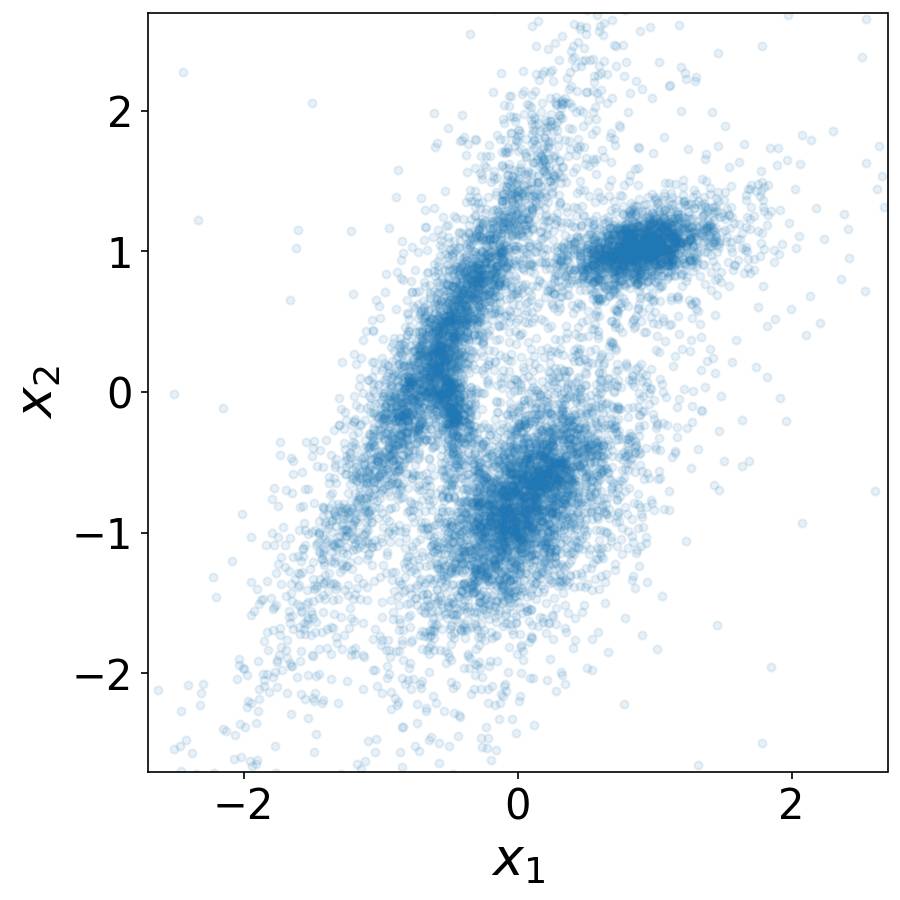}
    \end{minipage}
    \caption{GMM (left) and HT (right) targets.}
    \label{fig:gmm_ht_targets}
\end{wrapfigure}
a heavy-tailed (HT) target, where the score must be approximated. For HT we evaluate two families of estimators: MCMC samplers, namely HMC 
\cite{neal2011hmc}, the Barker proposal \cite{livingstone2020barker}, 
and NUTS \cite{hoffman2011nuts}, and a trained neural-network denoiser 
$\mathrm{NN}_{\theta}$ implemented as an MLP with EDM-style 
preconditioning. 
Using multiple MCMC samplers alongside the learned 
denoiser ensures that the observed gains are not an artefact of any 
specific approximation. We report NUTS and $\mathrm{NN}_{\theta}$ in 
the main text and defer HMC and Barker to 
\cref{app:appendix_experiments}.
We work in dimension $d=2$. The MCMC-based score estimator variance grows quickly with $d$, 
and for higher dimensions the resulting score 
becomes too noisy to support a clean isolation of the initialization 
effect. 
The GMM target is a 25-component mixture, with anisotropic covariances and 
Dirichlet-sampled weights. The HT target
is a mixture of $n=4$ multivariate $t$-Student components with $\nu=3$ 
degrees of freedom, anisotropic covariances, and Dirichlet-sampled 
weights.
We use fewer components for the HT target, since many modes amplify 
MCMC-based score approximation bias. Both targets are made to have mean zero and global standard deviation $1$.

For both targets, denoising follows the EDM scheduler \cite{karras2022elucidating} with 
$\sigma_{\max}=3$, $\sigma_{\min}=2\!\times\!10^{-4}$, $\rho=2$, and 
$40$ discretization steps. 
We evaluate the method at intermediate truncation points $\sigma_T$ along this schedule.
We use a first-order 
stochastic sampler in the VE parameterization \cite{ho2020denoising, karras2022elucidating}.
The flow $\discsolution[\param]{0}$ is a 
coupling flow with triangular-affine layers and ReLU6 activations; the 
base distribution is a standard multivariate Gaussian for the GMM 
target and a multivariate $t$-Student with $\nu=3$ for the HT target.

Further details on MCMC, tail generation, and flow training, together with tests on real-world architectures~\cite{zhai2025normalizing} and datasets (FFHQ-64 and two subsets of ImageNet~\cite{karras2022elucidating,Karras2023TrainingDiffusion}) assessing the potential of this architecture for modeling $\pforward{T}$ in high-dimensional settings, are deferred to \cref{app:appendix_experiments}.

\subsection{The Role of Initialization}
\begin{figure*}[t]
    \centering
    \begin{subfigure}[b]{0.32\textwidth}
        \centering
        \includegraphics[width=\textwidth]{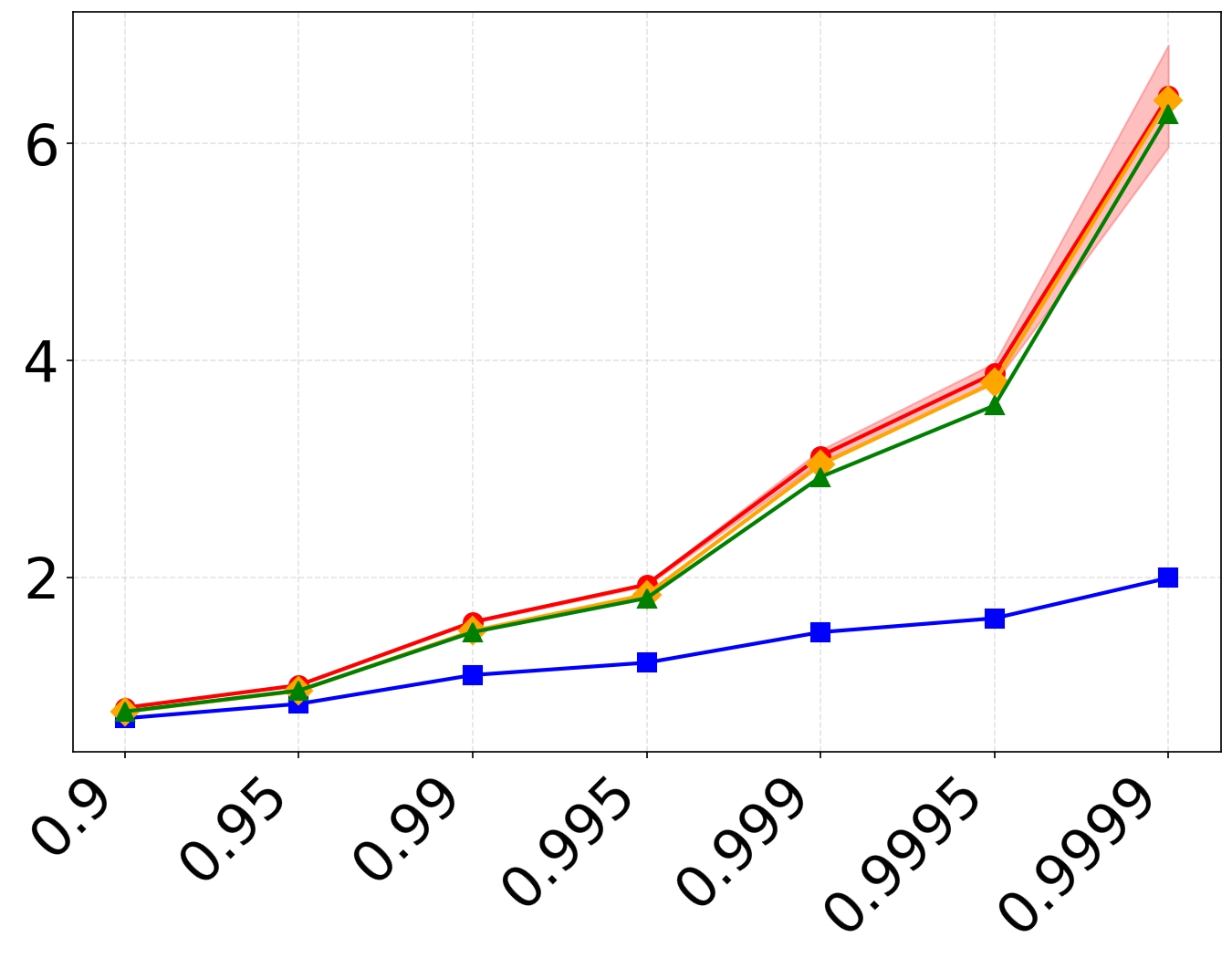}
        \caption{$\sigma_T = 0.801$}
    \end{subfigure}
    \hfill
    \begin{subfigure}[b]{0.32\textwidth}
        \centering
        \includegraphics[width=\textwidth]{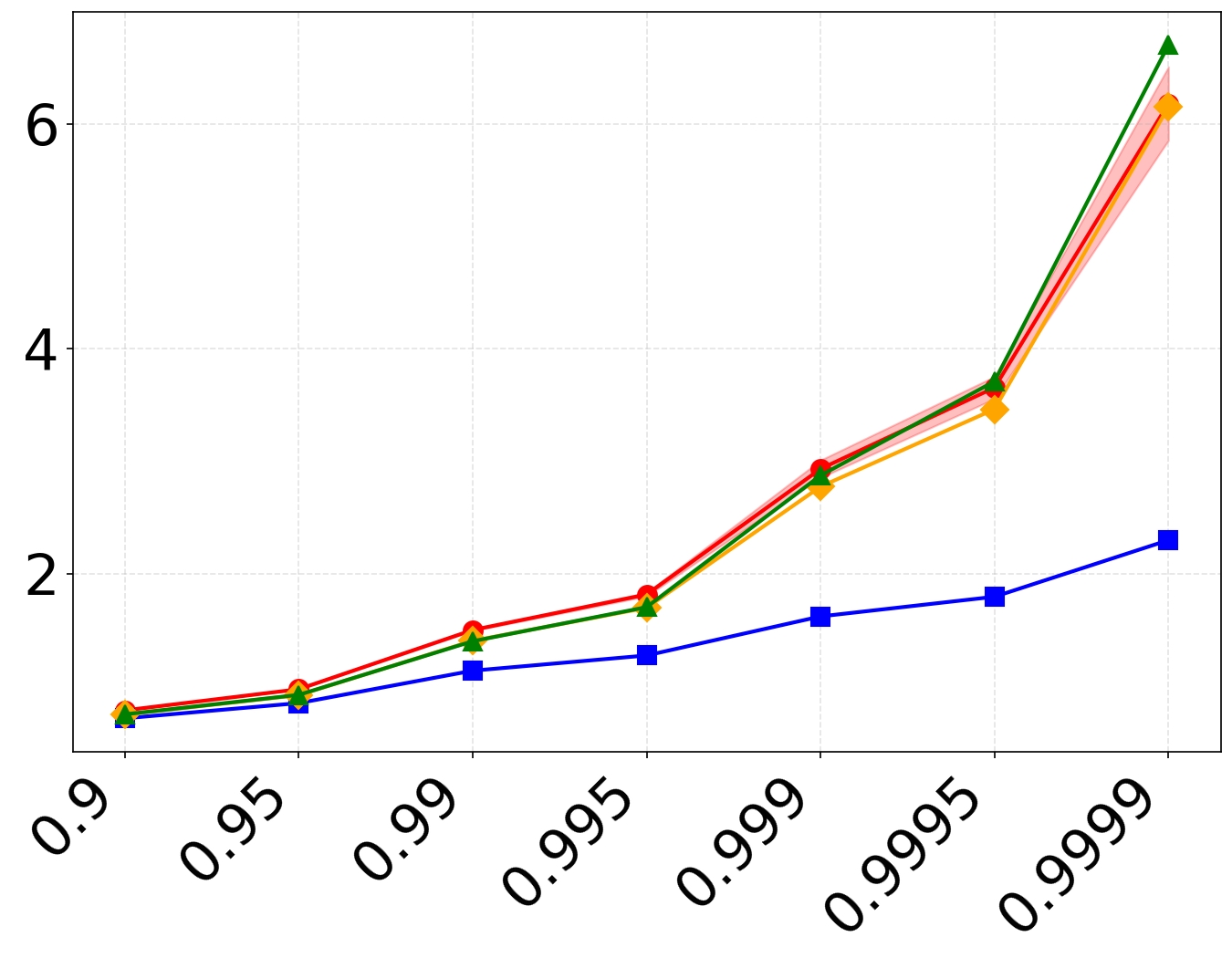}
        \caption{$\sigma_T = 1.055$}
    \end{subfigure}
    \hfill
    \begin{subfigure}[b]{0.32\textwidth}
        \centering
        \includegraphics[width=\textwidth]{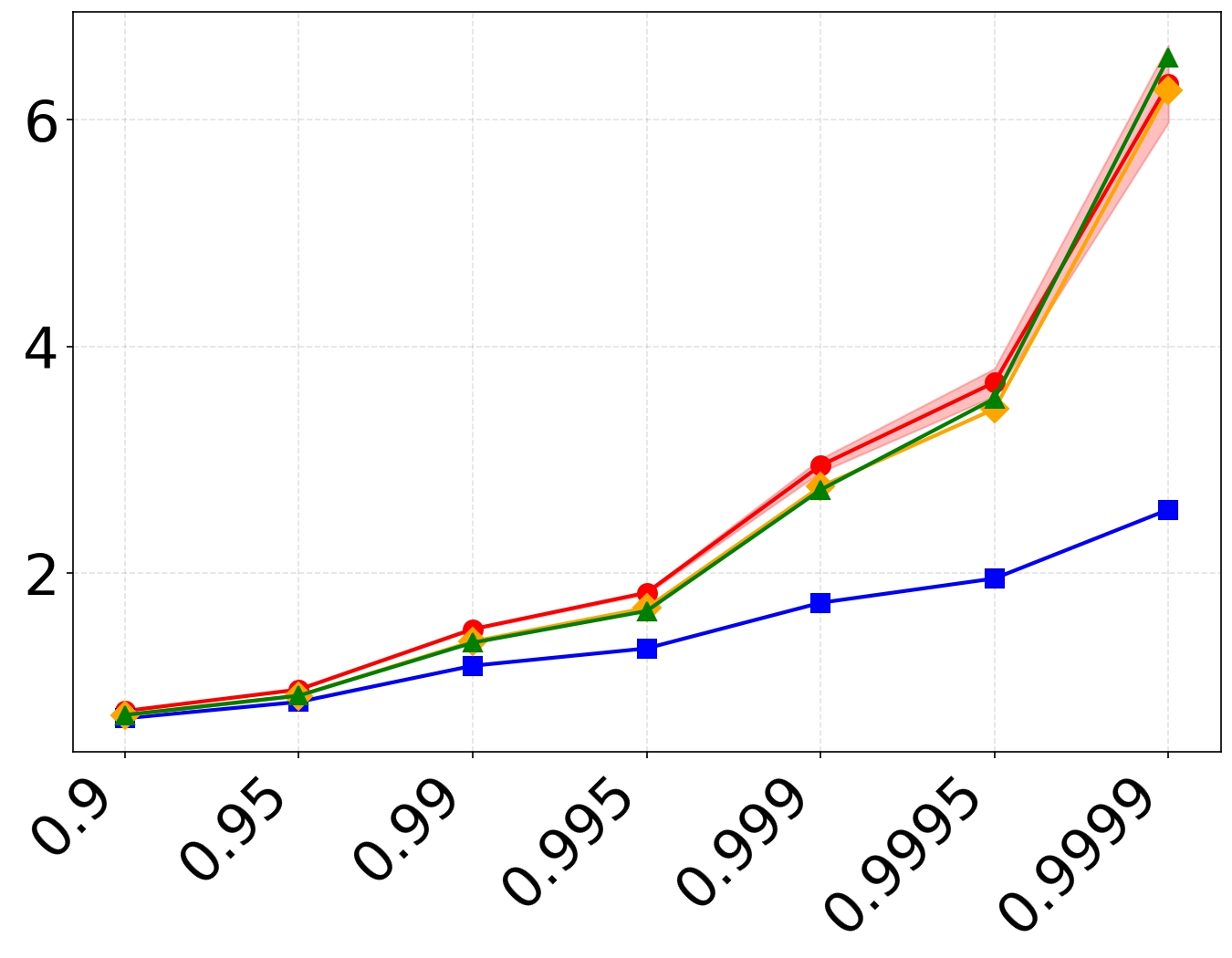}
        \caption{$\sigma_T = 1.244$}
    \end{subfigure}
    \caption{Tail quantiles (0.9--0.9999). \textcolor{blue}{$p_\infty$}, 
    \textcolor{orange}{$p_T$}, \textcolor{OliveGreen}{$p_\theta$}, \textcolor{red}{$\mudata$}. 
    NUTS denoiser.}
    \label{fig:ht_quantile_main_NUTS}
\end{figure*}
For both targets, we consider three noise levels $\sigma_T \in \{0.801, 1.055, 1.244\}$, 
corresponding to $\{21, 24, 26\}$ remaining EDM denoising steps. We 
compare four sampling strategies: denoising initialized from the standard 
Gaussian prior $\pinf \sim \mathcal{N}(0, \sigma_T^2 I_d)$, from $\pforward{T}$, or from $\discsolution[\param]{0}$, and 
direct sampling from a flow trained on clean data. For the first three strategies
we denoise with the analytic GMM score. The flow $\discsolution[\param]{0}$ 
and the flow trained on clean data are both trained on $10^4$ samples. $\discsolution[\param]{0}$ is trained in the dynamic-noise regime from \cref{alg:training}. 
For each strategy we generate $10^6$ samples. 

In the GMM case, we evaluate the global quality of the generated distribution
by computing the Max-Sliced Wasserstein (MSW) distance between generated data and 
$10^6$ reference samples, repeating the MSW computation $10$ times with 
independent draws from the target, and reporting mean and standard deviation. 
The MSW is performed via gradient-based optimization. 

In the HT case, we evaluate bulk and tail quality separately.
Bulk accuracy is measured by restricting both generated and reference samples to 
the $[0.1, 0.9]$ inter-quantile region along each dimension and computing the MSW 
distance between the two. We repeat the MSW computation $10$ times against 
independent reference draws and report mean and standard 
deviation over these runs.
Tail fidelity is assessed along a fixed worst-case projection direction, identified 
\begin{wraptable}{r}{0.65\textwidth}
  \centering
  \footnotesize
  \setlength{\tabcolsep}{4pt}
  \renewcommand{\arraystretch}{1.05}
   \caption{MSW (mean $\pm$ std). Global for GMM, bulk for HT. Flow baseline does not depend on $\sigma_T$}
  \label{tab:msw_combined}
  \begin{tabular}{cllccc}
    \toprule
    Target & Method & Init & $\sigma_T{=}0.801$ & $\sigma_T{=}1.055$ & $\sigma_T{=}1.244$ \\
    \midrule
    \multirow{4}{*}{GMM}
    & \multirow{3}{*}{Analytic}
      & $\pinf$       & $0.169\,{\pm}\,0.013$ & $0.105\,{\pm}\,0.007$ & $0.081\,{\pm}\,0.005$ \\
    & & $\discsolution[\param]{0}$ &  $0.027\,  \pm \, 0.004$ & $0.031 \,\pm\, 0.003$ & $0.034 \, \pm \, 0.003$ \\
    & & $\vec{p}_T$   & $0.024\,{\pm}\,0.004$ & $0.026\,{\pm}\,0.003$ & $0.030\,{\pm}\,0.003$ \\
    \cmidrule(lr){2-6}
    & Flow & -- & \multicolumn{3}{c}{$0.067\,{\pm}\,0.005$} \\
    \midrule
    \multirow{7}{*}{HT}
    & \multirow{3}{*}{NUTS}
      & $\pinf$       & $0.074\,{\pm}\,0.001$ & $0.037\,{\pm}\,0.006$ & $0.025\,{\pm}\,0.003$ \\
    & & $p_\theta$    & $0.014\,{\pm}\,0.002$ & $0.014\,{\pm}\,0.002$ & $0.012\,{\pm}\,0.004$ \\
    & & $\vec{p}_T$   & $0.012\,{\pm}\,0.002$ & $0.012\,{\pm}\,0.003$ & $0.012\,{\pm}\,0.003$ \\
    \cmidrule(lr){2-6}
    & \multirow{3}{*}{$\mathrm{NN}_\theta$}
      & $\pinf$       & $0.074\,{\pm}\,0.001$ & $0.035\,{\pm}\,0.009$ & $0.027\,{\pm}\,0.001$ \\
    & & $\discsolution[\param]{0}$    & $0.012\,{\pm}\,0.001$ & $0.011\,{\pm}\,0.002$ & $0.012\,{\pm}\,0.001$ \\
    & & $\vec{p}_T$   & $0.011\,{\pm}\,0.002$ & $0.011\,{\pm}\,0.002$ & $0.011\,{\pm}\,0.002$ \\
    \cmidrule(lr){2-6}
    & Flow & -- & \multicolumn{3}{c}{$0.015\,{\pm}\,0.001$} \\
    \bottomrule
  \end{tabular}
\end{wraptable}

as follows: for each $\sigma_T$, we run the MSW maximization between $10^6$ samples from 
$\pforward{T}$ and $10$ independent reference draws, retaining the direction achieving the 
highest MSW. Along this direction, we report quantiles at levels $\{0.9, ..., 0.9999\}$ 
for all sampling strategies and compare them against the reference mean $\pm 2$ standard deviations, 
estimated from $10$ independent reference draws.

MSW results are reported jointly in \cref{tab:msw_combined}; 
tail quantile plots are shown in \cref{fig:ht_quantile_main_NUTS} for NUTS and \cref{fig:ht_quantile_main_NN} for $\mathrm{NN}_{\theta}$.
Our method outperforms the flow on the bulk in the GMM case (where the strong multimodality challenges the flow), achieving better mode reconstruction. In the HT case, we still observe a bulk improvement over the flow, although the gap is smaller because the HT distribution has a simpler structure. Our method also yields substantially improved tail reconstruction compared to standard diffusion, and the bulk quality improves as a consequence: mass that $\pinf$ fails to place in the tail is otherwise redistributed into the bulk, degrading it.
\subsection{Flow training}
We investigate how the interplay between training data size and noise level affects the learning dynamics of normalizing flows, empirically validating the theoretical results of \cref{thm:tang2021empirical:thm3main} and the smoothing effect of noise (\cref{appendix:sec:estimating_D}). 
\begin{wrapfigure}[13]{r}{0.71\textwidth}
    \centering
    \begin{subfigure}[t]{0.35\textwidth}
        \centering
        \includegraphics[width=\textwidth]{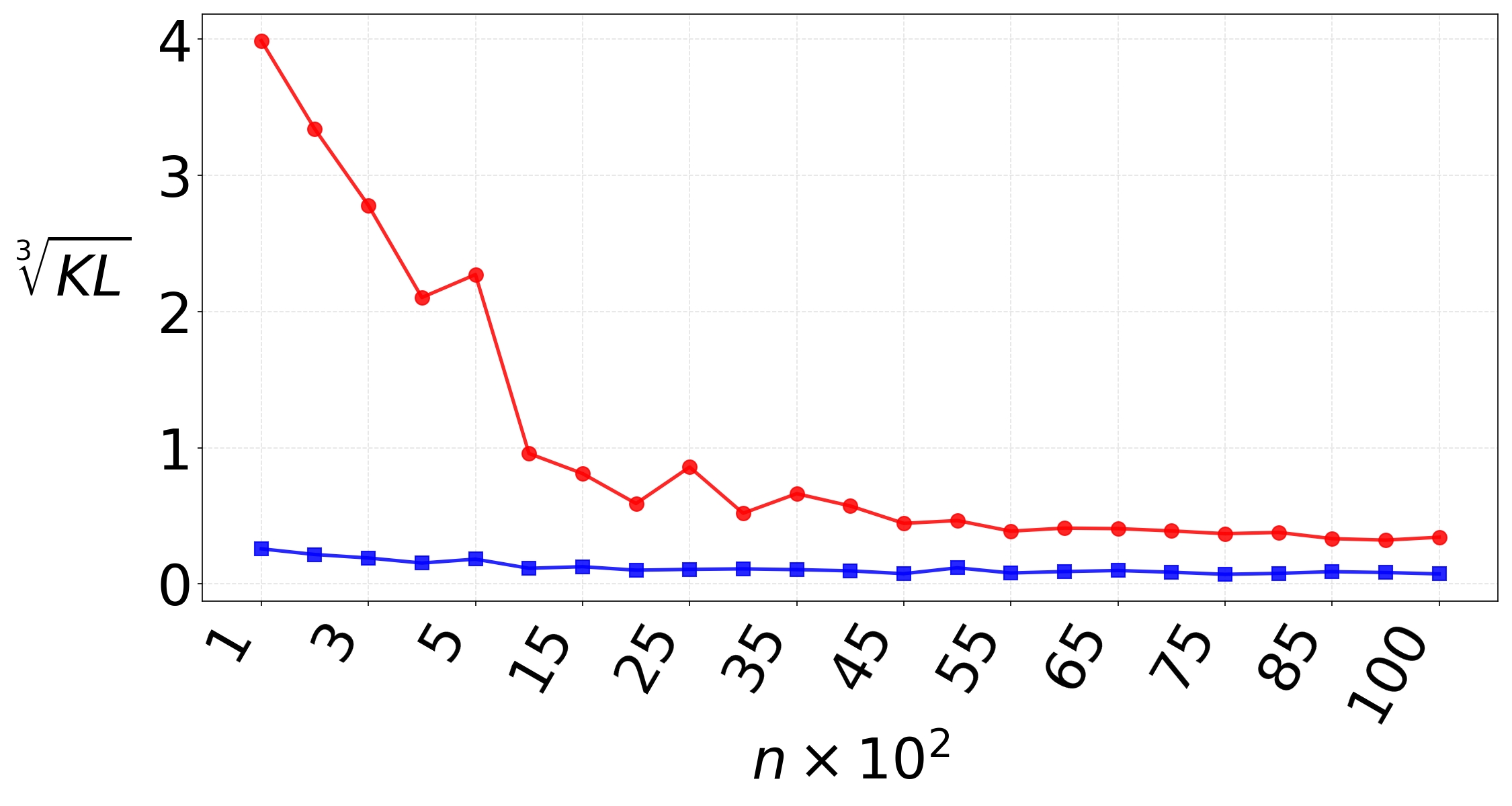}
        \caption{KL vs. training size at $\sigma = 1.055$ for \textcolor{blue}{dynamic} and \textcolor{red}{fixed} strategies.}
        \label{fig:kl_trainsize}
    \end{subfigure}
    \hfill
    \begin{subfigure}[t]{0.35\textwidth}
        \centering
        \includegraphics[width=\textwidth]{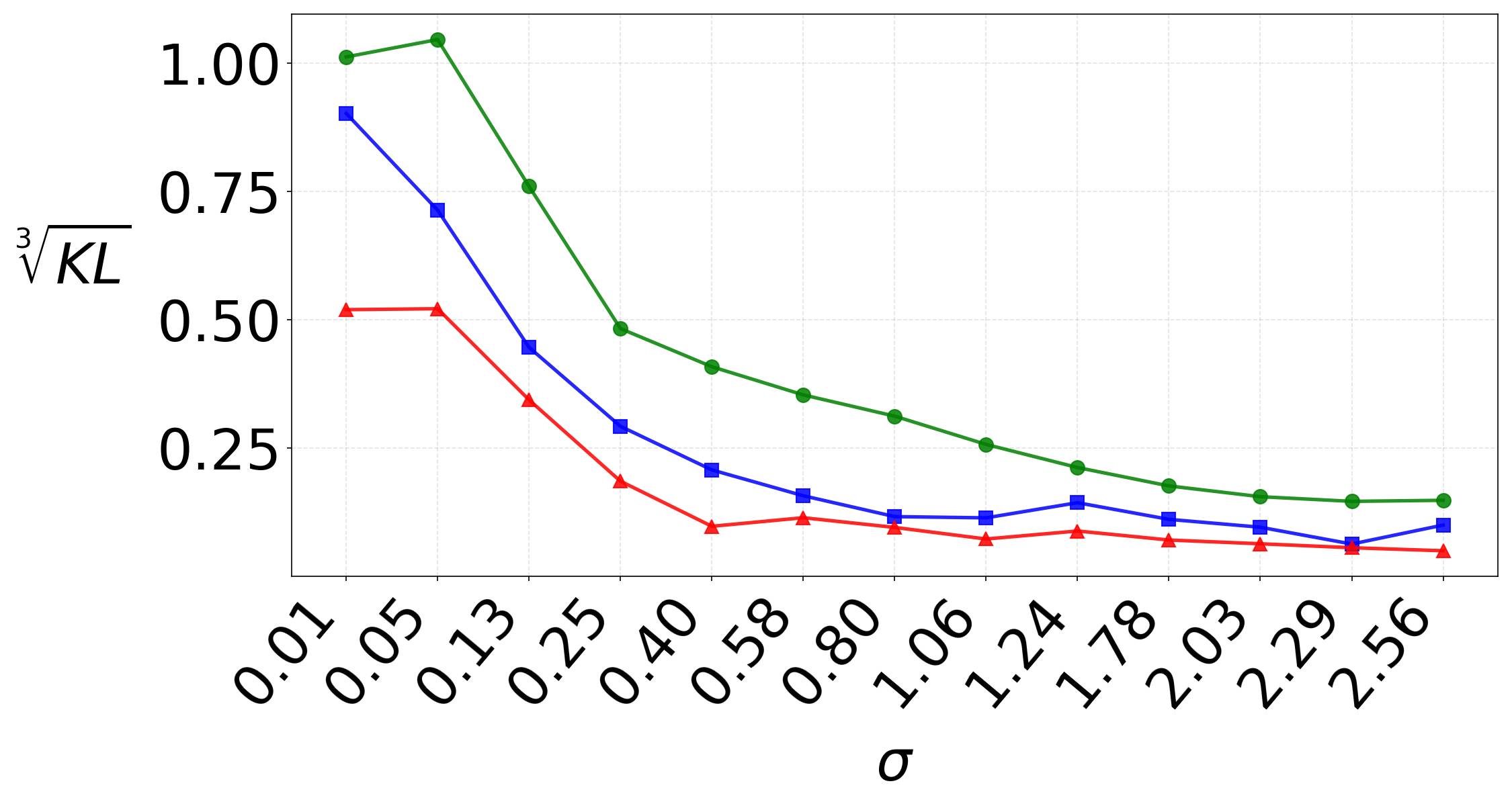}
        \caption{KL vs. noise level $\sigma$ for dynamic training, $\ndata \in \{\textcolor{OliveGreen}{10^2}, \textcolor{blue}{10^3}, \textcolor{red}{10^4}\}$.}
        \label{fig:kl_sigma}
    \end{subfigure}
    \caption{ $\kl{\pforward{T}}{\discsolution[\emin{\param}]{0}}$ for different $\ndata$ and $\sigma_T$.}
    \label{fig:flow_training}
\end{wrapfigure}
Specifically, we train flows on noised GMM samples for varying training sizes $\ndata$ and noise levels $\sigma_T$ and we evaluate $\kl{\pforward{T}}{\discsolution[\emin{\param}]{0}}$ (for GMM, $\pforward{T}$ is in closed form) using standard Monte Carlo.
We measure both the effect of increasing training set size and varying noise levels. 
\cref{fig:kl_trainsize} shows the KL divergence versus training size at fixed noise level $\sigma = 1.055$ for both training modalities in \cref{alg:training} (fixed and dynamic noise). The results reveal the strong regularization effect of dynamic noise training, with consistent improvement for both strategies as data increases.
\cref{fig:kl_sigma} shows the KL divergence versus noise level $\sigma_T$ for dynamic training across different training sizes $\ndata \in \{10^2, 10^3, 10^4\}$. The results demonstrate improved learning on smoothed distributions. We visualize the KL divergence as its cube root for improved readability. 
Recall also that the KL divergence is invariant under invertible reparameterizations; in particular, it is unaffected by scaling of the variables.
Further details on flow training and tests on real-world architectures and datasets are deferred to \cref{app:appendix_experiments}.
\begin{figure*}[t]
    \centering
    \begin{subfigure}[b]{0.32\textwidth}
        \centering
        \includegraphics[width=\textwidth]{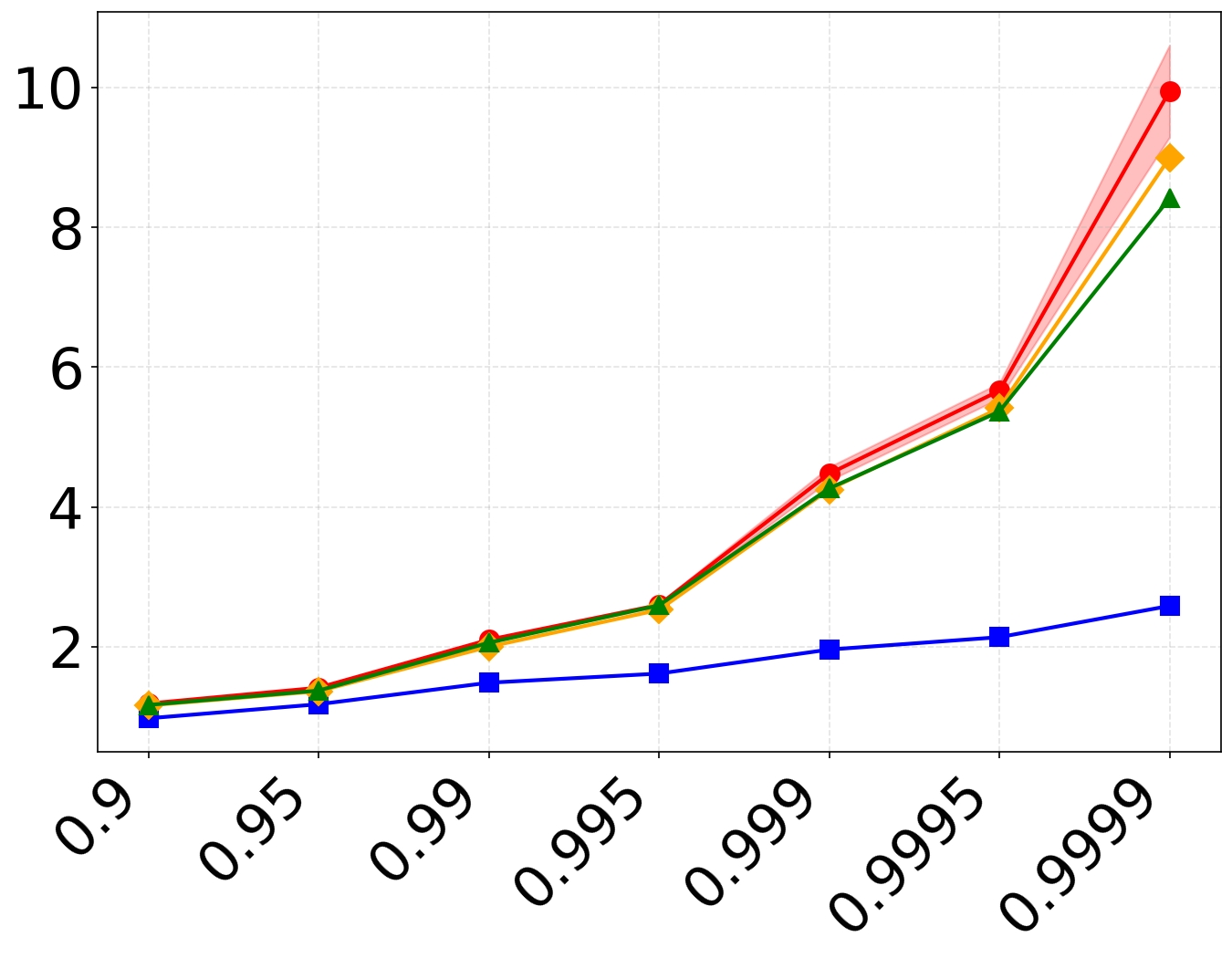}
        \caption{$\sigma_T = 0.801$}
    \end{subfigure}
    \hfill
    \begin{subfigure}[b]{0.32\textwidth}
        \centering
        \includegraphics[width=\textwidth]{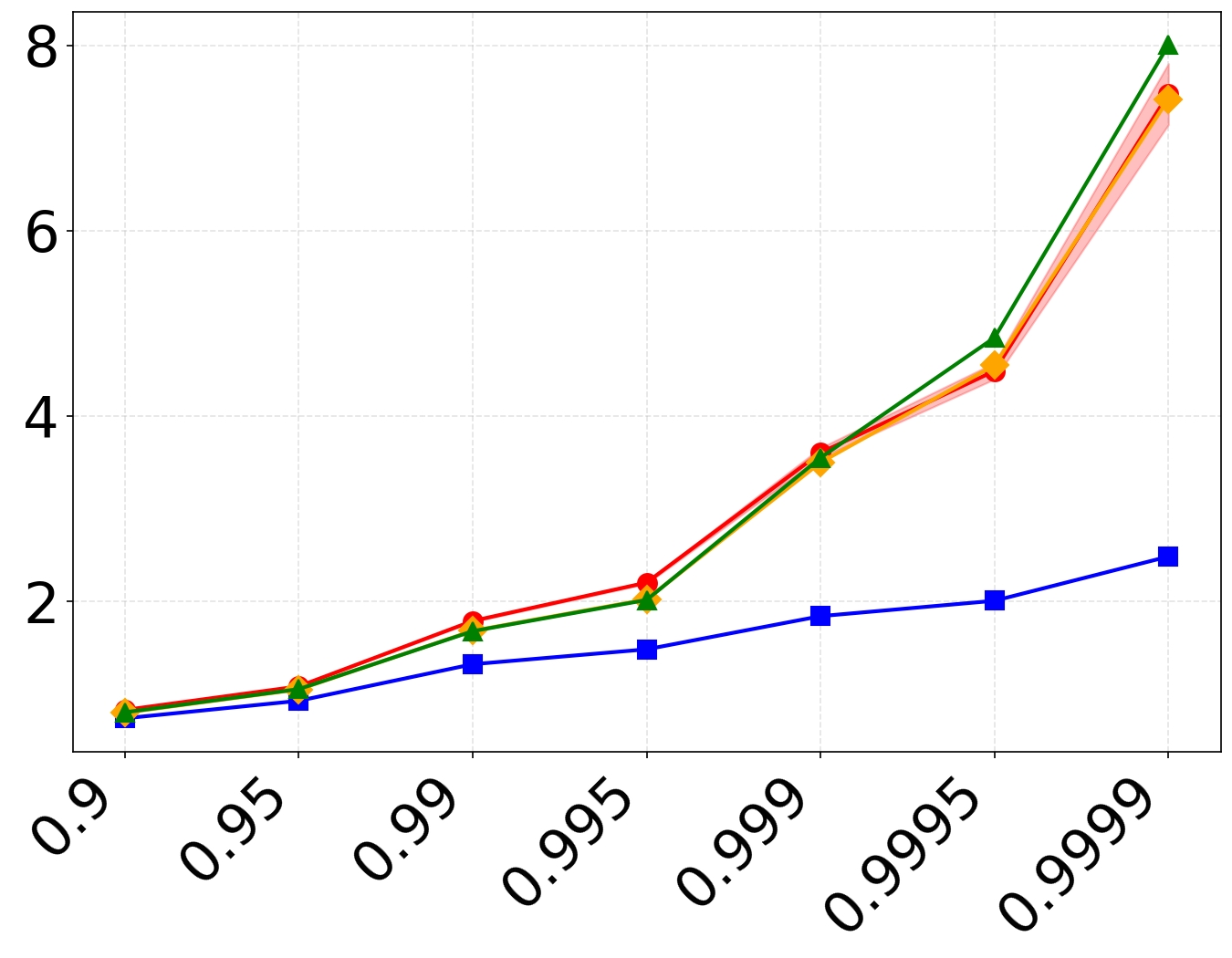}
        \caption{$\sigma_T = 1.055$}
    \end{subfigure}
    \hfill
    \begin{subfigure}[b]{0.32\textwidth}
        \centering
        \includegraphics[width=\textwidth]{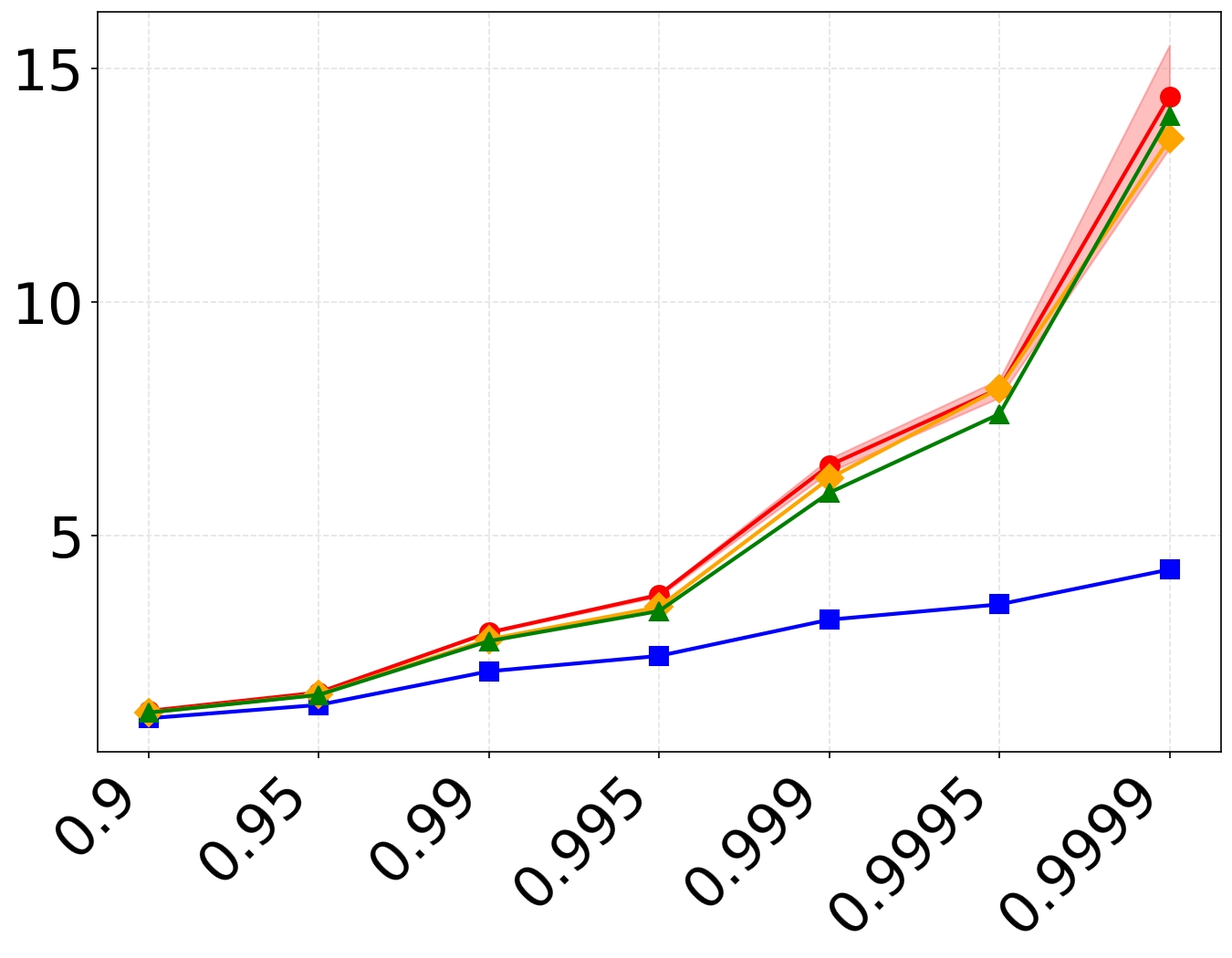}
        \caption{$\sigma_T = 1.244$}
    \end{subfigure}
    \caption{Tail quantiles (0.9--0.9999). \textcolor{blue}{$p_\infty$}, 
    \textcolor{orange}{$p_T$}, \textcolor{OliveGreen}{$p_\theta$}, \textcolor{red}{$\mudata$}. 
    Trained $\text{NN}_{\theta}$ denoiser, $10^4$ samples.}
    \label{fig:ht_quantile_main_NN}
\end{figure*}

\section{Discussion and Conclusion}
We introduced a theoretical framework for diffusion denoising that operates without
assumptions on the target distribution, identifying initialization and early stopping
as the two ingredients that decouple the convergence guarantees from the regularity of
the data. As a corollary, score-based generative models become applicable to heavy-tailed
targets, a property of growing relevance in scientific domains where diffusion models are
increasingly used. Along the way, we also established a convergence result for normalizing
flows that holds across a broad class of target distributions, including heavy-tailed ones.

This work provides a principled justification for initialization-aware sampling, a framework
previously employed heuristically, and opens several directions for future research:
(i) understanding when and how the denoiser can be accurately learned under heavy-tailed
targets; (ii) designing normalizing flows better suited to learning smoothed heavy-tailed
distributions; and (iii) deepening the theoretical understanding of flow training on
convolved distributions.

Concerning (i), our synthetic experiments show that for short-horizon diffusion sampling, a conventionally
trained denoiser achieves strong performance in reproducing both the bulk and the tail of
the target distribution we chose. However, whether this extends beyond the specific setting
considered (time horizon, network architecture, and low-dimensional regime) remains an
open question requiring further investigation.

Concerning (ii), the literature on normalizing flows for heavy-tailed distributions is already broad, yet most
existing approaches have been validated only on low-dimensional toy datasets.
Among existing methods, Comet Flow~\cite{mcdonald2022cometflows} stands out for its
tail-aware design: it maps data into a tractable domain $[0,1]^d$ before fitting a flow,
achieving sophisticated tail reproduction, though its main limitation lies in handling complex
multimodal or discrete distributions.
Our framework substantially simplifies the target distribution that the flow must learn.
Combining this with a tail-aware architecture such as Comet Flow, further empowered by a
high-capacity backbone such as TarFlow~\cite{zhai2025normalizing}, may represent a viable
path toward unconditional heavy-tailed generation in high-dimensional and image-like settings.

Concerning (iii), our experiments suggest that the dynamic-noise training strategy consistently improves flow
quality over fixed-noise training. We conjecture that noise-based smoothing benefits
training not only by regularizing the target density, but also by enabling effective data
augmentation. A theoretical understanding of this gap remains an open question.

\bibliographystyle{plainnat} 
\bibliography{biblio}
\clearpage

\section{Proofs of Results of \cref{thm:KL_bound}}
\label{appendix:proofs}

In this section, we provide the proofs of our main results. The presentation is organized as follows.
\begin{itemize}
    \item Theorem~\ref{thm:kl-control} in Section~\ref{sec:klcontrol} provides a bound on the KL divergence between solutions of SDEs. Variants of this result have been used in prior work (see, e.g., \cite{conforti2023kl,strasman2025analysis}). Since the precise statement and proof are not always presented in full detail in the literature, we provide a complete and self-contained proof tailored to our framework.    \item We discuss the existence and regularity of the forward and backward solutions in Section~\ref{sec:technical:res} and establish Lemmas~\ref{lem:regularity_VE} to \ref{lem:martingVE} required for the proof of our main result, Theorem~\ref{thm:KL_bound}.
    The explicit statement of the proof is also necessary to relax any sufficient hypothesis that is not necessary. 
    \item We prove Theorem~\ref{thm:KL_bound} in Section~\ref{sec:mainproof}.
    \item Additional results (Lemma~\ref{lem:conditional-properties} and Proposition~\ref{prop:KL-conditional}) required for the proof of Theorem \ref{thm:kl-control} are given in Section~\ref{sec:additional:technical}.
\end{itemize}

\paragraph{Additional notations. } If $\mu$ and $\nu$ are two measures on $ (\mathbb{R}^d, \mathcal{B}(\mathbb{R}^d))$, we write $\mu \ll \nu$ to state that $\mu$ is dominated by $\nu$. For all random variables $X$, we write $\mu_X$ the law of $X$ when there is no possible confusion.

\subsection{KL control of SDE solutions}
\label{sec:klcontrol}
\begin{theorem}(KL control for stochastic differential equation solutions)\label{thm:kl-control}
Consider the space of continuous functions
\( 
\mathcal{C}_T = \mathcal{C}([0,T],\mathbb{R}^d),
\)
endowed with the $\sigma$-algebra induced by the uniform convergence topology 
on the full product space $\prod_{t\in[0,T]} \R^d$. We consider the following SDE 
\begin{equation}\label{eq:SDE_X}
\rmd X_t = A_t(X)\,\rmd t + b_t\,\dbrown_t\eqsp,
\end{equation}
and
\begin{equation}\label{eq:SDE_Y}
\rmd Y_t = a_t(Y)\,\rmd t + b_t\,\dbrown_t\eqsp,
\end{equation}
where $A_t(X) = A(t,X_{[0,t]}), a_t(Y) = a(t,Y_{[0,t]})$ and $b_t$ is a strictly positive noise schedule.
We suppose the existence of strong solutions $X$ and $Y$ to respectively \cref{eq:SDE_X} and \cref{eq:SDE_Y} initialized with $X_0, Y_0$ on $ (\mathbb{R}^d, \mathcal{B}(\mathbb{R}^d))$ such that
\(
\mu_{X_0} \ll \mu_{Y_0}\eqsp.
\)
For all $x\in\R^d$, we also suppose the existence of $X^x$ and $Y^x$,  strong solutions to \cref{eq:SDE_X} and \cref{eq:SDE_Y} initialized at $x$. Assume that for $\mu_{X_0}$-almost every $x$
\begin{equation}
\label{eq:nov:shyriaev}
\mathbb{P}\!\left(
\int_0^T \frac{1}{b_s^2}
\bigl(\|A_s(X^x)\|^2 + \|a_s(X^x)\|^2\bigr)\,\rmd s
< \infty
\right) 
= 1\eqsp.
\end{equation}
Then, \(\mu_X \ll \mu_Y\)  and 
\begin{align*}
\kl{\mu_X}{\mu_Y} = \kl{\mu_{X_0}}{\mu_{Y_0}}  + \E \Bigg[\frac{1}{2} \int_0^T \frac{1}{b_s^2}
\,\!\|A_s(X) - a_s(X)\|^2\,\rmd s \Bigg].
\end{align*}
\end{theorem}
\begin{proof}
Using that the condition \( b_t > 0 \) ensures that \( (A_t(X) - a_t(X))/b_t \) is well defined, that there exist strong solutions to \eqref{eq:SDE_X} and \eqref{eq:SDE_Y} and by  \eqref{eq:nov:shyriaev}, for $\mu_{X_0}$-almost every $x$, we can apply \citet[Theorem 7.20]{liptser2001statistics} with $\xi = X^x$ and $\eta = Y^x$. Therefore,  $\mu_{X^x} \ll \mu_{Y^x}$ and
\[
\frac{\rmd \mu_{X^x}}{\rmd \mu_{Y^x}}(X^x)
= \exp\Biggl\{
\int_0^T \frac{A_t(X^x) - a_t(X^x)}{b_t^2} \, \rmd X^x_t - \int_0^T \frac{(A_t(X^x) - a_t(X^x))^T \cdot (A_t(X^x) + a_t(X^x))}{2 \, b_t^2} \rmd t
\Biggr\}\eqsp.
\]
The conditional law of $X$ given $X_0 = x$, 
written \ $\mu_{X\mid X_0}(\cdot \mid x)$, coincides with $\mu_{X^x}$, 
and similarly $\mu_{Y\mid Y_0}(\cdot \mid x)$ coincides with $\mu_{Y^x}$.  
Hence, for $\mu_{X_0}$-almost every $x$, we have
\[
\mu_{X\mid X_0}(\cdot \mid x) \ll \mu_{Y\mid Y_0}(\cdot \mid x)\eqsp.
\]
From Proposition \ref{prop:KL-conditional} applied to $E = \mathcal{C}_T$, $F = \R^d$ 
and $T$ the projection that associates to the solution its initialization we can state that $\mu_X \ll \mu_Y$ and 
\[
\kl{\mu_X}{\mu_Y} = \kl{\mu_{X_0}}{\mu_{Y_0}} 
+
\E\left[
\log\left(
\frac{\mathrm d\mu_{X\mid X_0}}{\mathrm d\mu_{Y\mid Y_0}}(X\mid X_0)
\right)
\right]\eqsp.
\]
In order to conclude the proof we have to show that 
\[
\E\left[
\log\left(
\frac{\mathrm d\mu_{X\mid X_0}}{\mathrm d\mu_{Y\mid Y_0}}(X\mid X_0)
\right)
\right] =  \E \Bigg[\frac{1}{2} \int_0^T \frac{1}{b_s^2}
\|A_s(X) - a_s(X)\|^2\,\rmd s \Bigg]\eqsp.
\]
We have
\begin{align}
\E\left[
\log\left(
\frac{\mathrm d\mu_{X\mid X_0}}{\mathrm d\mu_{Y\mid Y_0}}(X\mid X_0)
\right)
\right]
&= \int_{\mathcal{C}_T \times \R^d} 
\log\left(
\frac{\mathrm d\mu_{X\mid X_0}}{\mathrm d\mu_{Y\mid Y_0}}(x \mid x_0)
\right) \, \rmd \mu_{X, X_0}(x, x_0) \nonumber\\
&= \int_{\R^d} \int_{\mathcal{C}_T} 
\log\left(
\frac{\mathrm d\mu_{X\mid X_0}}{\mathrm d\mu_{Y\mid Y_0}}(x \mid x_0)
\right) \, \rmd \mu_{X\mid X_0}(x \mid x_0) \, \rmd \mu_{X_0}(x_0) \nonumber\\
&= \int_{\R^d} 
\E \left[
\log \frac{\mathrm d\mu_{X^{x_0}}}{\mathrm d\mu_{Y^{x_0}}}(X^{x_0})
\right] \, \rmd \mu_{X_0}(x_0)\eqsp.
\end{align}
We can also write
\begin{align}
\E \left[
\log \frac{\rmd \mu_{X^{x_0}}}{\rmd \mu_{Y^{x_0}}}(X^{x_0})
\right]
&= \E \left[ 
\int_0^T \frac{(A_t(X^x) - a_t(X^x))^T A_t(X^x)}{b_t^2} \, \rmd t
+ \int_0^T \frac{A_t(X^x) - a_t(X^x)}{b_t} \, \dbrown_t \right.\nonumber\\
&\hspace{4cm} \left.- \int_0^T \frac{(A_t(X^x) - a_t(X^x))^T(A_t(X^x) + a_t(X^x))}{2 \, b_t^2} \, \rmd t
\right] \nonumber\\
&= \E \left[
\int_0^T \frac{\|A_t(X^x) - a_t(X^x)\|^2}{2 \, b_t^2} \, \rmd t
\right]\eqsp.
\end{align}
So, we have
\begin{align}
\E\left[
\log\left(
\frac{\rmd \mu_{X\mid X_0}}{\rmd \mu_{Y\mid Y_0}}(X\mid X_0)
\right)
\right]
&= \int_{\R^d} 
\E \left[
\int_0^T \frac{\|A_t(X^x) - a_t(X^x)\|^2}{2 \, b_t^2} \, \rmd t
\right] \, \rmd \mu_{X_0}(x_0) \nonumber \\
&= \int_{\R^d}  \int_{\mathcal{C}_T} \int_0^T \frac{\|A_t(x) - a_t(x)\|^2}{2 \, b_t^2} \, \rmd t\, \rmd \mu_{X^{x_0}}(x) \, \rmd \mu_{X_0}(x_0) \nonumber \\
&= \int_{\R^d}  \int_{\mathcal{C}_T} \int_0^T \frac{\|A_t(x) - a_t(x)\|^2}{2 \, b_t^2} \, \rmd t \,\rmd \mu_{X \mid X_0}(x \mid x_0) \, \rmd \mu_{X_0}(x_0) \nonumber \\
&= \int_{\mathcal{C}_T \times \R^d}  \int_0^T \frac{\|A_t(x) - a_t(x)\|^2}{2 \, b_t^2} \, \rmd t \,\rmd \mu_{X , X_0}(x , x_0)  \nonumber \\
&= \int_{\mathcal{C}_T}  \int_0^T \frac{\|A_t(x) - a_t(x)\|^2}{2 \, b_t^2} \, \rmd t \,\rmd \mu_{X}(x)  \nonumber \\
&= \E \left[\frac{1}{2} \int_0^T \frac{1}{b_s^2}\,\!\|A_s(X) - a_s(X)\|^2\,\rmd s \right]\eqsp,
\end{align}
which concludes the proof.
\end{proof}

\subsection{Lemmas required for  Theorem \ref{thm:KL_bound}}
\label{sec:technical:res}

Following Theorem 5.2.1 in \citet[Theorem 2.5, Theorem 2.9]{karatzas1991brownian}, it is easy to show that for any random variable $\state_0: \Omega \rightarrow \R^d$ 
the existence and uniqueness of a solution to \cref{eq:VE_SDE} are guaranteed. 
Concerning the regularity of the forward density $\pforward{t}$ of \cref{eq:VE_SDE}, for our work, we need it to be in $C^2 ( [\delta, T ] \times \R^d )$ for any $\delta>0$ . We show in Lemma \ref{lem:regularity_VE} the regularity of the solution density. 

\begin{lemma} \label{lem:regularity_VE}
    In the VE setting \eqref{eq:VE_SDE}, $(t,x)\mapsto \pforward{t}(x)$ is in $C^{\infty}((0,T] \times \R^d)$.
\end{lemma}
\begin{proof}
    The density admits the representation
    \[
        \pforward{t}(x) = \int_{\mathbb{R}^d} \, q_t(x \mid y) \, \rmd \mudata(y),
    \]
    where $q_t(x \mid y) = (2\pi\sigma_t^2)^{-d/2} \exp\!\left(-\frac{\|x-y\|^2}{2\sigma_t^2}\right)$.
    We show $\pforward{t}(x)$ is $C^1$ in $(t,x)$; higher-order regularity follows by the same argument, since every higher partial derivative of $q_t(x\mid y)$ is again a product of $q_t(x\mid y)$ with a polynomial in $(x-y)/\sigma_t$ and a power of $\sigma_t^{-1}$, all of which remain bounded by the same reasoning.
    For showing the differentiability in $x$, direct computation gives
    \[
        \partial_{x_i} q_t(x \mid y) = -\frac{x_i - y_i}{\sigma_t^2}\, q_t(x \mid y).
    \]
    On any compact set $K \subset (0,T]$, $\sigma_t$ is bounded away from zero, so
    \[
        \left|\partial_{x_i} q_t(x \mid y)\right|
        \leq \frac{|x_i - y_i|}{\sigma_t}\cdot \frac{1}{\sigma_t} q_t(x\mid y)
        \leq \frac{C}{\sigma_t^{d+1}},
    \]
    using that $z \mapsto z\, e^{-z^2/2}$ is bounded. Since $\mudata$ is a probability measure, dominated convergence applies and $x \mapsto \pforward{t}(x)$ is $C^1$.
    For showing differentiability in $t$, a direct computation gives
    \[
        \partial_t q_t(x \mid y) = \left(\frac{\|x-y\|^2}{2\sigma_t^3} - \frac{d}{\sigma_t}\right)\dot\sigma_t\, q_t(x \mid y).
    \]
    The same bound applies: $z \mapsto z^2 e^{-z^2/2}$ is bounded, so $|\partial_t q_t(x\mid y)|$ is uniformly bounded in $y$, and dominated convergence gives differentiability in $t$.
    The same argument applies to all higher-order derivatives, using that $z \mapsto z^k e^{-z^2/2}$ is bounded for every $k \in \mathbb{N}$, which concludes the proof.
\end{proof}
\begin{proposition}\label{lem:existence-backward}
On the fixed space $(\Omega, \mathcal{F}, \mathbb{P})$, for any $\mudata,  T>0$, it always exists a strong solution $\state = \{\state_t\}_{t \in [0,T]}$ of \cref{eq:VE_SDE}. Furthermore for any $\delta>0$ it exists a brownian $\tilde{B}_t$ such that the process $\back{}=\{\back{t}\}_{t \in [0, T-\delta]}=\{\state_t\}_{t \in [\delta, T]}$ satisfies 
\begin{equation}
\back{t} = \back{0} + \barg_t^2\, \backscore{t} \rmd t + \barg_t\, \rmd \tilde{B}_t \eqsp ,
\end{equation}
for all $t \in [0, T-\delta]$.
Therefore $\back{}$ it is a weak solution to the backward equation 
\begin{equation}
    \rmd \back{t}= \barg_t^2\, \backscore{t} \rmd t + \barg_t\, \rmd W_t \eqsp ,
\end{equation}
initialized with $\back{0}=\state_T$.
Furthermore given a trained score $\trainedscore{t}$ and the discretized 
\begin{equation}
\rmd \trainedsolution{t}=  \barg_t^2 \trainedscore{t_k} \, \rmd t + \barg_t \, \rmd\tilde{\brown}_t\eqsp,
\end{equation}
for $t \in [t_k, t_{k+1}]$, where $t_{k+1} = t_k + h_{k+1}$, with $t_0 = 0$ and $t_N = T-\delta$ initialized with an independent $\trainedsolution{t_0} \sim \discsolution[\param]{0}$ 
the solution can be written in law as
\begin{equation}
\trainedsolution{t_{k+1}}= \trainedsolution{t_{k}} + (\sigma^2_{T-t_{k}}-\sigma^2_{T-t_{k+1}}) \trainedscore{t_k} + \left(\sigma^2_{T-t_{k}}-\sigma^2_{T-t_{k+1}}\right)^{\frac{1}{2}}Z \eqsp,
\end{equation}
with $k = 0,...,N-1$, $Z \sim \mathcal{N}(0, \Id_d)$.
\end{proposition}
\begin{proof}
The strong existence and uniqueness of the forward is a basic consequence of \citep[Theorem 2.5, Theorem 2.9]{karatzas1991brownian}
To show this result we use \citet[Theorem 2.1, Corollary 2.2]{haussmann1985time}. First, the term $\schedulerVE_t$ in Equation \eqref{eq:VE_SDE} trivially satisfies lipschitzianity in $x$ and sublinearity because it is independent of $x$ and bounded. 
Moreover, its forward solution admits a density $\pforward{t}(x)$ for all $0 < t \leq T$ that it is strongly regular for all $\delta>0$ on $[\delta, T]\times \R^d$. 
To satisfy the other hypothesis of \cite{haussmann1985time}, it suffices to prove that, for any $t_0 > 0$ and any bounded open set $A \subset \mathbb{R}^d$, $\int_{t_0}^T \int_A \pforward{t}(x)^2 + \schedulerVE_t^2 \|\nabla \pforward{t}(x)\|^2 \, \rmd x \, \rmd t < +\infty$. As in Lemma \ref{lem:regularity_VE}, $\pforward{t}(x)$ and $\nabla \pforward{t}(x)$ are bounded, and $\schedulerVE_t$ is bounded on any compact interval $[t_0,T]$ because it is continuous. Hence, the integral is finite, and the conclusion follows.
The non-trivial point that requires a precise discussion on the discretized backward is that, although $\trainedsolution{t_k}$ is built using the same Brownian motion $\tilde{B}$ that drives the noise increment on $[t_k, t_{k+1}]$, one can still replace the stochastic integral with an independent Gaussian sample $Z$: this requires checking that $\trainedsolution{t_k}$ and $\int_{t_k}^{t_{k+1}} \barg_t \, \rmd \tilde{B}_t$ are independent, which we now verify via a filtration argument. Concerning the discretization, let $\tilde{\mathcal{F}}_t := \sigma(\back{s} : s \le t) \vee \sigma(\trainedsolution{t_0})$ be the filtration generated 
by the backward process and the (independent) initialization, 
with respect to which $\tilde{B}_t$ remains a Brownian motion. 
By induction, $\trainedsolution{t_k}$ is 
$\tilde{\mathcal{F}}_{t_k}$-measurable: $\trainedsolution{t_0}$ is so by construction, and each step adds only $\tilde{\mathcal{F}}_{t_{k+1}}$-measurable quantities. The integral $\int_{t_k}^{t_{k+1}} \barg_t \, \rmd \tilde{B}_t$ has a deterministic integrand, hence by It\^o's isometry it is centered Gaussian with variance $\int_{t_k}^{t_{k+1}} \barg_t^2 \, \rmd t = \sigma_{T-t_k}^2 - \sigma_{T-t_{k+1}}^2$, and depending only on the increments $\{\tilde{B}_s - \tilde{B}_{t_k} : s \in [t_k, t_{k+1}]\}$ it is independent of $\tilde{\mathcal{F}}_{t_k}$, and therefore of $\trainedsolution{t_k}$. It can thus be written in law as $(\sigma_{T-t_k}^2 - \sigma_{T-t_{k+1}}^2)^{1/2} Z$ with $Z \sim \mathcal{N}(0, \Id_d)$ independent of $\trainedsolution{t_k}$, which yields the claimed recursion.
\end{proof}
The regularity result is necessary to obtain this generalized version of the Fokker-Planck equation, necessary for the analysis of the dynamics of the score.
\begin{lemma}[Fokker–Planck Equation]\label{lem:F-P}
We have, for all $t \in (0,T]$ and $x \in \R^d$, 
\[
\partial_t \pforward{t}(x) = \frac{\schedulerVE_t^2}{2} \Delta \pforward{t}(x) \eqsp .
\]
\end{lemma}
\begin{proof}
We report Fokker Planck equation proof for completeness. Consider the forward equation \eqref{eq:VE_SDE}.
We know that the density of the solution is sufficiently regular and consider a test function $\phi : \R^d \to \R$. By Ito's Formula, 
\begin{equation}
    \E[\phi(\state_t)] - \E[\phi(\state_{0})] = \int_{0}^t  \E\left[\frac{\schedulerVE_t^2}{2} \Delta \phi(\state_t)\right]\rmd t \eqsp .
\end{equation}
Therefore,
\begin{equation}
    \frac{\rmd}{\rmd t} \E[\phi(\state_t)] = \E\left[\frac{\schedulerVE_t^2}{2} \Delta \phi(\state_t) \right] \eqsp .
\end{equation}
Using the dominated convergence theorem yields
\begin{equation}
\realint \phi(x) \, \frac{\partial}{\partial t} \pforward{t}(x) \, \rmd x
=  \frac{\schedulerVE_t^2}{2} \realint \Delta \phi(x) \, \pforward{t}(x)\rmd x\eqsp.
\end{equation}
Thanks to regularity of $\pforward{t}$ and using integration by parts for the first term on the right hand side and Green's Formula for the second we can write
\begin{equation}
    \realint \phi(x) \, \frac{\partial}{\partial t} \pforward{t}(x) \, \rmd x
    = \frac{\schedulerVE_t^2}{2} \realint  \phi(x) \, \Delta \pforward{t}(x)\rmd x\eqsp.
\end{equation}
The last equality is true for any test function which concludes the proof.
\end{proof}
Consider a sufficiently regular function $\ausiliar:\R^d \rightarrow \R^+/\{0\}$ and define for all $x\in\R^d$,
\begin{equation}
\label{eq:def:modifiedscore}
\tilde{p}_{t}(x) = \frac{\pforward{t}(x)}{\ausiliar(x)}\eqsp.
\end{equation}
\begin{lemma}
\label{lem:fokker-tilde}
For all $x \in \R^d$ and all $0 \leq t < T$,
\begingroup
\begin{multline*}
 \partial_t \tilde{p}_{T-t}(x)\, \ausiliar(x) + \frac{\barg_t^2}{2} \Big[
    \ausiliar(x)\, \Delta \tilde{p}_{T-t}(x) 
    + \tilde{p}_{T-t}(x)\, \Delta \ausiliar(x) 
    + 2\, \nabla \tilde{p}_{T-t}(x) \cdot \nabla \ausiliar(x)
\Big] = 0 \eqsp.
\end{multline*}
\endgroup
\end{lemma}
\begin{proof}
From Lemma \ref{lem:F-P} we know the form of Fokker Planck equation for the density $\pforward{t}$ solution to \eqref{eq:VE_SDE}. 
We also know that for all $x \in \R^d$,  $\pforward{t}(x) = \tilde{p}_{t}(x)\rho(x)$. 
By simply injecting this last formula in the Fokker Planck equation, we get the result. 
\end{proof}
\begin{lemma}\label{lem:general_y_dynamic}
Let $Y_t = \backscoretilde{t}$,  where $\tilde{p}_{t}$ is defined in \eqref{eq:def:modifiedscore} and $\back{t}$ is solution to \cref{eq:VE_back_sde}. 
Then, under the hypotheses of  Lemma \ref{lem:fokker-tilde}, for all $\delta > 0$, on the interval $[0, T - \delta]$, the score satisfies the SDE:
\begin{equation}\label{eq:general_y_dynamic}
dY_t = - \barg_t^2 \left[\nabla^2 \log \ausiliar(\back{t})\, Y_t + \frac{1}{2} \nabla \left( \frac{\Delta \ausiliar(\back{t})}{\ausiliar(\back{t})} \right)\right]\, \rmd t  + \barg_t\, Z_t\, \dbrown_t \eqsp,
\end{equation}
where $Z_t = \nabla Y_t$, which is a.s. well defined.
\end{lemma}
\begin{proof}
We can reformulate the result of  Lemma~\ref{lem:fokker-tilde} as:
\begingroup
\begin{multline*}
 \partial_t \log \tilde{p}_{T-t}(x)\, \ausiliar(x) 
+ \frac{\barg_t^2}{2} \Big[
    \ausiliar(x)\, \frac{\Delta \tilde{p}_{T-t}(x) }{\tilde{p}_{T-t}(x)}
    + \Delta \ausiliar(x) 
    + 2\, \nabla \log \tilde{p}_{T-t}(x) \cdot \nabla \ausiliar(x)
\Big] = 0 \eqsp.
\end{multline*}
\endgroup
Using the general property
$$
\frac{\Delta f(x)}{f(x)} = \Delta f(x) + \|\nabla f(x) \|^2\eqsp,
$$
we can write 
\begingroup
\begin{multline*}
 \partial_t \log \tilde{p}_{T-t}(x)\, \ausiliar(x) 
+ \frac{\barg_t^2}{2} \Big[
    \ausiliar(x)\, [\Delta \log \tilde{p}_{T-t}(x) + \|\nabla \log \tilde{p}_{T-t}(x) \|^2]
    + \Delta \ausiliar(x) 
    + 2\, \nabla \log \tilde{p}_{T-t}(x) \cdot \nabla \ausiliar(x)
\Big] = 0 \eqsp ,
\end{multline*}
\endgroup
and so 
\begingroup
\begin{multline*}
 \partial_t \log \tilde{p}_{T-t}(x)
+ \frac{\barg_t^2}{2} \Big[
     \Delta \log \tilde{p}_{T-t}(x) + \|\nabla \log \tilde{p}_{T-t}(x) \|^2\Big]  \\
     = - \frac{\barg_t^2}{2}\Big[
     \frac{\Delta \ausiliar(x) }{\ausiliar (x)}
    + 2\, \nabla \log \tilde{p}_{T-t}(x) \cdot \nabla \log \ausiliar(x)
\Big]   \eqsp .
\end{multline*}
\endgroup
Using $\back{t}$ in the last equation and defining  $\Phi_t = \log \tilde{p}_{T-t}(\back{t})$ we finally write
\begin{equation}\label{eq:fp-phi}
\partial_t \Phi_t 
+ \frac{\barg_t^2}{2} \left( \| \nabla \Phi_t \|^2 + \Delta \Phi_t \right) 
= 
- \frac{\barg_t^2}{2} \left[ \frac{\Delta \rho}{\rho}(\back{t}) + 2 \nabla \log \rho \cdot  \nabla \Phi_t \right]\eqsp.
\end{equation}
The score function $Y_t = \backscoretilde{t}$ is well defined and we can apply Ito's Formula 
to Equation \eqref{eq:VE_back_sde} and noting that  $Y_t = \nabla \Phi_t $ yields
\begin{align}\label{eq:ito-phi}
\rmd Y_t
&= \left( \partial_t \nabla \Phi_t + \barg_t^2 \nabla^2 \Phi_t \left( \nabla \log \rho(\back{t}) + \nabla \Phi_t \right) + \frac{\barg_t^2}{2} \Delta \nabla \Phi_t \right) \rmd t + \barg_t \nabla^2 \Phi_t \rmd B_t \nonumber \\
&= \Bigg( \partial_t \nabla \Phi_t + \barg_t^2 \nabla^2 \Phi_t \nabla \Phi_t + \frac{\barg_t^2}{2} \Delta \nabla \Phi_t + \barg_t^2 \nabla^2 \Phi_t \nabla \log \rho(\back{t}) \Bigg) \rmd t + \barg_t \nabla^2 \Phi_t \rmd B_t \nonumber \\
&= \nabla \left\{ \partial_t \Phi_t + \frac{\barg_t^2}{2} \|\nabla \Phi_t\|^2 + \frac{\barg_t^2}{2} \Delta \Phi_t \right\} \rmd t + \barg_t^2 \nabla^2 \Phi_t \nabla \log \rho(\back{t}) \, \rmd t + \barg_t \nabla^2 \Phi_t \rmd B_t \nonumber\eqsp.
\end{align}
Combining Equation \eqref{eq:fp-phi} and the last equation concludes the proof.
\end{proof}

We established a general result for analyzing the dynamics of the score function in the setting of \cref{eq:VE_SDE}. 
Thanks to Lemma \ref{lem:general_y_dynamic} using $\rho(x) = 1$ we have that for all $\delta >0$, on the interval $[0, T-\delta]$,  $Y_t = \backscore{t}$ satisfies:
\[
\rmd Y_t = \barg_t Z_t\, \rmd\tilde{\brown}_t \eqsp.
\]
By applying Itô's formula like in \cite{conforti2023kl}, we also obtain:
\[
\rmd \|Y_t\|^2 = \barg_t^2 \|Z_t\|_{Fr}^2 \rmd t + 2 \barg_t^2 Y_t^\top Z_t\, \rmd\tilde{\brown}_t \eqsp.
\]
We show here that the norm of the score function $\backscore{t}$ is a martingale.
\begin{lemma}\label{lem:martingVE}
    For any $\mudata$, it holds
    \[
    \E\left[ \int_0^{T-\delta} \| Y_s^\top Z_s\|^2 \rmd s\right] < \infty \eqsp .
    \]
\end{lemma}

\begin{proof}
Note that for all vector $x$ and matrix $A$, 
$\|Ax\|\ \leq \|A\|_\mathrm{Fr}\|x\|$. Therefore, 
$$
\|Y_s^\top Z_s\|^2 \leq \|Y_s^\top\|^2 \|Z_s\|^2_{\mathrm{Fr}} \leq \frac{1}{2}\left(\|Y_s^\top\|^4 + \|Z_s\|_{\mathrm{Fr}}^4\right)  \eqsp .
$$
We first show that $\E[ \int_0^{T-\delta} \|Y_s^\top\|^4 \rmd s] < \infty$. We know that
\[
\pforward{t}(x) = \int_{\R^d}  \, q_t(x \mid y) \, \rmd \mudata(y)\eqsp,
\]
where the transition density is given by
\[
q_t(x \mid y) = \frac{1}{(2 \pi \sigma_t^2)^{d/2}} 
\exp\Biggl\{ - \frac{\|x-y\|^2}{2 \sigma_t^2} \Biggr\}\eqsp,
\]
with $\sigma_t^2 = \int_0^t \schedulerVE_s^2 \, \rmd s$. 
Therefore,
\begin{equation}
\begin{split}
\nabla \log \pforward{t}(x) 
&= \frac{\nabla \pforward{t}(x)}{\pforward{t}(x)} \\
&= \frac{1}{\pforward{t}(x)} \int \nabla q_t(x \mid y) \, \rmd \mudata(y) \\
&= - \frac{1}{\sigma_t^2 \, \pforward{t}(x)} \int (x - y) \,q_t(x \mid y) \, \rmd \mudata(y) \eqsp ,
\end{split}
\end{equation}
and
$$
\E[\|\nabla \log \vec{p}(\state_{T - s})\|^4] =  \sigma_t^{-8}\E[\|\E[\state_{T - s} - X_0|\state_{T - s}]\|^4] \leq \sigma_t^{-8}\E[\|\state_{T - s} -X_0\| ^ 4] < \infty\eqsp .
$$
Then, we show that $\E[ \int_0^{T-\delta}  \|Z_s\|_{\mathrm{Fr}}^4 \rmd s] < \infty$. Write
\[
\nabla^2\pforward{t}(x) 
= \frac{1}{\sigma_t^4} \int_{\R^d} 
\left[ (x - y)(x - y)^\top - \sigma_t^2 \Id_d \right] 
\, q_t(x \mid y) \, \rmd \mudata(y) \eqsp ,
\]
and 
\[
\nabla^2 \log \pforward{t}(x) 
= \frac{\nabla^2 \pforward{t}(x)}{\pforward{t}(x)} 
- \nabla \log \pforward{t}(x)\, \nabla \log \pforward{t}(x)^\top.
\]
So we have 
\begin{align*}
\nabla^2 \log \pforward{t}(x) 
&= \frac{1}{\sigma_t^4 \, \pforward{t}(x)} 
\int_{\mathbb{R}^d} \left[ (x - y)(x - y)^\top - \sigma_t^2 \Id_d \right] 
\, q_t(x \mid y)\, \rmd \mudata(y) \nonumber \\
&\quad - \frac{1}{\sigma_t^4 \, \pforward{t}(x)^2} 
\left( \int_{\mathbb{R}^d} (x - y)\,  q_t(x \mid y)\, \rmd \mudata(y) \right)
\left( \int_{\mathbb{R}^d} (x - y)\, q_t(x \mid y)\, \rmd \mudata(y) \right)^\top \eqsp .
\end{align*}
Finally, 
\begin{align*}
\nabla^2 \log \pforward{t}(\state_{T-s})
&= \frac{1}{\sigma_{T - s}^4} \, 
\E\left[
(\state_{T-s} - X_0)(\state_{T-s} - X_0)^\top - \sigma_{T - s}^2 \Id_d
\,\Big|\, \state_{T-s}
\right] \nonumber \\
&\quad - \frac{1}{\sigma_{T - s}^4} \, 
\E\left[X_0 - \state_{T-s} \,\Big|\, \state_{T-s} \right]
\E\left[X_0 - \state_{T-s} \,\Big|\, \state_{T-s} \right]^\top \eqsp .
\end{align*}
This yields 
\begin{align*} \nonumber
\E \left[
\left\|
\nabla^2 \log \pforward{t}(\state_{T-s})
\right\|^4
\right]
&= \E \left[
\left\|
\frac{1}{\sigma_{T-s}^4} \, 
\E\left[
(\state_{T-s} - X_0)(\state_{T-s} - X_0)^\top - \sigma_{T-s}^2 \Id_d
\,\Big|\, \state_{T-s}
\right] \right.\right. \\ \nonumber
&\quad \hspace{2cm} \left.\left.
- \frac{1}{\sigma_{T-s}^4} \,
\E\left[X_0 - \state_{T-s} \,\Big|\, \state_{T-s}\right]
\E\left[X_0 - \state_{T-s} \,\Big|\, \state_{T-s}\right]^\top
\right\|^4
\right] \\ \nonumber
&\leq \frac{C_1}{\sigma_{T-s}^{16}} \, \E \left[
\left\|
(\state_{T-s} - X_0)(\state_{T-s} - X_0)^\top - \sigma_{T-s}^2 \Id_d
\right\|^4
\right] \\ \nonumber
&\quad \hspace{2cm}+ \frac{C_2}{\sigma_{T-s}^{16}} \, \E \left[
\left\|
(X_0 - \state_{T-s})(X_0 - \state_{T-s})^\top
\right\|^4
\right] \eqsp,
\end{align*}
where $C_1,C_2$ are positive constants. The proof is concluded by noting that $\state_{T-s}-X_0$ is a  Gaussian distribution. 
\end{proof}

\subsection{Proof of Theorem~\ref{thm:KL_bound}}
\label{sec:mainproof}
We can finally prove \cref{thm:KL_bound} in the first form using \cref{eq:err:approxVE,eq:err:discrVE,eq:err:mixVE}
and in the form of \cref{eq:thm_simplified}.
We consider $\delta \geq 0$ and $\{\back{t}\}_{t \in [0,T-\delta]}$,  $\{\back{t}^x\}_{t \in [0,T-\delta]}$ of \cref{eq:VE_back_sde}.
Writing ${\trainedsolutioninit{t}{x}}$ the solution of the discretized backward with deterministic initialization $x$, the existence of  \( \{\trainedsolution{t}\}_{t \in [0, T-\delta]} \) and \( Y^x = \{ \trainedsolutioninit{t}{x} \}_{t \in [0, T-\delta]} \) follows from the explicit form of the solution.
Thanks to the fact that (by hypothesis) we know that $\discsolution{0} >0 $ $\lambda_{\R^d}$-almost surely and therefore $\back{0} \ll \trainedsolution{0}$. 

We have to show that in our context also the Novikov-like hypothesis of \cref{thm:kl-control} is satisfied. 
First we note that for all $t \in [\delta, T]$
\begin{align}
    \information(\pforward{t}) = \PE{}{||\scorep{t}(\state_t)||^2} 
    &= \frac{\PE{}{\left\| \PE{}{\state_t -\state_0\mid \state_t}\right\|^2}}{\sigma_t^4} 
    \leq \frac{\PE{}{\left\| \state_t -\state_0 \right\|^2}}{\sigma_t^4} \nonumber\\
    &= \frac{\PE{}{\left\|Z\right\|^2}}{\sigma_t^4}= \frac{d\sigma_t^2}{\sigma_t^4} = \frac{d}{\sigma_t^2} \eqsp .
\end{align}
Therefore 
\begin{equation}
    \int_{\delta}^{T} \PE{}{||\scorep{t}(\state_t)||^2} \rmd t < \infty
\end{equation}
We also have that, using the notation from \cref{thm:kl-control}
\begin{align}
    &\int_{\delta}^{T} \PE{}{||\scorep{t}(\state_t)||^2} \rmd t \nonumber\\
    &= \PE{}{\int_{\delta}^{T} ||\scorep{t}(\state_t)||^2\rmd t} \nonumber\\
    &= \int_{\mathcal{C}([\delta,T],\mathbb{R}^d)}\int_{\delta}^{T} ||\scorep{t}(z)||^2\rmd t \, \rmd \mu_X(z) \nonumber\\
    &= \int_{\R^d} \int_{\mathcal{C}([\delta,T],\mathbb{R}^d)}\int_{\delta}^{T} ||\scorep{t}(z)||^2\rmd t \,  \rmd \mu_{X \mid \back{0}}(z \mid x) \, \mu_{\back{0}}(x) \nonumber\\
    &= \int_{\R^d} \int_{\mathcal{C}([\delta,T],\mathbb{R}^d)}\int_{\delta}^{T} ||\scorep{t}(z)||^2\rmd t \,  \rmd \mu_{X \mid \back{0}}(z \mid x) \, \mu_{\back{0}}(x) \nonumber\\
    &= \int_{\R^d} \int_{\mathcal{C}([\delta,T],\mathbb{R}^d)}\int_{\delta}^{T} ||\scorep{t}(z)||^2\rmd t \,  \rmd \mu_{X^z}(z) \, \mu_{\back{0}}(x) \nonumber\\
    &= \int_{\R^d} \PE{}{\int_{0}^{T-\delta} ||\scorep{T-t}(\back{t}^x)||^2\rmd t \,} \mu_{\back{0}}(x) \eqsp.
\end{align}
Therefore $\mu_{\back{0}}$-a.s. 
\begin{equation}
    \PE{}{\int_{0}^{T-\delta} ||\scorep{T-t}(\back{t}^x)||^2\rmd t } < \infty
\end{equation}
and this simply implies that $\mu_{\back{0}}$-a.s. 
\begin{equation}
    \PP{\int_{0}^{T-\delta} ||\scorep{T-t}(\back{t}^x)||^2\rmd t } < \infty
\end{equation}
For the trained score, it is sufficient to note by hypothesis that 
\begin{align}
    \PE{}{||\trainedscoreforward{t_k}||^2} 
    &\leq \PE{}{||\scorep{t_k}(\state_{t_k}) - \trainedscoreforward{t_k}||^2} + \PE{}{||\scorep{t}(\state_t)||^2} \nonumber\\
    &< \epsilon + \frac{d}{\sigma_t^2} \eqsp.
\end{align}
We can therefore apply theorem \cref{thm:kl-control} with
\begin{itemize}
\item $X = \{\back{t}\}_{t \in [0,T-\delta]}$, $Y = \{\trainedsolution{t}\}_{t \in [0, T-\delta]}$.
\item $X^x= \{\back{t}^x\}_{t \in [0,T-\delta]}$, $Y^x = \{{\trainedsolutioninit{t}{x}}\}_{t \in [0, T-\delta]}$.
\item $X_0 = \back{0}$, $Y_0 \sim \discsolution{0}$.
\end{itemize} 
The KL bound we obtain is the following :
\begin{align}
\kl{\pforward{\delta}}{p_{T-\delta}^{\theta}} 
&\leq \kl{\pforward{T}}{p_0^{\theta}} +\nonumber \\ 
&= \frac{1}{2} \int_0^{T-\delta}  \frac{1}{\barg_t^2}\, \E \biggl[
  \bigl\| 
    \barg_t^2 \backscore{t} 
 - \sum_{k=0}^{N-1}   \barg_t^2\trainedscorex{t_k} 
  \ind_{[t_{k}, t_{k+1}]}(t)
  \bigr\|^2 
\biggr] \rmd t \eqsp , 
\end{align}
Using the square triangular inequality and extracting $\barg_t$, we get that 
\begin{align}
\kl{\pforward{\delta}}{p_{T-\delta}^{\theta}} 
&\leq \kl{\pforward{T}}{p_0^{\theta}} \\ \nonumber
&\quad + \sum_{k=0}^{N-1} \int_{t_{k}}^{t_{k+1}} \barg_t^2\ \rmd t \, \E \biggl[\bigl\| \backscore{t_{k}} - \trainedscorex{t_{k}} \bigr\|^2 \biggr]  \\ \nonumber
&\quad + \sum_{k=0}^{N-1} \int_{t_{k}}^{t_{k+1}} \barg_t^2\, \E \bigl[\| \backscore{t_{k}} -\backscore{t} \|^2\bigr] \rmd t \eqsp ,
\end{align}
We obtained \cref{thm:KL_bound} with terms \cref{eq:err:approxVE,eq:err:discrVE,eq:err:mixVE}.
To obtain \cref{eq:information}, using the definition of $Y_t$ write
$$\sum_{k=0}^{N-1} \int_{t_{k}}^{t_{k+1}} \barg_t^2\, \E \bigl[\| \backscore{t_{k}} -\backscore{t} \|^2\bigr] \rmd t  = \sum_{k=0}^{N-1} \int_{t_{k}}^{t_{k+1}} \barg_t^2\, \E \bigl[\| Y_{t_{k}} -Y_t \|^2\bigr] \rmd t\eqsp.
$$
The expected value inside the integral can be written
\[
\E\left[\left\| \int_{t_{k}}^t \barg_s Z_s \, \rmd B_s \right\|^2\right] = \E\left[\int_{t_{k}}^t \| \barg_s Z_s \|_{\mathrm{Fr}}^2 \, \rmd s \right]\eqsp.
\]
Then, note  that 
$$\sum_{k=0}^{N-1} \int_{t_{k}}^{t_{k+1}} \barg_t^2\, \E \left[\int_{t_{k}}^{t}\| \barg_s Z_s \|_{\mathrm{Fr}} ^2\rmd s \right] \rmd t \leq \sum_{k=0}^{N-1} \int_{t_{k}}^{t_{k+1}} \barg_t^2\, \E \left[\int_{t_{k}}^{t_{k+1}}\| \barg_s Z_s \|_{\mathrm{Fr}} ^2\rmd s \right] \rmd t \eqsp .$$
Thanks to Lemma \ref{lem:general_y_dynamic} and Lemma \ref{lem:martingVE} we know that the norm of the score evolves as an SDE and defines the diffusion part defines a martingal, so that the second term of the last inequality can be upper bounded as follows
$$\sum_{k=0}^{N-1} (\|Y_{t_{k+1}}\|^2 - \|Y_{t_{k}}\|^2) \int_{t_{k}}^{t_{k+1}} \barg_t^2\, \rmd t  \leq\max_{k = 0,\dots,N-1} \Bigg\{ \int_{t_{k}}^{t_{k+1}} \barg_t^2\, \rmd t \Bigg\}  \sum_{k=0}^{N-1} (\|Y_{t_{k+1}}\|^2 - \|Y_{t_{k}}\|^2)  \eqsp .
$$
We note the telescopic sum and we conclude 
$$
\discerror \leq \max_{k = 0,\dots,N-1} \Bigg\{ \int_{t_k}^{t_{k+1}} \barg_t^2 \, \rmd t \Bigg\} 
\;\cdot\; \Big( \information(\pforward{\delta}) - \information(\pforward{T}) \Big) \eqsp,
$$ 
where $\discerror $ is defined in Theorem~\ref{thm:KL_bound}.

\subsection{Results for the proof of Theorem \ref{thm:kl-control}}
\label{sec:additional:technical}
\begin{lemma}[Conditional absolute continuity]\label{lem:conditional-properties} 
Let $(E,\mathcal E)$ and $(F,\mathcal F)$ be Polish spaces.  
Let $X,Y : (\Omega,\mathcal A,\mathbb P)\to (E,\mathcal E)$ be random variables and let
$T:(E,\mathcal E)\to(F,\mathcal F)$ be a measurable map.
Denote by $\mu_X,\mu_Y$ the laws of $X,Y$ on $E$, by $\mu_{T(X)},\mu_{T(Y)}$ the laws of
$T(X),T(Y)$ on $F$, and by $\mu_{(X,T(X))}$, $\mu_{(Y,T(Y))}$ the joint laws of
$(X,T(X))$ and $(Y,T(Y))$ on $E\times F$.
Assume that regular conditional distributions
$\mu_{X\mid T(X)}(\cdot\mid\cdot)$ and $\mu_{Y\mid T(Y)}(\cdot\mid\cdot)$ exist.

Suppose that
\[
\mu_{X\mid T(X)}(\cdot\mid\cdot)\ll \mu_{Y\mid T(Y)}(\cdot\mid\cdot)
\quad\text{and}\quad
\mu_{T(X)}\ll \mu_{T(Y)}.
\]
Then $\mu_X\ll \mu_Y$ and $\mu_{(X,T(X))}\ll \mu_{(Y,T(Y))}$, and for $\mu_Y$-a.e.
$e\in E$,
\[
\frac{\mathrm d\mu_X}{\mathrm d\mu_Y}(e)
=
\frac{\mathrm d\mu_{(X,T(X))}}{\mathrm d\mu_{(Y,T(Y))}}\bigl(e,T(e)\bigr)
=
\frac{\mathrm d\mu_{X\mid T(X)}}{\mathrm d\mu_{Y\mid T(Y)}}\bigl(e\mid T(e)\bigr)
\,
\frac{\mathrm d\mu_{T(X)}}{\mathrm d\mu_{T(Y)}}\bigl(T(e)\bigr).
\]
\end{lemma}
\begin{proof}
Since $E$ and $F$ are Polish spaces, regular conditional distributions
$\mu_{X\mid T(X)}$ and $\mu_{Y\mid T(Y)}$ exist \citet[Appendix]{douc2018markov}.

Let $\phi : E \times F \to \mathbb R_+$ be a non-negative measurable test function.
By disintegration of $\mu_{(X,T(X))}$ with respect to $T(X)$, we have
\begin{align*}
\E[\phi(X,T(X))]
&= \int_{E\times F} \phi(e,f)\,\mathrm d\mu_{(X,T(X))}(e,f) \\
&= \int_F \int_E \phi(e,f)\,\mathrm d\mu_{X\mid T(X)}(e\mid f)\,
\mathrm d\mu_{T(X)}(f).
\end{align*}
By the assumed absolute continuity of the conditional laws and of the marginals,
this can be rewritten as
\begin{align*}
\E[\phi(X,T(X))]
&= \int_F \int_E \phi(e,f)\,
\frac{\mathrm d\mu_{X\mid T(X)}}{\mathrm d\mu_{Y\mid T(Y)}}(e\mid f)\,
\mathrm d\mu_{Y\mid T(Y)}(e\mid f)\,
\frac{\mathrm d\mu_{T(X)}}{\mathrm d\mu_{T(Y)}}(f)\,
\mathrm d\mu_{T(Y)}(f) \\
&= \int_{E\times F} \phi(e,f)\,
\frac{\mathrm d\mu_{X\mid T(X)}}{\mathrm d\mu_{Y\mid T(Y)}}(e\mid f)\,
\frac{\mathrm d\mu_{T(X)}}{\mathrm d\mu_{T(Y)}}(f)\,
\mathrm d\mu_{(Y,T(Y))}(e,f).
\end{align*}
Since this identity holds for all non-negative measurable $\phi$, it follows that
\[
\mu_{(X,T(X))} \ll \mu_{(Y,T(Y))}
\]
and that the Radon--Nikodym derivative is given $\mu_{(Y,T(Y))}$-a.e.\ by
\[
\frac{\mathrm d\mu_{(X,T(X))}}{\mathrm d\mu_{(Y,T(Y))}}(e,f)
=
\frac{\mathrm d\mu_{X\mid T(X)}}{\mathrm d\mu_{Y\mid T(Y)}}(e\mid f)\,
\frac{\mathrm d\mu_{T(X)}}{\mathrm d\mu_{T(Y)}}(f).
\]

Let's now consider $A \in \mathcal{E}$. We can write
\[
\begin{aligned}
\mu_X(A)
&= \mu_{(X,T(X))}(A \times T(A)) \\
&= \int_{E \times F} \ind_{A \times T(A)}(e,f) \,\mathrm d\mu_{(X,T(X))}(e,f) \\
&= \int_{E \times F} \ind_{A \times T(A)}(e,f) \,
\frac{\mathrm d\mu_{(X,T(X))}}{\mathrm d\mu_{(Y,T(Y))}}(e,f) \,\mathrm d\mu_{(Y,T(Y))}(e,f) \\
&= \int_E \int_F \ind_{A \times T(A)}(e,f) \,
\frac{\mathrm d\mu_{(X,T(X))}}{\mathrm d\mu_{(Y,T(Y))}}(e,f) \,
\mathrm d\mu_{T(Y)\mid Y}(f \mid e) \,\mathrm d\mu_Y(e) \\
&= \int_E \int_F \ind_{A \times T(A)}(e,f) \,
\frac{\mathrm d\mu_{(X,T(X))}}{\mathrm d\mu_{(Y,T(Y))}}(e,f) \,
\mathrm d\delta_{T(e)}(f) \,\mathrm d\mu_Y(e) \\
&= \int_E \ind_A(e) \,
\frac{\mathrm d\mu_{(X,T(X))}}{\mathrm d\mu_{(Y,T(Y))}}(e,T(e)) \,\mathrm d\mu_Y(e).
\end{aligned}
\]
\end{proof}

\begin{proposition}[Chain rule for the Kullback--Leibler divergence]
\label{prop:KL-conditional}
Let $(E,\mathcal E)$ and $(F,\mathcal F)$ be Polish spaces.
Let $X,Y:(\Omega,\mathcal A,\mathbb P)\to(E,\mathcal E)$ be random variables and
let $T:(E,\mathcal E)\to(F,\mathcal F)$ be a measurable map.

Denote by $\mu_X,\mu_Y$ the laws of $X,Y$ on $E$ and by
$\mu_{T(X)},\mu_{T(Y)}$ the laws of $T(X),T(Y)$ on $F$.
Assume that regular conditional distributions exist and that
\[
\mu_{X\mid T(X)}(\cdot\mid\cdot)\ll \mu_{Y\mid T(Y)}(\cdot\mid\cdot),
\qquad
\mu_{T(X)}\ll \mu_{T(Y)}.
\]
Then $\mu_X\ll \mu_Y$ and
\[
\kl{X}{Y}
=
\kl{T(X)}{T(Y)}
+
\E\!\left[
\log\!\left(
\frac{\mathrm d\mu_{X\mid T(X)}}{\mathrm d\mu_{Y\mid T(Y)}}(X\mid T(X))
\right)
\right].
\]
\end{proposition}
\begin{proof}
Since $E$ and $F$ are Polish spaces, regular conditional distributions
$\mu_{X\mid T(X)}$ and $\mu_{Y\mid T(Y)}$ exist \citet[Appendix]{douc2018markov}.
Under the stated absolute continuity assumptions, Lemma~\ref{lem:conditional-properties}
ensures that $\mu_X\ll\mu_Y$ and that, for $\mu_X$-a.e.\ $e\in E$,
\[
\frac{\mathrm d\mu_X}{\mathrm d\mu_Y}(e)
=
\frac{\mathrm d\mu_{T(X)}}{\mathrm d\mu_{T(Y)}}\bigl(T(e)\bigr)
\frac{\mathrm d\mu_{X\mid T(X)}}{\mathrm d\mu_{Y\mid T(Y)}}\bigl(e\mid T(e)\bigr).
\]

By definition of the Kullback--Leibler divergence,
\begin{align*}
\kl{X}{Y}
&= \mathbb E\!\left[
\log\!\left(\frac{\mathrm d\mu_X}{\mathrm d\mu_Y}(X)\right)
\right] \\
&= \mathbb E\!\left[
\log\!\left(
\frac{\mathrm d\mu_{T(X)}}{\mathrm d\mu_{T(Y)}}(T(X))
\right)
\right]
+
\mathbb E\!\left[
\log\!\left(
\frac{\mathrm d\mu_{X\mid T(X)}}{\mathrm d\mu_{Y\mid T(Y)}}(X\mid T(X))
\right)
\right] \\
&= \kl{T(X)}{T(Y)}
+
\mathbb E\!\left[
\log\!\left(
\frac{\mathrm d\mu_{X\mid T(X)}}{\mathrm d\mu_{Y\mid T(Y)}}(X\mid T(X))
\right)
\right].
\end{align*}
\end{proof}

\section{Proofs of Results of \cref{thm:tang2021empirical:thm3main}}\label{appendix:m_estimation}
The goal of this section is to establish conditions to bound the estimation error $\kl{\pforward{T}}{\discsolution[\emin{\param}]{0}}$ in the context of Normalizing Flows. To do so, we begin by defining the pointwise log-density difference
\begin{equation*}
    \label{eq:def:mfunction}
    \mfun: \paramsp \times \rset^{\statedim} \ni (\param, x) \mapsto \mfun[\param][x] \eqdef \log \pforward{T}(x) - \log \discsolution[\param]{0}(x) \in \rset \eqsp,
\end{equation*}
the empirical minimizer, i.e., the maximum likelihood estimator over the parametric family $\{\discsolution[\param]{0}\}_{\param \in \paramsp}$,
\begin{equation*}
    \label{eq:def:emp_min}
    \emin{\param} \in \argmin_{\param \in \paramsp} \sum_{i=1}^{\ndata} - \log \discsolution[\param]{0}(\state_T^i) = \argmin_{\param \in \paramsp} M_n(\param) \eqsp,
\end{equation*}
and the theoretical minimizer, i.e., the best approximation of $\pforward{T}$ achievable within the model class,
\begin{equation*}
    \label{eq:def:true_min}
    \tmin{\param} \in \argmin_{\param \in \paramsp} \kl{\pforward{T}}{\discsolution[\param]{0}} = \argmin_{\param \in \paramsp} M^*(\param) \eqsp,
\end{equation*}
where $\paramsp \subseteq \rset^{\paramdim}$, $M_n(\param) = \ndata^{-1}\sum_{i=1}^{\ndata} \mfun[\param][\state_T^i]$, and $M^*(\param) = \E[\mfun[\param][\state_T]]$. 
We also define the excess loss function, which compares the loss at $\param$ to that at the theoretical minimizer,
\begin{equation*}
    \label{eq:def:zeromfun}
    \zeromfun: \paramsp \times \rset^{\statedim} \ni (\param, x) \mapsto \zeromfun[\param][x] \eqdef \mfun[\param][x] - \mfun[\tmin{\param}][x] \in \rsetpos \eqsp,
\end{equation*}
and the following two functional sets, where $\simplexg$ extends $\spaceg$ by including all rescalings by a factor $a \in (0,1]$, making it star-shaped around zero:
\begin{align}
    \label{eq:def:spaceg} \spaceg &= \{ \zeromfun[\param] \mid \param \in \paramsp \} \eqsp, \\
    \label{eq:def:simplexg} \simplexg &= \{ \zeromfun \mid \exists a \in (0,1], \zeromfun[\param] \in \spaceg \text{ such that } \zeromfun = a \, \zeromfun[\param] \} \eqsp.
\end{align}
Furthermore, we define the local Rademacher complexity \cite{wainwright2019highdimensional} for a generic family of functions $G$ as
\begin{equation*}
    \label{eq:def:localrad}
    \localrad{\delta}{G} = \E_{\pforward{T}} \left[ \E_{\epsilon} \left[ \sup_{\zeromfun \in G, \lpnorm{\zeromfun}[2][\pforward{T}] \leq \delta} \left| \frac{1}{\ndata} \sum_{i=1}^{\ndata} \epsilon_i \zeromfun(\state_T^i) \right| \right] \right] \eqsp,
\end{equation*}
where $\epsilon_i$ are i.i.d. $\text{Bernoulli}(1/2)$ variables and $\lpnorm{\zeromfun}[2][\pforward{T}]^2 = \int \zeromfun^2(x) \pforward{T}(x) \text{d}x$. Compared to the global Rademacher complexity, the local version restricts attention to functions with $L^2(\pforward{T})$ norm at most $\delta$.
\begin{theorem}[{Theorem 1 in \cite{tang2021empirical} adapted to our setting}]
\label{thm:1_tang}
Assume \hypref{assump:tang2021empirical:A} and the existence of $\delta_n > 0$ satisfying the following two conditions:
\begin{enumerate}
    \item $\localrad{\delta_n}{\simplexg} \leq \frac{\delta_n^2}{D \log \ndata}$,
    \item $\frac{\ndata \min\{\delta_n,\delta_n^2\}}{D^2 \log^2 \ndata} \geq \log \log \frac{D}{\delta_n}$.
\end{enumerate}
Then, there exist constants $c_0, c_1, c_2 \in \R_+$ such that
\begin{multline}\label{eq:thm:kl_bound_tang}
\prob \Biggl( \kl{\pforward{T}} {\discsolution[\emin{\param}]{0}} \leq \inf_{\gamma > 0} \Bigl[ 2 (1+\gamma) \kl{\pforward{T}} {\discsolution[\tmin{\param}]{0}} + c_2 (1+\gamma^{-1}) \frac{\max\{\delta_n, \delta_n^2\}}{D \log \ndata} \Bigr] \Biggr) \\
\geq 1 - c_0 \exp\Bigl(-c_1 \frac{\ndata \min\{\delta_n,\delta_n^2\}}{D^2 \log^2 \ndata}\Bigr) \eqsp.
\end{multline}
\end{theorem}
The proof of \Cref{thm:1_tang} relies on two auxiliary lemmas. 
\Cref{lem:tang_11} controls the second moment of the loss function $\mfun[\param][\state_T]$ in terms of the KL divergence and a variance term that vanishes as $\ndata \to \infty$.
\Cref{lem:tang_12} establishes a uniform concentration inequality showing that the empirical fluctuations of the excess loss, normalized by a data-dependent scale, remain small with high probability.
\begin{lemma}
    \label{lem:tang_11}
    Under \hypref{assump:tang2021empirical:A}, there exists a constant $c \in \R$ such that for all $\param \in \paramsp$:
    \begin{equation*}
        \E[\mfun[\param][\state_T]^2] \leq c \left( \kl{\pforward{T}}{\discsolution[\param]{0}} + \frac{D^2 \log^2 \ndata}{\ndata} \right) \eqsp.
    \end{equation*}
\end{lemma}
\begin{proof}[Proof of \Cref{lem:tang_11}]
    We do not provide here the proof, because it is a straightforward adaptation of the proof of \citet[Lemma 11]{tang2021empirical} to our setting. 
\end{proof}
\begin{lemma}
    \label{lem:tang_12}
    In the context of \Cref{thm:1_tang}, for all $\param \in \paramsp$, it exists $c_4>0$ such that the following concentration inequality holds:
    \begin{equation*}
        \prob\left( \frac{| M_n(\param) - M_n(\tmin{\param}) - (M^*(\param) - M^*(\tmin{\param})) |}{\delta_n + \lpnorm{\zeromfun[\param]}[2][\pforward{T}]} \leq c_3\frac{\delta_n}{D \log \ndata} \right) \geq 1 - c_0 \exp\Bigl(-c_1 \frac{\ndata \min\{\delta_n,\delta_n^2\}}{D^2 \log^2 \ndata}\Bigr) \eqsp.
    \end{equation*}
\end{lemma}
\begin{proof}[Proof of \Cref{lem:tang_12}]
    The proof of the lemma is an adaptation to our setting of \citet[Lemma 12 ]{tang2021empirical}.
    By \hypref{assump:tang2021empirical:A}, since $x \leq e^x - 1 = \psi_1(x)$ for all $x \geq 0$, we have:
    \begin{equation*}
        \frac{1}{\lambda} \sup_{\param \in \paramsp} \E[\zeromfun[\param][\state_T]] \leq \E\left[ \frac{\sup_{\param \in \paramsp} |\zeromfun[\param][\state_T]|}{\lambda} \right] \leq \E\left[ \psi_1\left( \frac{\sup_{\param \in \paramsp} |\zeromfun[\param][\state_T]|}{\lambda} \right) \right] \eqsp.
    \end{equation*}
    This implies by the definition of the Orlicz norm that 
    \begin{equation*}
        \onorm{\sup_{\param \in \paramsp} \PE{}{\zeromfun[\param][\state_T]}}[1] \leq \onorm{\sup_{\param \in \paramsp} |\zeromfun[\param][\state_T]|}[1] \leq D \eqsp,
    \end{equation*}
    where we used the bound from \hypref{assump:tang2021empirical:A}.
    Observing that the empirical fluctuation of the excess loss can be rewritten as
    $$\frac{1}{n}\sum_{i=1}^{n}\zeromfun[\param][\state_T^i] - \PE{}{\zeromfun[\param][\state_T]} = M_n(\param) - M_n(\tmin{\param}) -M^*(\param) + M^*(\tmin{\param}),$$
    and applying the triangle inequality to the Orlicz norm together with the definition of $\zeromfun[\param]$ and \hypref{assump:tang2021empirical:A}, we can verify that
    \begin{equation*}
        \onorm{\sup_{\param \in \paramsp} \left|\zeromfun[\param][\state_T] - \PE{}{\zeromfun[\param][\state_T]} \right|}[1] \leq 4 D \eqsp.
    \end{equation*}
    Using the bound $\lpnorm{\zeromfun[\param]}[2][\pforward{T}] \leq 2 \onorm{\zeromfun[\param][\state_T]}[1]$ together with the maximal inequality
    $$ \onorm{\max_i \sup_{\param \in \paramsp} | \zeromfun[\param][\state_T^i] - \PE{}{\zeromfun[\param][\state_T]} |}[1] \leq  K_{\psi_1} \frac{\log (\ndata + 1)}{\ndata} \onorm{\sup_{\param \in \paramsp} | \zeromfun[\param][\state_T^i] - \PE{}{\zeromfun[\param][\state_T]} |}[1] \eqsp, $$
    we can deduce that $\simplexg$ (and also $\spaceg$) is contained in a ball in the space  $L^2(\pforward{T})$ of radius $2^J \delta_n$ for $J \geq \log(D/\delta_\ndata) + 3$ and that for any subspace $\paramsp_0 \subset \paramsp$,
    $$ \onorm{\max_i \sup_{\param \in \paramsp_0} | \zeromfun[\param][\state_T^i] - \PE{}{\zeromfun[\param][\state_T]} |}[1]\leq 4 K_{\psi_1} \frac{\log (\ndata + 1)}{\ndata} D \eqsp .$$
    To control the uniform empirical fluctuations over the function class, we introduce, for a generic $r>0$ and a family of functions $G$, the worst-case concentration variable
    $$ \worstconcentrationvar{r}{G} = \sup_{\zeromfun \in G, \lpnorm{g}[2][\pforward{T}] \leq r} \left| \frac{1}{\ndata} \sum_{i=1}^{\ndata} \zeromfun[][\state_T^i]- \PE{}{[\zeromfun[][\state_T]]}\right|.$$
    Applying \citet[Theorem 4]{adamczak2008tail} to the centered functions $( \zeromfun[\param][\state_T^i] - \PE{}{\zeromfun[\param][\state_T]})/n$ over $\spaceg$, we obtain that for any $r, t > 0$ and $0 < \eta < 1$, there exists a constant $L=4 C(\eta) K_{\psi_1}$ such that
    \begin{equation*}
        \prob \left( \worstconcentrationvar{r}{\spaceg} \geq (1+\eta) \PE{}{\worstconcentrationvar{r}{\spaceg}} + t \right) \leq 3\exp\left( - \ndata \min \left\{ \frac{t^2}{4r^2}, \frac{t}{L r D \log(\ndata + 1)} \right\} \right) \eqsp.
    \end{equation*}
    Furthermore, by \citet[Proposition 4.11]{wainwright2019highdimensional}, we have
    \begin{equation*}
        \PE{}{\worstconcentrationvar{r}{\spaceg}}\leq 2 \localrad{r}{\spaceg} \leq 2 \localrad{r}{\simplexg} \eqsp.
    \end{equation*}
    We now have all the ingredients needed to establish the concentration inequality. For any $\epsilon_n \geq 0$, a peeling argument over shells of radius $2^j \delta_n$ gives (where for $j=0$, $2^{j-1}$ is interpreted as $0$):
\begin{align*}
    \prob \Biggl( &\frac{| M_n(\param) - M_n(\tmin{\param}) - (M^*(\param) - M^*(\tmin{\param})) |}{\delta_n + \lpnorm{\zeromfun[\param]}[2][\pforward{T}]} \geq \epsilon_n \Biggr) \\
    &\hspace{3cm} \leq \prob \left( \sup_{\param \in \paramsp} \frac{| \frac{1}{\ndata}\sum_{i=1}^{\ndata}\zeromfun[\param][\state_T^i] - \PE{}{\zeromfun[\param][\state_T]} |}{\delta_n + \lpnorm{\zeromfun[\param]}[2][\pforward{T}]} \geq \epsilon_n \right) \\
    &\hspace{3cm}= \prob \left( \max_{j=0,\dots,J} \sup_{\substack{\param \in \paramsp \\ 2^{j-1} \delta_n < \lpnorm{\zeromfun[\param]}[2][\pforward{T}] \leq 2^j \delta_n}} \frac{|\frac{1}{\ndata}\sum_{i=1}^{\ndata}\zeromfun[\param][\state_T^i] - \PE{}{\zeromfun[\param][\state_T]}|}{\delta_n + \lpnorm{\zeromfun[\param]}[2][\pforward{T}]} \geq \epsilon_n \right) \\
    &\hspace{3cm}\leq \sum_{j=0}^J \prob \left( \sup_{\substack{\param \in \paramsp \\ \lpnorm{\zeromfun[\param]}[2][\pforward{T}] \leq 2^j \delta_n}} \frac{|\frac{1}{\ndata}\sum_{i=1}^{\ndata}\zeromfun[\param][\state_T^i] - \PE{}{\zeromfun[\param][\state_T]}|}{\delta_n + \lpnorm{\zeromfun[\param]}[2][\pforward{T}]} \geq \epsilon_n \right) \\
    &\hspace{3cm}\leq \sum_{j=0}^J \prob \left( \sup_{\substack{\param \in \paramsp \\ \lpnorm{\zeromfun[\param]}[2][\pforward{T}] \leq 2^j \delta_n}} \left| \frac{1}{\ndata}\sum_{i=1}^{\ndata}\zeromfun[\param][\state_T^i]- \PE{}{\zeromfun[\param][\state_T]} \right| \geq \epsilon_n (2^{j-1} \delta_n + \delta_n) \right) \\
    &\hspace{3cm} = \sum_{j=0}^J 
    \prob \left( \worstconcentrationvar{2^j \delta_n}{\spaceg} \geq \epsilon_n 2^{j-1} \delta_n \right) \\
    &\hspace{3cm} \leq \prob \left( \worstconcentrationvar{\delta_n}{\spaceg} \geq \epsilon_n  \delta_n \right)  + \sum_{j=1}^J \prob \left( \worstconcentrationvar{2^j \delta_n}{\spaceg} \geq 2^{j-1} \epsilon_n \delta_n \right) \eqsp . \\
\end{align*}
We now bound each term separately. We set $\epsilon_n =  2 c_4 \delta_n/(D \log(\ndata))$ and use the hypothesis $\localrad{\delta_n}{\simplexg} \leq \delta_n^2/(D \log{\ndata})$ to bound the $j=0$ term:
\begin{align*}
 \prob \left( \worstconcentrationvar{\delta_n}{\spaceg} \geq \epsilon_n  \delta_n \right) &= \prob \left( \worstconcentrationvar{\delta_n}{\spaceg} \geq \frac{2c_4 \delta_n^2}{D \log(\ndata)} \right) \\ \nonumber
&= \prob \left( \worstconcentrationvar{\delta_n}{\spaceg} \geq \frac{\delta_n^2}{D \log(\ndata)} + (2c_4-1)\frac{\delta_n^2}{D \log(\ndata) }  \right) \\ \nonumber
&= \prob \left( \worstconcentrationvar{\delta_n}{\spaceg} \geq (1 + \eta)\frac{\delta_n^2}{D \log(\ndata)} -\eta\frac{\delta_n^2}{D \log(\ndata)} +  (2c_4-1)\frac{\delta_n^2}{D \log(\ndata) } \right) \\ \nonumber
&\leq \prob \left( \worstconcentrationvar{\delta_n}{\spaceg} \geq (1 + \eta) {\localrad{\delta_n}{\simplexg}} -\eta\frac{\delta_n^2}{D \log(\ndata)} +  (2c_4-1)\frac{\delta_n^2}{D \log(\ndata) } \right) \\ \nonumber
&\leq \prob \left( \worstconcentrationvar{\delta_n}{\spaceg} \geq (1 + \eta)\PE{}{\worstconcentrationvar{\delta_n}{\spaceg}} +(2c_4-\eta-1)\frac{\delta_n^2}{D \log(\ndata)}  \right) \\ \nonumber
&\leq 3 \exp\left( - \ndata \min \left\{ \frac{(2c_4-\eta)^2 \delta_n^2}{4 D^2 \log^2(\ndata)}, \frac{(2c_4-\eta-1) \delta_n}{L D^2 \log(\ndata)\log(\ndata+1)} \right\} \right) \eqsp ,\nonumber
\end{align*}
where $L$ is a constant depending on $\eta$. For the remaining terms $j \geq 1$, we use the fact that $\simplexg$ contains the entire segment between $0$ and $\zeromfun[\param]$ for any $\param \in \paramsp$, which implies the following scaling property of the local Rademacher complexity:
\begin{equation*}
\localrad{2^j \delta_n}{\simplexg} \leq 2^j \localrad{\delta_n}{\simplexg}.
\end{equation*}
Applying the same concentration argument as above to each shell $j \geq 1$ yields:
\begin{align}
    \label{eq:zleqr}
    \prob \left( \worstconcentrationvar{2^j \delta_n}{\spaceg} \geq \epsilon_n 2^{j-1} \delta_n \right) \\ \nonumber 
    &= \prob \left( \worstconcentrationvar{2^j\delta_n}{\spaceg} \geq 2^{j}\frac{c_4 \delta_n^2}{D \log(\ndata)} \right) \\ \nonumber 
    &=\prob \Big( \worstconcentrationvar{2^j\delta_n}{\spaceg} \geq (1 + \eta)2^{j}\frac{\delta_n^2}{D \log(\ndata)} -\eta2^{j}\frac{\delta_n^2}{D \log(\ndata)} \\ \nonumber 
    &\qquad + (c_4-1)2^{j}\frac{\delta_n^2}{D \log(\ndata) } \Big) \\ \nonumber 
    &\leq \prob \left( \worstconcentrationvar{2^j\delta_n}{\spaceg} \geq (1 + \eta) {\localrad{2^j \delta_n}{\simplexg}} +(c_4 - \eta -1)\frac{2^j\delta_n^2}{D \log(\ndata)} \right) \\ \nonumber 
    &\leq \prob \left( \worstconcentrationvar{2^j\delta_n}{\spaceg} \geq (1 + \eta) \PE{}{\worstconcentrationvar{2^j \delta_n}{\spaceg}} +(c_4 -\eta)\frac{2^j\delta_n^2}{D \log(\ndata)} \right) \\ \nonumber 
    &\leq 3 \exp\left( - \ndata \min \left\{ \frac{(c_4-\eta)^2 \delta_n^2}{4 D^2 \log^2(\ndata)}, \frac{(c_4-\eta) \delta_n}{L D^2 \log(\ndata) \log(\ndata+1)} \right\} \right) \eqsp .
\end{align}
Summing over all $J+1$ shells and using the bound $J \leq \log(D/\delta_n) + 3$, we obtain:
\begin{align*}
    \prob \Biggl( &\frac{| M_n(\param) - M_n(\tmin{\param}) - (M^*(\param) - M^*(\tmin{\param})) |}{\delta_n + \lpnorm{\zeromfun[\param]}[2][\pforward{T}]} \geq \epsilon_n \Biggr) \\
    &\hspace{1.5cm}\leq 3 \exp\Big( - \ndata \min \Big\{ \frac{(2c_4-\eta-1)^2 \delta_n^2}{4 D^2 \log^2(\ndata)}, \frac{(2c_4-\eta-1) \delta_n}{L D^2 \log(\ndata)\log(\ndata+1)} \Big\} \Big) \\
    &\hspace{4cm}+ \sum_{j=1}^{J} 3 \exp\Big( - \ndata \min \Big\{ \frac{(c_4-\eta-1)^2 \delta_n^2}{4 D^2 \log^2(\ndata)}, \frac{(c_4-\eta-1) \delta_n}{L D^2 \log(\ndata) \log(\ndata+1)} \Big\} \Big) \\        
    &\hspace{1.5cm}\leq 6 (J+1) \exp\Big( - \ndata \min \Big\{ \frac{(c_4-\eta-1)^2 \delta_n^2}{4 D^2 \log^2(\ndata)}, \frac{(c_4-\eta-1) \delta_n}{L D^2 \log^2(\ndata)}, \\
    &\hspace{5.2cm} \frac{(2c_4-\eta-1)^2 \delta_n^2}{4 D^2 \log^2(\ndata)}, \frac{(2c_4-\eta-1) \delta_n}{L D^2 \log^2(\ndata)} \Big\} \Big) \\
    &\hspace{1.5cm}\leq 12 J \exp\Big( - \ndata \min \Big\{ \frac{(c_4-\eta-1)^2 \delta_n^2}{4 D^2 \log^2(\ndata)}, \frac{(c_4-\eta-1) \delta_n}{L D^2 \log^2(\ndata)}, \\
    &\hspace{5.2cm} \frac{(2c_4-\eta-1)^2 \delta_n^2}{4 D^2 \log^2(\ndata)}, \frac{(2c_4-\eta-1) \delta_n}{L D^2 \log^2(\ndata)} \Big\} \Big) \\
    &\hspace{1.5cm}\leq 12 \Big( \log\Big(\frac{D}{\delta_n}\Big) + 4 \Big) \exp\Big( - \ndata \min \Big\{ \frac{(c_4-\eta-1)^2 \delta_n^2}{4 D^2 \log^2(\ndata)}, \\
    &\hspace{5.2cm} \frac{(c_4-\eta-1) \delta_n}{L D^2 \log^2(\ndata)}, \frac{(2c_4-\eta-1)^2 \delta_n^2}{4 D^2 \log^2(\ndata)}, \\
    &\hspace{5.2cm} \frac{(2c_4-\eta-1) \delta_n}{L D^2 \log^2(\ndata)} \Big\} \Big) \\
    &\hspace{1.5cm}\lesssim \log\Big(\frac{D}{\delta_n}\Big)\exp\Big( - \ndata \min \Big\{ \frac{(c_4-\eta-1)^2 \delta_n^2}{4 D^2 \log^2(\ndata)}, \\
    &\hspace{5.2cm} \frac{(c_4-\eta-1) \delta_n}{L D^2 \log^2(\ndata)}, \frac{(2c_4-\eta-1)^2 \delta_n^2}{4 D^2 \log^2(\ndata)}, \\
    &\hspace{5.2cm} \frac{(2c_4-\eta-1) \delta_n}{L D^2 \log^2(\ndata)} \Big\} \Big) \eqsp .
\end{align*}
Finally, using condition 2 of \Cref{thm:1_tang}, namely $\ndata \min\left\{\delta_\ndata, \delta_\ndata^2\right\} /(D^2 \log^2 \ndata) \geq \log \log (D/\delta_\ndata)$, the logarithmic can be absorbed into the exponential, giving:
Denoting 
\begin{equation*}
    \Psi_n = \min \left\{(c_4-\eta-1)^2 \delta_n^2, \frac{(c_4-\eta-1) \delta_n}{L} , (2c_4-\eta-1)^2 \delta_n^2, \frac{(2c_4-\eta-1) \delta_n}{L} \right\},
\end{equation*}
we have:
\begin{equation*}
\prob \Biggl( \frac{| M_n(\param) - M_n(\tmin{\param}) - (M^*(\param) - M^*(\tmin{\param})) |}{\delta_n + \lpnorm{\zeromfun[\param]}[2][\pforward{T}]} \geq \epsilon_n \Biggr) \leq c_0  \exp\left(\min\left\{\delta_n, \delta_n^2\right\}\frac{\ndata}{D^2 \log^2 \ndata} - \frac{\ndata}{D^2\log^2(\ndata)} \Psi_n \right) \eqsp . 
\end{equation*}
\end{proof}
Choosing $\eta=0.5$, we have
\begin{multline*}
\prob \Biggl( \frac{| M_n(\param) - M_n(\tmin{\param}) - (M^*(\param) - M^*(\tmin{\param})) |}{\delta_n + \lpnorm{\zeromfun[\param]}[2][\pforward{T}]} \geq \epsilon_n \Biggr) \\
\hspace{1cm}\leq c_0 \exp\Bigg(
\Big( \min\{\delta_n, \delta_n^2\}
- \min \Big\{ (c_4-1.5)^2 \delta_n^2, \frac{(c_4-1.5)\delta_n}{L}, \\
\hspace{4.5cm} (2c_4-1.5)^2 \delta_n^2, \frac{(2c_4-1.5)\delta_n}{L} \Big\} \Big)
\frac{\ndata}{D^2 \log^2 \ndata}
\Bigg) \eqsp.
\end{multline*}
Therefore for $c_4>3$, the minimum over the four terms simplifies, and the exponent becomes negative:
\begin{multline*}
            \prob \Biggl( \frac{| M_n(\param) - M_n(\tmin{\param}) - (M^*(\param) - M^*(\tmin{\param})) |}{\delta_n + \lpnorm{\zeromfun[\param]}[2][\pforward{T}]} \geq \epsilon_n \Biggr) \\
            \leq c_0  \exp\left(\left(\min\left\{\delta_n, \delta_n^2\right\} - \min \left\{ \delta_n, \delta_n^2 \right\} \min\{\frac{1}{L},1\}(c_4-1.5)\right)\frac{\ndata}{D^2 \log^2 \ndata}\right).
\end{multline*}
Choosing $c_4$ such that $\min\{L^{-1},1\}(c_4-1.5) > 1$, defining $c_1 =  \min\{L^{-1},1\}(c_4-1.5) - 1$ and noting $c_3 = 2c_4$ the exponent is strictly negative and we obtain the result.
\begin{proof}[Proof of \Cref{thm:1_tang}]
    We begin by bounding the $L^2(\pforward{T})$ norm of the excess loss. For all $\param \in \paramsp$, by the triangle inequality for $L^2$ norms:
    \begin{equation*}
        \lpnorm{\zeromfun[\param]}[2][\pforward{T}]^2
        \leq 2 \lpnorm{\mfun[\param][\cdot]}[2][\pforward{T}]^2 + 2 \lpnorm{\mfun[\tmin{\param}][\cdot]}[2][\pforward{T}]^2.
    \end{equation*}
    Applying \Cref{lem:tang_11} and the subadditivity of the square root, there exists a constant $c$ such that
    \begin{equation*}
        \lpnorm{\zeromfun[\param]}[2][\pforward{T}]\leq 
        c \left(\left(D\log(\ndata)\kl{\pforward{T}}{\discsolution[\param]{0}}\right)^{1/2}+ \left(D\log(\ndata)\kl{\pforward{T}}{\discsolution[\tmin{\param}]{0}}\right)^{1/2} + \frac{D \log \ndata}{\ndata^{1/2}} \right)\eqsp.
    \end{equation*}
    Next, since $\emin{\param}$ minimizes $M_n$ over $\paramsp$, we have $M_n(\emin{\param}) \leq M_n(\tmin{\param})$, and therefore 
\begin{align*}
    \kl{\pforward{T}}{\discsolution[\param]{0}} - \kl{\pforward{T}}{\discsolution[\tmin{\param}]{0}} 
    &= M^*(\emin{\param}) - M^*(\tmin{\param}) \\
    &\leq M_n(\tmin{\param}) - M_n(\emin{\param}) + M^*(\emin{\param}) - M^*(\tmin{\param}) \\
    &= |M_n(\emin{\param}) - M_n(\tmin{\param}) + M^*(\tmin{\param}) - M^*(\emin{\param})| \eqsp.
\end{align*}
Combining the two bounds above, we can obtain
\begin{align*}
    &\prob \Biggl( \frac{ \kl{\pforward{T}}{\discsolution[\emin{\param}]{0}} - \kl{\pforward{T}}{\discsolution[\tmin{\param}]{0}}}{\delta_n + c \left(\left(D\log(\ndata)\kl{\pforward{T}}{\discsolution[\param]{0}}\right)^{1/2}+ \left(D\log(\ndata)\kl{\pforward{T}}{\discsolution[\tmin{\param}]{0}}\right)^{1/2} + \frac{D \log \ndata}{\ndata^{1/2}} \right)} \leq\epsilon_n \Biggr) \\
    &\quad \geq \prob \Biggl( \frac{|M_n(\emin{\param}) - M_n(\tmin{\param}) + M^*(\tmin{\param}) - M^*(\emin{\param}) |}{\delta_n + c \left(\left(D\log(\ndata)\kl{\pforward{T}}{\discsolution[\emin{\param}]{0}}\right)^{1/2}+ \left(D\log(\ndata)\kl{\pforward{T}}{\discsolution[\tmin{\param}]{0}}\right)^{1/2} + \frac{D \log \ndata}{\ndata^{1/2}} \right)} \leq \epsilon_n \Biggr) \\
    &\quad \geq \prob \Biggl( \frac{|M_n(\emin{\param}) - M_n(\tmin{\param}) + M^*(\tmin{\param}) - M^*(\emin{\param}) |}{\delta_n + \lpnorm{\zeromfun[\emin{\param}]}[2][\pforward{T}]} \geq \epsilon_n \Biggr).
\end{align*}
Applying \Cref{lem:tang_12}, with probability at least $1 - c_0 \exp\Bigl(-c_1 \frac{\ndata \min\{\delta_n,\delta_n^2\}}{D^2 \log^2 \ndata}\Bigr)$ we have:
\begin{align*}
    &\frac{M^*(\emin{\param}) - M^*(\tmin{\param})}{\delta_n + c \left(\left(D\log(\ndata)\kl{\pforward{T}}{\discsolution[\emin{\param}]{0}}\right)^{1/2}+ \left(D\log(\ndata)\kl{\pforward{T}}{\discsolution[\tmin{\param}]{0}}\right)^{1/2} + \frac{D \log \ndata}{\ndata^{1/2}} \right)} \leq \frac{c_3 \delta_n}{D \log(\ndata)}.
\end{align*}
Multiplying both sides by the denominator and rearranging gives:
\begin{align*}
    \kl{\pforward{T}}{\discsolution[\emin{\param}]{0}} - \kl{\pforward{T}}{\discsolution[\tmin{\param}]{0}} &\leq \frac{c_3 \delta_n^2}{D \log(\ndata)} + c \frac{c_3 \delta_n}{\sqrt{D \log(\ndata)}}\left(\sqrt{\kl{\pforward{T}}{\discsolution[\emin{\param}]{0}}}+\sqrt{\kl{\pforward{T}}{\discsolution[\tmin{\param}]{0}} }\right)\\
    &+ \frac{c c_3 \delta_n}{\sqrt{\ndata}}. 
\end{align*}
Using $2\sqrt{ab}\leq \gamma a + \frac{1}{\gamma}b$, valid for any $\gamma>0$, we can state that for any $\bar{\gamma_1},\bar{\gamma_2}>0$:
\begin{align*}
    &\kl{\pforward{T}}{\discsolution[\emin{\param}]{0}} - \kl{\pforward{T}}{\discsolution[\tmin{\param}]{0}} \\
    &\hspace{2.5cm}\leq \frac{c_3 \delta_n^2}{D \log(\ndata)} + \frac{c c_3 \delta_n}{\sqrt{D \log(\ndata)}}\left(\sqrt{\kl{\pforward{T}}{\discsolution[\emin{\param}]{0}}}+\sqrt{\kl{\pforward{T}}{\discsolution[\tmin{\param}]{0}} }\right) + \frac{c_3 \delta_n}{\sqrt{\ndata}} \\
    &\hspace{2.5cm}\leq \frac{c_3 \delta_n^2}{D \log(\ndata)} + c c_3 \Bigg( \frac{ \delta_n^2}{2 \bar{\gamma_1}D \log(\ndata)} +  \frac{\bar{\gamma_1}}{2}\kl{\pforward{T}}{\discsolution[\emin{\param}]{0}} +  \frac{ \delta_n^2}{2\bar{\gamma_2} D \log(\ndata)}\\
    &\hspace{9cm} +  \frac{\bar{\gamma_2}}{2} \kl{\pforward{T}}{\discsolution[\tmin{\param}]{0}}  +  \frac{\delta_n}{\sqrt{\ndata}}  \Bigg).
\end{align*}
Rescaling $\gamma_i = c c_3 \bar{\gamma}_i$ and collecting terms:
\begin{align*}
    \kl{\pforward{T}}{\discsolution[\emin{\param}]{0}} - \kl{\pforward{T}}{\discsolution[\tmin{\param}]{0}} &\leq \frac{c_3 \delta_n^2}{D \log(\ndata)} + c^2 c_3^2 \left( \frac{ \delta_n^2}{2\gamma_1 D \log(\ndata)} + \frac{ \delta_n^2}{2\gamma_2 D \log(\ndata)} +  \frac{\delta_n}{\sqrt{\ndata}} \right)\\
    & +  \frac{\gamma_1}{2}\kl{\pforward{T}}{\discsolution[\emin{\param}]{0}} +  \frac{\gamma_2}{2} \kl{\pforward{T}}{\discsolution[\tmin{\param}]{0}}.
\end{align*}
Moving the $\kl{\pforward{T}}{\discsolution[\emin{\param}]{0}}$ term to the left-hand side, we obtain with the desired probability:
\begin{align}
    (1 -\frac{\gamma_1}{2})\kl{\pforward{T}}{\discsolution[\emin{\param}]{0}} &\leq (1 + \frac{\gamma_2}{2})\kl{\pforward{T}}{\discsolution[\tmin{\param}]{0}} \\
    &+ \frac{c_3 \delta_n^2}{D \log(\ndata)} + c^2 c_3^2 \left( \frac{ \delta_n^2}{2\gamma_1 D \log(\ndata)} + \frac{ \delta_n^2}{2\gamma_2 D \log(\ndata)} +  \frac{\delta_n}{\sqrt{\ndata}} \right).
\end{align}
Setting $\gamma_1 = 1$ to normalize the left-hand side:
\begin{align*}
    &\kl{\pforward{T}}{\discsolution[\emin{\param}]{0}} \leq 2(1 + \gamma_2)\kl{\pforward{T}}{\discsolution[\tmin{\param}]{0}} + \frac{2c_3 \delta_n^2}{D \log(\ndata)} + c^2 c_3^2 \left( \frac{ \delta_n^2}{ D \log(\ndata)} + \frac{ \delta_n^2}{\gamma_2 D \log(\ndata)} +  \frac{2\delta_n}{\sqrt{\ndata}} \right).
\end{align*}
Choosing $c \geq 1$ and grouping all estimation error terms using $\max\{\delta_n, \delta_n^2\}$:
\begin{align*}
\kl{\pforward{T}}{\discsolution[\emin{\param}]{0}} &\leq2(1 + \gamma_2)\kl{\pforward{T}}{\discsolution[\tmin{\param}]{0}} + 2(c^2 c_3^2)\left(\frac{\delta_n^2}{D \log(\ndata)} +  \frac{ \delta_n^2}{ D \log(\ndata)} + \frac{ \delta_n^2}{\gamma_2 D \log(\ndata)} +  \frac{\delta_n}{\sqrt{\ndata}} \right) \\
&\leq 2(1 + \gamma_2)\kl{\pforward{T}}{\discsolution[\tmin{\param}]{0}} + 6(c^2 c_3^2)\left(\frac{\max\{\delta_n,\delta_n^2\}}{D \log(\ndata)} + \frac{\max\{\delta_n,\delta_n^2\}}{\gamma_2 D \log(\ndata)} \right) \\
&= 2(1 + \gamma_2)\kl{\pforward{T}}{\discsolution[\tmin{\param}]{0}} + 6(c^2 c_3^2)\left(1 + \frac{1}{\gamma_2}\right)\left(\frac{\max\{\delta_n,\delta_n^2\}}{D \log(\ndata)} \right) \eqsp . 
\end{align*}
Renaming $\gamma_2 = \gamma$ and $c_2 = 6c^2 c_3^2$, and taking the infimum over $\gamma > 0$, we obtain the desired result.
\end{proof}
We now restrict our attention to the case where $\discsolution{0}$ belongs to the family of normalizing flows \cite{rezende2015variational}. Specifically, we introduce a reference measure with density $\refnfmeas(x)$ with respect to the Lebesgue measure $\lambda_{\R^d}$, and a parametric family of transformations
\begin{equation*}
    \flow: \rset^{\statedim} \times \paramsp \ni (x, \param) \rightarrow \flow[\param][x] \in \rset^{\statedim}\eqsp,
\end{equation*}
such that for all $\param \in \paramsp$, $\flow[\param]: \rset^{\statedim} \rightarrow \rset^{\statedim}$ is an invertible map, and $\discsolution{0}$ is the density of the pushforward measure $\pushforward{\flow[\param]}{\refnfmeas}$. In other words, the model distribution is obtained by pushing the reference density through the learned transformation $\flow[\param]$. By the change of variables formula, the log-density of the model admits the explicit expression
\begin{equation*}
    \log \discsolution{0}(x) = \log\refnfmeas(\invflow[\param][x]) - \log(\det{\jflow[\param][x]})\eqsp,
\end{equation*}
where $\jflow[\param][x]$ denotes the Jacobian of $\flow[\param]$ evaluated at $x$. Under \hypref{assump:tang2021empirical:C1}, which provides regularity conditions on the flow family, the general bound of \Cref{thm:1_tang} can be made explicit, as stated in the following result.
\begin{theorem}{Theorem 3 in \cite{tang2021empirical} adapted to our setting}
    \label{thm:3_tang}
    For clarity and ease of reading we restate here \cref{thm:tang2021empirical:thm3main}.
    Under \hypref{assump:tang2021empirical:A}, \hypref{assump:tang2021empirical:C1}, there exist constants $c_0, c_1, c_2 > 0$ such that 
\begin{align*}
\label{eq:thm:kl_bound_tang_parametric}
\prob \Biggl( 
    &\kl{\pforward{T}} {\discsolution[\emin{\param}]{0}} \leq \inf_{\gamma > 0} \Bigl[ 2 (1+\gamma) \kl{\pforward{T}} {\discsolution[\tmin{\param}]{0}} + c_2 (1+\gamma^{-1}) \frac{ \paramdim D \log(\ndata) \log(\ndata \paramdim)}{\sqrt{\ndata}} \Bigr] \Biggr) \\
    &\hspace{7cm}\geq 1 - c_0 \exp\Bigl(-c_1 \min \left\{ \paramdim \log(\ndata \paramdim) , \frac{\sqrt{\ndata \paramdim \log(\ndata \paramdim)}}{D \log \ndata} \right\}\Bigr)\eqsp.
\end{align*}
\end{theorem}
To prove this result, we need two technical lemmas. 
\begin{lemma}
    \label{lem:tang_13} 
    There exists a constant $0<\rho\leq c_0D \log(\ndata)$ and a constant $c$ such that for any $\delta>0$:
    \begin{equation*}
        \PE{}{\worstconcentrationvar{\delta}{\simplexg}}
        \leq \PE{}{\sup_{\zeromfun \in G, \lpnorm{g}[2][\pforward{T}] \leq r} \left| \frac{1}{\ndata} \sum_{i=1}^{\ndata} \zeromfun[][\state_T^i]\ind_{ A(\state_T^i)}- \PE{}{\zeromfun[][\state_T\ind_{A(\state_T)}]}\right| + c \frac{D \log(\ndata)}{\ndata}} \eqsp , 
    \end{equation*}
    where $A(X) = \{\sup_{\zeromfun \in \spaceg} |\zeromfun[][X]| \leq \rho\}$.
\end{lemma}
\Cref{lem:tang_13} shows that the expected worst-case empirical fluctuation can be reduced to a truncated version of the problem, up to a negligible remainder term. 
This truncation allows us to work with bounded functions, on which covering number arguments can be applied. To this end, we introduce the empirical distance between two functions $\zeromfun, \zeromfun' \in \simplexg$:
\begin{equation*}
    \distn{\zeromfun}{\zeromfun'} = \sqrt{\frac{1}{\ndata}\sum_{i=1}^{n}\left(\zeromfun[][\state_T^i] - \zeromfun[][\state_T^i]'\right)^2} \eqsp , 
\end{equation*}
and the truncated function classes
\begin{itemize}
    \item $\spaceg_{A} = \{{\zeromfun}_A\eqdef\zeromfun\ind_{ A(\cdot)} |  \zeromfun \in \spaceg\} \eqsp,$
    \item $\simplexg_{A} = \{{\zeromfun}_A\eqdef\zeromfun\ind_{ A(\cdot)} |  \zeromfun \in \simplexg\} \eqsp,$
\end{itemize}
to which the distance $\distn$ naturally extends. We also define the diameter of $\spaceg_A$ under this distance:
\begin{equation*}
\maxray = \sup_{{\zeromfun}_A, {\zeromfun}'_A \in {\spaceg}_A} \distn{{\zeromfun}_A}{{\zeromfun'}_A} \leq 2 \rho \eqsp.
\end{equation*}
The second lemma provides a bound on the Dudley entropy integral for the truncated class, which will be the key ingredient to control the local Rademacher complexity.
\begin{lemma}\label{lem:tang_14}
    Under \hypref{assump:tang2021empirical:A}, there exists a constant $c$ depending on $\onorm{\mlip[\state_T]}[1]$ and $c_{\infty}=\sup_{\param \in \paramsp} \|\param\|_\infty$ such that 
\begin{align*}
                &\PE{}{\int_0^{\maxray}\sqrt{\paramdim \log\left(\frac{6c_{\infty}\sqrt{\frac{\paramdim}{\ndata}\sum_{i=1}^{\ndata}\mlip[\state_T^i]}}{\epsilon}\right) + \log\left( \frac{2\rho}{\epsilon} \right) \rmd \epsilon }} \\
                &\hspace{3cm} \leq c \left( \rho \sqrt{\paramdim} \sqrt{-\PE{}{\frac{\maxray^2}{4 \rho^2}}\log\left(\PE{}{\frac{\maxray^2}{4 \rho^2}} \right) + \PE{}{\frac{\maxray^2}{4 \rho^2}}} + \PE{}{\maxray \sqrt{\paramdim \log (\paramdim)}}\right) \eqsp .
\end{align*}
\end{lemma}
We are now ready to prove \Cref{thm:3_tang}.
\begin{proof}[Proof of \Cref{thm:3_tang}]  
    We choose
    \begin{equation*}
        \delta_n = c_3\sqrt{\paramdim \frac{\log(\ndata \paramdim)}{\ndata}}D \log(\ndata) \eqsp,
    \end{equation*}
    which is the natural parametric rate scaled by the complexity parameters $D$ and $\log(\ndata)$. It is straightforward to verify that for sufficiently large $c_3$, condition $$ \frac{\ndata \min\{\delta_n,\delta_n^2\}}{D^2 \log^2 \ndata} \geq \log \log \frac{D}{\delta_n}$$ holds for all $\ndata \geq \ndata_0$.
    We now focus on bounding $\localrad{\delta_n}{\simplexg_A}$ and show that this is sufficient for obtaining the conclusion of \cref{thm:1_tang} with this choise of $\delta_\ndata$.
    The strategy is to derive a recursive bound on $\localrad{\delta_n}{\simplexg_A}$ via a covering number argument.
    Conditionally on $\{\state_T^1,..., \state_T^{\ndata}\}$, $\distn$ is a deterministic distance. Denoting by $\nballs{{\simplexg}_A}{{\distn}}{\epsilon}$ and $\nballs{\spaceg_A}{{\distn}}{\epsilon}$ the $\epsilon$-covering numbers of $\simplexg_A$ and $\spaceg_A$ respectively, we claim that:
    $$\log \nballs{\simplexg_A}{ \distn}{\epsilon} \leq \nballs{\spaceg_A}{ \distn}{\epsilon/2} + \log \frac{2 \rho}{\epsilon} \eqsp . $$
    To see this, note that any $\bar{\zeromfun}_A \in \simplexg_A$ can be written as $\bar{\zeromfun}_A = a {\zeromfun}_A$ for some ${\zeromfun}_A \in \spaceg_A$ and $a \in (0,1]$. Letting ${\zeromfun}_A^c$ be the center of the nearest $\epsilon/2$-ball in $\spaceg_A$, and choosing $k$ such that $k\frac{2\epsilon}{2 \rho} < a \leq (k+1)\frac{2\epsilon}{2 \rho}$, the triangle inequality gives:
$$ {\distn{ \bar{\zeromfun}_A }{ (k+1)\frac{\epsilon}{2\rho}{\zeromfun}_A^c }} \leq {\distn{ \bar{\zeromfun}_A }{ (k+1)\frac{\epsilon}{2\rho}{\zeromfun}_A }} + {\distn{ (k+1)\frac{\epsilon}{2\rho}{\zeromfun}_A }{ (k+1)\frac{\epsilon}{2\rho}{\zeromfun}_A^c }} \leq \frac{\epsilon}{2} + \frac{\epsilon}{2} = \epsilon \eqsp , $$
which establishes the covering number bound for $\simplexg_A$.
We now use \hypref{assump:tang2021empirical:C1} to bound $\nballs{\spaceg_A}{ \distn}{\epsilon/2}$. For any ${\zeromfun}_A, {\zeromfun}'_A \in \spaceg_A$, the Lipschitz condition gives:
\begin{align*}
        \distn{{\zeromfun}_A}{{\zeromfun}'_A}^2 &= \frac{1}{n}\sum_{i=1}^{\ndata}\left(\zeromfun[\param][\state_T^i]\ind_{ A(\state_T^i)} - \zeromfun[\param'][\state_T^i]\ind_{ A(\state_T^i)}\right)^2 \\
        &\leq \frac{1}{n}\sum_{i=1}^{\ndata} \left( \mlip[\state_T^i] \|\param - \param'\|_2\right)^2 \\
        &= \left(\frac{1}{n}\sum_{i=1}^{\ndata} \mlip[\state_T^i]^2 \right) \|\param - \param'\|_2^2 \eqsp. \\
\end{align*}
This means that the distance $\distn$ on $\spaceg_A$ is dominated by the data-dependent distance on the parameter space $\distnparam{\param}{\param'} = \sqrt{\frac{\sum_{i=1}^{\ndata} \mlip[\state_T^i]^2}{n}} \|\param - \param'\|$, and therefore:
$$\nballs{\spaceg_A}{\distn}{\epsilon/2} \leq \nballs{\paramsp}{\distnparam}{\epsilon/2} \eqsp . $$
Since $\paramsp$ is compact with $\sup_{\param \in \paramsp}\|\param\|_\infty < c_{\infty}$, it is contained in a Euclidean ball of radius $c_{\infty}\sqrt{\paramdim}$. A ball of radius $\epsilon/2$ in the $\distnparam$ metric corresponds to a Euclidean ball of radius $\epsilon\sqrt{n}/(2(\sum_{i=1}^{\ndata} \mlip[\state_T^i]^2)^{1/2})$, so :
$$\nballs{\paramsp}{\distnparam}{\epsilon/2} \leq \left(\frac{6 c_{\infty}\sqrt{\paramdim}}{\epsilon} \cdot \sqrt{\frac{\sum_{i=1}^{\ndata} \mlip[\state_T^i]^2}{n}}\right)^{\paramdim}.$$
Combining the two covering bounds, we obtain the following number of balls bound on the space $\simplexg_A$:
\begin{equation*}
    \log \nballs{\simplexg_A}{\distn}{\epsilon}  \leq \paramdim \log\left(\frac{6c_{\infty}\sqrt{\paramdim}}{\epsilon} \cdot \sqrt{\frac{\sum_{i=1}^{\ndata} \mlip[\state_T^i]^2}{n}}\right) + \log \frac{2\rho}{\epsilon} \eqsp .
\end{equation*}
We now apply Dudley's inequality \citet[Equation 8.13]{Vershynin2018} to control the Rademacher complexity. To verify the sub-Gaussian increment condition required by Dudley's theorem, we first establish a $\psi_2$ norm bound for the Rademacher average. 
\begin{lemma}\label{lem:bernoulli_rademacher}
For a fixed vector $c = (c_1,...,c_\ndata)$ and $\epsilon_i$ indipendent Bernoulli with values in 
$ \{ -1,1 \} $, the random variable 
$$S = \frac{1}{n}\sum_{i=1}^{\ndata} c_i \epsilon_i$$
satisfies , 
\begin{equation*}
    \|S\|_{\psi_2} \leq \frac{2\|c\|_2}{\ndata} \eqsp .
\end{equation*}
\end{lemma}
\begin{proof}
The moment generating function of $S$ satisfies $\PE{}{e^{\lambda S}}\leq e^{\lambda^2 \|c\|_2^2 /2 \ndata^2}$ for all $\lambda>0$ (using $\cosh(x)\leq e^{x^2/2}$), which gives via Markov's inequality:
\begin{equation*}
    \PP{S>u} \leq e^{-\lambda u} \PE{}{e^{\lambda S}} \leq e^{-\lambda u + \lambda^2 \|c\|_2^2 /2 \ndata^2} \eqsp .
\end{equation*}
Optimizing over $\lambda$ yields $\PP{S>u} \leq e^{-\ndata^2 u^2/(2\|c\|_2^2)}$. Using the two-sided bound $\PP{|S|>u} \leq 2e^{-\ndata^2 u^2/(2\|c\|_2^2)}$, we compute the moment generating function of $S^2$:
\begin{equation*}
    \PE{}{e^{S^2/t^2}} = \int_0^\infty \frac{2u}{t^2} e^{u^2/t^2} \PP{|S|>u} \rmd u 
    \leq \frac{4}{t^2}\int_0^\infty u\, e^{-\alpha u^2} \rmd u = \frac{2}{t^2 \alpha} \eqsp ,
\end{equation*}
where $\alpha = \frac{\ndata^2}{2\|c\|_2^2} - \frac{1}{t^2} > 0$ and we used $\int_0^\infty u e^{-\alpha u^2} du = \frac{1}{2\alpha}$.
Choosing $t = \frac{2\|c\|_2}{\ndata}$ gives $\alpha = \frac{\ndata^2}{4\|c\|_2^2}$, so that:
\begin{equation*}
    \PE{}{e^{S^2/t^2}} \leq \frac{2}{\frac{4\|c\|_2^2}{\ndata^2} \cdot \frac{\ndata^2}{4\|c\|_2^2}} = 2 \eqsp ,
\end{equation*}
and hence by definition of the $\psi_2$ norm:
\begin{equation*}
    \|S\|_{\psi_2} \leq \frac{2\|c\|_2}{\ndata} \eqsp .
\end{equation*}
\end{proof}
\Cref{lem:bernoulli_rademacher} shows that, conditionally on $\{\state_T^1,..., \state_T^{\ndata}\}$, the Rademacher process indexed by $\simplexg_A$ has sub-Gaussian increments with respect to $\distn$: for all ${\zeromfun[A]}, {\zeromfun[A]'} \in \simplexg_A$,
$$\onorm{\left|\frac{1}{n} \sum_{i=1}^{\ndata} \epsilon_i \zeromfun[A][\state_T^i] \right|- \left|\frac{1}{n} \sum_{i=1}^{\ndata} \epsilon_i \zeromfun[A][\state_T^i]' \right|}[2] \leq \frac{2}{\sqrt{n}}\distn{\zeromfun[A]}{\zeromfun[A]'} \eqsp , $$
i.e., \citet[Equation 8.1]{Vershynin2018} is satisfied  $K = \frac{2}{\ndata}$. Dudley's inequality therefore gives, for a universal constant $C_{d}$:
\begin{equation*}
    \localrad{r}{\simplexg_A} \leq \PE{}{ 
    C_d\frac{2}{\sqrt{n}} \int_{0}^{\maxray} \sqrt{\log \nballs{\simplexg_A}{\distn}{u}} \rmd u} \eqsp .
\end{equation*}
From this and \ref{lem:tang_14}, we obtain a constant $c_1$ such that
\begin{equation*}
        \localrad{r}{\simplexg_A} \leq c_1 \left( \rho \sqrt{\paramdim} \sqrt{-\PE{}{\frac{\maxray^2}{4 \rho^2}}\log\left(\PE{}{\frac{\maxray^2}{4 \rho^2}} \right) + \PE{}{\frac{\maxray^2}{4 \rho^2}}} + \PE{}{\maxray \sqrt{\paramdim \log (\paramdim)}}\right) \eqsp .
\end{equation*}
Before using this last equation we also note that 
\begin{align*}
        \PE{}{\maxray^2} &\leq \PE{}{\sup_{\zeromfun[A] \in \simplexg_A , \lpnorm{\zeromfun[A]}[2][\pforward{T}] \leq r} \frac{2}{n} \sum_{i=1}^{n}\zeromfun[A][\state_T^i]^2} \\
        &\leq \PE{}{\sup_{\zeromfun[A] \in \simplexg_A , \lpnorm{\zeromfun[A]}[2][\pforward{T}] \leq r} \frac{2}{n} \sum_{i=1}^{n}\left(\zeromfun[A][\state_T^i] - \PE{}{\zeromfun[A][\state_T]}\right)^2} + 4 r^2 
\end{align*}
that combined with \citet[Equation 3.84]{wainwright2019highdimensional} and the fact that $\worstconcentrationvar{r}{\simplexg_A} \leq 2 \localrad{r}{\simplexg_A}$ allows to state the existence of a constant $c$ such that 
\begin{equation*}
        \PE{}{\maxray^2} \leq c (r^2 + \rho \localrad{r}{\simplexg_A}) \eqsp . 
\end{equation*}
We can write thanks to the fact that $-x\log x + x$ is growing for $x \leq 1$ 
\begin{align*}
-\PE{}{\frac{\maxray^2}{4 \rho^2}} \log \left( \PE{}{\frac{\maxray^2}{4 \rho^2}} \right) &+ \PE{}{\frac{\maxray^2}{4 \rho^2}} \\
&= -c \PE{}{\frac{\maxray^2}{4 c \rho^2}} \log \left( \PE{}{\frac{\maxray^2}{4 c \rho^2}}\right) + \log(c) + c \PE{}{\frac{\maxray^2}{4 c \rho^2}}  \\
&= c \left( - \PE{}{\frac{\maxray^2}{4 c \rho^2}} \log \left( \PE{}{\frac{\maxray^2}{4 c \rho^2}} \right) + \PE{}{\frac{\maxray^2}{4 c \rho^2}} \right) - \log(c)\PE{}{\frac{\maxray^2}{4  \rho^2}}\\
&\leq c \left( - \frac{r^2 + \rho \localrad{r}{\simplexg_A}}{4\rho^2} \log \left( \frac{r^2 + \rho \localrad{r}{\simplexg_A}}{4 \rho^2} \right) \right. \notag \\
&\left. + \frac{r^2 + \rho \localrad{r}{\simplexg_A}}{4\rho^2} \right) - \log(c)\PE{}{\frac{\maxray^2}{4 \rho^2}}\\
&= c \frac{r^2 + \rho \localrad{r}{\simplexg_A}}{4\rho^2}(-\log \left( \frac{r^2 + \rho \localrad{r}{\simplexg_A}}{4 \rho^2}\right) +1 ) - \log(c)\PE{}{\frac{\maxray^2}{4  \rho^2}}\\
&= c \frac{r^2 + \rho \localrad{r}{\simplexg_A}}{4\rho^2}(\log \left( \frac{4 \rho^2}{r^2 + \rho \localrad{r}{\simplexg_A}}\right) +1 ) - \log(c)\PE{}{\frac{\maxray^2}{4  \rho^2}}\\
&\leq c \frac{r^2 + \rho \localrad{r}{\simplexg_A}}{4\rho^2}(2\log \left( \frac{\rho}{r}\right) +\log(4)+1 ) - \log(c)\PE{}{\frac{\maxray^2}{4  \rho^2}}\\
&\leq c \frac{r^2 + \rho \localrad{r}{\simplexg_A}}{4\rho^2}(2\log \left( \frac{\rho}{r}\right) +\log(4)+1 ) 
\end{align*}
On the other hand we have that by Jensen Inequality
\begin{equation*}
    \PE{}{\maxray} \leq \sqrt{\PE{}{\maxray^2}} \leq \sqrt{c(r^2 + \localrad{r}{\simplexg_A})}
\end{equation*}
Combining the two, we obtain
\begin{equation*}
    \localrad{r}{\simplexg_A} \leq \frac{c_1}{\sqrt{\ndata}}\left( \rho \sqrt{\paramdim} \sqrt{c \frac{r^2 + \rho \localrad{r}{\simplexg_A}}{4\rho^2}}\left(\sqrt{2\log \left( \frac{\rho}{r}\right) +\log(4)+1} +  \sqrt{\log\paramdim}\right) \right)
\end{equation*}
By $\sqrt{a} + \sqrt{b} \leq \sqrt{2}\sqrt{a+b}$
that brings to
\begin{equation*}
    \localrad{r}{\simplexg_A} \lesssim \frac{1}{ \ndata} \sqrt{\paramdim} \sqrt{r^2 + \rho \localrad{r}{\simplexg_A}}\sqrt{\log \left( \frac{\rho}{r}\right) + 1 + \log\paramdim}
\end{equation*}
We show that by a reduction ad absurdum that fixing
\begin{equation*}
    \localrad{\delta_\ndata}{\simplexg_A}\leq \frac{\delta_\ndata^2}{D \log \ndata} \eqsp . 
\end{equation*}
In fact, supposing that 
\begin{equation*}
    \localrad{\delta_\ndata}{\simplexg_A}> \frac{\delta_\ndata^2}{D \log \ndata} \eqsp , 
\end{equation*}
and remembering that $\rho \leq c_0 D \log \ndata$
we obtain the following chain : 
\begin{align*}
    \localrad{r}{\simplexg_A} &\lesssim \frac{1}{n} \sqrt{p} \sqrt{\delta_n^2 + \rho \localrad{\delta_n}{\simplexg_A}} \sqrt{\log \left( \frac{\rho}{\delta_n} \right) + 1 + \log \paramdim} \\
    &\leq \frac{1}{n} \sqrt{\paramdim} \sqrt{ D \log n \localrad{\delta_n}{\simplexg_A} + c_0 D \log n \localrad{\delta_n}{\simplexg_A}}\\
    & \cdot \hspace{4cm} \sqrt{\frac{1}{2}\log \left( \frac{c_0 n}{c_3 \paramdim \log(n \paramdim)} \right) + 1 + \log \paramdim} \\
    &\leq \frac{1}{n} \sqrt{\paramdim} \sqrt{ D \log n \localrad{\delta_n}{\simplexg_A} + c_0 D \log n \localrad{\delta_n}{\simplexg_A}} \\
    & \hspace{4cm}\cdot \sqrt{\frac{1}{2}\log \left( \frac{c_0 n}{c_3 \paramdim \log(n \paramdim)} \right) + 1 + \log \paramdim} \\
    &\leq \frac{1}{n} \sqrt{\paramdim} \sqrt{ D \log n \localrad{\delta_n}{\simplexg_A} + c_0 D \log n \localrad{\delta_n}{\simplexg_A}} \\
    & \hspace{4cm} \cdot \sqrt{\frac{1}{2}\log \left(  n\right) + 1 + \log(c_0 )- \frac{1}{2}\log\left( {c_3 \paramdim \log(n \paramdim)} \right) + \log \paramdim} \\
    &\leq \frac{1}{n} \sqrt{\paramdim} \sqrt{ D \log n \localrad{\delta_n}{\simplexg_A} + c_0 D \log n \localrad{\delta_n}{\simplexg_A}} \sqrt{\frac{1}{2}\log \left( \ndata\right)  + \log \paramdim} \\
    & \lesssim \frac{1}{n} \sqrt{\paramdim} \sqrt{ D \log n \localrad{\delta_n}{\simplexg_A}} \sqrt{\log \ndata + \log\paramdim} \eqsp . 
\end{align*}
which implies 
\begin{equation*}
    \localrad{r}{\simplexg_A} \lesssim \frac{1}{n}\paramdim D \log n \log \ndata \paramdim = \frac{1}{c_3} \frac{\delta_\ndata^2}{D \log \ndata} \eqsp . 
\end{equation*}
which for big enough $c_3$ brings to an abdurdity.
Thanks to \cref{lem:tang_13} we know that 
\begin{align*}
    \PE{}{\worstconcentrationvar{\delta_\ndata}{\simplexg}} &\leq \worstconcentrationvar{\delta_\ndata}{\simplexg_A} + c \frac{D \log \ndata}{\ndata} \\
    &\leq \localrad{\delta_\ndata}{\simplexg_A} + c \frac{D \log \ndata}{\ndata} \\
    &= c_3 \frac{\log \ndata \paramdim}{\ndata}\paramdim D \log \ndata + c \frac{D \log \ndata}{\ndata} \\
    &\leq 2 c_3 \frac{\log \ndata \paramdim}{\ndata}\paramdim D \log \ndata \\
    &\leq 2 \frac{\delta_n^2}{D \log \ndata} \eqsp . 
\end{align*}
Using this inequality in the chain of lower probabilty bounding of \cref{eq:zleqr} in \cref{lem:tang_12} we obtain the result. 
\end{proof}

\subsection{Estimating $D$ in different settings}
\label{appendix:sec:estimating_D}
We focus now on upper bounding the value of $D$ in \cref{thm:tang2021empirical:thm3main} in different scenarios.
To do so, we will focus on a specific type of normalizing flow, namely those defined by the coupling layers, such as for example the Real NVP \citep{dinh2017realnvp}.
\subsubsection{Marginal preserving coupling Layers.}
We define slightly-modified version of the Real NVP that preserve the marginals. 
\begin{definition}[Marginal tail preserving RealNVP]
    \label{def:modrealnvp}
    Let $k < \statedim$, $m: \rset^k \rightarrow [-a, a]^{\statedim-k}$, $b: \rset^k \rightarrow \rset^{\statedim-k}$, $\varsigma: \{1, \cdots, \statedim\} \rightarrow \{1, \cdots, \statedim\}$ an injective function and for $(i, j) \in \{1, \cdots, \statedim\}^2$ with $i < j$
    $\varsigma_{i:j}$ the restriction of $\varsigma$ to the domain $\{i, \cdots, j\}$.

    For $x \in \rset^\statedim$, define $x_{\varsigma_{i:j}} = (x_{\varsigma(i)}, \cdots, x_{\varsigma(j)})$. Then, we define the coupling layer $(m, b, \varsigma, k)$ as
    \begin{equation}
        \label{eq:flow:coupling_layer}
        \flow[m, b, \varsigma, k]: \rset^{\statedim} \ni x \rightarrow  (x_{\varsigma_{1:k}}, x_{\varsigma_{k+1:\statedim}} \exp(m(x_{\varsigma_{1:k}})) + b(x_{\varsigma_{1:k}}))_{\varsigma^{-1}} \in \rset^\statedim \eqsp.
    \end{equation}
\end{definition}

Note that for the usual Real NVP, choosing the reference density $\refnfmeas$ to share the tail behavior of $\pforward{T}$ does not, in general, suffice for the boundness of the log ratio of densities. 
Indeed, the flow transformation $\flow[\param]$ can redistribute mass across coordinates and break the tail alignment, as the coordinate permutations are not usually "undone". 
A simple illustration is a target on $\R^2$ with independent components, Gaussian along the first axis and Student-$t$ along the second, with $\refnfmeas$ having the same component-wise tail structure. 
A coordinate swap $\flow[\param](x_1, x_2) = (x_2, x_1)$ then yields a model with mismatched tails along both axes; the ratio $\pforward{T}/\discsolution[\param]{0}$ explodes (Gaussian over polynomial), and \hypref{assump:tang2021empirical:A} fails.
That is the reason we introduced \cref{def:modrealnvp}
It is then straightforward to note that the standard properties of the RealNVP are preserved.
\begin{proposition}
    For any $(m, b, \varsigma, k)$, $\flow[m, b, \varsigma, k]$ is a bijection. 
    Furthermore,
    \begin{equation*}
        \|\flow[m, b, \varsigma, k][x] - x\| \leq C \|x\| + \|b(x_{\varsigma_{1:k}})\|\eqsp,
    \end{equation*}
    where $C = \max\{|1 - \exp(-a)|,|1 - \exp(a)| \}$.
    
    If $m, b$ are smooth, the bijection is smooth and 
    \begin{equation*}
        \det \jflow[m, b, \varsigma, k][x] = \prod_{i=1}^{\statedim - k}\exp(m(x_{\varsigma_{1:k}}))_{i} \eqsp.
    \end{equation*}
\end{proposition}
\begin{corollary}
    If $b$ is $L$-Lipschitz and $b(0) = 0$, then
    \begin{equation*}
         \|\flow[m, b, \varsigma, k][x] - x\| \leq C \|x\|\eqsp,
    \end{equation*}
    where $C = \max\{|1 - \exp(-a)|,|1 - \exp(a)| \} + L$.
\end{corollary}
\begin{corollary}
    \label{cor:nf:invflow:lip}
     If $b$ is $L$-Lipschitz and $b(0) = 0$, then
     \begin{equation*}
         \|\invflow[m, b, \varsigma, k][x] - x\| \leq C \|x\|\eqsp,
    \end{equation*}
    where $C = \max\{|1 - \exp(-a)|^{-1},|1 - \exp(a)|^{-1} \} + L\max\{|1 - \exp(-a)|^{-1},|1 - \exp(a)|^{-1} \}$.
\end{corollary}
Property \cref{cor:nf:invflow:lip} is useful for bounding $\onorm{\log \fwmargd{T}{\state_T} - \log \refnfmeas(\state_T)}[1]$ because of the following proposition, which is a direct application of 
Taylor's integral formula.
\begin{proposition}
    \label{prop:nf:taylor_ratio}
    Let $\discsolution[\theta]{0}(x) = \refnfmeas(\invflow[\theta][x]) \det \jinvflow[\theta][x]$.%
    Then, for all $x \in \rset^d$, it holds that
    \begin{multline*}
        \left|\log\frac{\fwmargd{T}{x}}{\discsolution[\theta]{0}(x)}\right| \leq \left|\log\frac{\fwmargd{T}{x}}{\refnfmeas(x)}\right| +  \left|\log \det \jinvflow[\theta][x]\right| \\
        + \left\|\int_{0}^{1}\nabla \log \refnfmeas(\invflow[\theta][x]t + (1-t)x)\rmd t\right\| \left\|x-\invflow[\theta][x]\right\|\eqsp.
    \end{multline*}
\end{proposition}
\begin{proof}
    \begin{align*}
        \left|\log\frac{\refnfmeas(\invflow[\theta][x])}{\refnfmeas(x)}\right| 
        &= \left|\int_{0}^{1}\nabla \log \refnfmeas(\invflow[\theta][x]t + (1-t)x)^t(\invflow[\theta][x] - x)\rmd t\right|\\
        &\leq \left\|\int_{0}^{1}\nabla \log \refnfmeas(\invflow[\theta][x]t + (1-t)x)\rmd t\right\| \left\|x-\invflow[\theta][x]\right\|\eqsp.
    \end{align*}
\end{proof}
From now on, all theorems will be established under assumption~\hypref{assum:flow:modrealnvp}.
\begin{hypH}
    \label{assum:flow:modrealnvp}
    There exists $r \in \mathbb{N}$ such that $\flow[\param] = \flow[m_{\param_r^m}, b_{\param_r^b}, \varsigma_r, k_r] \circ \cdots \circ \flow[m_{\param_1^m}, b_{\param_1^b}, \varsigma_1, k_1] $ where for each $i \in \{1, \cdots r\}$, $m_{\param_i^m}$ is a Neural network with bounded final activation and $b_{\param_i^b}$ is a bias-free neural network.
\end{hypH}
\subsubsection{Heavy tail case}
We consider the case of $\fwmarg{0}=\mudata$ has at most exponential tails.
\begin{hypH}
    \label{assum:nf:heavy_Tails}
    $\onorm{\state_0}[1] < \infty$. Also, that $\|\nabla \log \refnfmeas(x)\| \leq L_{\refnfmeas}$ for all $x \in \rset^\statedim$.
\end{hypH}
We are now ready to state the main result of this section.
\begin{proposition}
        \label{prop:nf:onormratio_bound:heavy}
    Assume assumptions~\hypref{assum:nf:heavy_Tails} and \hypref{assum:flow:modrealnvp}. 
    If $\onorm{\log \nofrac{\fwmargd{T}{\state_T}}{\log \refnfmeas(\state_T)}}[1] < R_T$, then
    \begin{equation*}
        \onorm{\sup_{\param \in \paramsp}\left|\log\frac{\fwmargd{T}{x}}{\discsolution[\param]{0}(x)}\right|}[1] \leq R_T + L_{\refnfmeas} C_{\operatorname{flow, 1}} \left(\onorm{\state_0}[1] + \sigma_T \onorm{Z}[1] \right) + C_{\operatorname{flow, 2}} \eqsp,
    \end{equation*}
    where
    \begin{align*}
        C_{\operatorname{flow, 1}} &= \left(\max\{|1 - \exp(-a)|^{-1},|1 - \exp(a)|^{-1} \} + L\max\{|1 - \exp(-a)|^{-1},|1 - \exp(a)|^{-1} \}\right)^{r(d-k)} \eqsp, \\
        C_{\operatorname{flow, 2}} &=  r a (d-k)\eqsp.
    \end{align*}
\end{proposition}
\begin{proof}
    It is straightforward to see that $\sup_{\param \in \paramsp }\left|\log \det \jinvflow[\theta][x]\right| \leq r a$.
    By \cref{cor:nf:invflow:lip}, we have that $\sup_{\param \in \paramsp}\left\|x-\invflow[\theta][x]\right\| \leq C_{\operatorname{flow, 1}} \|x\|$.
    Therefore, by \cref{prop:nf:taylor_ratio} using the fact that $\|\nabla \log \refnfmeas(x)\| \leq L_{\refnfmeas}$ and  by the triangular inequality we can conclude.
\end{proof}
\begin{corollary}
    Suppose the prior $\refnfmeas$ is chosen as the time-dependent law of $\sigma_T\,\tilde{U}$, where $\tilde{U} \sim \tilde{\refnfmeas}$ has score bounding constant $L_{\tilde{\refnfmeas}}$. Then the score of $\refnfmeas$ is bounded by $L_{\refnfmeas} = L_{\tilde{\refnfmeas}}/\sigma_T$, and
    \begin{equation*}
        \onorm{\sup_{\param \in \paramsp}\left|\log\frac{\fwmargd{T}{\state_T}}{\discsolution[\param]{0}(\state_T)}\right|}[1] 
        \;\leq\; R_T \;+\; L_{\tilde{\refnfmeas}}\, C_{\operatorname{flow},1} \left(\sigma_T^{-1} \, \onorm{\state_0}[1] +\onorm{Z}[1]\right) \;+\; C_{\operatorname{flow},2} \eqsp.
    \end{equation*}
\end{corollary}
We focus now on the necessary conditions one must impose to $\refnfmeas$ to ensure that $\onorm{\log \nofrac{\fwmargd{T}{\state_T}}{\log \refnfmeas(\state_T)}}[1] < R_T$.
\begin{hypH}
    \label{assum:nf:bound_ratio}
    There exists $\densratioub > 0$ such that for all $x \in \rset^d$, $\fwmargd{T}{x} / \refnfmeas(x) \leq \densratioub$.
    Furthermore, $\int \indi{\fwmarg{T} \geq \refnfmeas}(x) \fwmargd{T}{x}\rmd x < 1$ and there exists $\kappa > \log \left(\densratioub / \int \indi{\fwmarg{T} \geq \refnfmeas}(x) \fwmargd{T}{x}\rmd x\right)$ %
    such that
    \begin{equation*}
        \int \left(\frac{\refnfmeas(x)}{\fwmargd{T}{x}}\right)^{\frac{1}{\kappa}}\refnfmeas(x)\rmd x < \infty\eqsp.
    \end{equation*}
\end{hypH}
\begin{lemma}
    Assume assumption~\hypref{assum:nf:bound_ratio}. Then, 
    \begin{equation*}
        \onorm{\left|\log \fwmargd{T}{\state_T} - \log\refnfmeas(\state_T)\right|}[1] \leq -\log \densratioub / \log\left(\int \indi{\fwmarg{T} \geq \refnfmeas}(x) \fwmargd{T}{x}\rmd x\right)\eqsp.
    \end{equation*}
\end{lemma}
\begin{proof}
    Let $\lambda \in [1, \infty)$.
    \begin{align*}
        \PE{}{\exp\frac{1}{\lambda}\left|\log\frac{\fwmargd{T}{\state_T}}{\refnfmeas(\state_T)}\right|} 
            &= \PE{}{\indi{\fwmarg{T}\geq \refnfmeas}(\state_T)\left(\frac{\fwmargd{T}{\state_T}}{\refnfmeas(\state_T)}\right)^{\frac{1}{\lambda}}} 
            + \PE{}{\indi{\fwmarg{T} < \refnfmeas}(\state_T)\left(\frac{\refnfmeas(\state_T)}{\fwmargd{T}{\state_T}}\right)^{\frac{1}{\lambda}}}\\
            &\leq \PE{}{\indi{\fwmarg{T} \geq \refnfmeas}(\state_T)}^{1-\frac{1}{\lambda}}\densratioub^{\frac{1}{\lambda}} 
            + \PE{}{\indi{\fwmarg{T} < \refnfmeas}(\state_T)}^{1-\frac{1}{\lambda}} \PE{U \sim \refnfmeas}{\indi{\fwmarg{T} < \refnfmeas}(U)}^{\frac{1}{\lambda}} \\
            &\leq 1 + \PE{}{\indi{\fwmarg{T} \geq \refnfmeas}(\state_T)}^{1-\frac{1}{\lambda}}\densratioub^{\frac{1}{\lambda}} \eqsp.
    \end{align*}
    Thus, provided that $\lambda \geq  -\log \densratioub / \log\left(\int \indi{\fwmarg{T} \geq \refnfmeas}(x) \fwmargd{T}{x}\rmd x\right)$ we have that $\PE{}{\exp\frac{1}{\lambda}\left|\log\frac{\fwmargd{T}{\state_T}}{\refnfmeas(\state_T)}\right|} \leq 2$.
\end{proof}
\subsubsection{Light-tailed distributions}
\begin{proposition}
    \label{prop:nf:onormratiobound:gaussian}
    If $\onorm{\state_0}[2] < \infty$, then for all $\sigma_T > 0$ it holds that 
    \begin{equation}
        \onorm{\log\!\left(\frac{\fwmargd{T}{\state_T}}{\gausspdf(\state_T; 0, \sigma_T^2\Id)}\right)}[1] 
        \leq \frac{\log 2 + 1}{\log 2} + \frac{\onorm{\state_0}[2]^2}{\sigma_T^2}\left(\frac{1}{\log 2} + \frac{8\statedim}{6}\right) + \frac{1}{2}\frac{\onorm{\state_0}[2]^4}{\sigma_T^4} \eqsp.
    \end{equation}
\end{proposition}
\begin{proof}
By Young's inequality, for any $\zeta \in (0,1)$,
\begin{align*}
    (1 - \zeta^2)\|x\|^2 + (1 - \zeta^{-2})\|x_0\|^2 \leq \|x - x_0\|^2 \leq (1 + \zeta^2)\|x\|^2 + (1 + \zeta^{-2})\|x_0\|^2 \eqsp.
\end{align*}
Substituting into 
$
\fwmargd{T}{x} = (2\pi\sigma_T^2)^{-d/2}\,\PE{}{\exp(-\|x-\state_0\|^2/(2\sigma_T^2))}$ and dividing by $\gausspdf(x; 0, \sigma_T^2\Id) = (2\pi\sigma_T^2)^{-d/2}\exp(-\|x\|^2/(2\sigma_T^2)) \eqsp , 
$
we obtain
\begin{align*}
    \exp\!\left(-\tfrac{\zeta^2}{2\sigma_T^2}\|x\|^2\right) A_+ 
    \;\leq\; \frac{\fwmargd{T}{x}}{\gausspdf(x; 0, \sigma_T^2\Id)} 
    \;\leq\; \exp\!\left(\tfrac{\zeta^2}{2\sigma_T^2}\|x\|^2\right) A_- \eqsp,
\end{align*}
where 
\begin{equation*}
    A_\pm \coloneqq \PE{}{\exp\!\left(-\tfrac{1 \pm \zeta^{-2}}{2\sigma_T^2}\|\state_0\|^2\right)}\eqsp.
\end{equation*}
Taking logarithms,
\begin{equation*}
    \left|\log \frac{\fwmargd{T}{x}}{\gausspdf(x; 0, \sigma_T^2\Id)}\right| 
    \;\leq\; \tfrac{\zeta^2}{2\sigma_T^2}\|x\|^2 + \max\!\left(\log A_-,\; -\log A_+\right) \eqsp.
\end{equation*}

Choosing 
\begin{equation*}
    \zeta^2 \coloneqq \frac{\onorm{\state_0}[2]^2 / \sigma_T^2}{\onorm{\state_0}[2]^2 / \sigma_T^2 + 2} \in (0,1)\eqsp,
\end{equation*}
one has $(\zeta^{-2}-1)/(2\sigma_T^2) = 1/\onorm{\state_0}[2]^2$, so by the definition of the Orlicz $\psi_2$-norm,
\begin{equation*}
    A_- = \PE{}{\exp\!\left(\frac{\|\state_0\|^2}{\onorm{\state_0}[2]^2}\right)} \leq 2\,, \qquad \text{hence} \qquad \log A_- \leq \log 2\eqsp.
\end{equation*}
 
The function $u \mapsto \exp(-c u)$ is convex on $\rset_{\geq 0}$ for any $c > 0$, so by Jensen's inequality
\begin{equation*}
    A_+ \geq \exp\!\left(-\tfrac{1+\zeta^{-2}}{2\sigma_T^2}\E\!\left[\|\state_0\|^2\right]\right)\,, \qquad \text{hence} \qquad -\log A_+ \leq \tfrac{1+\zeta^{-2}}{2\sigma_T^2}\E\!\left[\|\state_0\|^2\right]\eqsp.
\end{equation*}
With our choice of $\zeta$, $1+\zeta^{-2} = 2\left(\onorm{\state_0}[2]^2 + \sigma_T^2\right)/\onorm{\state_0}[2]^2$, and from the inequality $1+u \leq \exp(u)$ together with the definition of the $\psi_2$-norm we have $\E[\|\state_0\|^2] \leq \onorm{\state_0}[2]^2$. Therefore,
\begin{equation*}
    -\log A_+ \leq \frac{\onorm{\state_0}[2]^2}{\sigma_T^2} + 1\eqsp.
\end{equation*}
We have 
$\max(\log A_-, -\log A_+) \leq \log 2 + \onorm{\state_0}[2]^2/\sigma_T^2 + 1$, and substituting $\zeta^2 = \left(\onorm{\state_0}[2]^2/\sigma_T^2\right) / \left(\onorm{\state_0}[2]^2 / \sigma_T^2 + 2\right)$ yields the pointwise bound
\begin{equation*}
    \left|\log \frac{\fwmargd{T}{x}}{\gausspdf(x; 0, \sigma_T^2\Id)}\right| 
    \leq \log 2 + 1 + \frac{\onorm{\state_0}[2]^2}{\sigma_T^2} + \frac{ \onorm{\state_0}[2]^2}{\sigma_T^2}\left(\frac{\onorm{\state_0}[2]^2}{\sigma_T^2} + 2\right)^{-1} \frac{\|x\|^2}{2\sigma_T^2}\eqsp.
\end{equation*}
Applying $\onorm{\,\cdot\,}[1]$ at $x = \state_T$ and using the triangular inequality together with the identity $\onorm{a}[1] = |a|/\log 2$ for any constant $a$ and $\onorm{\|\state_T\|^2}[1] = \onorm{\state_T}[2]^2$, we obtain
\begin{equation*}
    \onorm{\log\!\left(\frac{\fwmargd{T}{\state_T}}{\gausspdf(\state_T; 0, \sigma_T^2\Id)}\right)}[1] 
    \leq \frac{\log 2 + 1}{\log 2} + \frac{\onorm{\state_0}[2]^2}{\sigma_T^2 \log 2} + \frac{ \onorm{\state_0}[2]^2}{\sigma_T^2}\left(\frac{\onorm{\state_0}[2]^2}{\sigma_T^2} + 2\right)^{-1} \frac{\onorm{\state_T}[2]^2}{2\sigma_T^2} \eqsp .
\end{equation*}
Using the fact that $\onorm{\state_T}[2]^2 \leq 2\left(\onorm{\state_0}[2]^2 + 8 d \sigma_T^2/3 \right)$  and that $\onorm{\state_0}[2]^2 / \sigma_T^2 + 2\geq 2 $ we conclude.
\end{proof}
\begin{corollary}
    Assume assumption~\hypref{assum:flow:modrealnvp} and that $\onorm{\state_0}[2] < \infty$. Then, for all $\sigma_T > 0$,
    \begin{align*}
        \onorm{\sup_{\param \in \paramsp}\left|\log\frac{\fwmargd{T}{\state_T}}{\discsolution[\param]{0}(\state_T)}\right|}[1] 
        &\leq \frac{\log 2 + 1}{\log 2} + \frac{\onorm{\state_0}[2]^2}{\sigma_T^2}\left(\frac{1}{\log 2} + \frac{8\statedim}{6}\right) + \frac{1}{2}\frac{\onorm{\state_0}[2]^4}{\sigma_T^4} \\
        &\quad + \sigma_T C_{\operatorname{flow, 1}} \left(\frac{\onorm{\state_0}[1]}{\sigma_T} +  \statedim M_{\psi_1}\right) + \frac{\sigma_T^2}{2} C_{\operatorname{flow, 2}}^2 \left(\frac{\onorm{\state_0}[2]^2}{\sigma_T^2} + \statedim \frac{M_{\psi_1}^2}{\sigma_T^2}\right) \eqsp,
    \end{align*}
    where $M_{\psi_1} = \onorm{Z}[1]$ with $Z\sim \mathcal{N}(0, 1)$ and 
    \begin{align*}
        C_{\operatorname{flow, 1}} &= \left[\max\{|1 - \exp(-a)|^{-1},|1 - \exp(a)|^{-1} \}(1 + L)\right]^{r(d-k)} \eqsp, \\
        C_{\operatorname{flow, 2}} &=  r a (d-k)\eqsp.
    \end{align*}
\end{corollary}
\begin{proof}
    Simply note that
    \begin{multline*}
        \int_{0}^{1} \nabla \log \refnfmeas\!\left(t\,\invflow[\theta][x] + (1-t)x\right) \rmd t 
        = - \frac{1}{\sigma_T^2}\int_{0}^{1} \left(t\,\invflow[\theta][x] + (1-t)x\right) \rmd t \\
        = -\frac{1}{2\sigma_T^2}\!\left(\invflow[\theta][x] - x\right) \eqsp.
    \end{multline*}
    Therefore,
    \begin{align*}
        &\left|\log \det \jinvflow[\theta][x]\right| + \left\|\int_{0}^{1}\nabla \log \refnfmeas(\invflow[\theta][x]t + (1-t)x)\rmd t\right\| \left\|x-\invflow[\theta][x]\right\|\\
        &\qquad \leq \left|\sum_{\ell=1}^{r} \sum_{i=1}^{d - k} m_{\theta_{\ell}^m}(x_{\varsigma_{1:d}})_{i}\right| + \frac{1}{2\sigma_T^2}\left\|x-\invflow[\theta][x]\right\|^2 \\
        &\qquad \leq C_{\operatorname{flow, 2}}\|x\| + \frac{C_{\operatorname{flow, 1}}}{2\sigma_T^2}\|x\|^2 \eqsp.
    \end{align*}
\end{proof}

\section{Appendix experiments}\label{app:appendix_experiments}

In this section we report experiments details and further extend the discussion of results. 

\subsection{The targets and the employed score estimators}
\paragraph{Targets and score estimators.}
 Since the score of $\pforward{t}$ is not available in closed form for any heavy-tailed target, we deliberately consider two distributions, each chosen to isolate a different aspect of our method.
The first is a Gaussian mixture model (GMM), for which the score is analytically tractable: this allows us to study the effect of the initialization in isolation, free from any bias introduced by score approximation.

The second is a heavy-tailed (HT) target, for which the score must be approximated. This second setting reflects the regime our method is ultimately designed for, and lets us assess whether the gains observed on the GMM persist once an inexact score is plugged in.
For the HT target we consider two families of score estimators: MCMC-based approximations of $\PE{}{\state_0 \mid \state_t = x}$, using HMC~\cite{neal2011hmc}, the Barker proposal~\cite{livingstone2020barker}, and NUTS~\cite{hoffman2011nuts}, and a learned neural-network denoiser $\mathrm{NN}_{\theta}$, implemented as an MLP with EDM-style preconditioning.
The MCMC samplers are used to produce approximate samples from the conditional law $\pforward{0 \mid t}(y \mid x)$, which is known only up to a normalization constant; the approximate conditional mean is then obtained by averaging these samples.
Comparing several MCMC samplers alongside the learned denoiser is meant as a robustness check: any improvement that survives across heterogeneous score approximations is unlikely to be an artefact of one specific estimator. 

We work in dimension $d=2$ throughout this section. This is a deliberate choice: the variance of MCMC-based score estimators grows quickly with $d$, and in higher dimensions the resulting score becomes too noisy to cleanly separate the contribution of the initialization from that of score approximation error, which is precisely the effect we wish to isolate.

The two targets are constructed as follows. The GMM target is a 25-component mixture with anisotropic covariances and Dirichlet-sampled weights, designed to exhibit a rich multimodal structure on which the initialization can be stress-tested. The HT target is a mixture of $n=4$ multivariate Student-$t$ components, each with $\nu=3$ degrees of freedom, anisotropic covariances, and Dirichlet-sampled weights. We use fewer components in the HT case because the bias of MCMC-based score approximations grows with the number of modes; keeping $n$ small ensures that the score estimators remain reliable enough to draw meaningful conclusions. Both targets are normalized to have zero mean and unit global standard deviation, so that comparisons are not confounded by differences in scale. The two targets are visualized in \cref{fig:gmm_ht_targets}.

For both targets, denoising follows the EDM scheduler of~\cite{karras2022elucidating} with $\sigma_{\max} = 3$, $\sigma_{\min} = 2 \times 10^{-4}$, $\rho = 2$, and $40$ discretization steps. We evaluate our method at intermediate truncation points $\sigma_T$ along this schedule, which corresponds to running the reverse process from a short rather than full horizon. Sampling is performed with a first-order stochastic solver in the VE parameterization~\cite{ho2020denoising, karras2022elucidating}.
\begin{figure}[!h]
  \centering
  \includegraphics[width=0.6\textwidth]{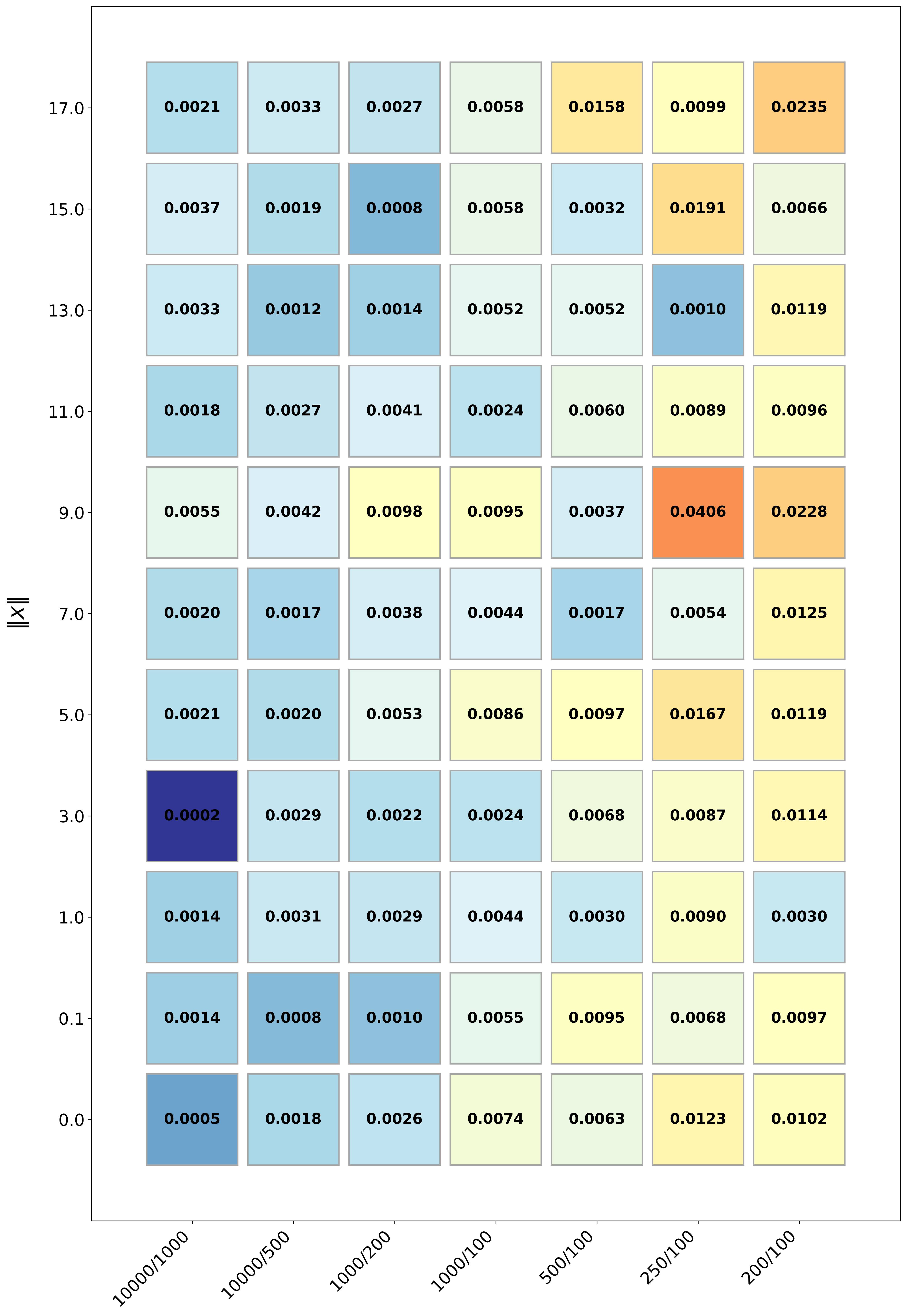}
  \caption{%
    Denoising quality for the NUTS sampler at $\sigma_t=1.0$.
    The heatmap reports the $\text{diff}_{\text{MCMC}}$ metric, which compares posterior means across configurations
    (rows: input norms $\|x\|$; columns: $(n_{\text{mcmc}}, n_{\text{warmup}})$ pairs).
    Warmer colors (yellow/red) indicate larger deviations from the oracle, while cooler colors (blue) indicate smaller deviations.
    The selected configuration $(1000,\,200)$ attains accuracy comparable to the oracle at a markedly lower computational cost.
  }
  \label{fig:mcmc_nuts_quality}
\end{figure}
\paragraph{MCMC for denoising.}
Running denoising through MCMC samplers is computationally expensive even in dimension $d=2$. We therefore aim to balance computational cost against denoising accuracy by carefully tuning the number of MCMC samples ($n_{\text{mcmc}}$) and warmup steps ($n_{\text{warmup}}$) used to estimate the posterior mean.
For each configuration, defined by the MCMC method, the pair $(n_{\text{mcmc}}, n_{\text{warmup}})$, the noise level $\sigma_t$, and the input norm, we draw $N_{\text{repeats}}=200$ independent posterior expectation estimations. We assess accuracy by comparing this estimate against an oracle baseline through the metric
\begin{equation}
  \text{diff}_{\text{MCMC}} = \left\| \frac{1}{N_{\text{repeats}}} \sum_{r=1}^{N_{\text{repeats}}} \hat{\mu}_{0,\text{config}}^{(r)} - \frac{1}{N_{\text{repeats}}} \sum_{r=1}^{N_{\text{repeats}}} \hat{\mu}_{0,\text{oracle}}^{(r)} \right\|_2,
\end{equation}
where $\hat{\mu}_{0,\text{oracle}}^{(r)}$ denotes the $r$-th posterior mean approximation of $\PE{}{\state_0 \mid \state_t=x}$ obtained under the oracle configuration: high-precision NUTS with $n_{\text{mcmc}}=10000$ and $n_{\text{warmup}}=1000$.
We benchmarked several configurations for all three MCMC methods (NUTS, Barker, and HMC) over noise levels $\sigma_t \in [0.01, 2.0]$ and different input $x = \lambda \cdot (1,1)$ ($\lambda>0$) of the posterior $\PE{}{\state_0 \mid \state_t=x}$ with norms $\|x\| \in [0, 17]$. Based on this empirical analysis, we adopted $n_{\text{mcmc}}=1000$ and $n_{\text{warmup}}=200$ as a favorable trade-off, preserving denoising quality across all tested settings while substantially reducing runtime. This configuration was applied uniformly to all MCMC-based denoisers. Results for NUTS are reported in \cref{fig:mcmc_nuts_quality}.
For noise levels $\sigma_t \in [0.01, 2.0]$ and input norms $\|x\| \in [0, 17]$, we selected the configuration $n_{\text{mcmc}}=1000$ and $n_{\text{warmup}}=200$ as an effective computational trade-off. This configuration was then applied uniformly across all MCMC-based denoisers (NUTS, Barker, HMC). 
Results for $\sigma_t=1$ and MCMC NUTS are displayed in \cref{fig:mcmc_nuts_quality}.
\paragraph{Diffusion and Flow Architectures.}
For the diffusion denoiser neural network we trained a MLP like neural network using training strategies from \cite{karras2022elucidating}. See \cref{tab:toy:ht:config} for details on the denoiser network.
\begin{table}[htbp]
    \scriptsize
    \centering
    \caption{Configuration of the MLP denoiser used for the HT case. The model uses four fully connected blocks with GroupNorm and SiLU activations. The noise schedule is continuous with log-normal sampling. FNet refers to the denoiser of \cite{karras2022elucidating} and we used Adam optimizer \cite{kingma2015adam}.}    
\label{tab:toy:ht:config}
    \begin{tabular}{|c|c|c|c|c|c|c|c|c|}
        \toprule
        Learning rate & Batch size &  Dataset size & Blocks & Channel Multipliers & Normalization & Activation & Positional Embedding \\
        \midrule
         $10^{-4}$ & 20000 & $10^3-10^4-10^5$ & 4 & [2,4,4,2] & GroupNorm & SiLU & FNet-style  \\
        \bottomrule
    \end{tabular}
    \vspace{0.2cm}
\end{table}

For the normalizing flows we trained coupling flows using the FlowJAX library. We used a t$-$Student prior (scale=$\sigma_T$, $\nu=3$) for the HT case and a Gaussian prior (scale=$\sigma_T$). Details on the flow architecture and training scheme are reported in \cref{tab:flow:config}.
\begin{table}[htbp]
    \footnotesize
    \centering
    \caption{Configuration of the normalizing flow models used for heavy-tail and GMM cases. Both models use coupling flows with triangular affine transformations. Training employs Adam optimizer with dynamic or fixed noise strategies. The flow operates in a noise-conditioned setting where samples are corrupted with Gaussian noise during training.}
            \label{tab:flow:config}
    \begin{tabular}{|c|c|c|c|c|c|c|c|}
        \toprule
        Dimension & Flow Layers & NN Width & NN Depth & Activation & Learning Rate & Batch Size & N. Epochs \\
        \midrule
        2 & 5 & 50 & 3 & ReLU6 & $10^{-3}$ & 1000 & 3000 \\
        \bottomrule
    \end{tabular}
    \vspace{0.2cm}
    \end{table}
\begin{table}[ht]
  \centering
  \small
  \caption{Bulk MSW (mean $\pm$ std) for all configurations.}
  \label{tab:bulk_msw_full}
  \begin{tabular}{llccc}
    \toprule
      & Init & $\sigma=0.8012415$ & $\sigma=1.05527906$ & $\sigma=1.24404052$ \\
    \midrule
    \multirow{3}{*}{NUTS}
    & $\pinf$      & 0.074 $\pm$ 0.001 & 0.037 $\pm$ 0.006 & 0.025 $\pm$ 0.003 \\
    & $p_\theta$    & 0.014 $\pm$ 0.002 & 0.014 $\pm$ 0.002 & 0.012 $\pm$ 0.004 \\
    & $\vec{p}_{T}$ & 0.012 $\pm$ 0.002 & 0.012 $\pm$ 0.003 & 0.012 $\pm$ 0.003 \\
    \midrule
    \multirow{3}{*}{$\text{NN}_{\param}$}
    & $\pinf$      & 0.074 $\pm$ 0.001 & 0.035 $\pm$ 0.009 & 0.027 $\pm$ 0.001 \\
    & $p_\theta$    & 0.012 $\pm$ 0.001 & 0.011 $\pm$ 0.002 & 0.012 $\pm$ 0.001 \\
    & $\vec{p}_{T}$ & 0.011 $\pm$ 0.002 & 0.011 $\pm$ 0.002 & 0.011 $\pm$ 0.002 \\
    \midrule
    \multirow{3}{*}{BARKER}
    & $\pinf$      & 0.065 $\pm$ 0.012 & 0.032 $\pm$ 0.004 & 0.019 $\pm$ 0.003 \\
    & $p_\theta$    & 0.017 $\pm$ 0.001 & 0.013 $\pm$ 0.001 & 0.013 $\pm$ 0.002 \\
    & $\vec{p}_{T}$ & 0.015 $\pm$ 0.000 & 0.015 $\pm$ 0.001 & 0.015 $\pm$ 0.000 \\
    \midrule
    \multirow{3}{*}{HMC}
    & $\pinf$      & 0.072 $\pm$ 0.001 & 0.033 $\pm$ 0.006 & 0.022 $\pm$ 0.001 \\
    & $p_\theta$    & 0.015 $\pm$ 0.001 & 0.013 $\pm$ 0.002 & 0.014 $\pm$ 0.003 \\
    & $\vec{p}_{T}$ & 0.013 $\pm$ 0.002 & 0.013 $\pm$ 0.002 & 0.013 $\pm$ 0.002 \\
    \midrule
    Flow & -- & \multicolumn{3}{c}{0.015 $\pm$ 0.001} \\
    \midrule
    Reference  & -- & -- &0.002 $\pm$ 0.001 & -- \\
    \bottomrule
  \end{tabular}
\end{table}

\begin{figure}[!h]
  \centering
  \begin{tabular}{cccc}
    \includegraphics[width=0.32\textwidth]{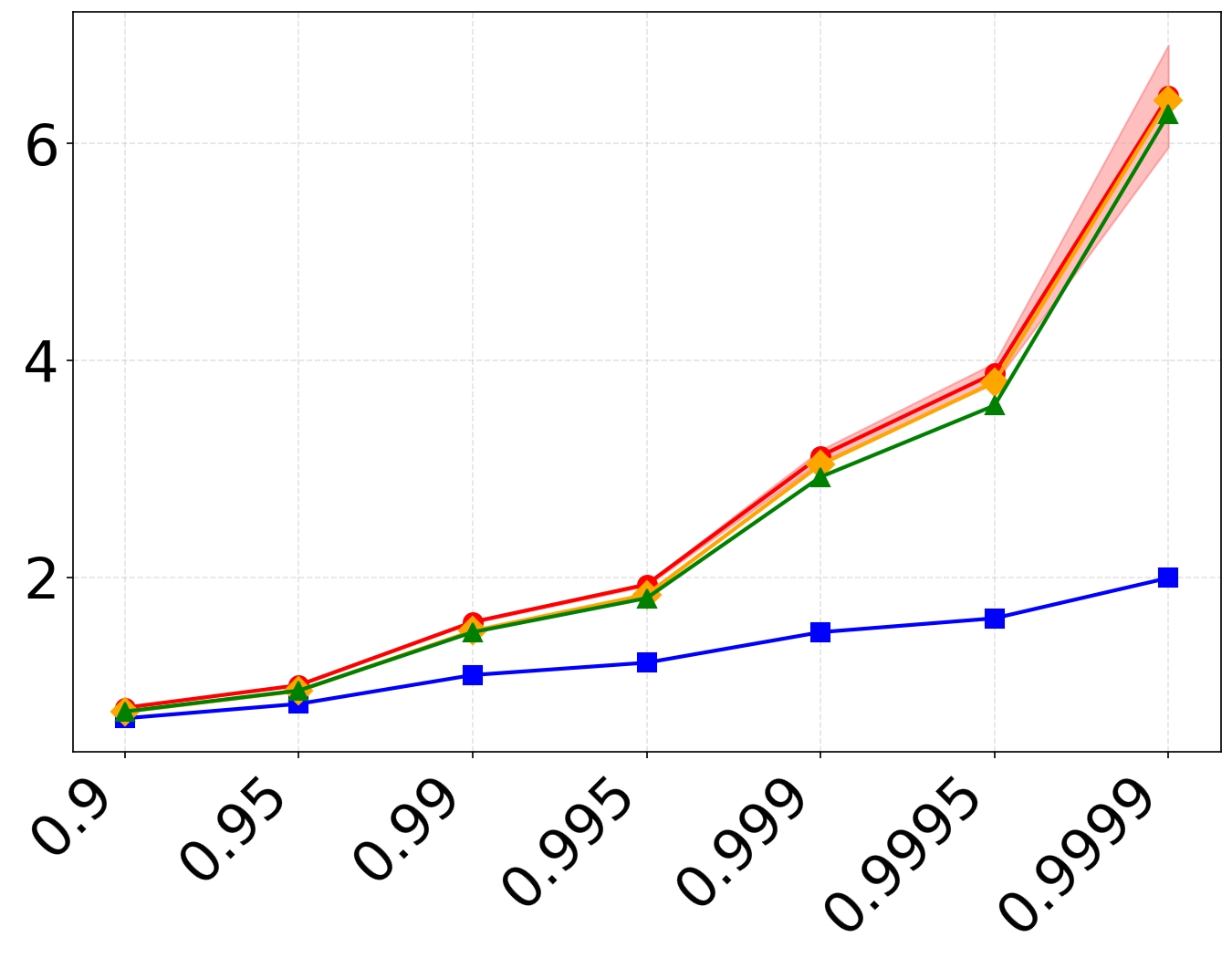} &
    \includegraphics[width=0.32\textwidth]{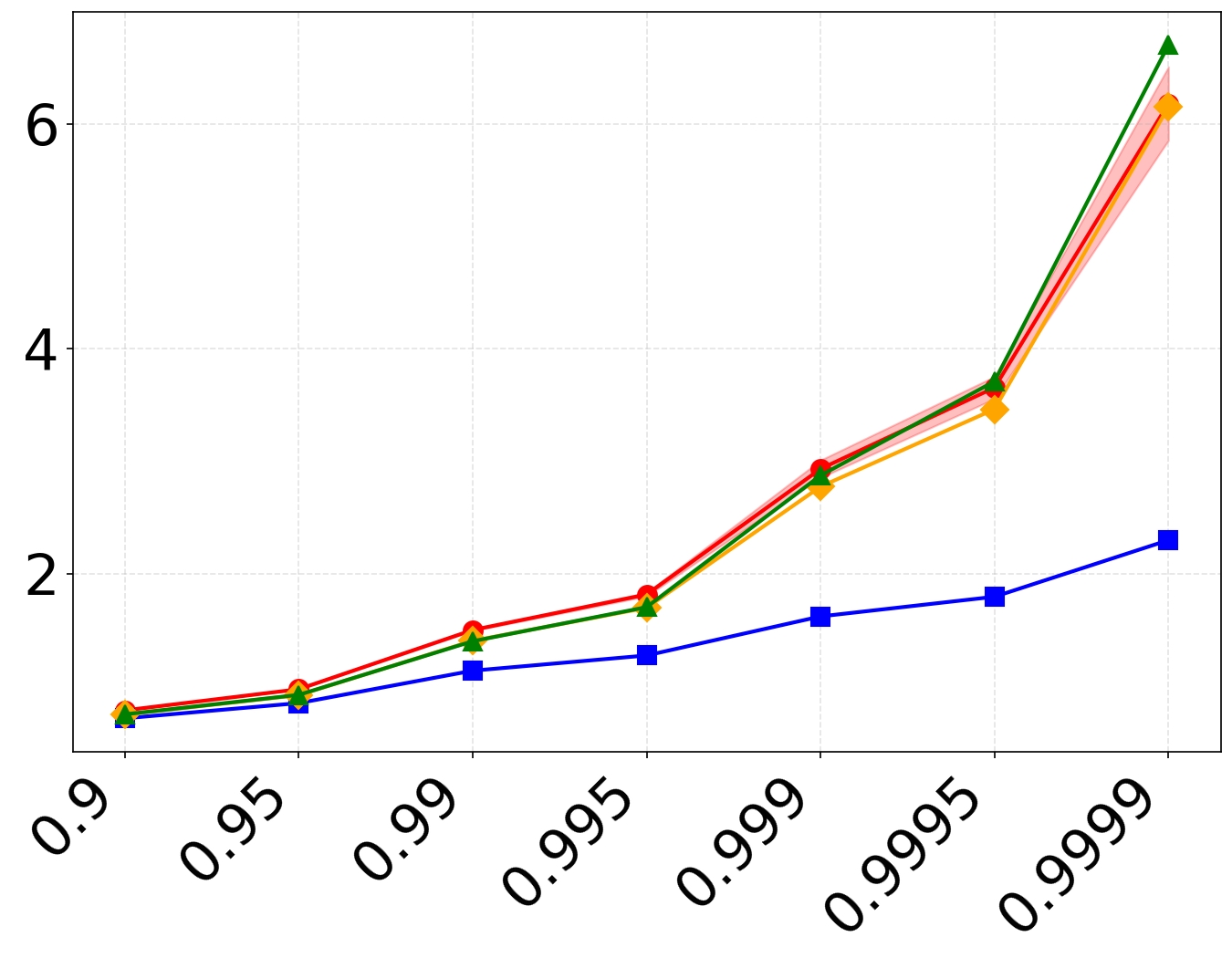} &
    \includegraphics[width=0.32\textwidth]{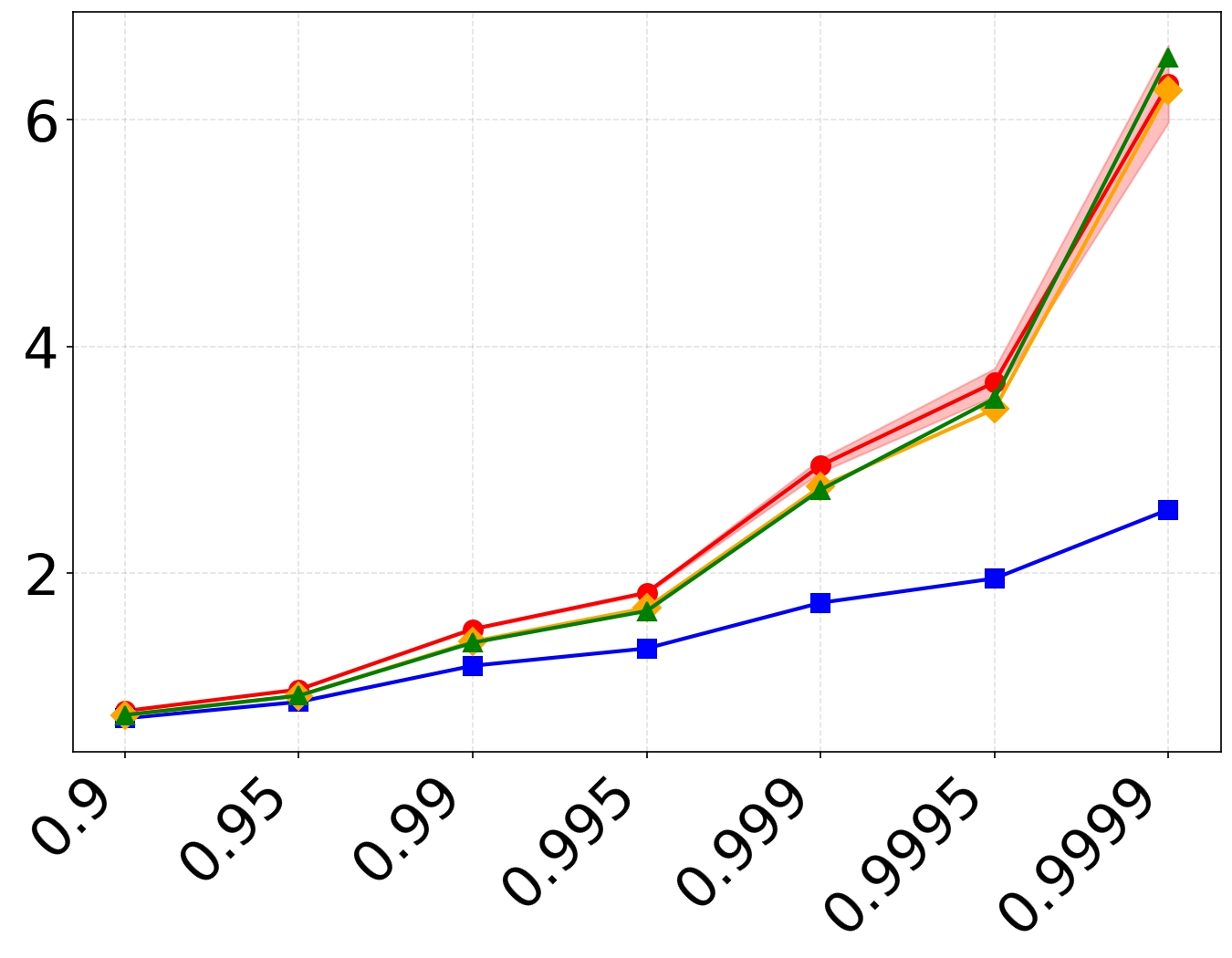} \\
    \small (a) NUTS, $\sigma_T=0.801$ & (b) NUTS, $\sigma_T=1.055$ & (c) NUTS, $\sigma_T=1.244$ \\[0.5em]
    \includegraphics[width=0.32\textwidth]{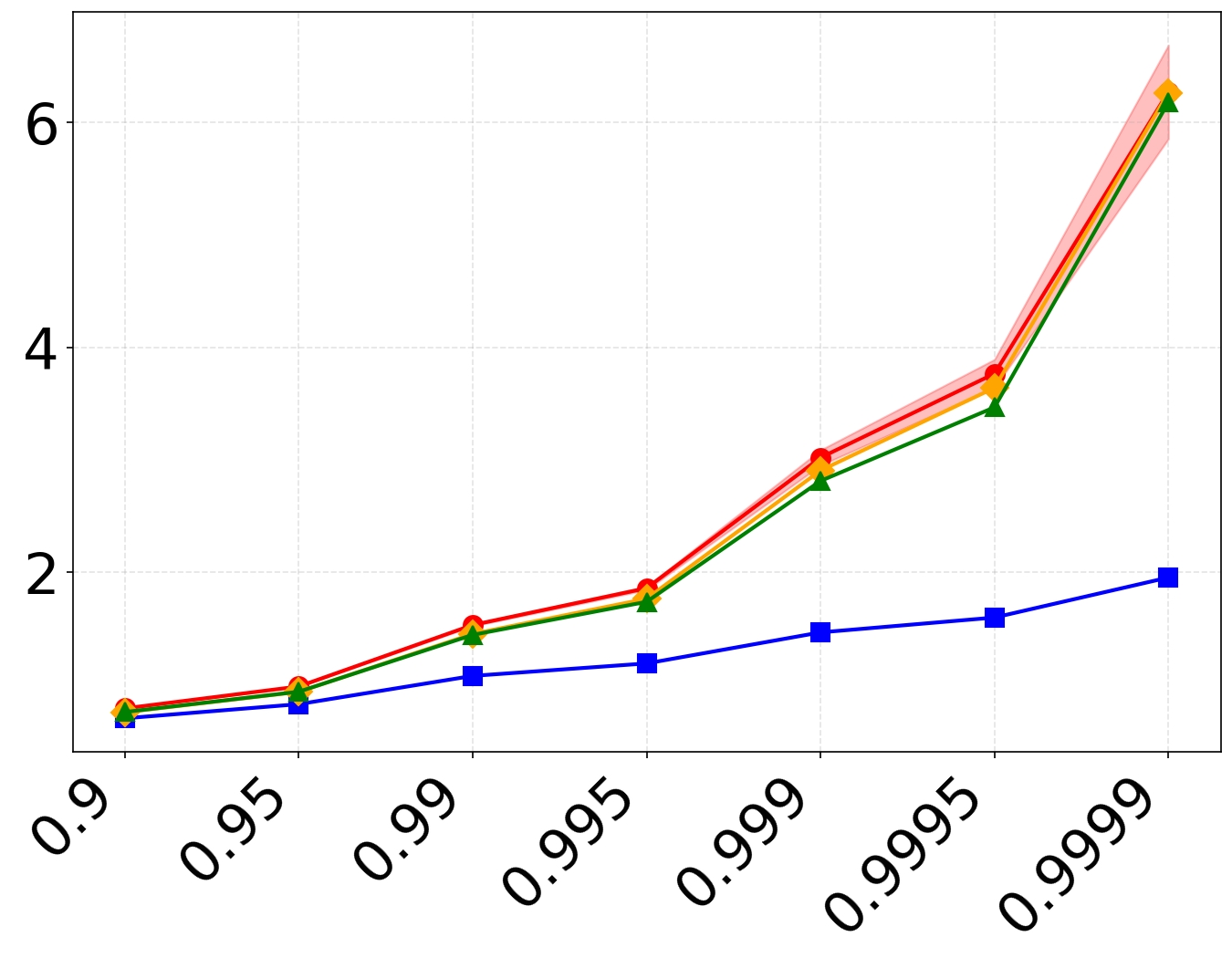} &
    \includegraphics[width=0.32\textwidth]{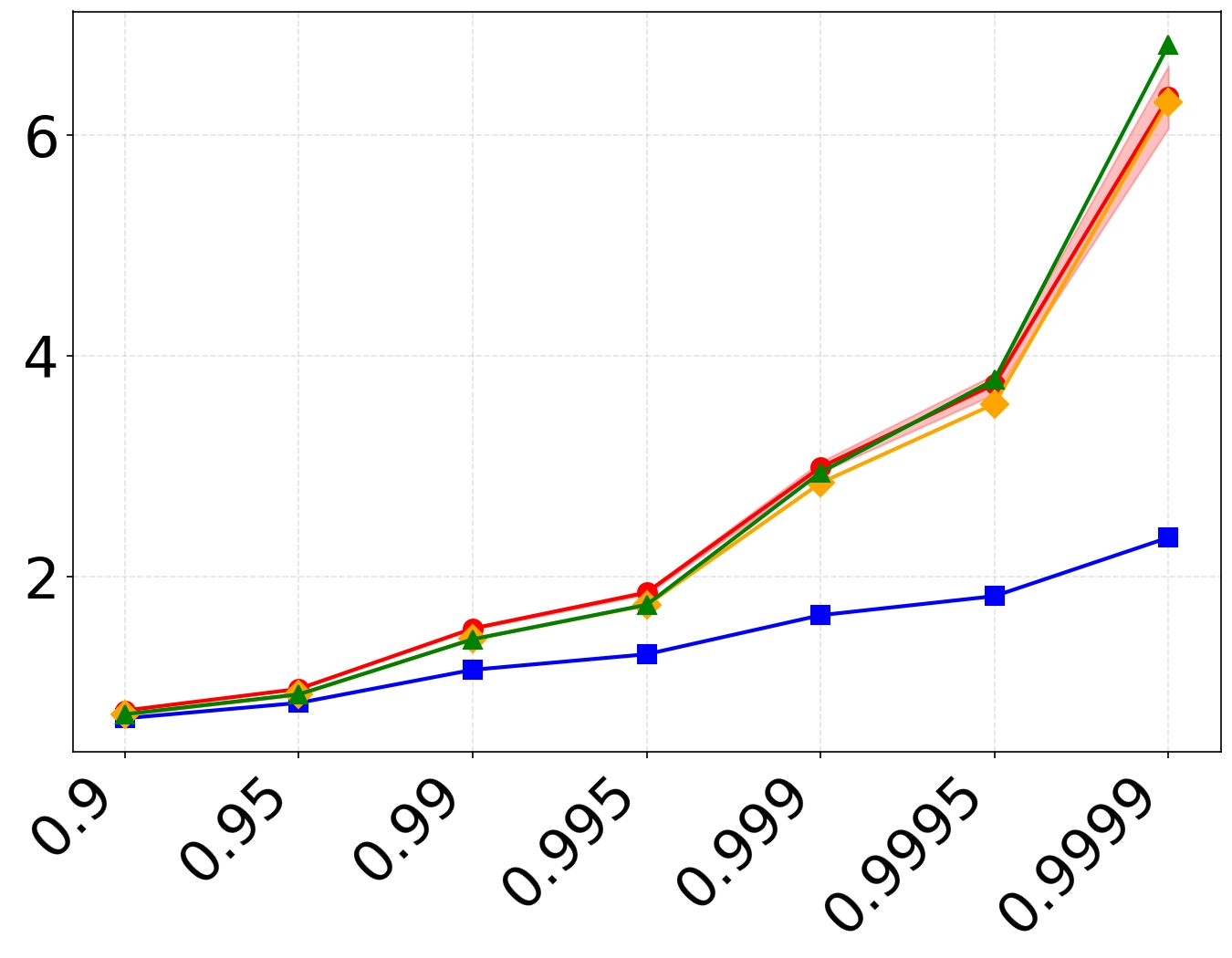} &
    \includegraphics[width=0.32\textwidth]{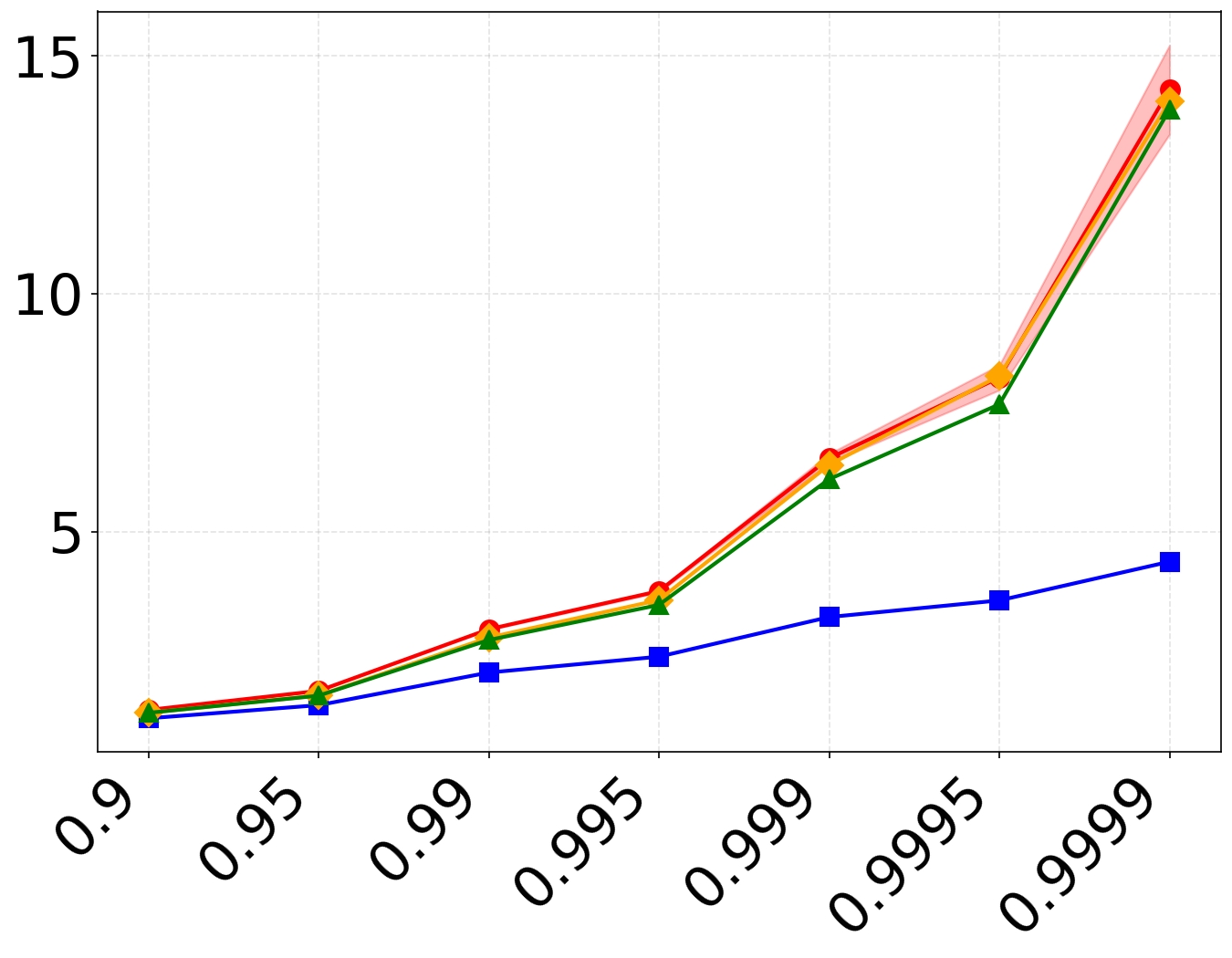} \\
    \small (d) BARKER, $\sigma_T=0.801$ & (e) BARKER, $\sigma_T=1.055$ & (f) BARKER, $\sigma_T=1.244$ \\[0.5em]
    \includegraphics[width=0.32\textwidth]{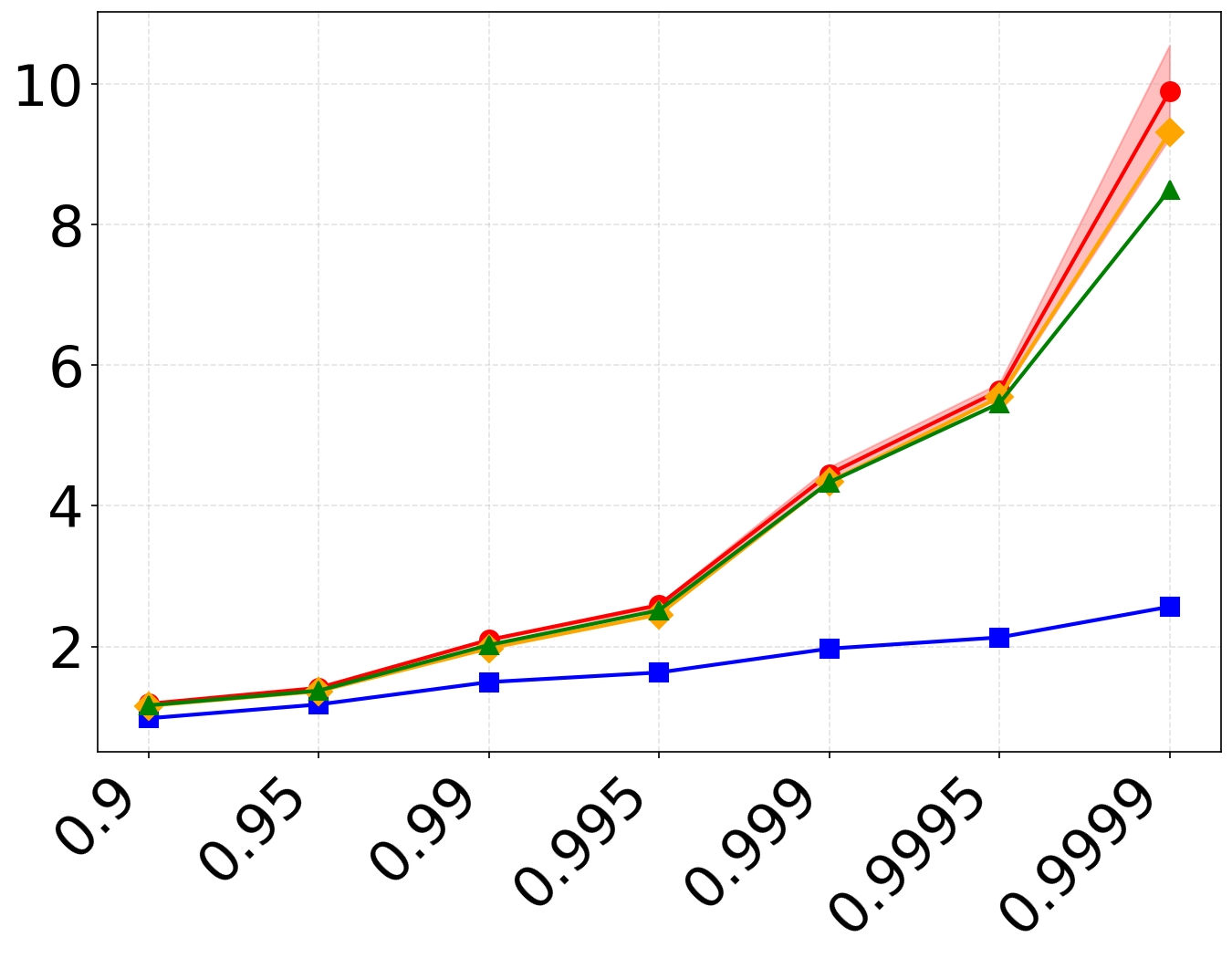} &
    \includegraphics[width=0.32\textwidth]{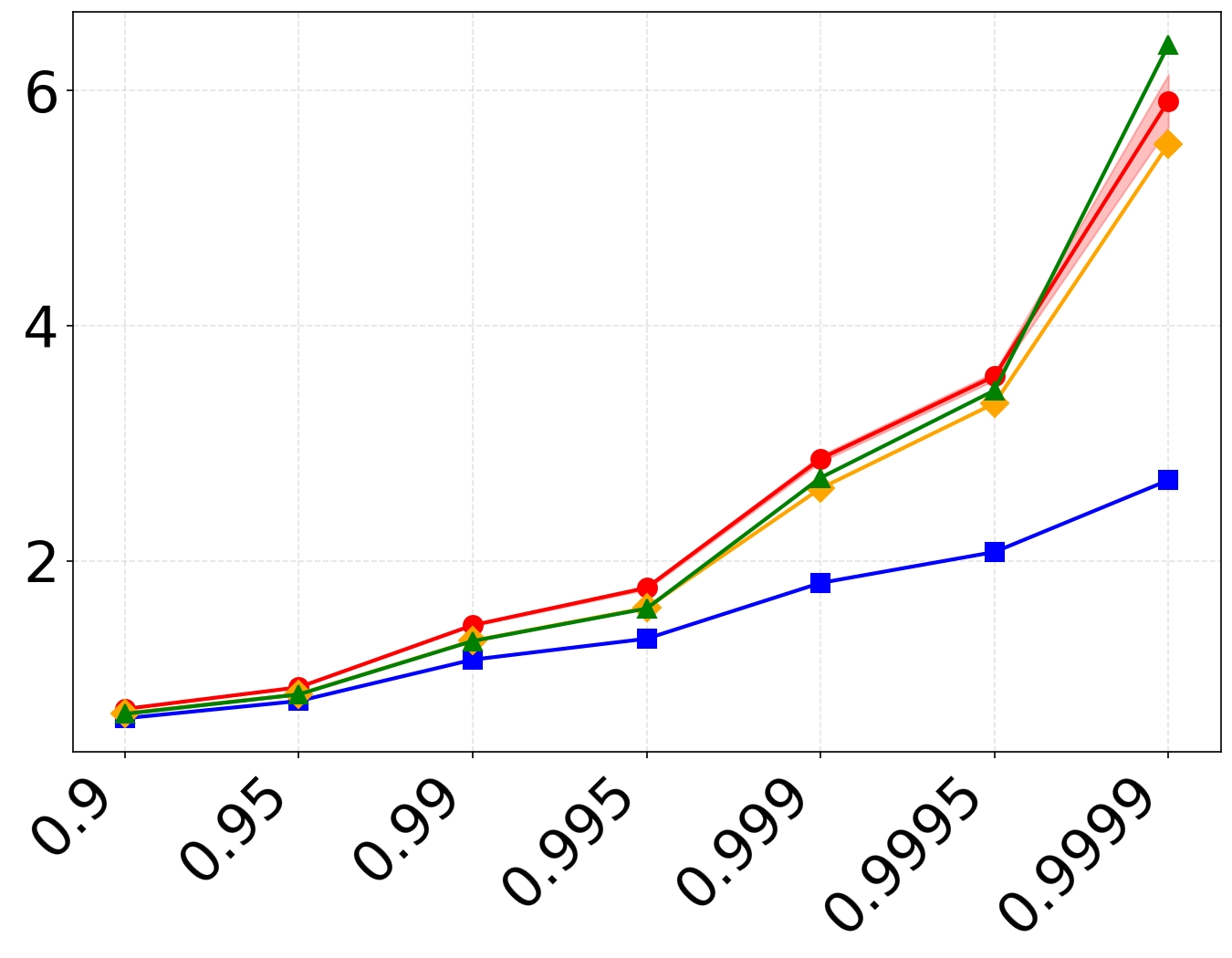} &
    \includegraphics[width=0.32\textwidth]{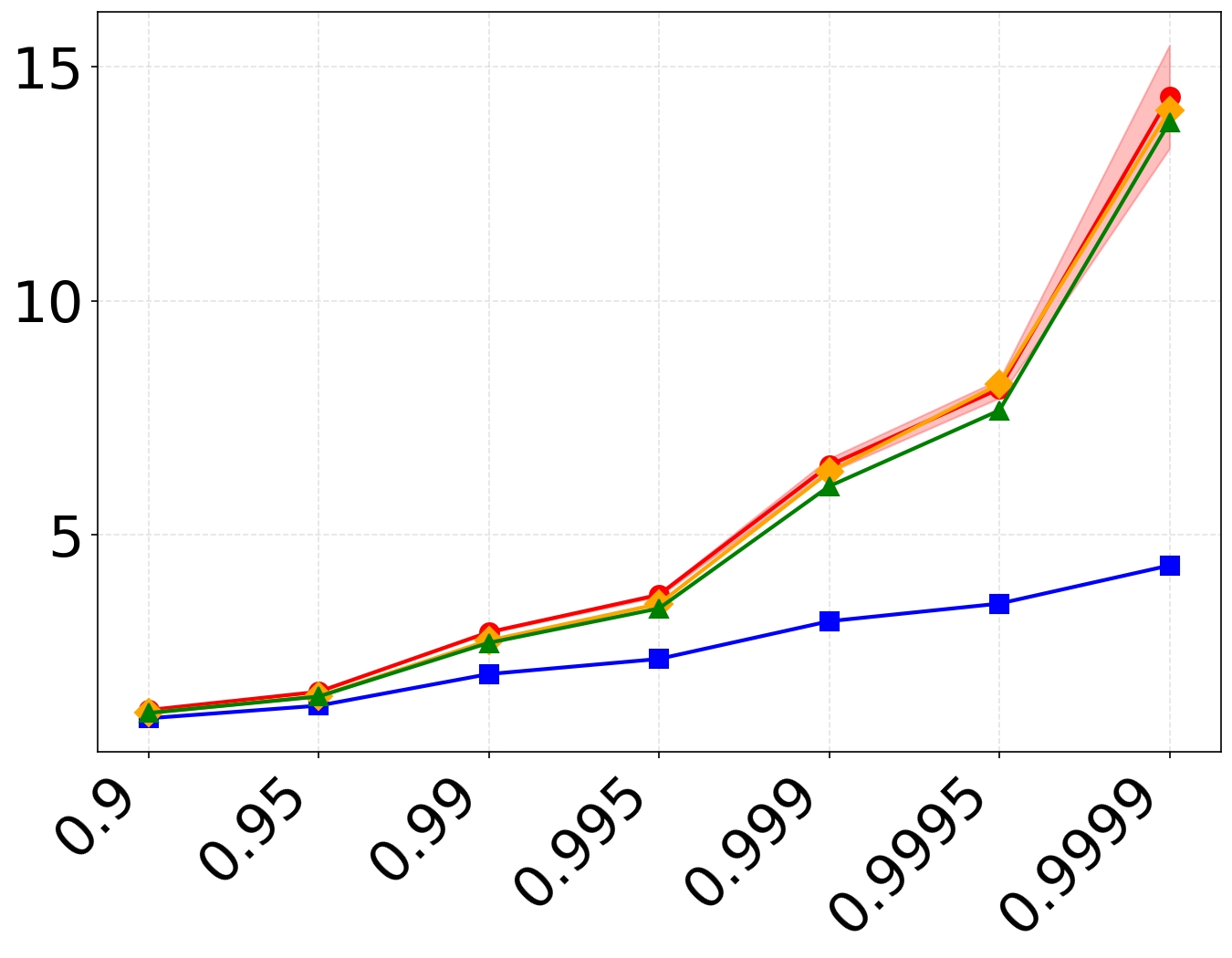} \\
    \small (g) HMC, $\sigma_T=0.801$ & (h) HMC, $\sigma_T=1.055$ & (i) HMC, $\sigma_T=1.244$ \\
  \end{tabular}
  \caption{Tail quantiles (0.9--0.9999). \textcolor{blue}{$p_\infty$}, 
    \textcolor{orange}{$p_T$}, \textcolor{green}{$p_\theta$}, \textcolor{red}{Reference}. 
    We test different MCMC-based denoisers}
    \label{fig:MCMC_quantile}
\end{figure}
\begin{figure}[!h]
  \centering
  \begin{tabular}{cccc}
    \includegraphics[width=0.32\textwidth]{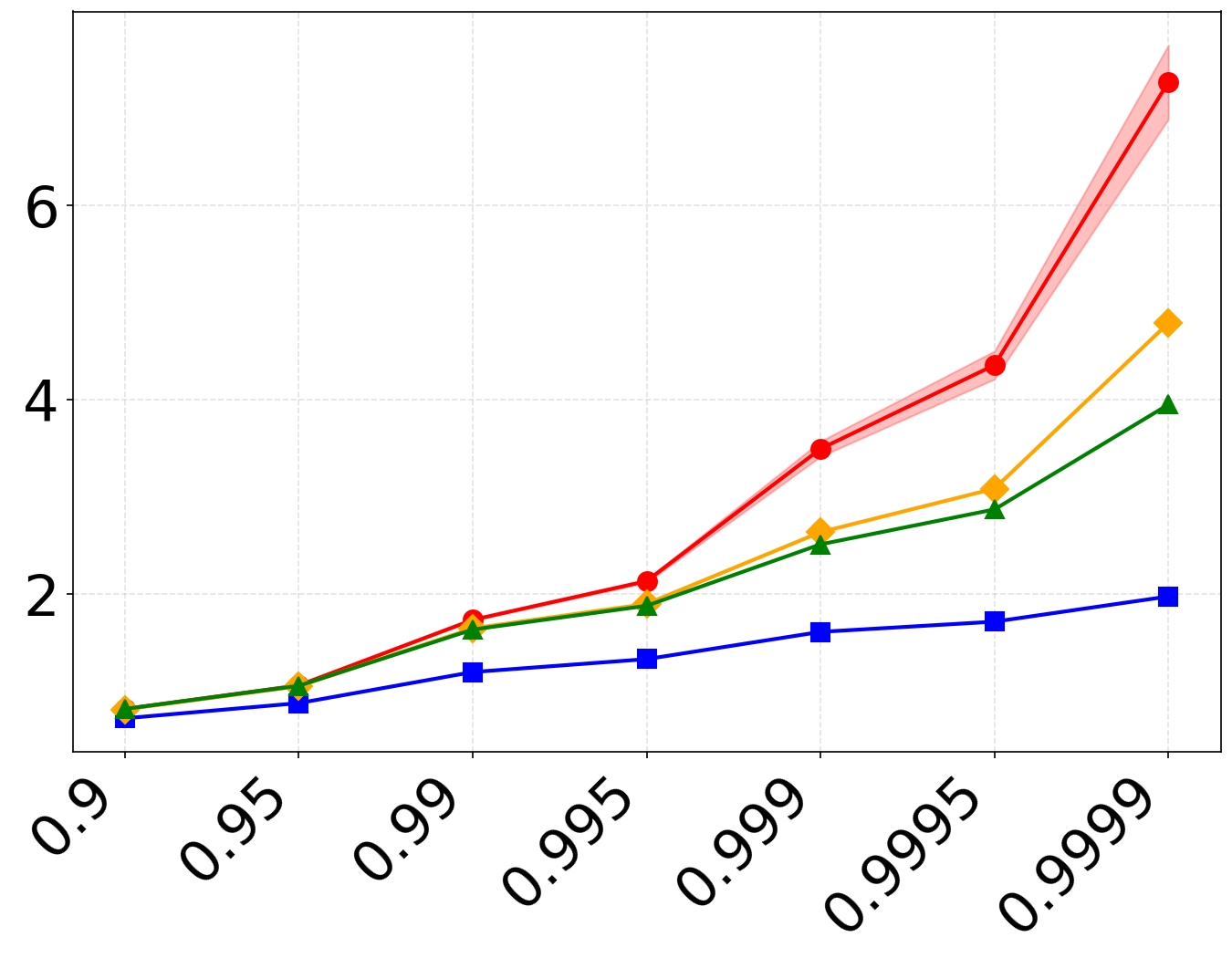} &
    \includegraphics[width=0.32\textwidth]{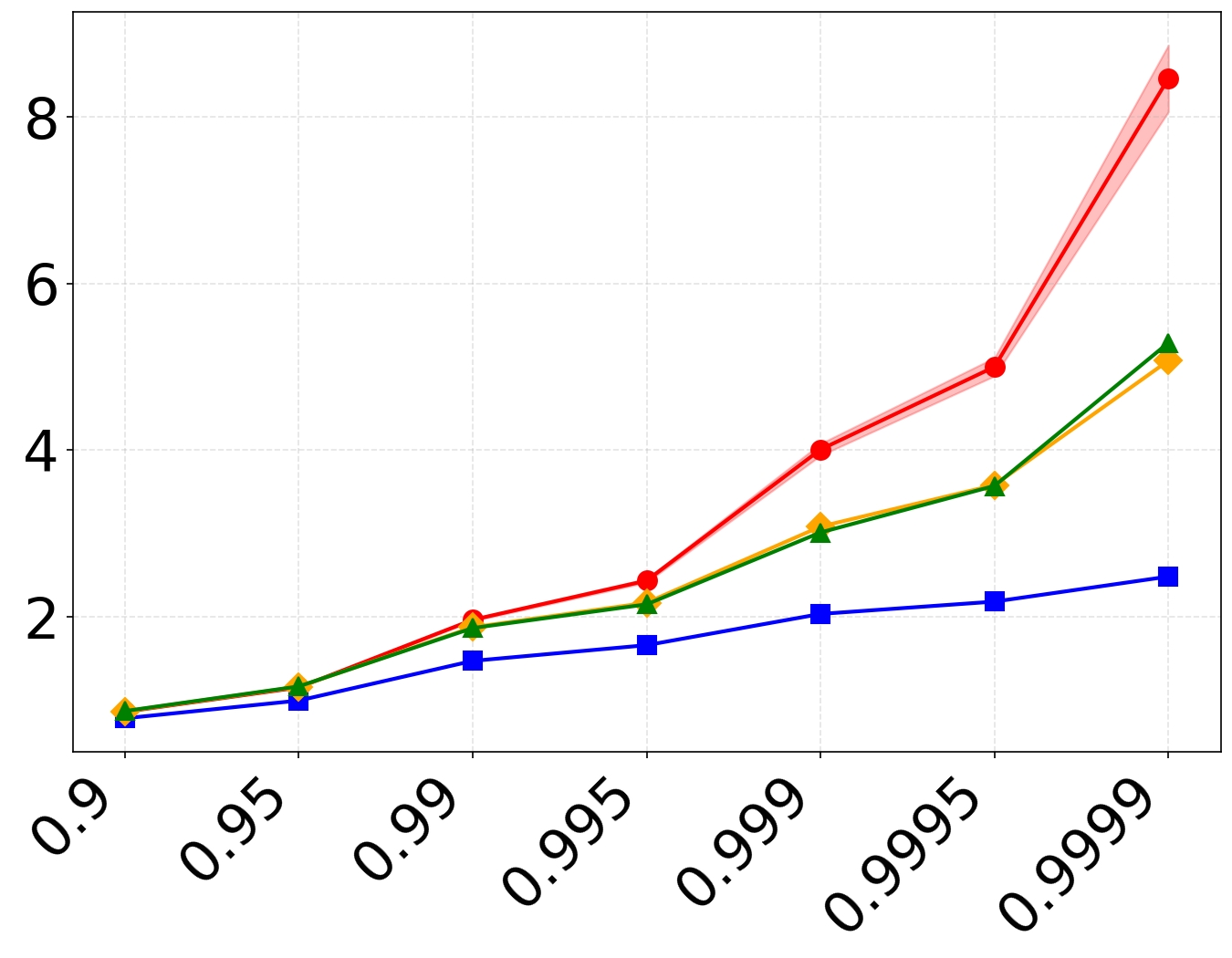} &
    \includegraphics[width=0.32\textwidth]{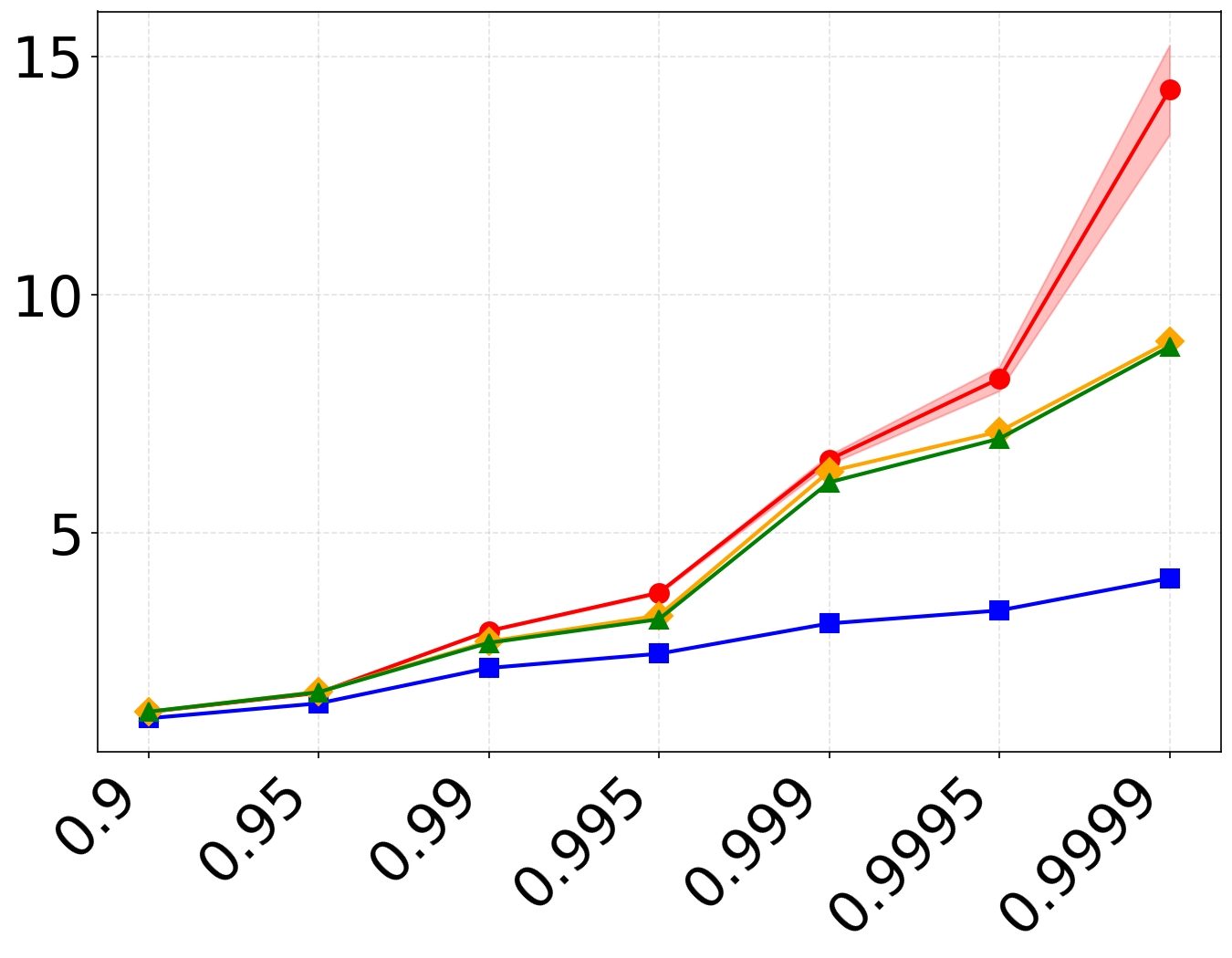} \\
    \small (a) $n=10^3$, $\sigma_T=0.801$ & (b) $n=10^3$, $\sigma_T=1.055$ & (c) $n=10^3$, $\sigma_T=1.244$ \\[0.5em]
    \includegraphics[width=0.32\textwidth]{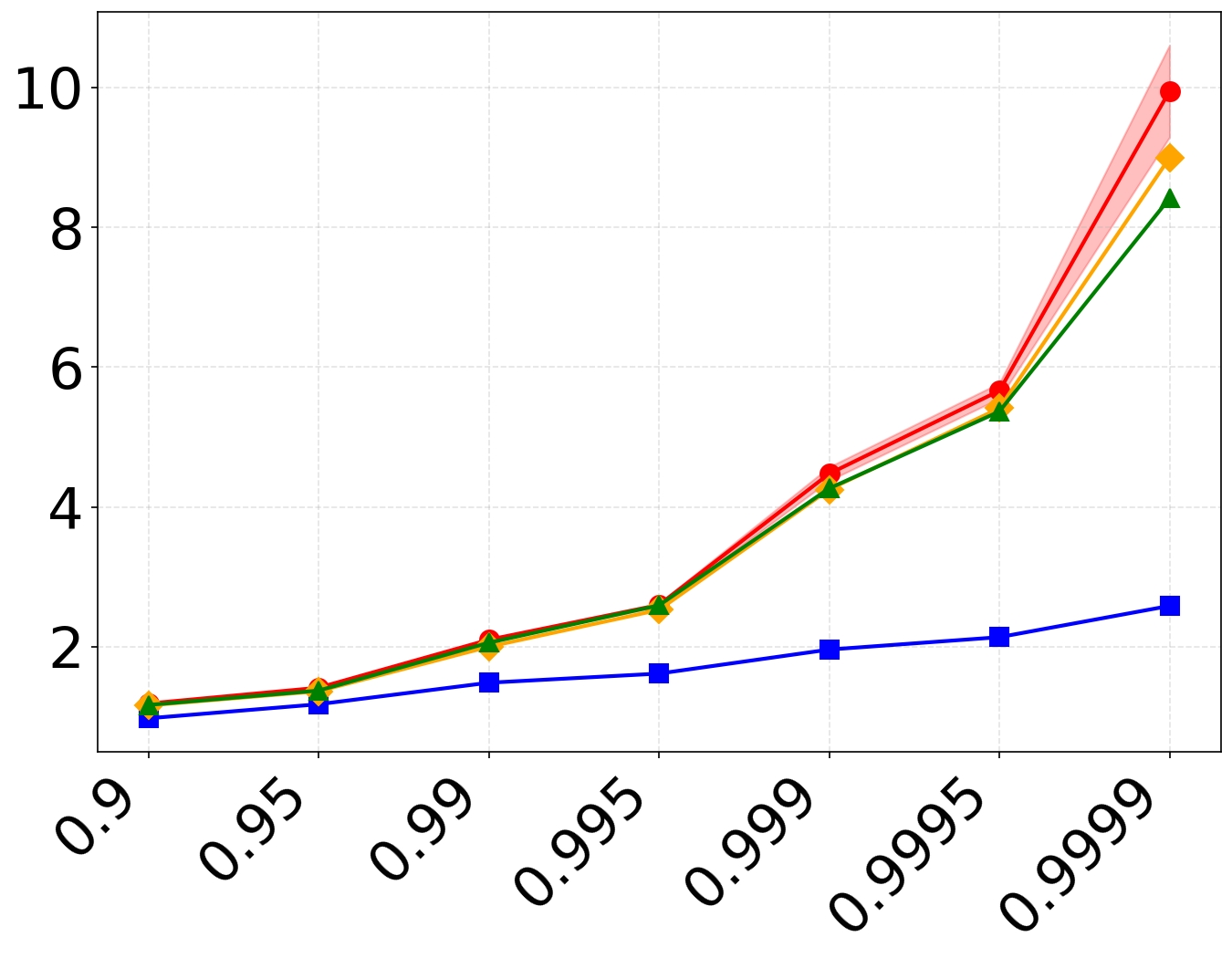} &
    \includegraphics[width=0.32\textwidth]{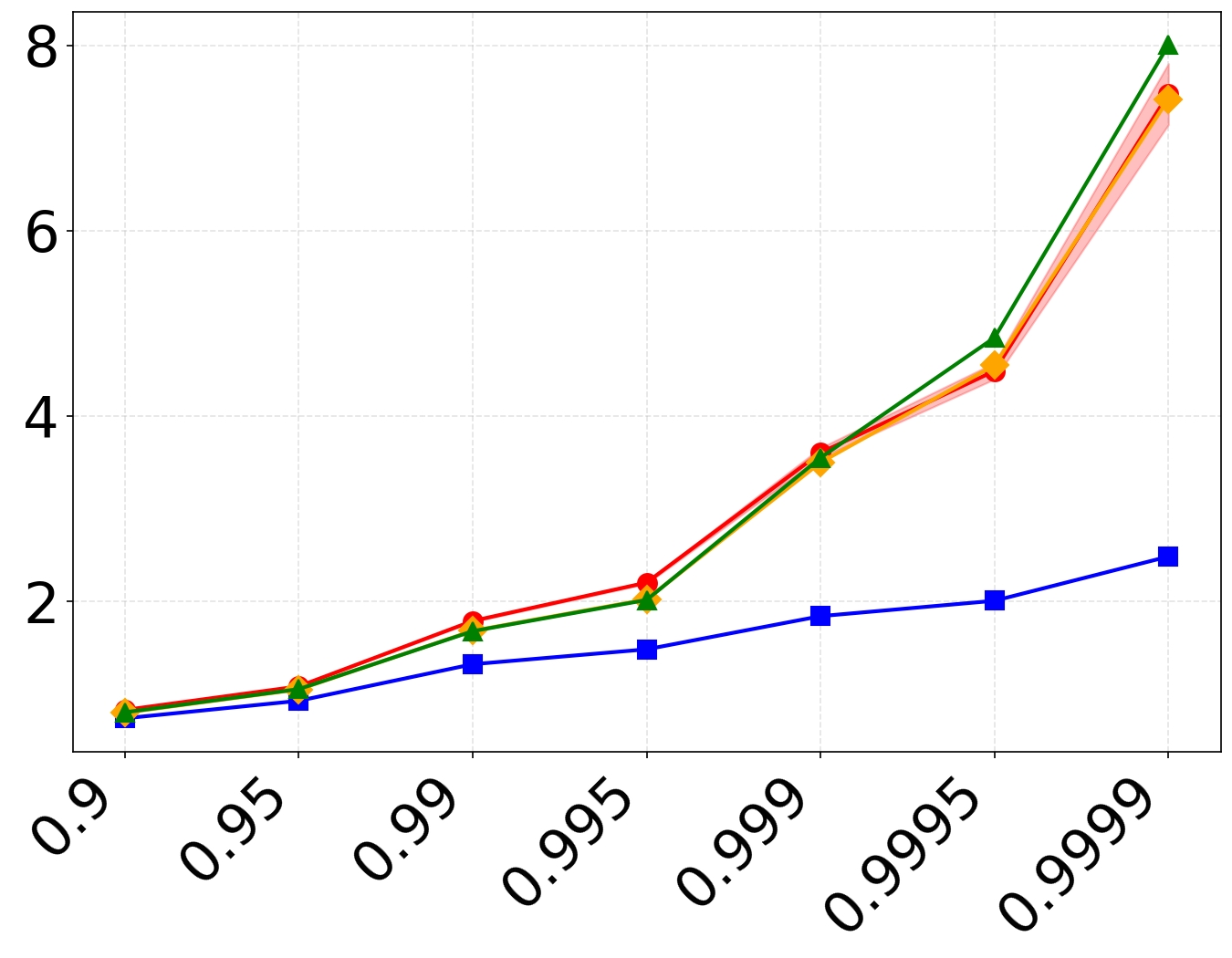} &
    \includegraphics[width=0.32\textwidth]{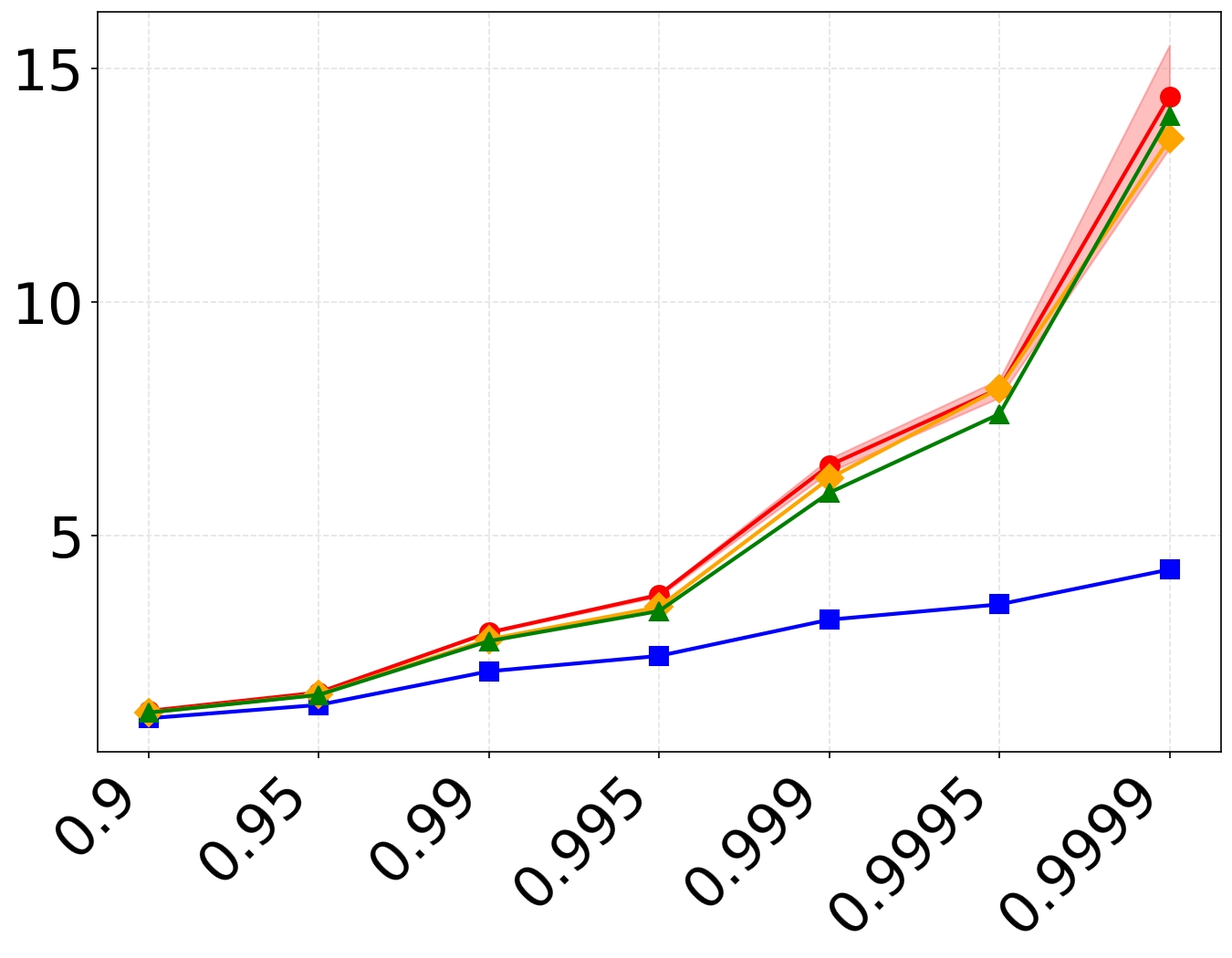} \\
    \small (d) $n=10^4$, $\sigma_T=0.801$ & (e) $n=10^4$, $\sigma_T=1.055$ & (f) $n=10^4$, $\sigma_T=1.244$ \\[0.5em]
    \includegraphics[width=0.32\textwidth]{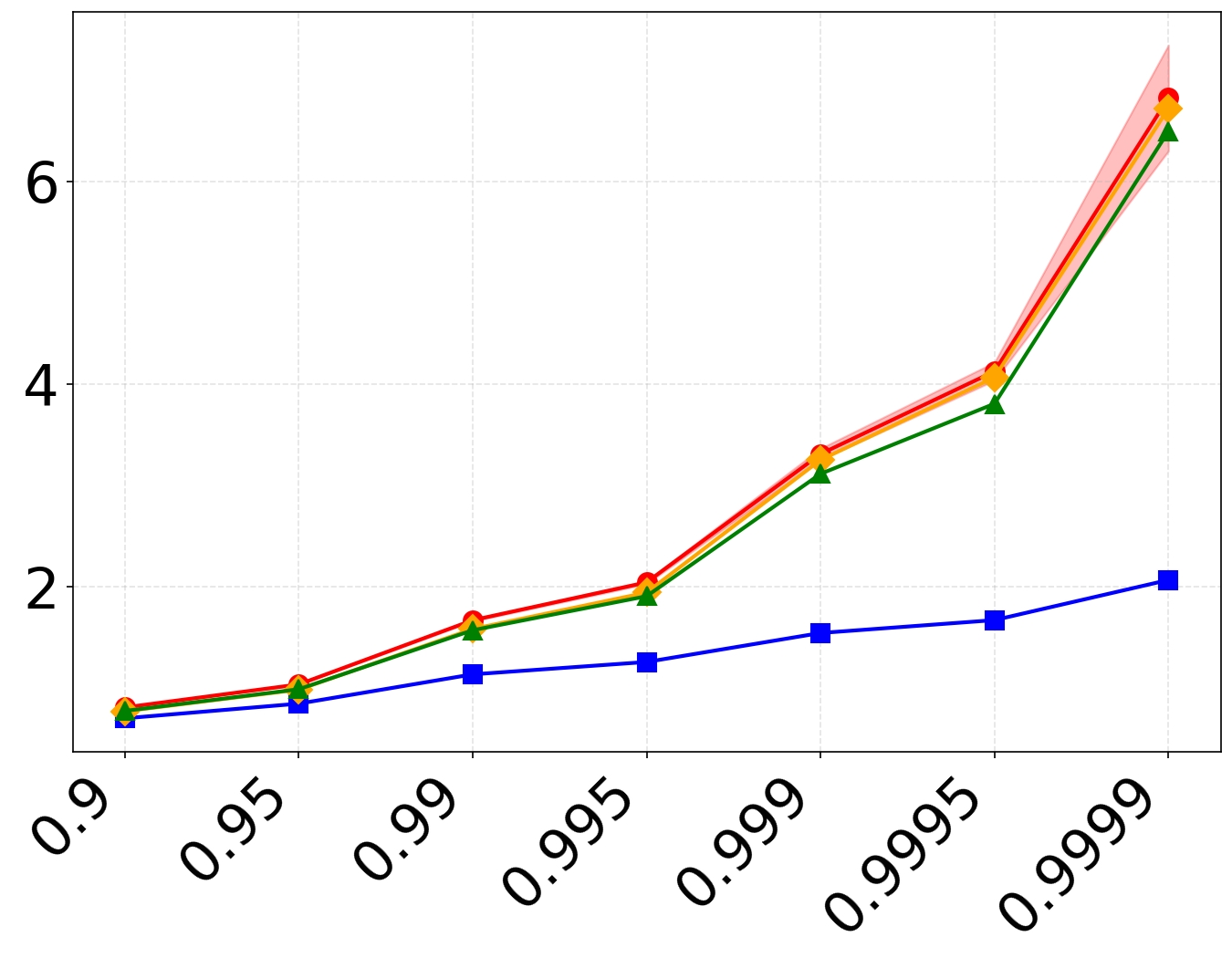} &
    \includegraphics[width=0.32\textwidth]{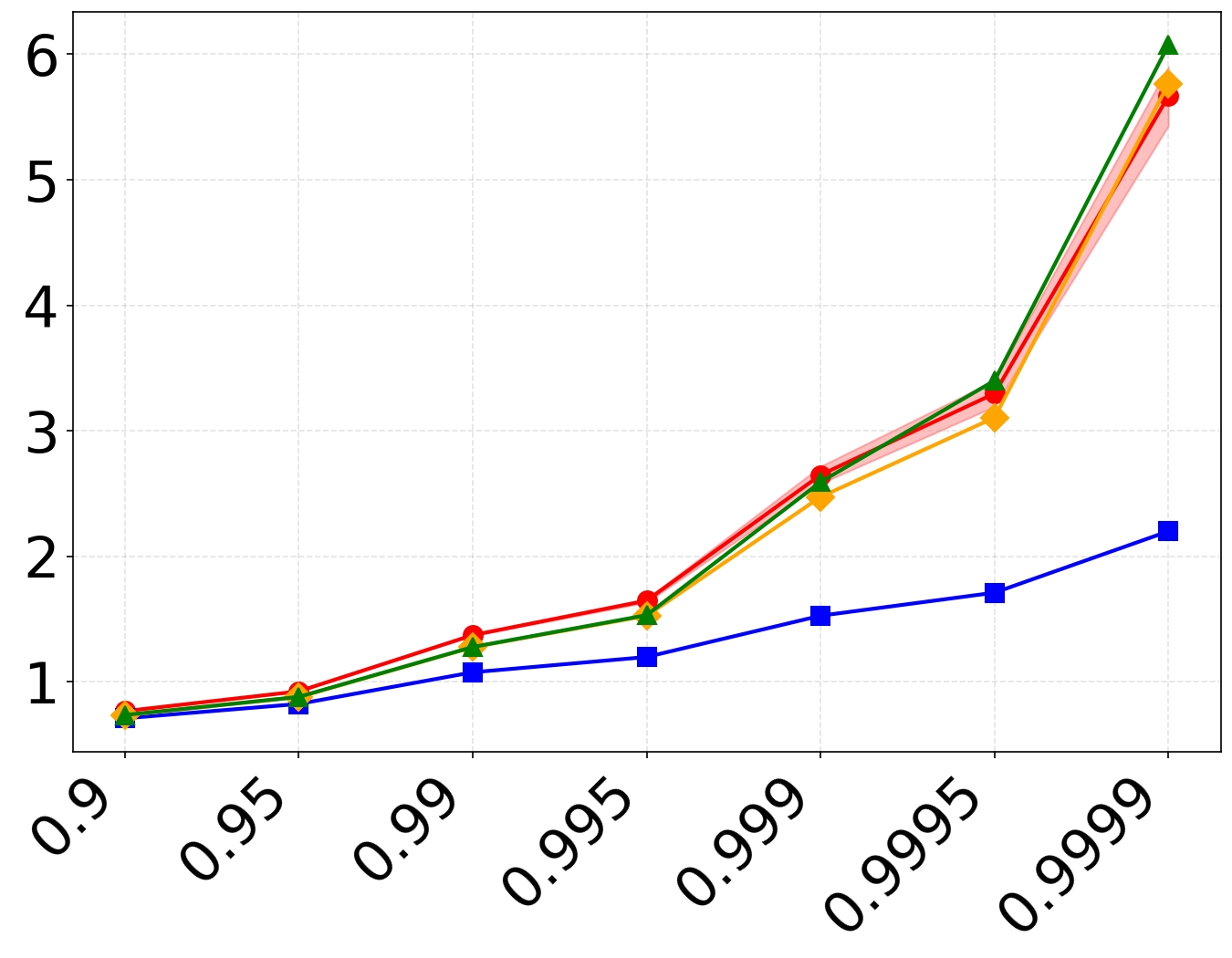} &
    \includegraphics[width=0.32\textwidth]{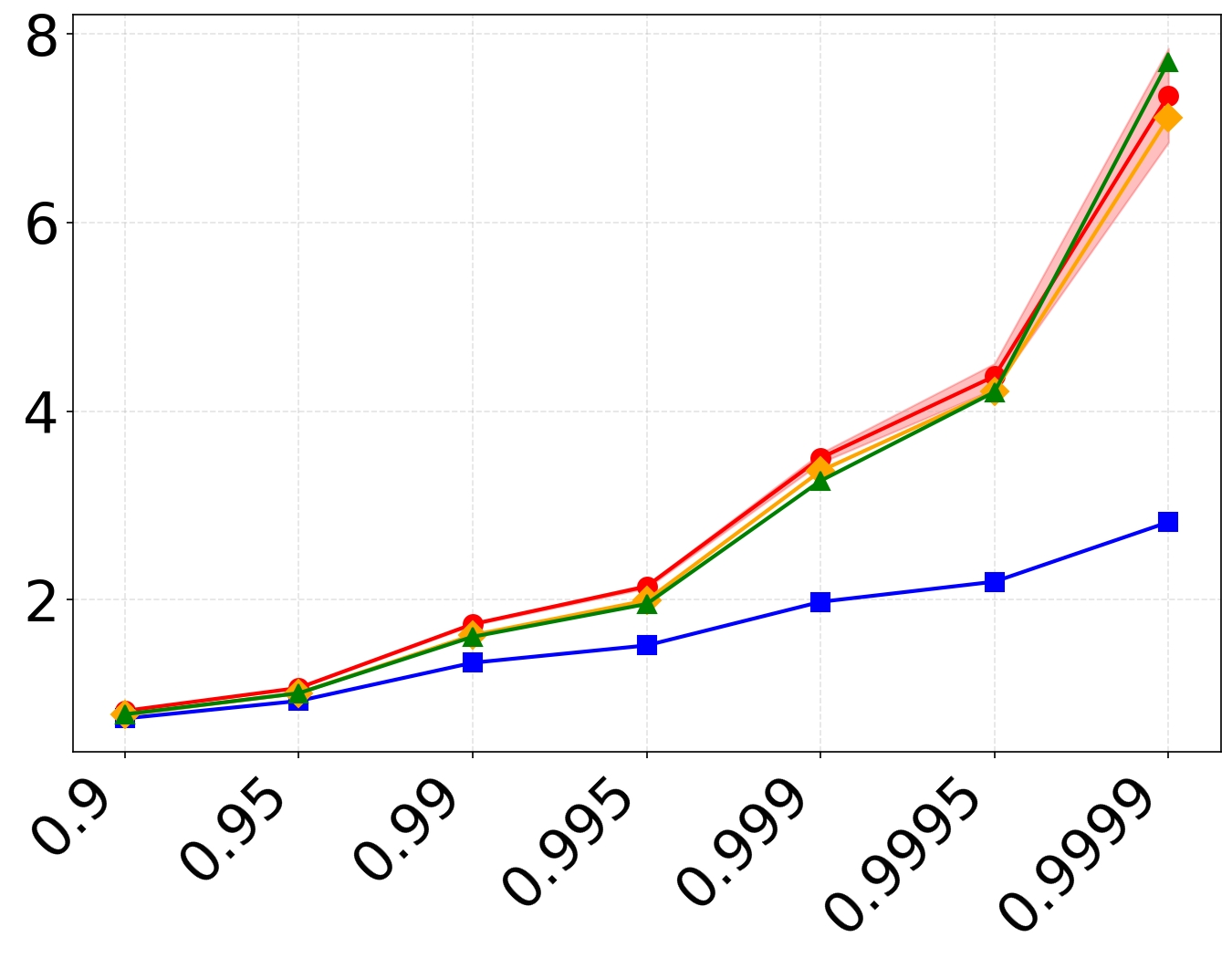} \\
    \small (g) $n=10^5$, $\sigma_T=0.801$ & (h) $n=10^5$, $\sigma_T=1.055$ & (i) $n=10^5$, $\sigma_T=1.244$ \\
  \end{tabular}
  \caption{Tail quantiles (0.9--0.9999). \textcolor{blue}{$p_\infty$}, 
    \textcolor{orange}{$p_T$}, \textcolor{green}{$p_\theta$}, \textcolor{red}{Reference}. 
    We test Neural Network denoisers trained on different size of training datasets.}
\label{fig:NN_quantile}
\end{figure}
\subsection{Testing initialization}
As already mentioned in the main text, we test the effect of $\pinf$, $\pforward{T}$, and $\discsolution[\param]{0}$ on the sampling performance in the GMM and HT case. 
We already mentioned that the test for GMM were perfomed sampling $10^6$ points and comparing MSW with $10^6$ reference points. The comparaison was made $10$ times and the results are reported in mean$\pm$std. 
For the HT case we divided the evaluation of bulk and tail using the procedures described in \cref{sec:experiments_toy}. Results for GMM are already reported in the main text, while for the HT case we need to report results for HMC and Barker. In \cref{tab:bulk_msw_full} we report the results for the MSW on the bulk for all denoising methods. 
In \cref{fig:MCMC_quantile} we report the tail quantiles for all MCMC methods, including HMC and Barker. In \cref{fig:NN_quantile} we report the tail quantiles for the neural network denoiser trained on different size of training datasets.
\begin{figure}[!h]
  \centering
  \includegraphics[width=0.7\textwidth]{results_flow_gmm/heatmap_fix.jpg}
  \caption{%
    KL divergence between $\discsolution[\param]{0}$ and $\pforward{T}$ for the flow trained under the \emph{fixed-noise} regime of \cref{alg:training}.
    Rows index the training set size $n_{\text{train}}$, columns the noise level $\sigma_T$.
    The cubic root of the KL divergence is shown for readability; warmer colors correspond to larger divergences.
    The fixed-noise strategy yields low KL only when the dataset is sufficiently large. 
  }
  \label{fig:flow_kl_fix}
\end{figure}
\subsection{Training the Coupling Layer Flow}
\cref{thm:tang2021empirical:thm3main} expresses the rate of decay of the KL divergence as a function of the number of training samples.
Moreover, following \cite{silveri2025beyond}, the noising operation can be interpreted as a smoothing of the target distribution, which can be beneficial for the training of $\discsolution[\param]{0}$.
\begin{figure}[!h]
  \centering
  \includegraphics[width=0.7\textwidth]{results_flow_gmm/heatmap_dynamic.jpg}
  \caption{%
    KL divergence between $\discsolution[\param]{0}$ and $\pforward{T}$ for the flow trained under the \emph{dynamic-noise} regime of \cref{alg:training}.
    Rows index the training set size $n_{\text{train}}$, columns the noise level $\sigma_T$.
    The cubic root of the KL divergence is shown for readability; warmer colors correspond to larger divergences.
    Compared to the fixed-noise regime (\cref{fig:flow_kl_fix}), the dynamic strategy produces an improved performance the bigger $\sigma_T$ it is, confirming the benefit of training the flow on noised data.
  }
  \label{fig:flow_kl_dynamic}
\end{figure}
\paragraph{Assessing the Coupling Layer Flow.}
To assess the joint effect of training set size and noise conditioning on the coupling layerflow accuracy, we trained $\discsolution[\param]{0}$ on the GMM target while varying both the number of training samples $n_{\text{train}} \in [100, 10000]$ and the noise level $\sigma_T \in [0.01, 2.56]$. We chose the GMM target as it admits a tractable density $\pforward{T}$, which makes the KL divergence between the model and the target amenable to Monte Carlo estimation.
We compare the two training strategies of \cref{alg:training} Results for the two strategies are reported in \cref{fig:flow_kl_fix} and \cref{fig:flow_kl_dynamic}, respectively. Results of \cref{fig:kl_sigma,fig:kl_trainsize} are just a subset of the these full results. 
For better readability, the heatmaps display the cubic root of the KL divergence; both figures share a common color scale to enable direct visual comparison.
\subsection{TarFlow as a candidate $\discsolution{0}$ on image data}
We now turn our attention on Tarflow \cite{zhai2025normalizing}, a Transformer-based normalizing flow, that could serve as a viable model for $\discsolution{0}$ for high dimensional datawithin the sampling pipeline sampling introduced in the main text. To this end, we run TarFlow on standard image benchmarks and, in parallel, investigate whether the choice of noise sampling strategy during training, fixed versus dynamical (\cref{alg:training}), meaningfully affects the resulting initialization. 
The image generation results reported below should therefore be read primarily as a diagnostic of TarFlow's suitability as $\discsolution{0}$, rather than as an attempt to compete with state-of-the-art image generators. The underlying rationale is the following: if the proposed pipeline can match or outperform the benchmark established by \cite{karras2022elucidating, Karras2023TrainingDiffusion}, then TarFlow can be confidently employed within our framework to model light-tailed targets, and, by extension, light-tailed transformations of $\pforward{T}$, in the spirit of Comet Flow~\cite{mcdonald2022cometflows}.
We report these results in the Appendix as they are still preliminary.
\paragraph{Datasets.}
We consider three benchmarks: FFHQ-64, used for unconditional generation, and two subsets of ImageNet-512 used for conditional generation. We denote by ImageNet$_{\text{dogs}}$ the subset obtained by randomly selecting 50 dog classes, and by ImageNet$_{\text{birds}}$ the analogous subset of 50 bird classes. The two ImageNet subsets allow us to evaluate both class-conditional fidelity and the robustness of the method across closely related categories.

\paragraph{Pretrained denoisers.}
For FFHQ-64 we rely on the pretrained VE model of \cite{karras2022elucidating}. For ImageNet-512 we use the size-S denoiser of \cite{Karras2023TrainingDiffusion}, which operates on the latent representations produced by Stable VAE~\cite{stabilityai_sd-vae-ft-mse}; consequently, our short-horizon sampling on ImageNet is performed in latent space. In all cases, we use the EDM sampler~\cite{karras2022elucidating}.

\paragraph{TarFlow as $\discsolution{0}$.}
We use TarFlow~\cite{zhai2025normalizing} as the model class for $\discsolution{0}$. Compared to the original TarFlow setup, our instantiation is at least $10\times$ smaller; a detailed comparison is given in \cref{tab:benchmark:config}. We train $\discsolution{0}$ on noised data with short time horizon $\sigma_T = 7$. A practical issue arises from the fact that training TarFlow directly on noised data can push inputs outside the architecture's intended dynamic range (the original model is trained on images in $[-1,1]$). We address this by dividing the noised inputs by a \emph{Training Factor} during training and multiplying by the same factor at sampling initialization, which we found to substantially improve generation quality. Further architectural and optimization details are reported in \cref{tab:benchmark:config}.
we evaluate both the fixed and dynamic noise training strategies of \cref{alg:training}.
\paragraph{Sampling strategies.}
For all datasets we use the EDM sampler and evaluate five initialization strategies. The first is a long-horizon Gaussian start ($\pinf$ at $\sigma_T = 80$): the standard EDM setup, with 40 steps and $\rho=7$ for FFHQ-64, and 32 steps and $\rho=7$ for ImageNet-512, following \cite{karras2022elucidating, Karras2023TrainingDiffusion}. The second and third are short-horizon, flow-based initializations using $\discsolution{0}$ trained respectively under the fixed and the dynamical variant, denoted $\discsolution{0}$ (fix) and $\discsolution{0}$ (dyn). The fourth is a short-horizon Gaussian start at $\sigma_T = 7$, denoted $\pinf$ ($\sigma_T = 7$), which uses the short horizon but no learned initializer. The fifth is a short-horizon empirical initialization $\pforward{T}$, obtained by convolving the empirical training distribution with Gaussian noise of level $\sigma_T = 7$. For ImageNet, we perform class-conditional sampling with a flow-side guidance weight $\text{CFG}_{\text{flow}} \in [0.5, 2]$, as natively supported by TarFlow, and a denoiser-side weight $\text{CFG}_{\text{SGM}} = 1.9$, following \cite{Karras2023TrainingDiffusion}. All short-horizon strategies use 20 denoising steps and $\rho = 5$; we observed stable quality across $\rho$ values.

\paragraph{Metrics and evaluation protocol.}
We report Fréchet Inception Distance (FID)~\cite{heusel2017gans} together with several complementary metrics, motivated by recent work pointing out the limitations of FID~\cite{jayasumana2024rethinking,stein2023exposing}: the unbiased Kernel Inception Distance (KID)~\cite{binkowski2018demystifying}, the Dino Fréchet Distance (DinoFD)~\cite{oquab2024dinov,stein2023exposing,Karras2023TrainingDiffusion}, the Sliced Wasserstein Distance (SWD), and the Max Sliced Wasserstein Distance (MSW). FID, KID and DinoFD assess perceptual realism, whereas SWD and MSW measure the fidelity of the underlying distribution, a property that is particularly relevant when realism is not the primary objective. On ImageNet-512, MSW, SWD are computed in latent space, the others in pixel space.

For each dataset, we train two TarFlow models, one per training variant, and run, for each sampling strategy, three independent generations. FID, KID, and DinoFD are computed on $5\times 10^4$ samples per run, and we report the minimum across the three runs. SWD and MSW are computed four times per run with $1.75\times 10^4$ samples and $2\times 10^4$ projection slices each; we report the global mean and standard deviation over the resulting $3 \times 4 = 12$ evaluations. All comparisons are made against the training data, following the protocol of \cite{karras2022elucidating, Karras2023TrainingDiffusion}.

Aggregated metrics are reported in \cref{tab:ffhq} (FFHQ-64), \cref{tab:in_birds} (ImageNet$_{\text{birds}}$), and \cref{tab:in_dogs} (ImageNet$_{\text{dogs}}$). Qualitative comparisons are shown in \cref{fig:photos:ffhq,fig:photos:birds,fig:photos:dogs}: each row corresponds to a sampling strategy, and entries are obtained by selecting a training image and retrieving, for every method, its $L^2$ nearest neighbour among the generated samples (in latent space for the ImageNet subsets). Class labels follow the convention of \url{https://deeplearning.cms.waikato.ac.nz/user-guide/class-maps/IMAGENET/}.

\begin{table}
    \caption{Aggregated Results for ImageNet$_{\text{birds}}$}
\label{tab:in_birds}
\centering
\begin{tabular}{lccccc}
\toprule
Model & FID & DINO FD & KID & SWD & MSW \\
\midrule
$\pinf (\sigma_T = 80)$ (standard) & 5.90 & 67.13 & 0.0018 & 0.042 $\pm$ 0.001 & 5.095 $\pm$ 0.129 \\
$\discsolution{0}$ - CFG$_{\text{flow}} = 0.5$  (fixed)  & 3.38 & 58.72 & 0.0009 & 0.016 $\pm$ 0.000 & 1.955 $\pm$ 0.112 \\
$\discsolution{0}$ - CFG$_{\text{flow}} = 1.0$  (fixed)  & 3.30 & 63.26 & 0.0009 & 0.023 $\pm$ 0.000 & 1.984 $\pm$ 0.041 \\
$\discsolution{0}$ - CFG$_{\text{flow}} = 1.5$  (fixed)  & 3.11 & 69.58 & 0.0008 & 0.035 $\pm$ 0.000 & 4.781 $\pm$ 0.124 \\
$\discsolution{0}$ - CFG$_{\text{flow}} = 2.0$ (fixed) & 2.93 & 79.57 & 0.0006 & 0.048 $\pm$ 0.001 & 6.992 $\pm$ 0.145 \\
$\discsolution{0}$ - CFG$_{\text{flow}} = 0.5$ (dynamical) & 4.35 & 59.56 & 0.0013 & 0.027 $\pm$ 0.000 & 2.999 $\pm$ 0.062 \\
$\discsolution{0}$ - CFG$_{\text{flow}} = 1.0$ (dynamical) & 4.82 & 65.86 & 0.0015 & 0.028 $\pm$ 0.000 & 3.069 $\pm$ 0.051 \\
$\discsolution{0}$ - CFG$_{\text{flow}} = 1.5$ (dynamical) & 5.38 & 72.84 & 0.0016 & 0.031 $\pm$ 0.000 & 3.693 $\pm$ 0.047 \\
$\discsolution{0}$ - CFG$_{\text{flow}} = 2.0$ (dynamical) & 5.92 & 81.48 & 0.0017 & 0.035 $\pm$ 0.000 & 4.509 $\pm$ 0.050 \\
$\pforward{T}$ & 3.91 & 53.49 & 0.0011 & 0.022 $\pm$ 0.000 & 1.258 $\pm$ 0.633 \\
$\pinf (\sigma_T = 7)$ & 4.15 & 68.42 & 0.0011 & 0.041 $\pm$ 0.000 & 8.174 $\pm$ 0.063 \\
\bottomrule
\end{tabular}
\end{table}

\begin{figure}[t]
    \centering
    \small
    \setlength{\tabcolsep}{0pt}%
    \renewcommand{\arraystretch}{0}%
    \begin{tabular}{lcccccccccc}
        & 9 & 23 & 81 & 86 & 87 & 88 & 92 & 93 & 94 & 145 \vspace{0.05cm}
\\%        
        \raisebox{1ex}{{\scriptsize Train}} & 
        \includegraphics[width=0.095\linewidth]{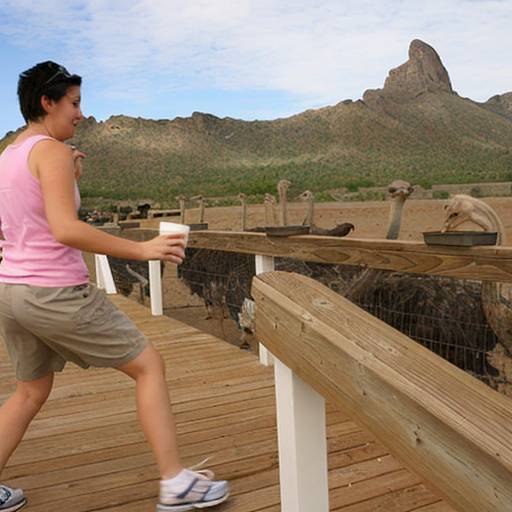}&%
        \includegraphics[width=0.095\linewidth]{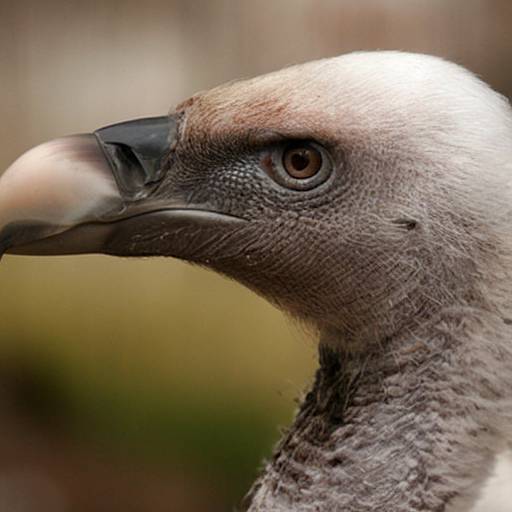}&%
        \includegraphics[width=0.095\linewidth]{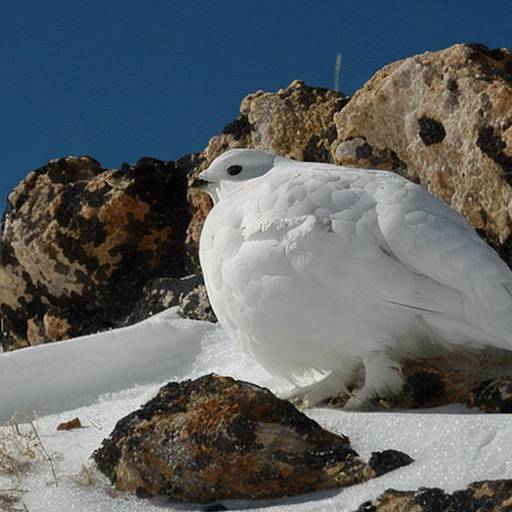}&%
        \includegraphics[width=0.095\linewidth]{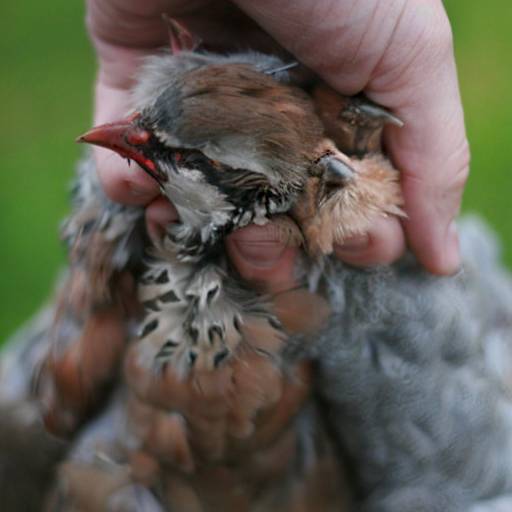}&%
        \includegraphics[width=0.095\linewidth]{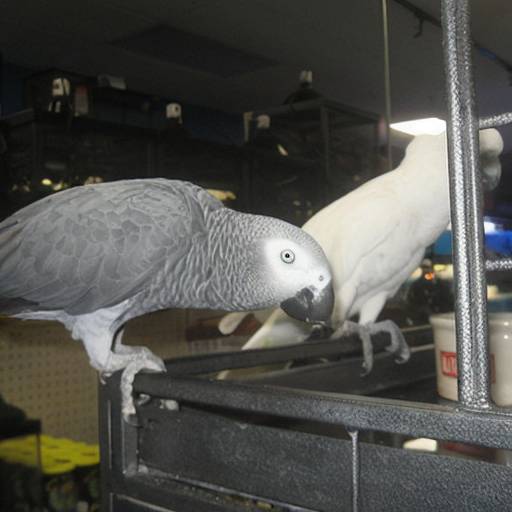}&%
        \includegraphics[width=0.095\linewidth]{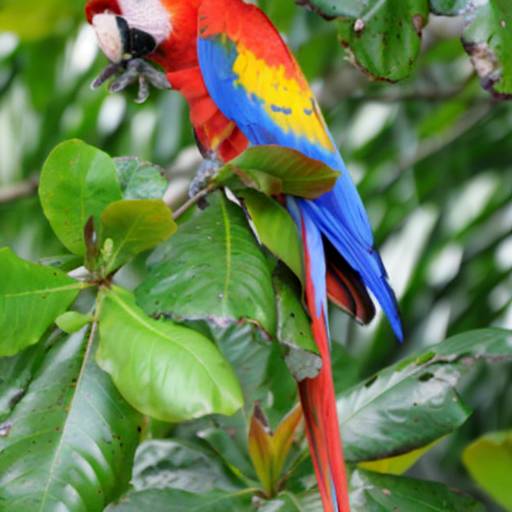}&%
        \includegraphics[width=0.095\linewidth]{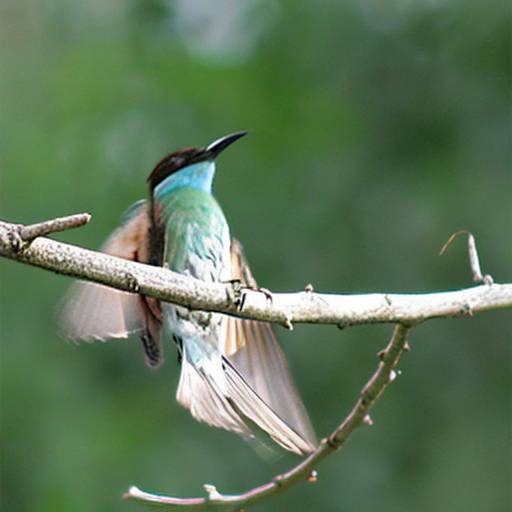}&%
        \includegraphics[width=0.095\linewidth]{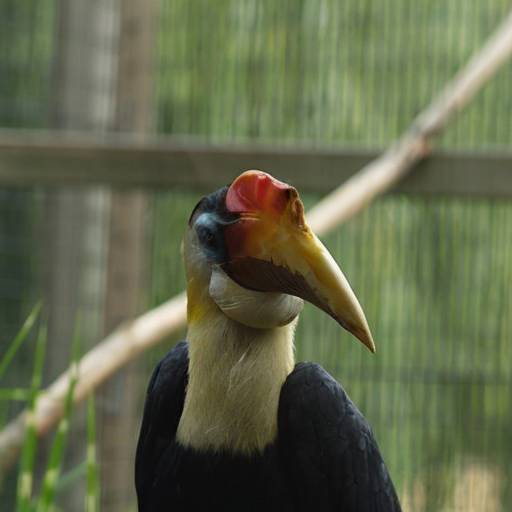}&%
        \includegraphics[width=0.095\linewidth]{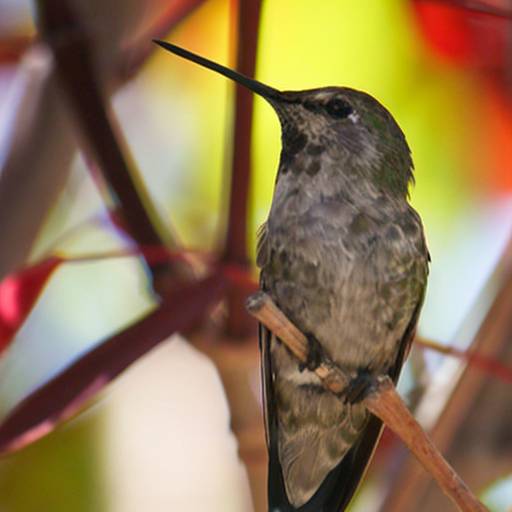}&%
        \includegraphics[width=0.095\linewidth]{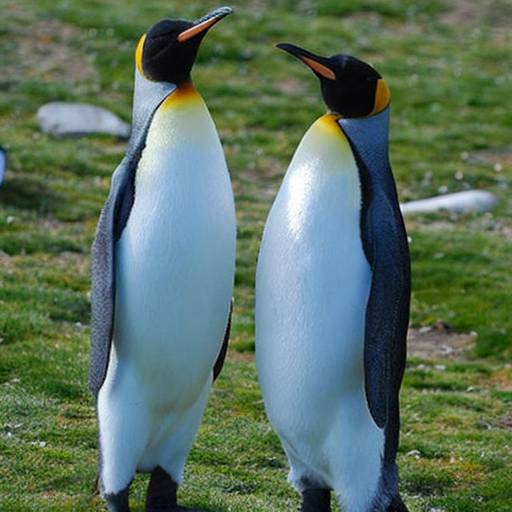}\\%

        \raisebox{0.5ex}{\scriptsize \shortstack[l]{$\pinf$ \\ $(\sigma_T = 80)$}}  & 
        \includegraphics[width=0.095\linewidth]{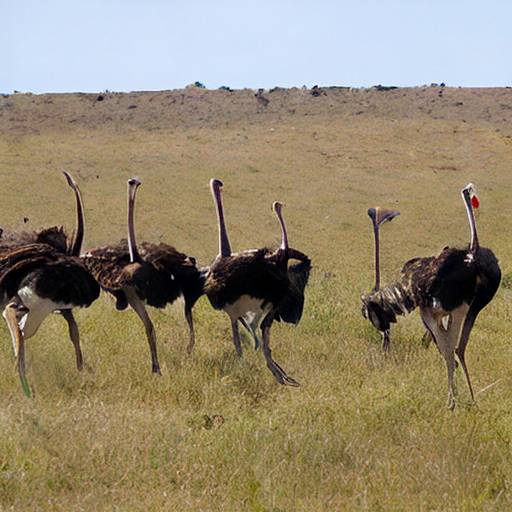}&%
        \includegraphics[width=0.095\linewidth]{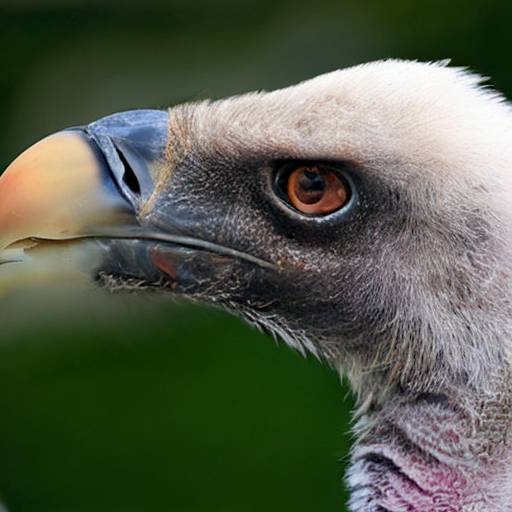}&%
        \includegraphics[width=0.095\linewidth]{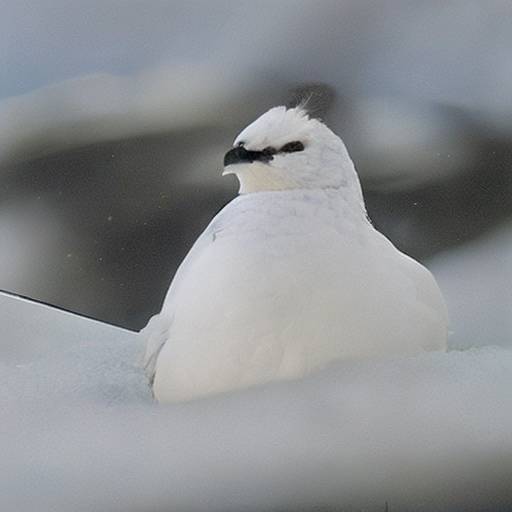}&%
        \includegraphics[width=0.095\linewidth]{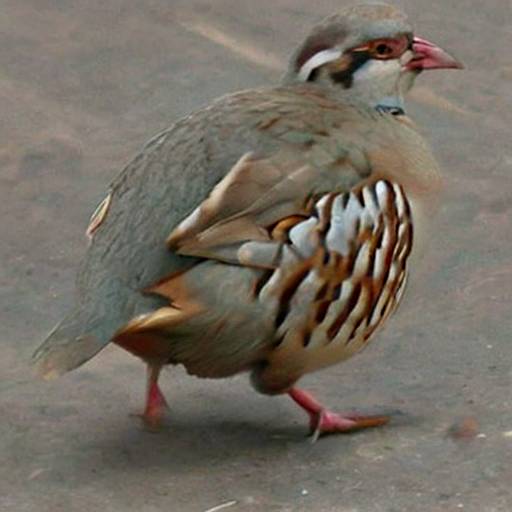}&%
        \includegraphics[width=0.095\linewidth]{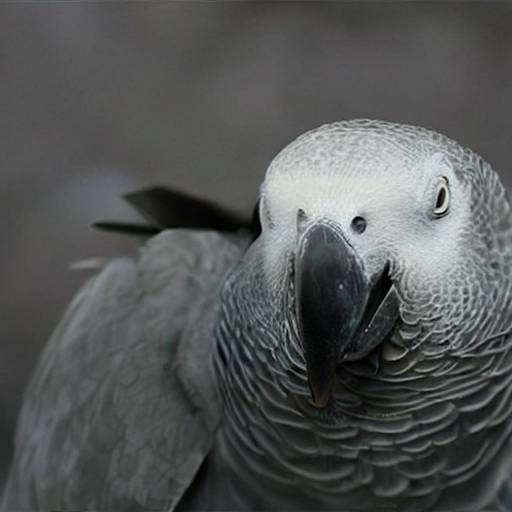}&%
        \includegraphics[width=0.095\linewidth]{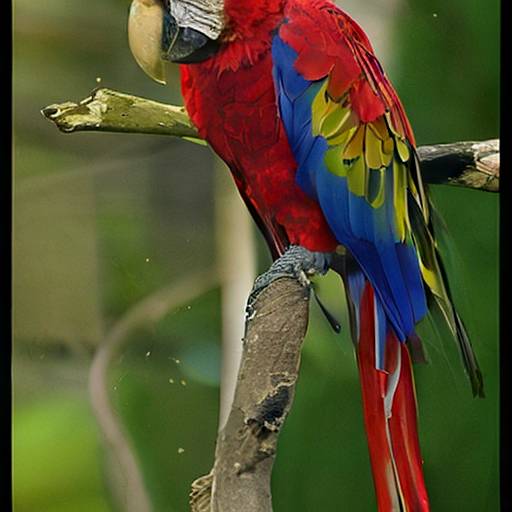}&%
        \includegraphics[width=0.095\linewidth]{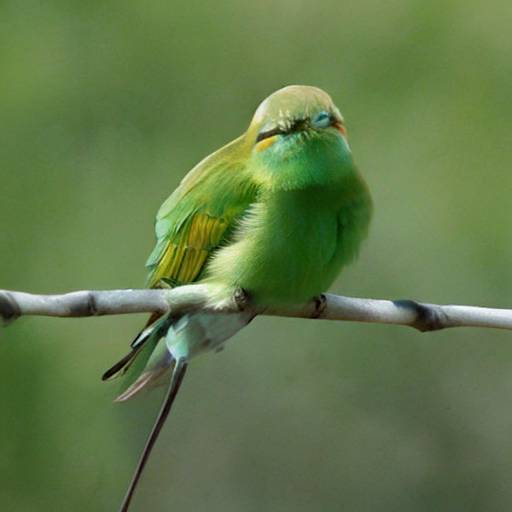}&%
        \includegraphics[width=0.095\linewidth]{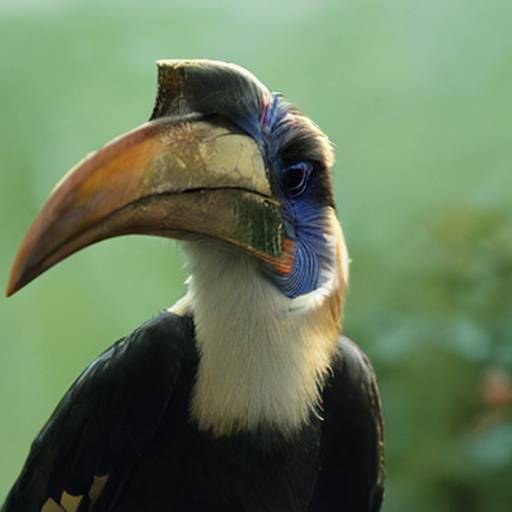}&%
        \includegraphics[width=0.095\linewidth]{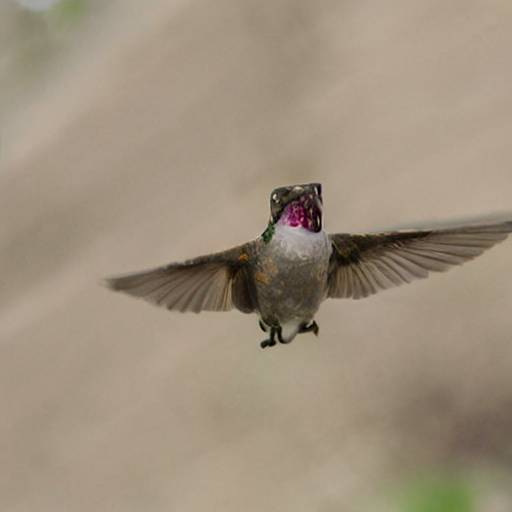}&%
        \includegraphics[width=0.095\linewidth]{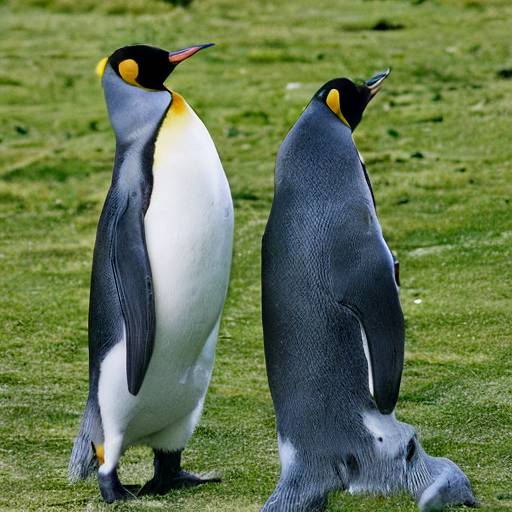}\\%
        
        \raisebox{1ex}{{\scriptsize $\discsolution{0}$} (fix)} & 
        \includegraphics[width=0.095\linewidth]{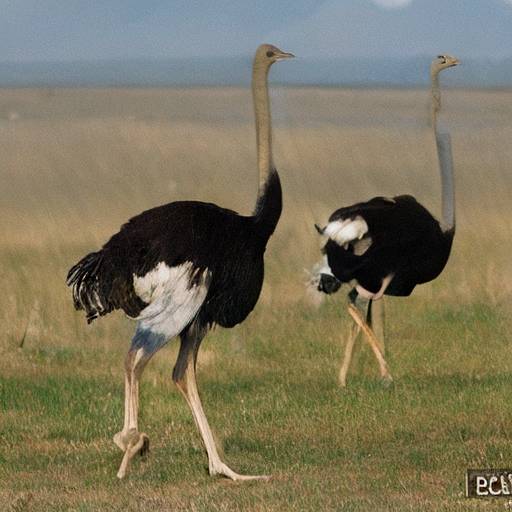}&%
        \includegraphics[width=0.095\linewidth]{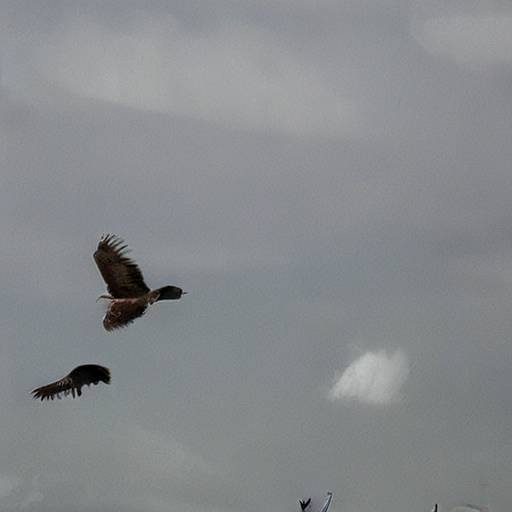}&%
        \includegraphics[width=0.095\linewidth]{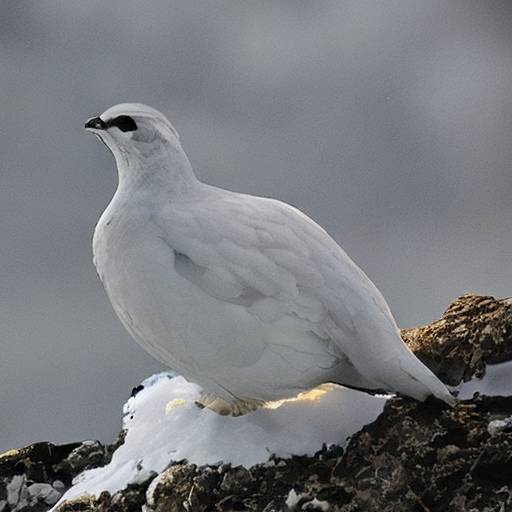}&%
        \includegraphics[width=0.095\linewidth]{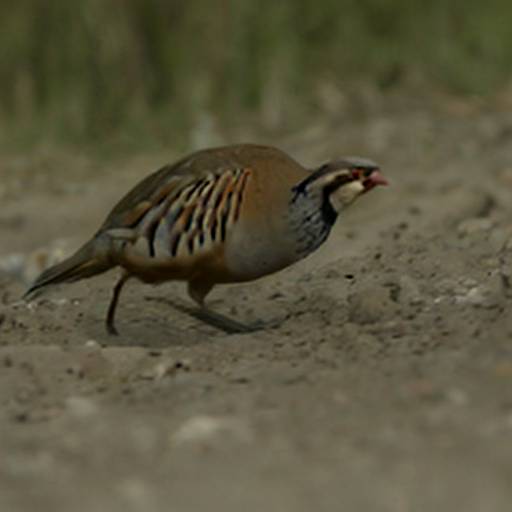}&%
        \includegraphics[width=0.095\linewidth]{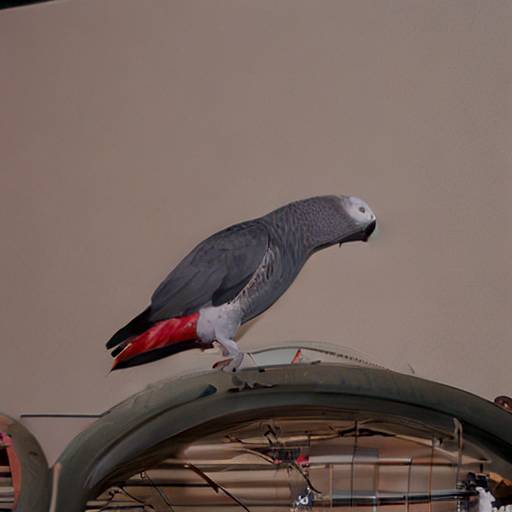}&%
        \includegraphics[width=0.095\linewidth]{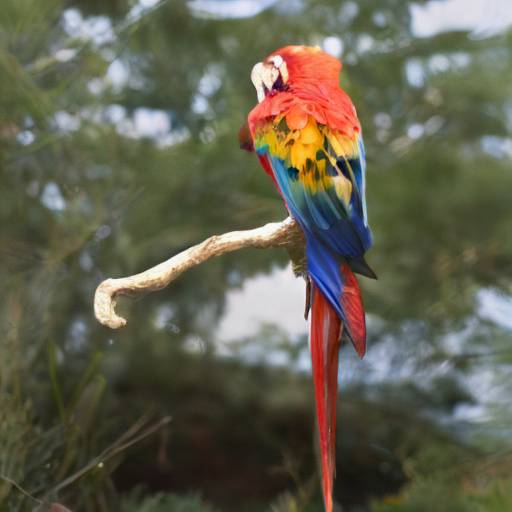}&%
        \includegraphics[width=0.095\linewidth]{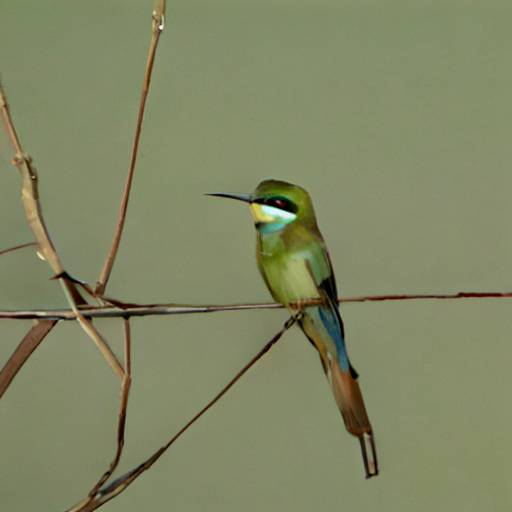}&%
        \includegraphics[width=0.095\linewidth]{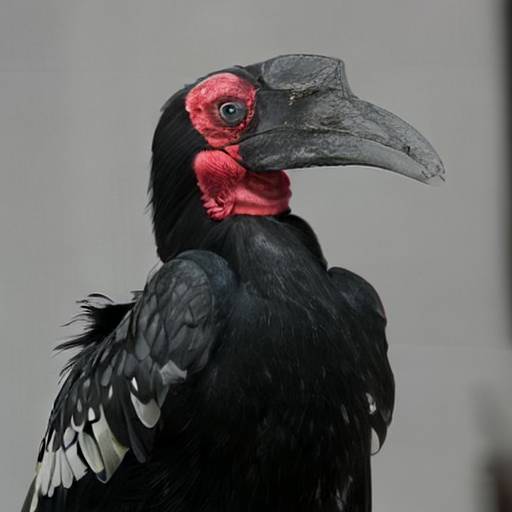}&%
        \includegraphics[width=0.095\linewidth]{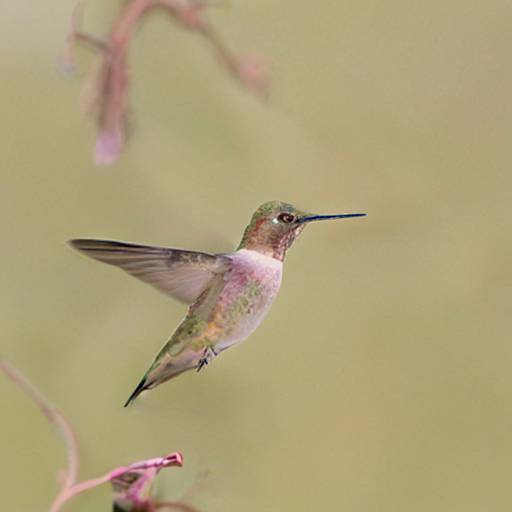}&%
        \includegraphics[width=0.095\linewidth]{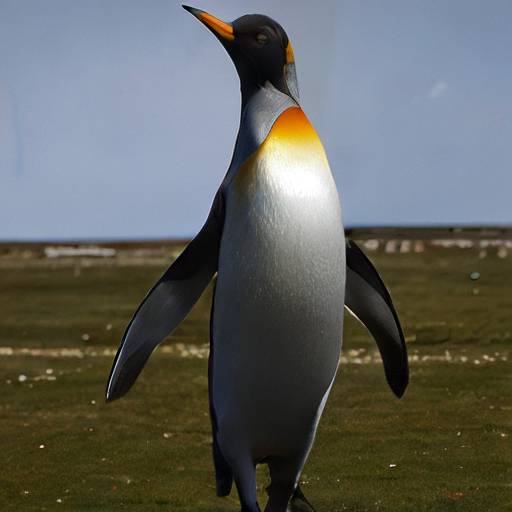}\\%
        
        \raisebox{1ex}{{\scriptsize $\discsolution{0}$ (dyn)}} & 
        \includegraphics[width=0.095\linewidth]{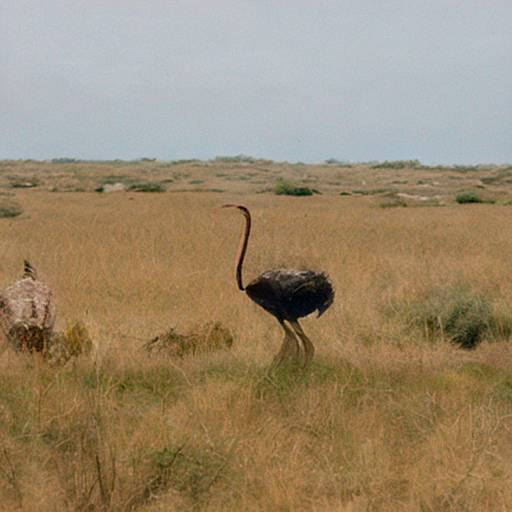}&%
        \includegraphics[width=0.095\linewidth]{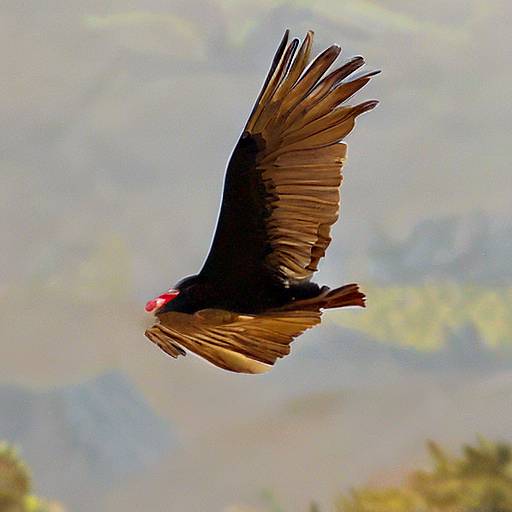}&%
        \includegraphics[width=0.095\linewidth]{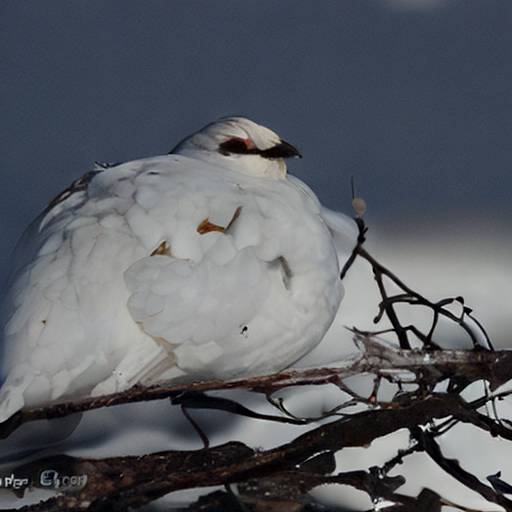}&%
        \includegraphics[width=0.095\linewidth]{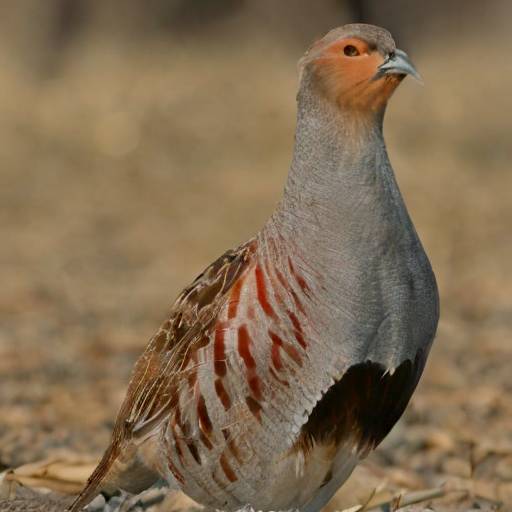}&%
        \includegraphics[width=0.095\linewidth]{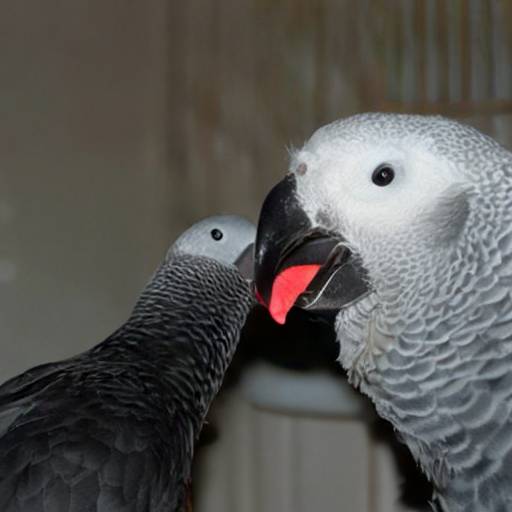}&%
        \includegraphics[width=0.095\linewidth]{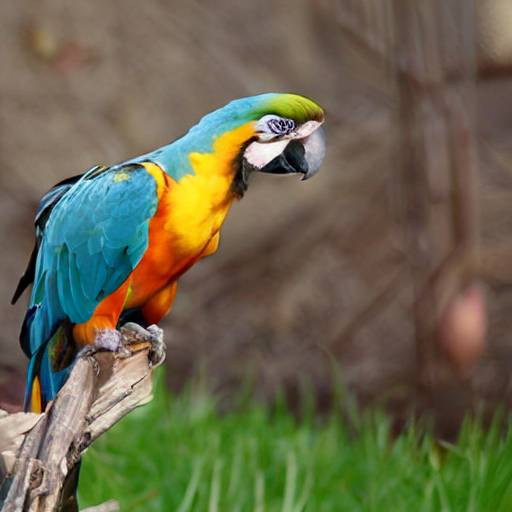}&%
        \includegraphics[width=0.095\linewidth]{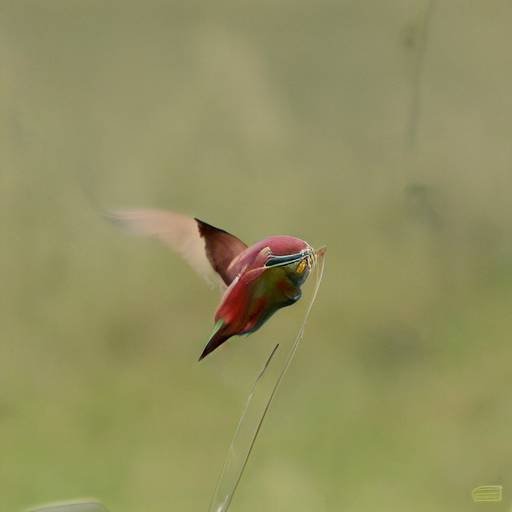}&%
        \includegraphics[width=0.093\linewidth]{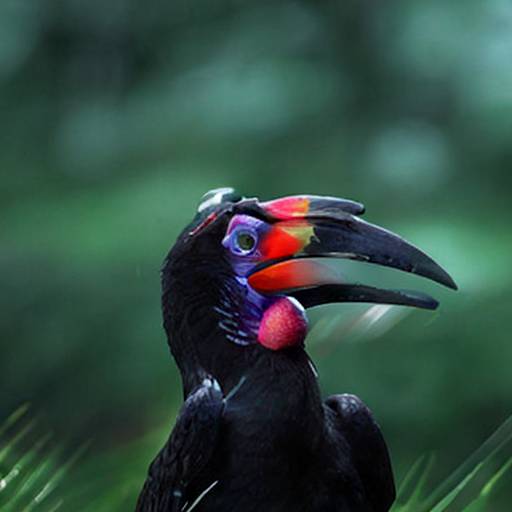}&%
        \includegraphics[width=0.095\linewidth]{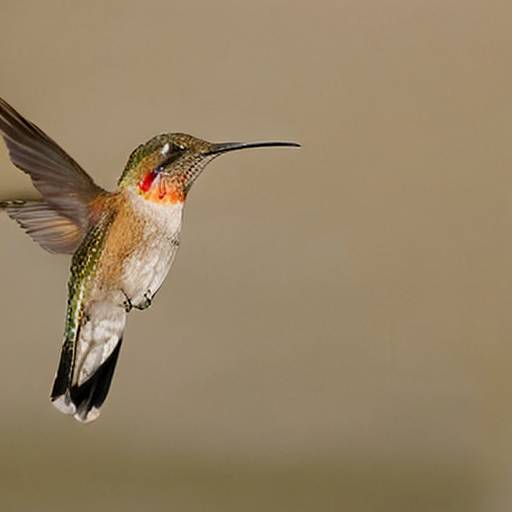}&%
        \includegraphics[width=0.095\linewidth]{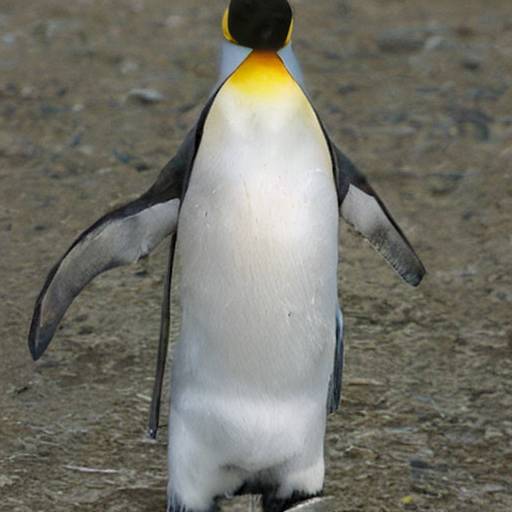}\\%

        \raisebox{1ex}{{\scriptsize $\pforward{T}$}} & 
        \includegraphics[width=0.095\linewidth]{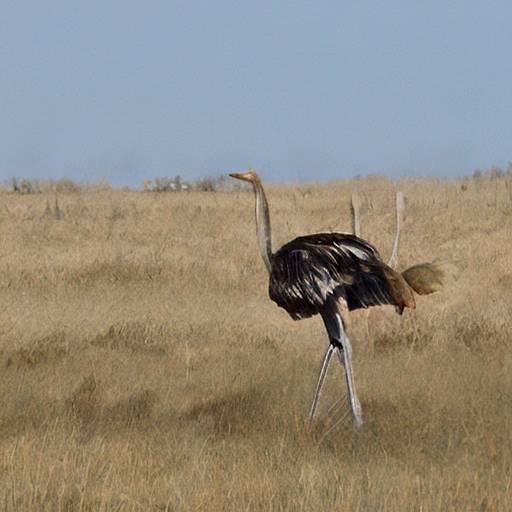}&%
        \includegraphics[width=0.095\linewidth]{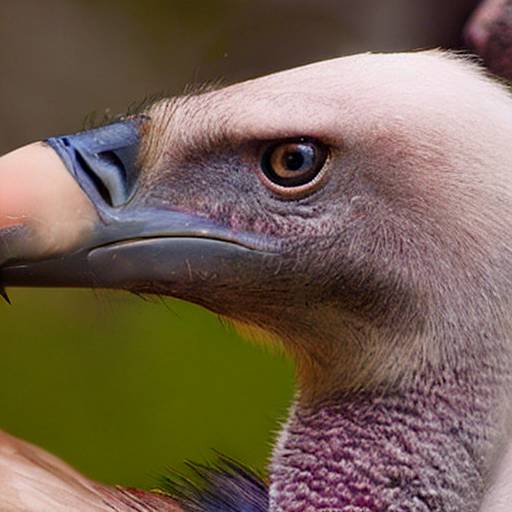}&%
        \includegraphics[width=0.095\linewidth]{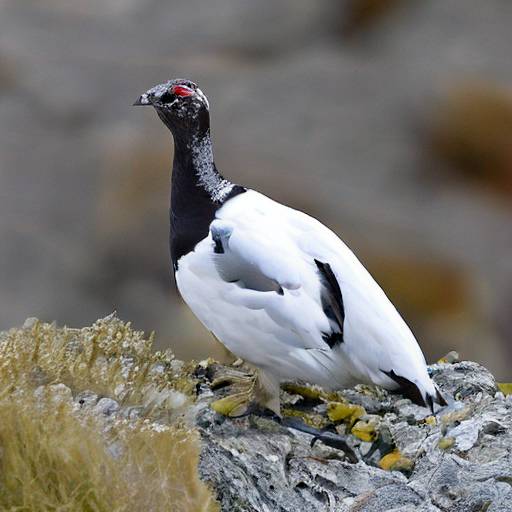}&%
        \includegraphics[width=0.095\linewidth]{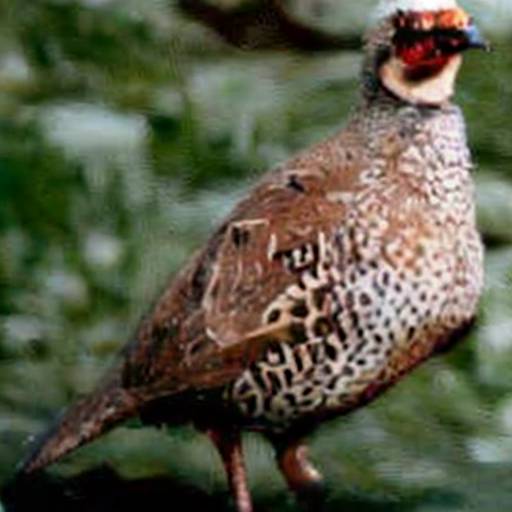}&%
        \includegraphics[width=0.095\linewidth]{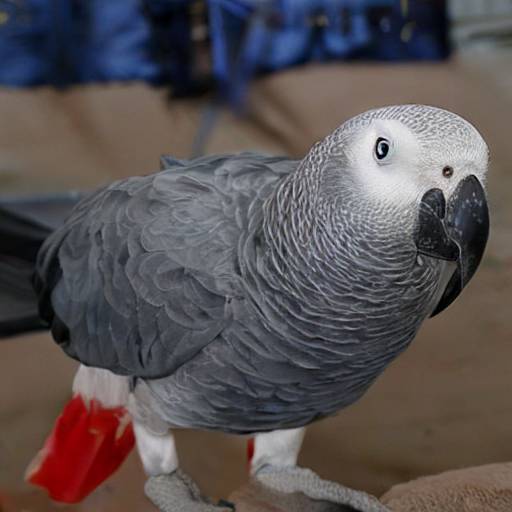}&%
        \includegraphics[width=0.095\linewidth]{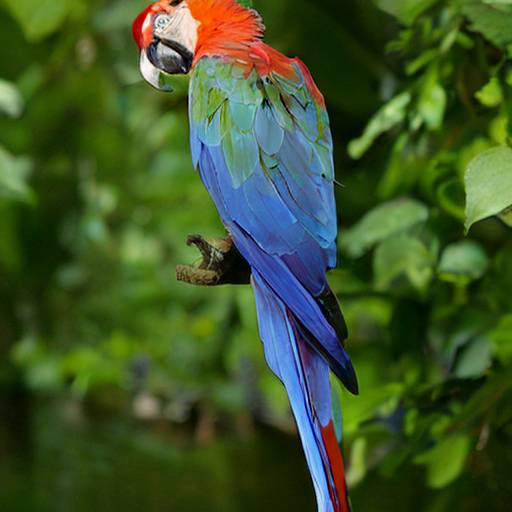}&%
        \includegraphics[width=0.095\linewidth]{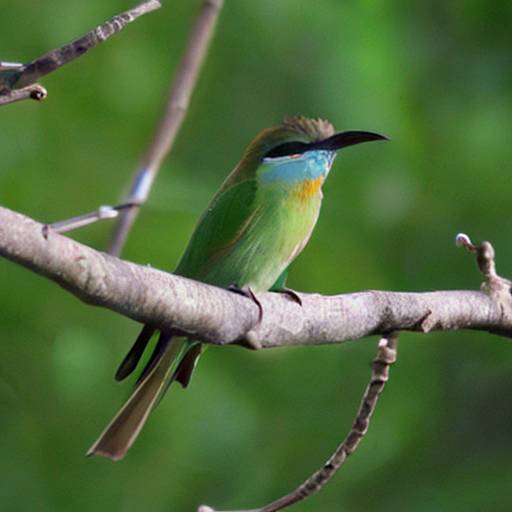}&%
        \includegraphics[width=0.095\linewidth]{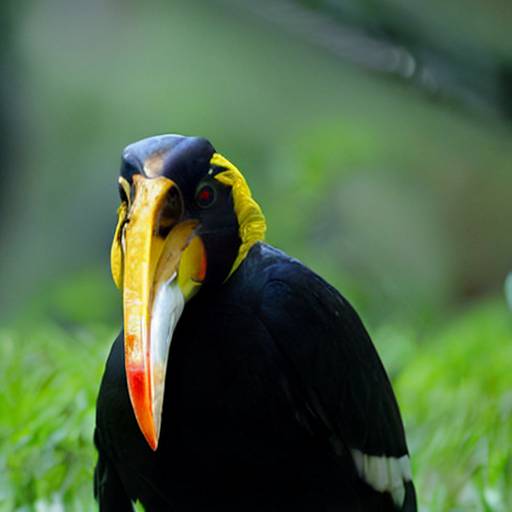}&%
        \includegraphics[width=0.095\linewidth]{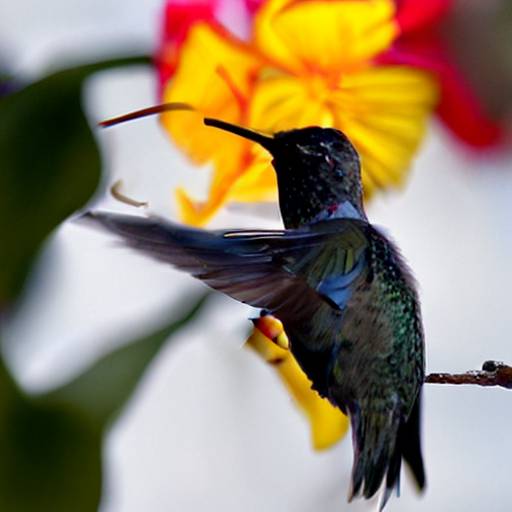}&%
        \includegraphics[width=0.095\linewidth]{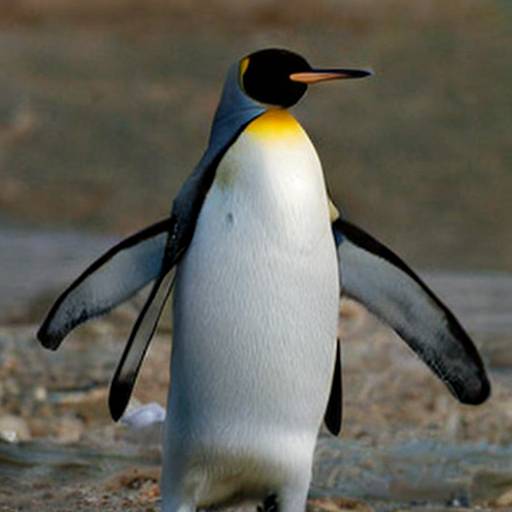}\\%
        
        \raisebox{0.5ex}{\scriptsize \shortstack[l]{$\pinf$ \\ $(\sigma_T = 7)$}}  & 
        \includegraphics[width=0.095\linewidth]{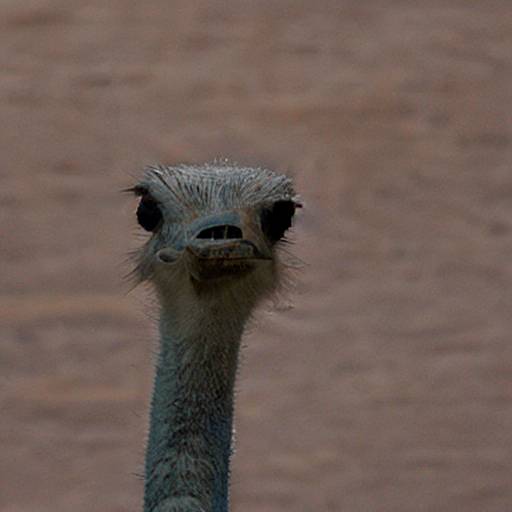}&%
        \includegraphics[width=0.095\linewidth]{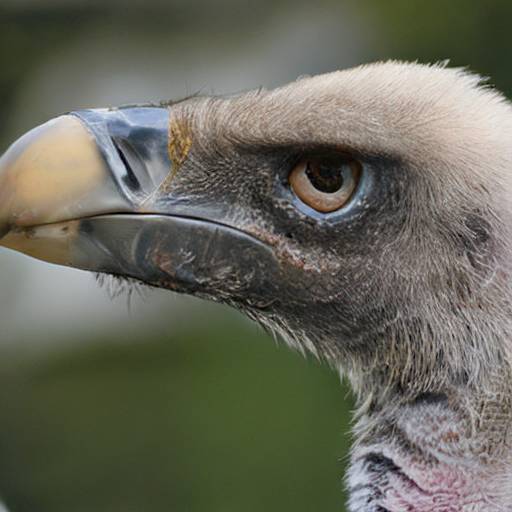}&%
        \includegraphics[width=0.095\linewidth]{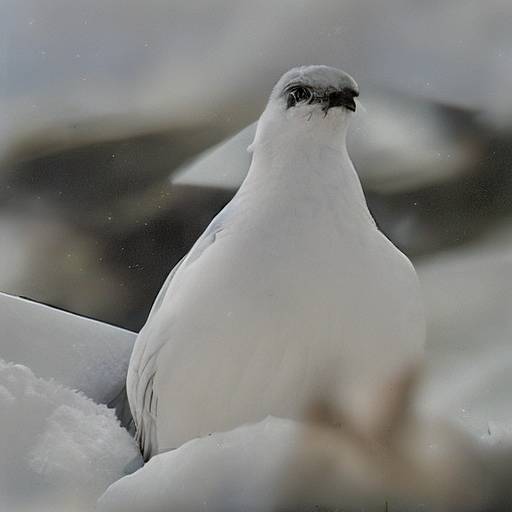}&%
        \includegraphics[width=0.095\linewidth]{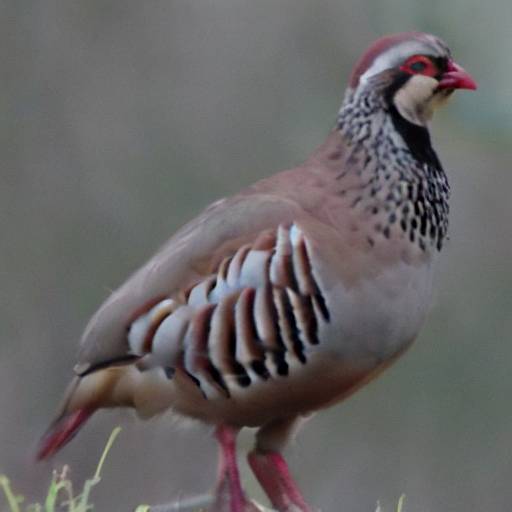}&%
        \includegraphics[width=0.095\linewidth]{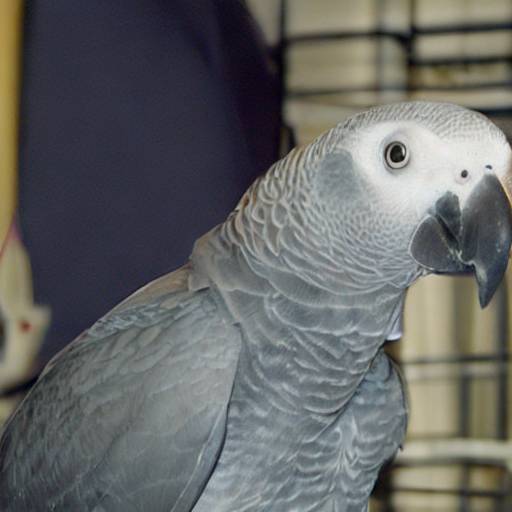}&%
        \includegraphics[width=0.095\linewidth]{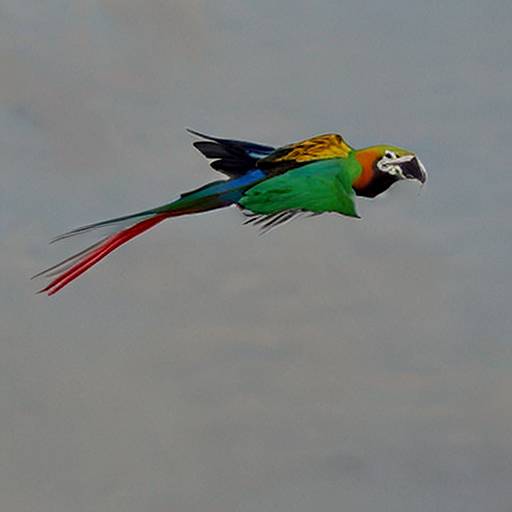}&%
        \includegraphics[width=0.095\linewidth]{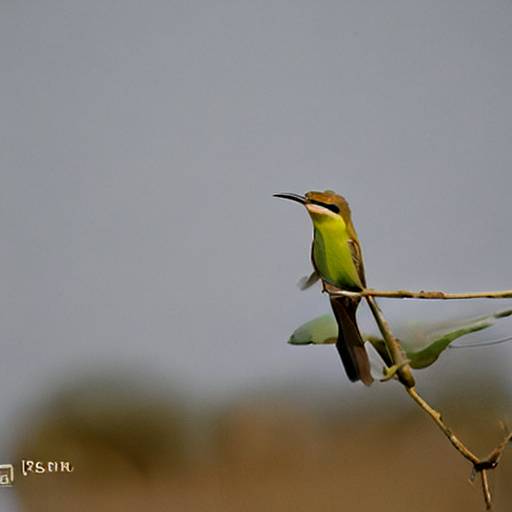}&%
        \includegraphics[width=0.095\linewidth]{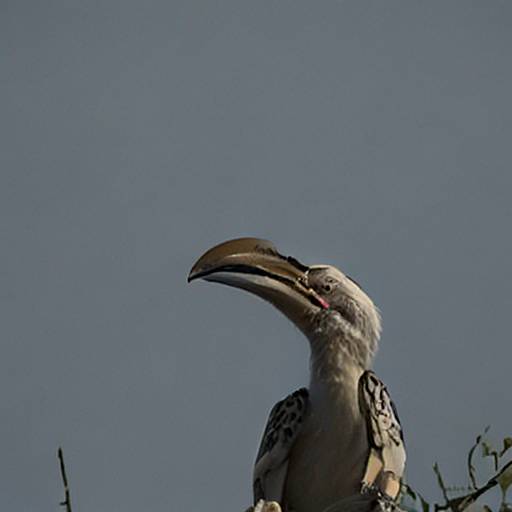}&%
        \includegraphics[width=0.095\linewidth]{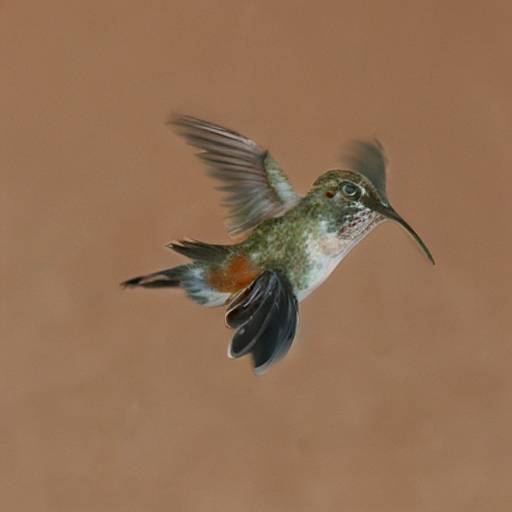}&%
        \includegraphics[width=0.095\linewidth]{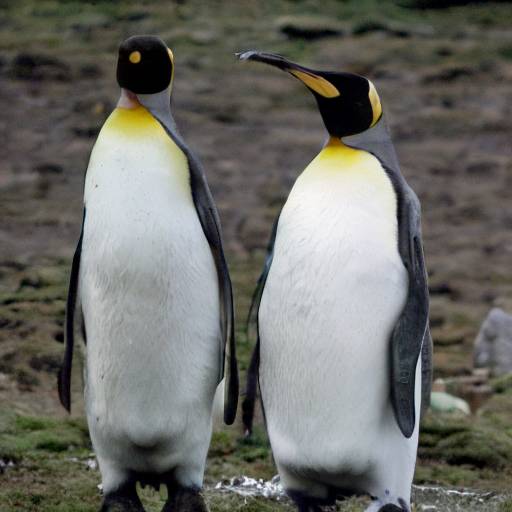}\\%
    \end{tabular}
    \caption{Nearest-neighbour grid comparison for ImageNet$_{\text{birds}}$. Each row corresponds to a different initialization strategy; columns correspond to the queried training classes.}
    \label{fig:photos:birds}
\end{figure}
\vfill

\begin{center}
\begin{table}
    \centering
    \caption{TarFlow~\cite{zhai2025normalizing} configurations: original model versus the lightweight variant we use as $\discsolution{0}$.}
    \resizebox{\textwidth}{!}{
    \begin{tabular}{|c|c|c|c|c|c|c|c|c|}
        \toprule
        Model & Channel size & \# blocks & Layers per block & Epochs & Batch size & Learning rate & Training Factor & Guidance \\
        \midrule
        Original Model & $512$ & $8$ & $8$ & $100$ & $128$ (default) & $10^{-4}$ (default) & ,  & $0$ uncond / $1$ cond \\
        \midrule
        $\discsolution{0}$, FFHQ & $128$ & $1$ & $8$ & $700$ & $500$ & $3\cdot 10^{-4}$ & $14$ & $0$ (uncond) \\
        \midrule
        $\discsolution{0}$, ImageNet$_{\text{birds}}$ & $128$ & $1$ & $8$ & $700$ & $1000$ & $3\cdot 10^{-4}$ & $14$ & $0.5$--$2$ (cond) \\
        \midrule
        $\discsolution{0}$, ImageNet$_{\text{dogs}}$ & $128$ & $1$ & $8$ & $700$ & $1000$ & $3\cdot 10^{-4}$ & $14$ & $0.5$--$2$ (cond) \\
        \bottomrule
    \end{tabular}
    }
    \label{tab:benchmark:config}
\end{table}
\end{center}

\begin{table}
\caption{Aggregated Results for FFHQ}
\centering
\begin{tabular}{lccccc}
\toprule
Model & FID & DINO FD & KID & SWD & MSW \\
\midrule
$\pinf (\sigma_T = 80)$ (standard) & 2.53 & 198.34 & 0.0013 & 0.041 $\pm$ 0.003 & 5.380 $\pm$ 0.282 \\
$\discsolution{0}$  (fixed) & 3.00 & 210.86 & 0.0019 & 0.036 $\pm$ 0.002 & 3.295 $\pm$ 0.331 \\
$\discsolution{0}$ (dynamical) & 3.54 & 221.39 & 0.0024 & 0.062 $\pm$ 0.002 & 6.646 $\pm$ 0.274 \\
$\pforward{T}$ & 2.47 & 188.70 & 0.0014 & 0.011 $\pm$ 0.001 & 0.085 $\pm$ 0.289 \\
$\pinf (\sigma_T = 7)$ & 16.80 & 322.57 & 0.0152 & 0.255 $\pm$ 0.001 & 25.138 $\pm$ 0.134 \\
\bottomrule
\end{tabular}
\label{tab:ffhq}
\end{table}

\paragraph{Results on TarFlow.}
Three messages emerge from these benchmarks. First, across all three datasets, both the empirical short-horizon initialization $\pforward{T}$ and the flow-based initialization $\discsolution{0}$ behave remarkably well, consistently matching, and on ImageNet often improving upon, the classical long-horizon baseline $\pinf\,(\sigma_T = 80)$. The empirical initialization $\pforward{T}$ attains the lowest scores throughout, confirming that whenever faithful samples of the noised target are available, they provide a near-ideal starting distribution. The flow-based variant $\discsolution{0}$ tracks $\pforward{T}$ closely and outperforms the uninformed short-horizon Gaussian $\pinf\,(\sigma_T = 7)$ across FID, KID, DinoFD, and SWD on ImageNet$_{\text{dogs}}$ and ImageNet$_{\text{birds}}$, while remaining competitive on FFHQ-64 (slightly higher FID and DinoFD, but improved SWD and MSW). The qualitative grids in \cref{fig:photos:ffhq,fig:photos:birds,fig:photos:dogs} confirm that the generated samples are diverse and do not collapse onto training images. Taken together, these results indicate that TarFlow has a strong potential as a model class for $\discsolution{0}$ within our pipeline: it succeeds in approximating, through a learned parametric model, the near-ideal regime accessed by $\pforward{T}$.
Second, as a positive side effect, the short-horizon strategy drastically reduces the number of sampling steps required to reach competitive quality. Whereas the standard long-horizon pipeline relies on $40$ steps on FFHQ-64 and $32$ steps on ImageNet-512, both $\pforward{T}$ and $\discsolution{0}$ achieve comparable, and frequently better, scores using only $20$ steps. The comparison between the short- and long-horizon Gaussian initializations further suggests that $\sigma_T = 80$ is more of a default than an optimal choice: paired with an informed prior, a shorter horizon is enough to match, or even surpass, the standard pipeline at a fraction of its computational cost.
Third, on the specific question of fixed versus dynamical training, the fixed variant of $\discsolution{0}$ tends to outperform the dynamical strategy, particularly
on FFHQ.
Overall, these results support the use of TarFlow as $\discsolution{0}$ in our pipeline. Both short-horizon variants provide a practical alternative to the standard long-horizon pipeline, delivering comparable or better quality at a substantially reduced sampling cost.
\vfill
\begin{figure}[t]
    \centering
    \small
    \setlength{\tabcolsep}{0pt}%
    \renewcommand{\arraystretch}{0}%
    
    \begin{tabular}{lcccccccccc}
        \raisebox{1ex}{{\scriptsize Train}} & 
        \includegraphics[width=0.095\linewidth]{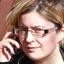}&%
        \includegraphics[width=0.095\linewidth]{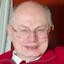}&%
        \includegraphics[width=0.095\linewidth]{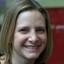}&%
        \includegraphics[width=0.095\linewidth]{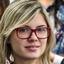}&%
        \includegraphics[width=0.095\linewidth]{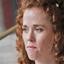}&%
        \includegraphics[width=0.095\linewidth]{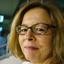}&%
        \includegraphics[width=0.095\linewidth]{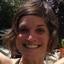}&%
        \includegraphics[width=0.095\linewidth]{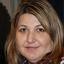}&%
        \includegraphics[width=0.095\linewidth]{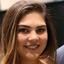}&%
        \includegraphics[width=0.095\linewidth]{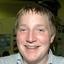}\\%

        \raisebox{0.5ex}{\scriptsize \shortstack[l]{$\pinf$ \\ $(\sigma_T = 80)$}} & 
        \includegraphics[width=0.095\linewidth]{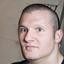}&%
        \includegraphics[width=0.095\linewidth]{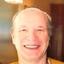}&%
        \includegraphics[width=0.095\linewidth]{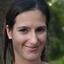}&%
        \includegraphics[width=0.095\linewidth]{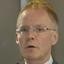}&%
        \includegraphics[width=0.095\linewidth]{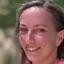}&%
        \includegraphics[width=0.095\linewidth]{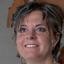}&%
        \includegraphics[width=0.095\linewidth]{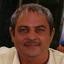}&%
        \includegraphics[width=0.095\linewidth]{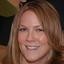}&%
        \includegraphics[width=0.095\linewidth]{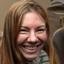}&%
        \includegraphics[width=0.095\linewidth]{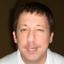}\\%
        
        \raisebox{1ex}{{\scriptsize $\discsolution{0}$ (fix)}} & 
        \includegraphics[width=0.095\linewidth]{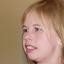}&%
        \includegraphics[width=0.095\linewidth]{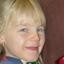}&%
        \includegraphics[width=0.095\linewidth]{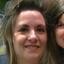}&%
        \includegraphics[width=0.095\linewidth]{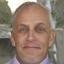}&%
        \includegraphics[width=0.095\linewidth]{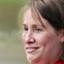}&%
        \includegraphics[width=0.095\linewidth]{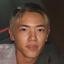}&%
        \includegraphics[width=0.095\linewidth]{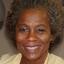}&%
        \includegraphics[width=0.095\linewidth]{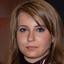}&%
        \includegraphics[width=0.095\linewidth]{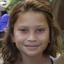}&%
        \includegraphics[width=0.095\linewidth]{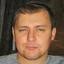}\\%
        
        \raisebox{1ex}{{\scriptsize $\discsolution{0}$ (dyn)}} & 
        \includegraphics[width=0.095\linewidth]{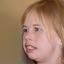}&%
        \includegraphics[width=0.095\linewidth]{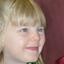}&%
        \includegraphics[width=0.095\linewidth]{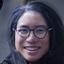}&%
        \includegraphics[width=0.095\linewidth]{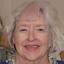}&%
        \includegraphics[width=0.095\linewidth]{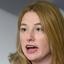}&%
        \includegraphics[width=0.095\linewidth]{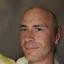}&%
        \includegraphics[width=0.095\linewidth]{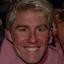}&%
        \includegraphics[width=0.095\linewidth]{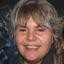}&%
        \includegraphics[width=0.095\linewidth]{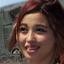}&%
        \includegraphics[width=0.095\linewidth]{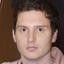}\\%
         
        \raisebox{1ex}{{\scriptsize $\pforward{T}$}} & 
        \includegraphics[width=0.095\linewidth]{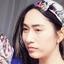}&%
        \includegraphics[width=0.095\linewidth]{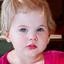}&%
        \includegraphics[width=0.095\linewidth]{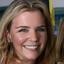}&%
        \includegraphics[width=0.095\linewidth]{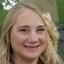}&%
        \includegraphics[width=0.095\linewidth]{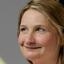}&%
        \includegraphics[width=0.095\linewidth]{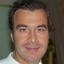}&%
        \includegraphics[width=0.095\linewidth]{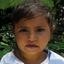}&%
        \includegraphics[width=0.095\linewidth]{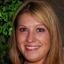}&%
        \includegraphics[width=0.095\linewidth]{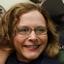}&%
        \includegraphics[width=0.095\linewidth]{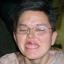}\\%
        
        \raisebox{0.5ex}{\scriptsize \shortstack[l]{$\pinf$ \\ $(\sigma_T = 7)$}}  & 
        \includegraphics[width=0.095\linewidth]{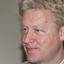}&%
        \includegraphics[width=0.095\linewidth]{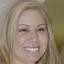}&%
        \includegraphics[width=0.095\linewidth]{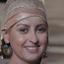}&%
        \includegraphics[width=0.095\linewidth]{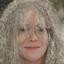}&%
        \includegraphics[width=0.095\linewidth]{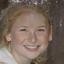}&%
        \includegraphics[width=0.095\linewidth]{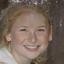}&%
        \includegraphics[width=0.095\linewidth]{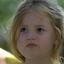}&%
        \includegraphics[width=0.095\linewidth]{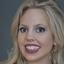}&%
        \includegraphics[width=0.095\linewidth]{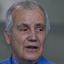}&%
        \includegraphics[width=0.095\linewidth]{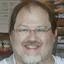}\\%
    \end{tabular}
    \caption{Nearest-neighbour grid comparison for FFHQ-64. Each row corresponds to a different initialization strategy.}
    \label{fig:photos:ffhq}
\end{figure}
\vfill

\begin{table}
    \caption{Aggregated Results for ImageNet$_{\text{dogs}}$}
\label{tab:in_dogs}
\centering
\begin{tabular}{lccccc}
\toprule
Model & FID & DINO FD & KID & SWD & MSW \\
\midrule
$\pinf (\sigma_T = 80)$ (standard) & 10.46 & 95.30 & 0.0050 & 0.040 $\pm$ 0.001 & 3.739 $\pm$ 0.088 \\
$\discsolution{0}$ - CFG$_{\text{flow}} = 0.5$  (fixed)  & 5.53 & 70.35 & 0.0022 & 0.029 $\pm$ 0.001 & 2.911 $\pm$ 0.073 \\
$\discsolution{0}$ - CFG$_{\text{flow}} = 1.0$  (fixed)  & 5.25 & 69.59 & 0.0021 & 0.034 $\pm$ 0.001 & 3.034 $\pm$ 0.071 \\
$\discsolution{0}$ - CFG$_{\text{flow}} = 1.5$  (fixed)  & 4.76 & 70.36 & 0.0018 & 0.041 $\pm$ 0.001 & 4.359 $\pm$ 0.171 \\
$\discsolution{0}$ - CFG$_{\text{flow}} = 2.0$  (fixed)  & 4.25 & 71.51 & 0.0014 & 0.050 $\pm$ 0.001 & 6.146 $\pm$ 0.175 \\
$\discsolution{0}$ - CFG$_{\text{flow}} = 0.5$ (dynamical) & 6.66 & 73.08 & 0.0029 & 0.033 $\pm$ 0.001 & 2.810 $\pm$ 0.068 \\
$\discsolution{0}$ - CFG$_{\text{flow}} = 1.0$ (dynamical) & 7.44 & 76.88 & 0.0033 & 0.034 $\pm$ 0.001 & 2.771 $\pm$ 0.058 \\
$\discsolution{0}$ - CFG$_{\text{flow}} = 1.5$ (dynamical) & 8.10 & 81.47 & 0.0036 & 0.035 $\pm$ 0.000 & 2.882 $\pm$ 0.054 \\
$\discsolution{0}$ - CFG$_{\text{flow}} = 2.0$ (dynamical) & 8.69 & 87.11 & 0.0039 & 0.038 $\pm$ 0.000 & 3.816 $\pm$ 0.060 \\
$\pforward{T}$ & 6.55 & 68.93 & 0.0030 & 0.024 $\pm$ 0.000 & 0.376 $\pm$ 0.665 \\
$\pinf (\sigma_T = 7)$& 7.20 & 85.17 & 0.0028 & 0.034 $\pm$ 0.000 & 7.287 $\pm$ 0.072 \\
\bottomrule
\end{tabular}
\end{table}

\begin{figure}[t]
    \centering
    \small
    \setlength{\tabcolsep}{0pt}%
    \renewcommand{\arraystretch}{0}%
    \begin{tabular}{lcccccccccc}
        & 157 & 196 & 206 & 219 & 225 & 228 & 234 & 235 & 236 & 268 \vspace{0.05cm}
\\%        
        \raisebox{1ex}{{\scriptsize Train}} & 
        \includegraphics[width=0.095\linewidth]{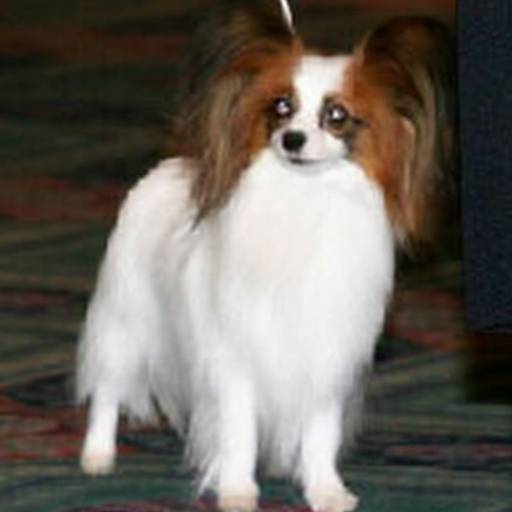}&%
        \includegraphics[width=0.095\linewidth]{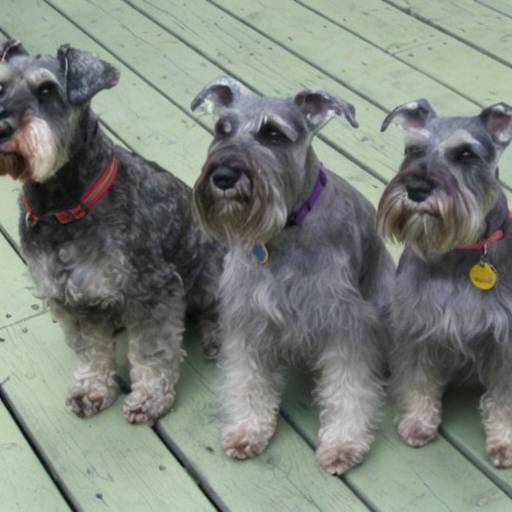}&%
        \includegraphics[width=0.095\linewidth]{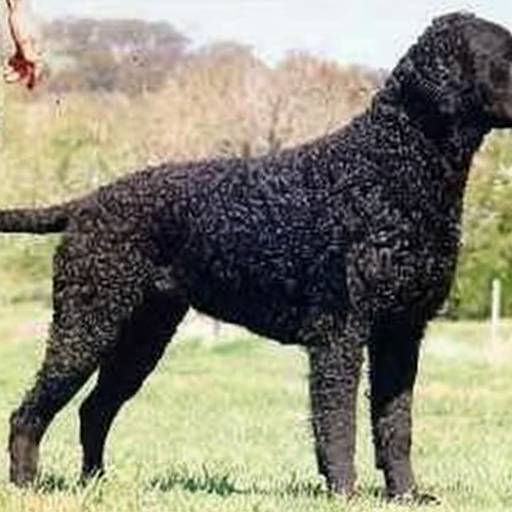}&%
        \includegraphics[width=0.095\linewidth]{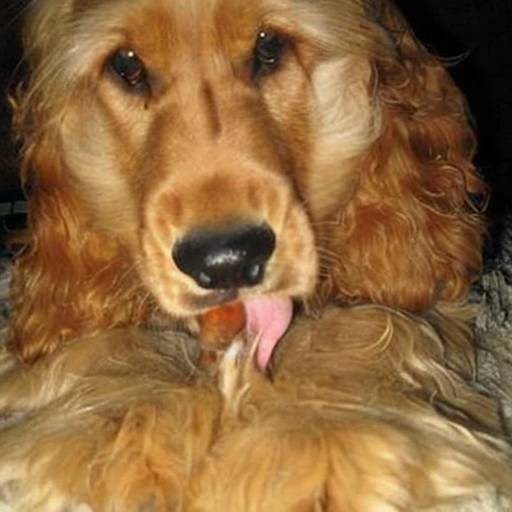}&%
        \includegraphics[width=0.095\linewidth]{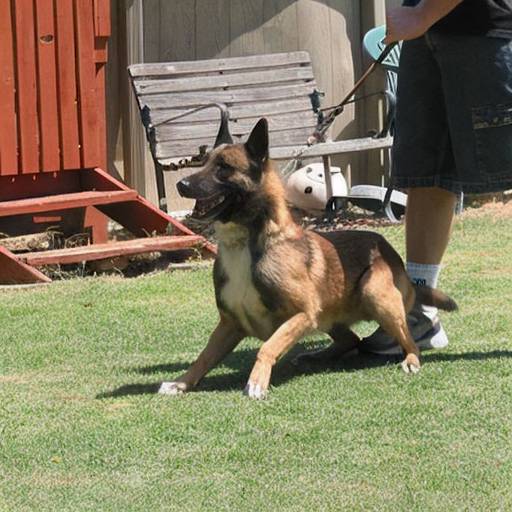}&%
        \includegraphics[width=0.095\linewidth]{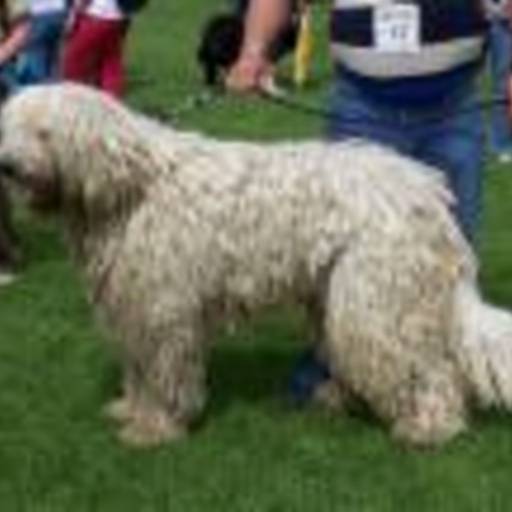}&%
        \includegraphics[width=0.095\linewidth]{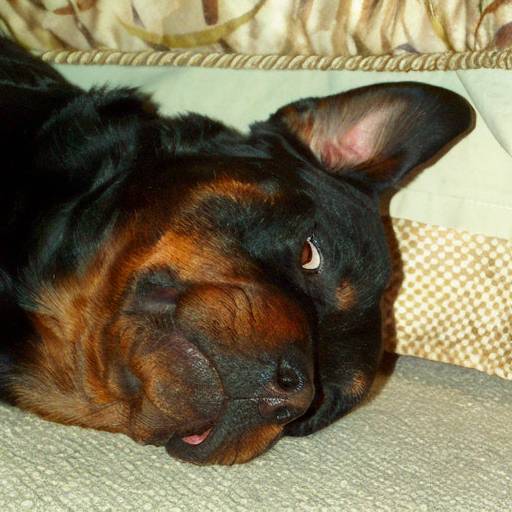}&%
        \includegraphics[width=0.095\linewidth]{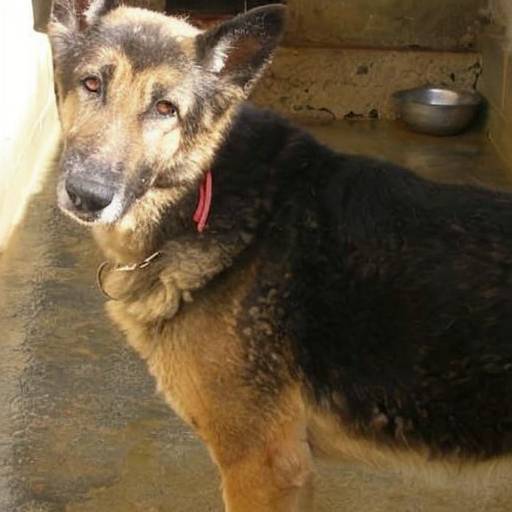}&%
        \includegraphics[width=0.095\linewidth]{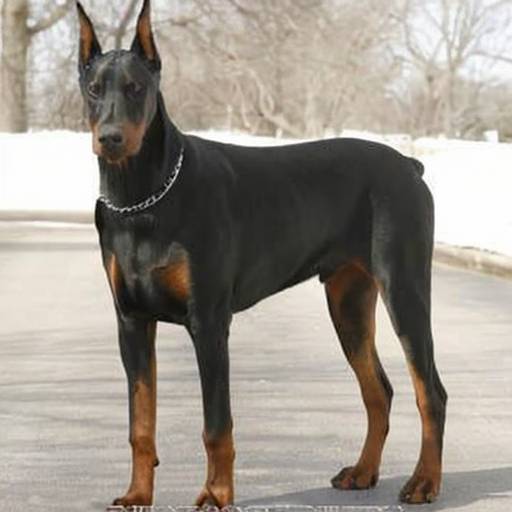}&%
        \includegraphics[width=0.095\linewidth]{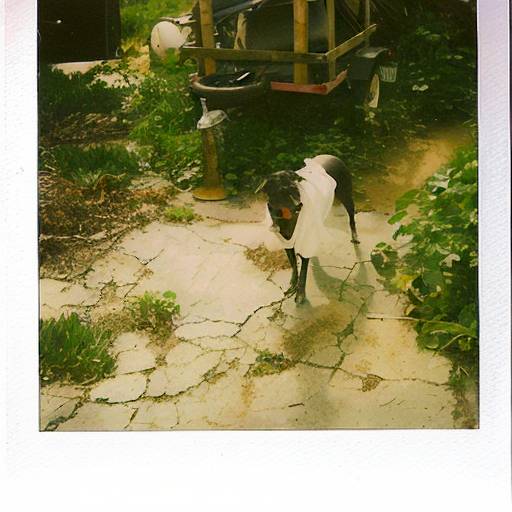}\\%

        \raisebox{0.5ex}{\scriptsize \shortstack[l]{$\pinf$ \\ $(\sigma_T = 80)$}} &
        \includegraphics[width=0.095\linewidth]{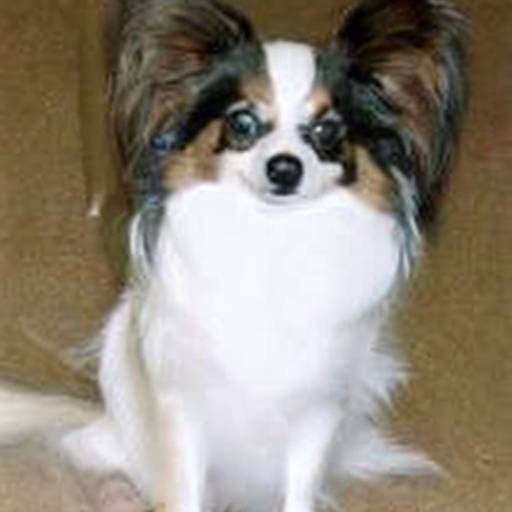}&%
        \includegraphics[width=0.095\linewidth]{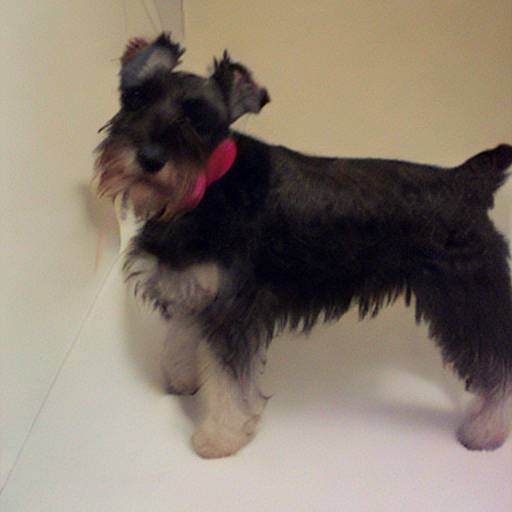}&%
        \includegraphics[width=0.095\linewidth]{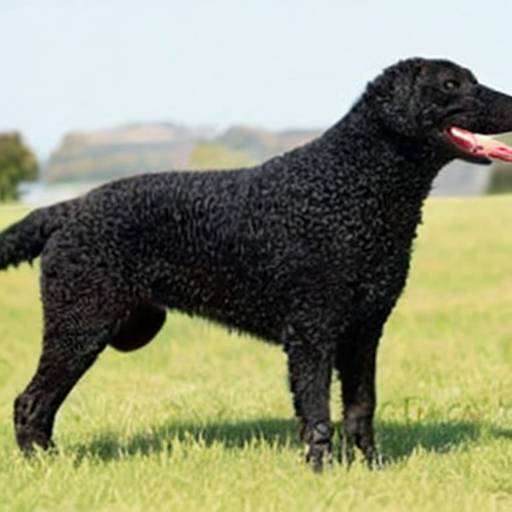}&%
        \includegraphics[width=0.095\linewidth]{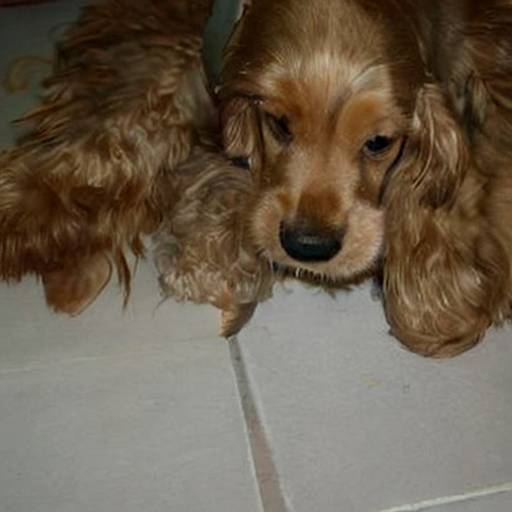}&%
        \includegraphics[width=0.095\linewidth]{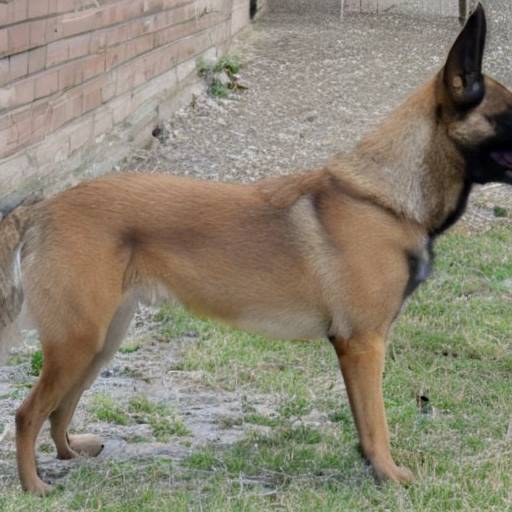}&%
        \includegraphics[width=0.095\linewidth]{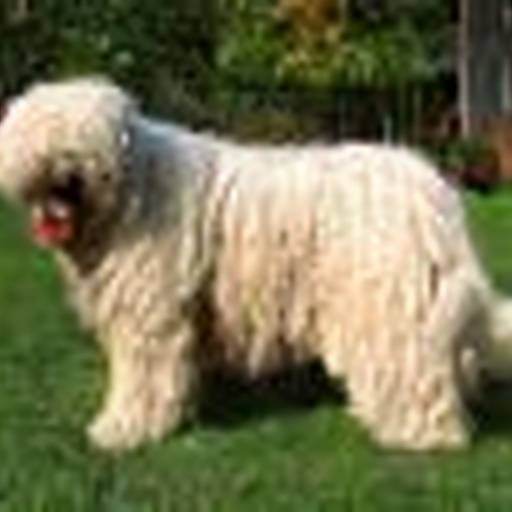}&%
        \includegraphics[width=0.095\linewidth]{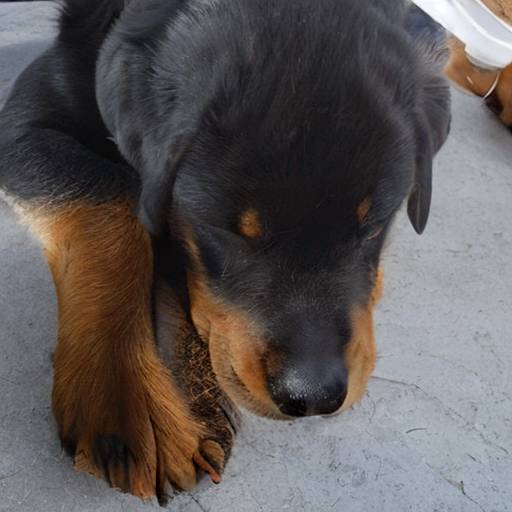}&%
        \includegraphics[width=0.095\linewidth]{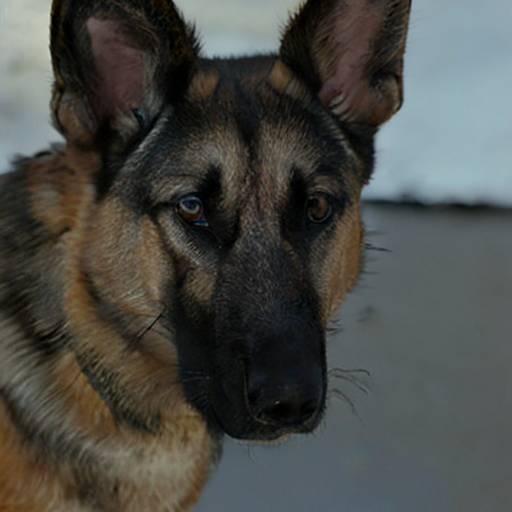}&%
        \includegraphics[width=0.095\linewidth]{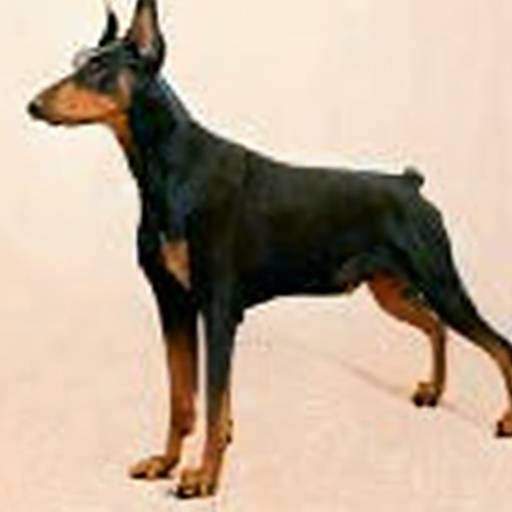}&%
        \includegraphics[width=0.095\linewidth]{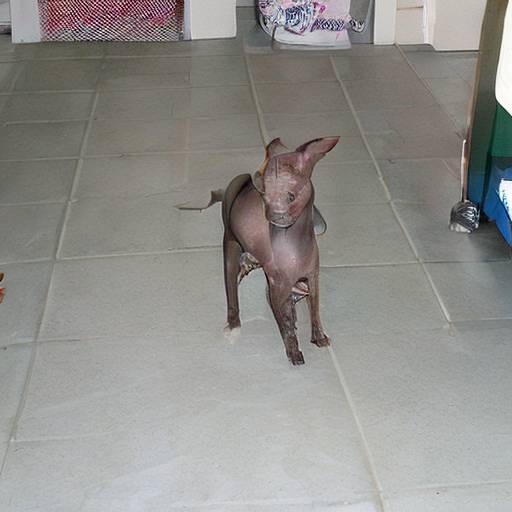}\\%
        
        \raisebox{1ex}{{\scriptsize $\discsolution{0}$ (fix)}} & 
        \includegraphics[width=0.095\linewidth]{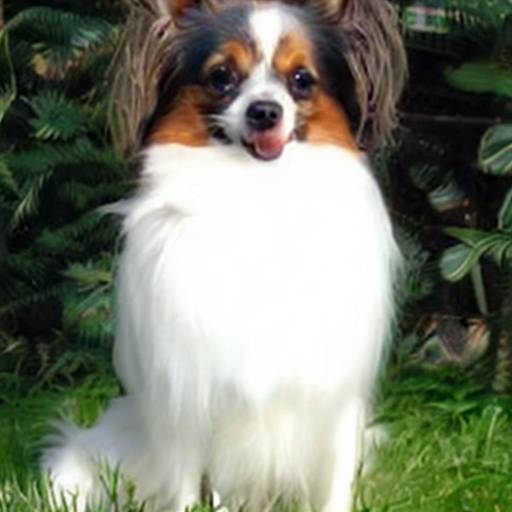}&%
        \includegraphics[width=0.095\linewidth]{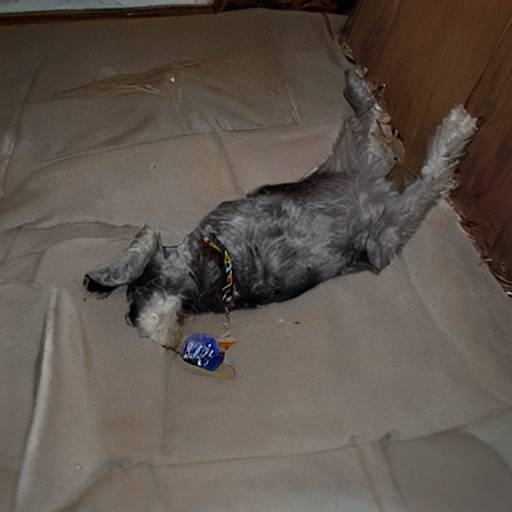}&%
        \includegraphics[width=0.095\linewidth]{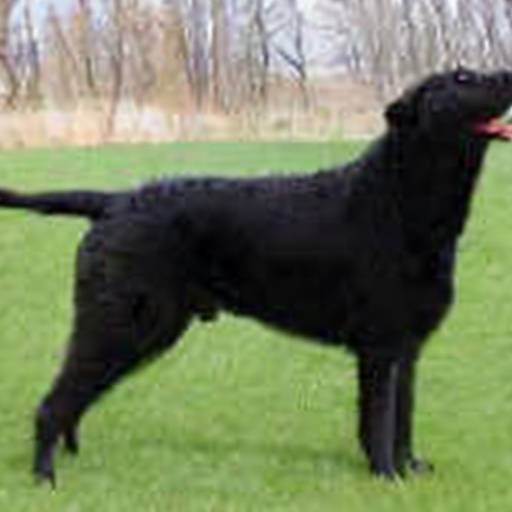}&%
        \includegraphics[width=0.095\linewidth]{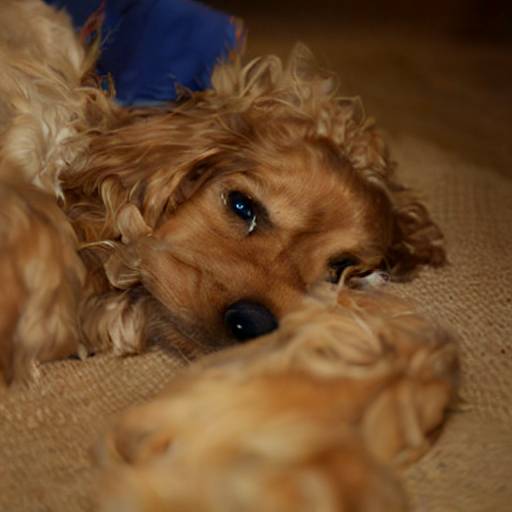}&%
        \includegraphics[width=0.095\linewidth]{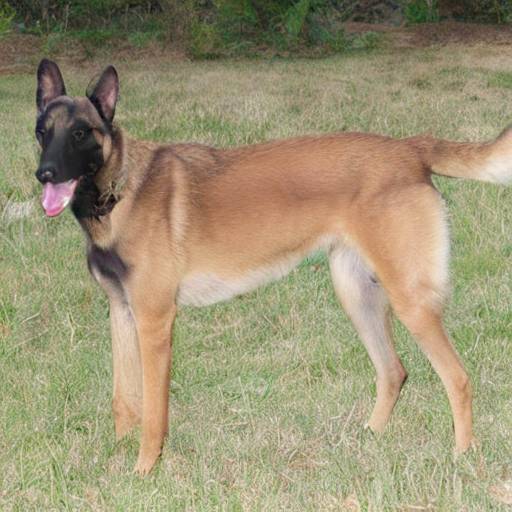}&%
        \includegraphics[width=0.095\linewidth]{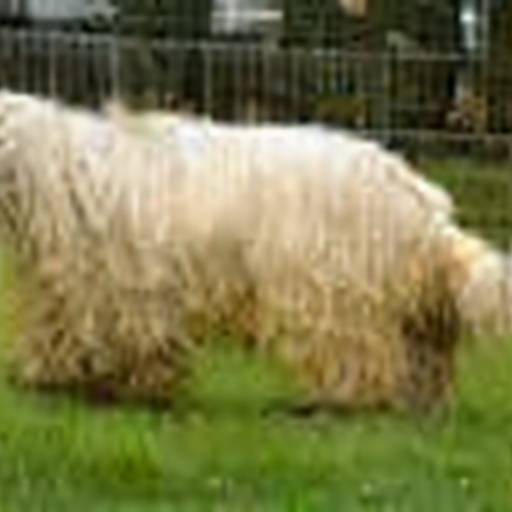}&%
        \includegraphics[width=0.095\linewidth]{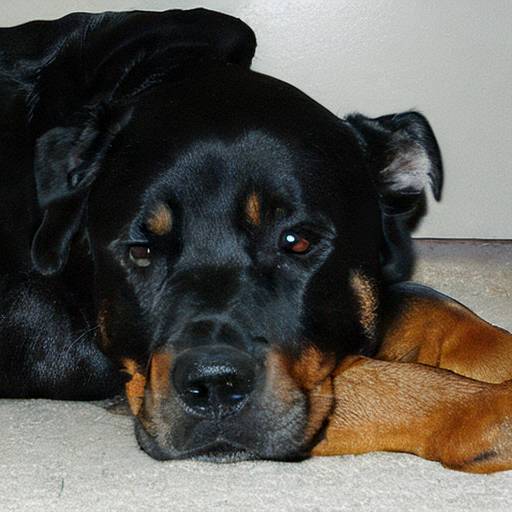}&%
        \includegraphics[width=0.095\linewidth]{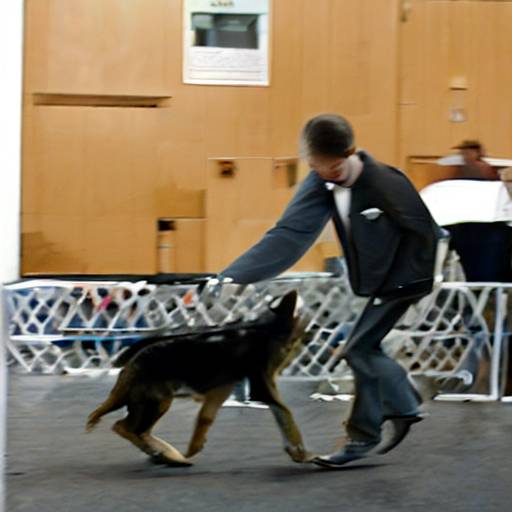}&%
        \includegraphics[width=0.095\linewidth]{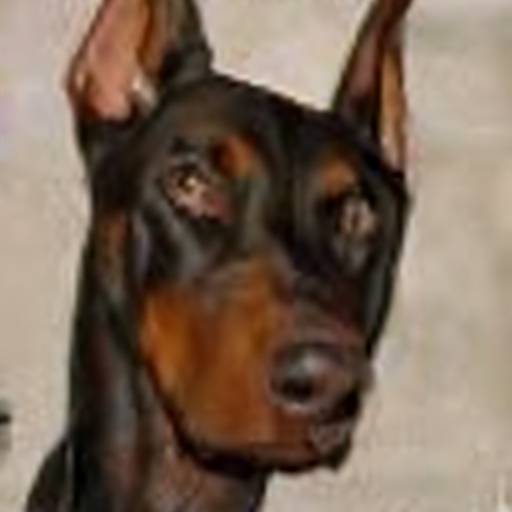}&%
        \includegraphics[width=0.095\linewidth]{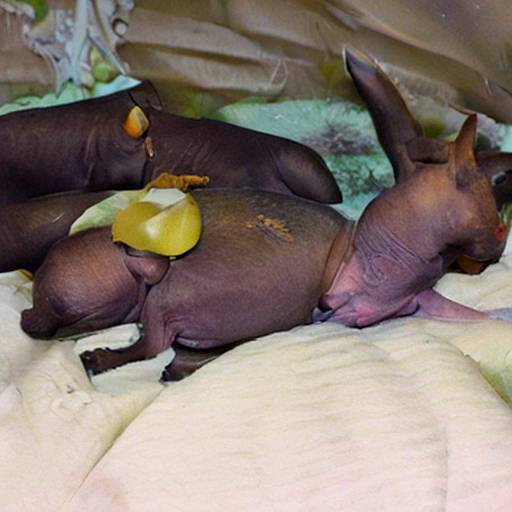}\\%
        
        \raisebox{1ex}{{\scriptsize $\discsolution{0}$ (dyn)}} & 
        \includegraphics[width=0.095\linewidth]{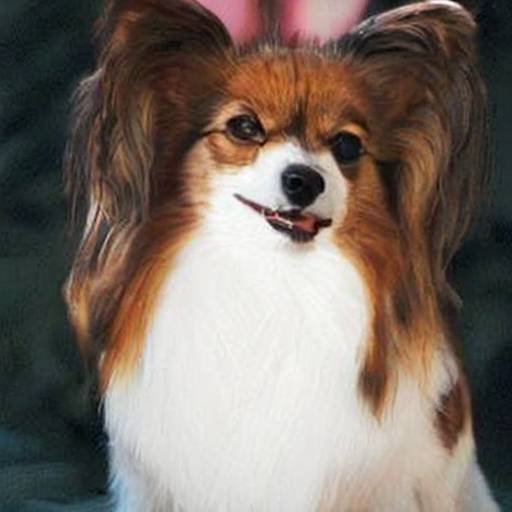}&%
        \includegraphics[width=0.095\linewidth]{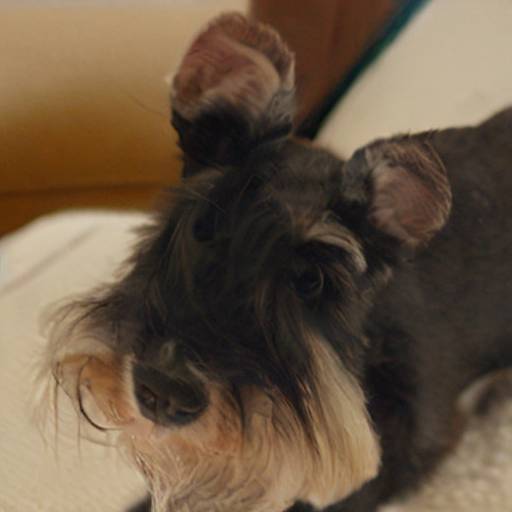}&%
        \includegraphics[width=0.095\linewidth]{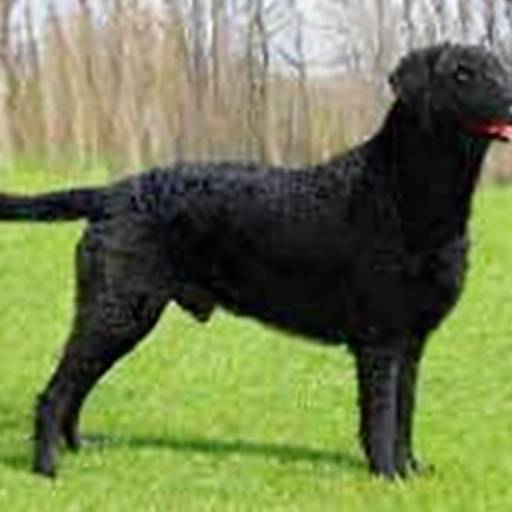}&%
        \includegraphics[width=0.095\linewidth]{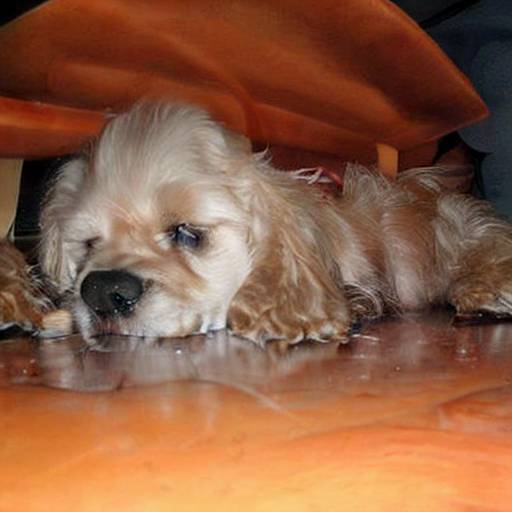}&%
        \includegraphics[width=0.095\linewidth]{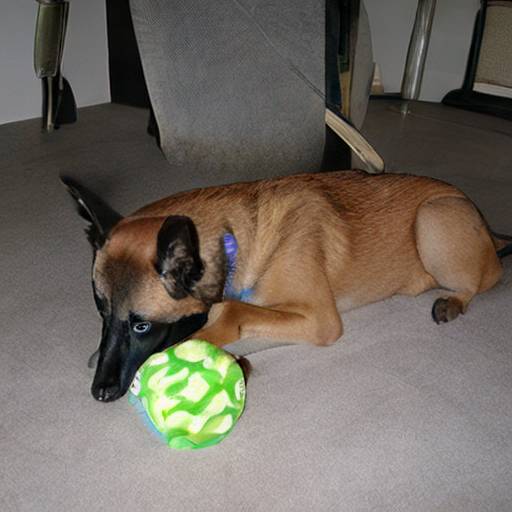}&%
        \includegraphics[width=0.095\linewidth]{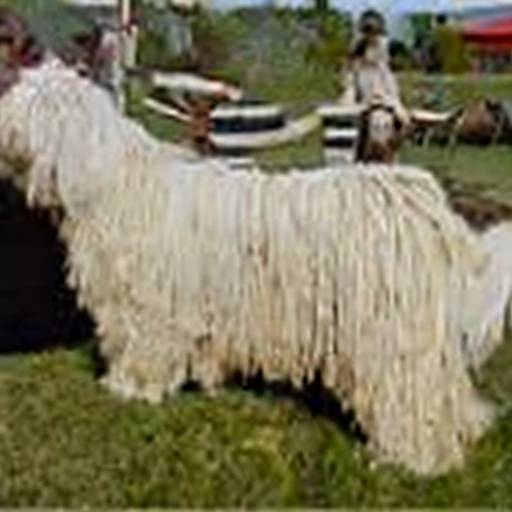}&%
        \includegraphics[width=0.095\linewidth]{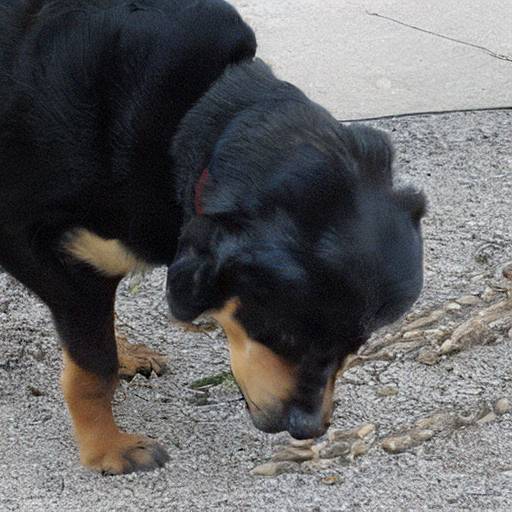}&%
        \includegraphics[width=0.093\linewidth]{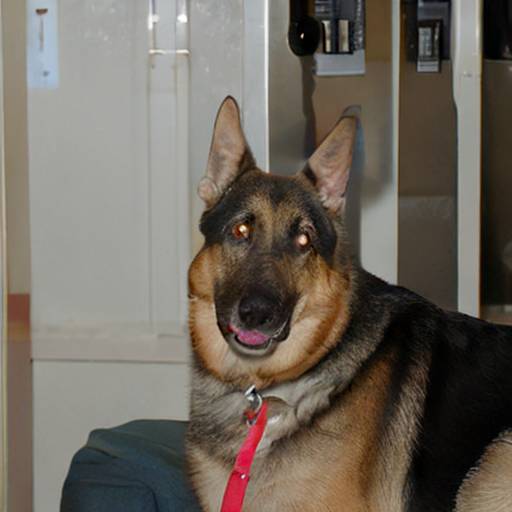}&%
        \includegraphics[width=0.095\linewidth]{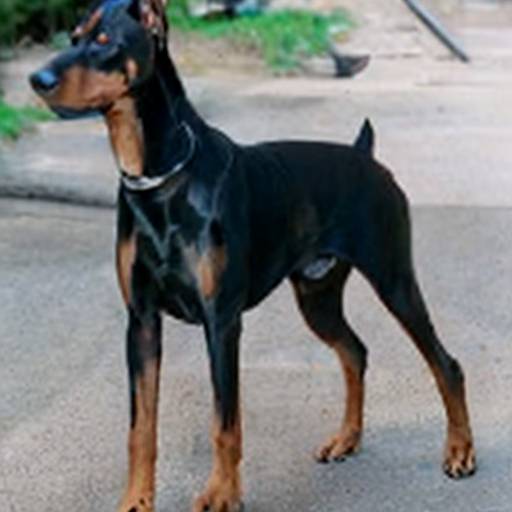}&%
        \includegraphics[width=0.095\linewidth]{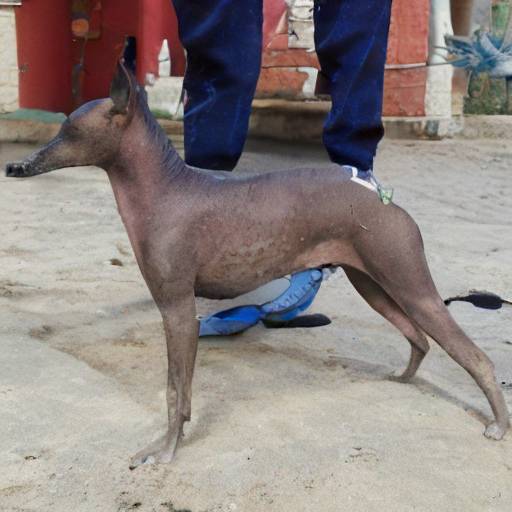}\\%

        \raisebox{1ex}{{\scriptsize $\pforward{T}$}} & 
        \includegraphics[width=0.095\linewidth]{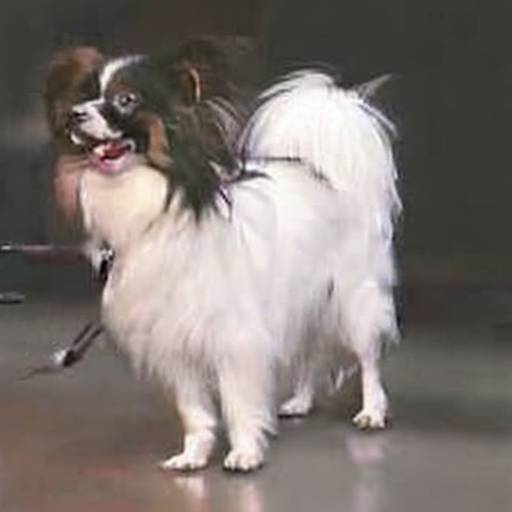}&%
        \includegraphics[width=0.095\linewidth]{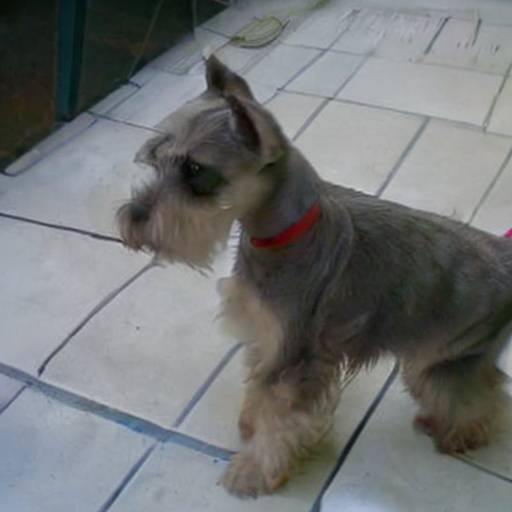}&%
        \includegraphics[width=0.095\linewidth]{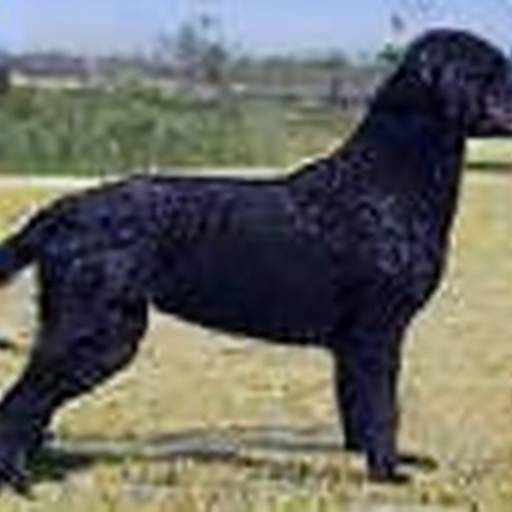}&%
        \includegraphics[width=0.095\linewidth]{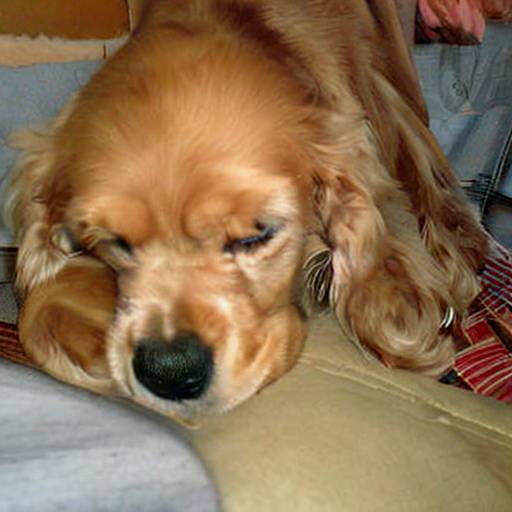}&%
        \includegraphics[width=0.095\linewidth]{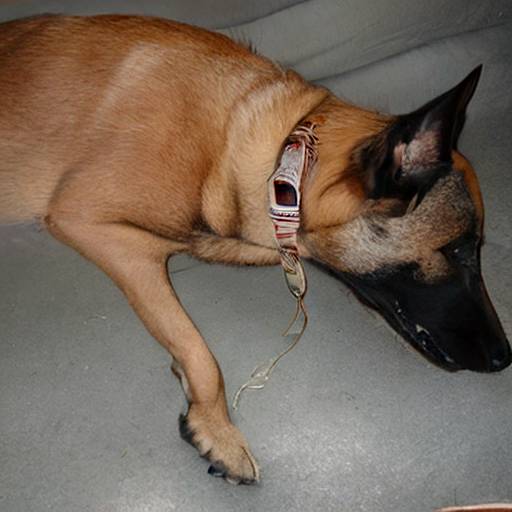}&%
        \includegraphics[width=0.095\linewidth]{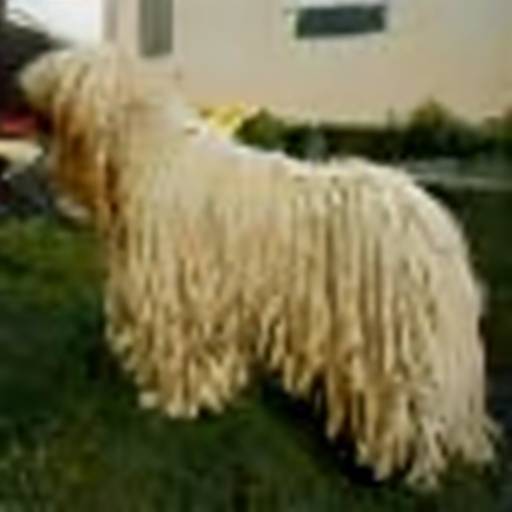}&%
        \includegraphics[width=0.095\linewidth]{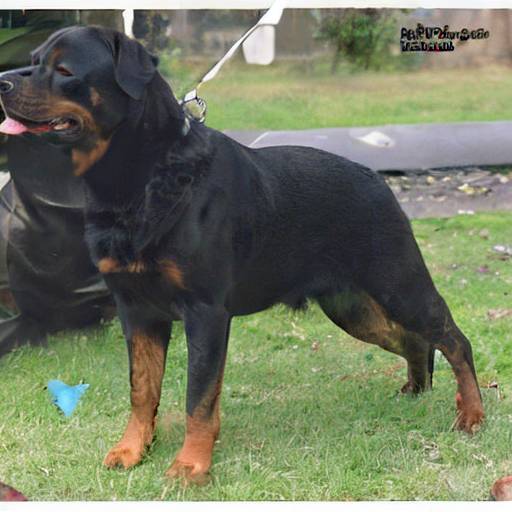}&%
        \includegraphics[width=0.095\linewidth]{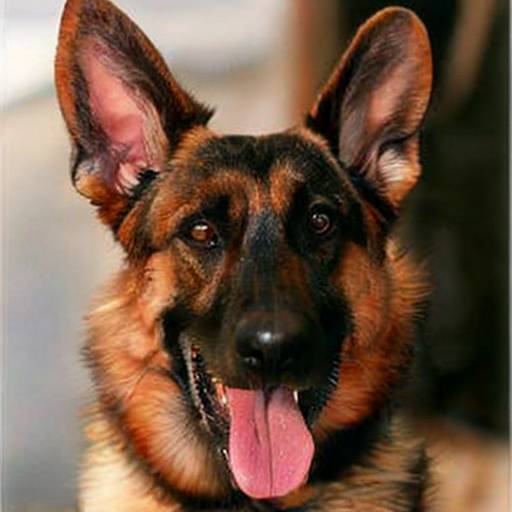}&%
        \includegraphics[width=0.095\linewidth]{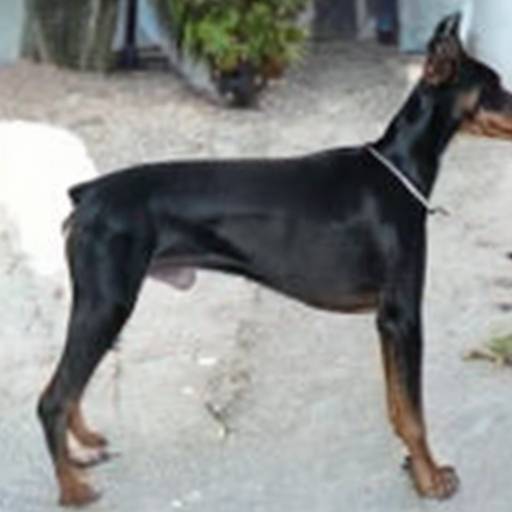}&%
        \includegraphics[width=0.095\linewidth]{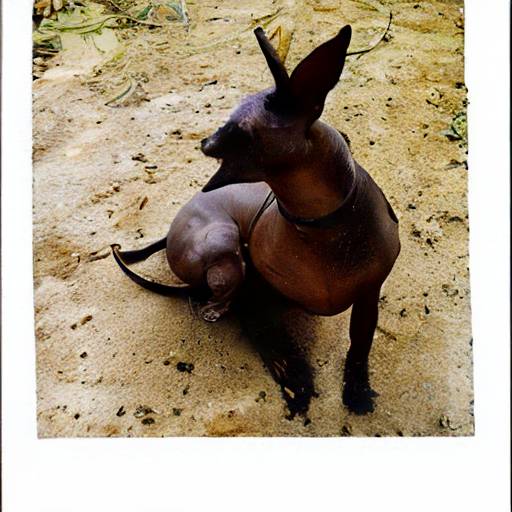}\\%

        \raisebox{0.5ex}{\scriptsize \shortstack[l]{$\pinf$ \\ $(\sigma_T = 7)$}}  & 
        \includegraphics[width=0.095\linewidth]{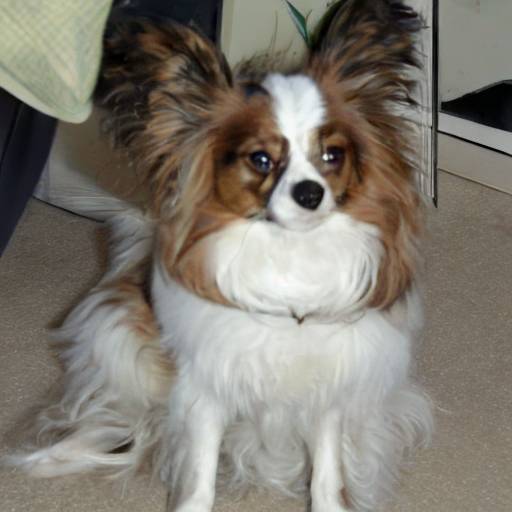}&%
        \includegraphics[width=0.095\linewidth]{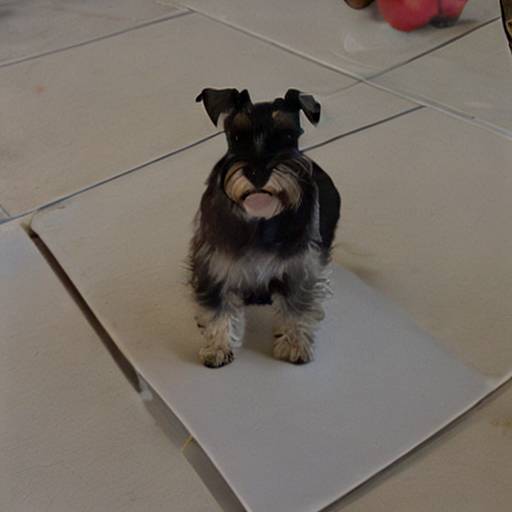}&%
        \includegraphics[width=0.095\linewidth]{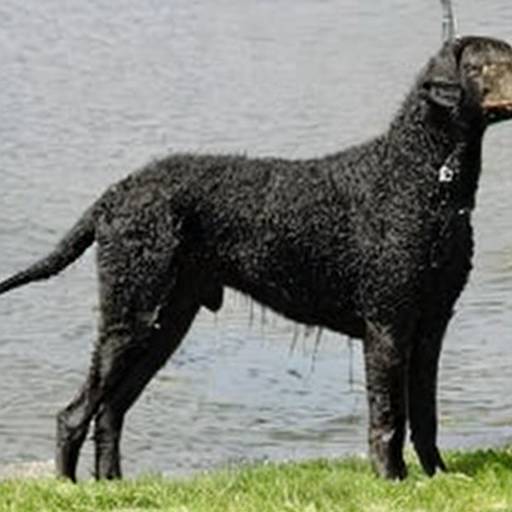}&%
        \includegraphics[width=0.095\linewidth]{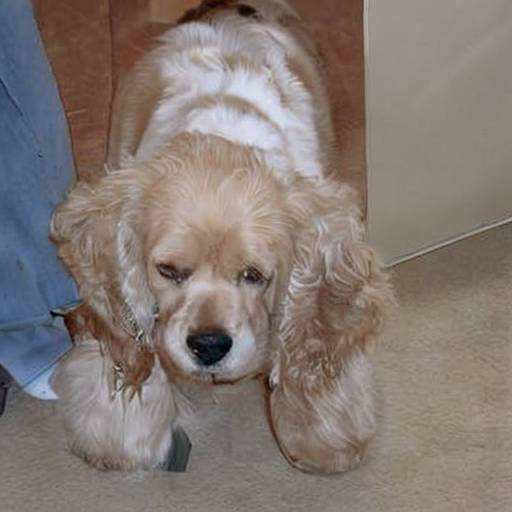}&%
        \includegraphics[width=0.095\linewidth]{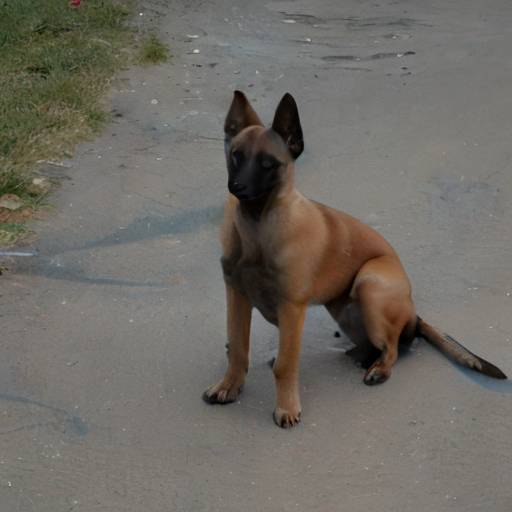}&%
        \includegraphics[width=0.095\linewidth]{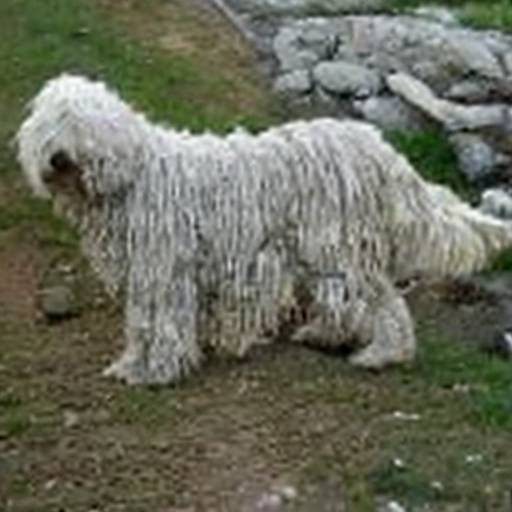}&%
        \includegraphics[width=0.095\linewidth]{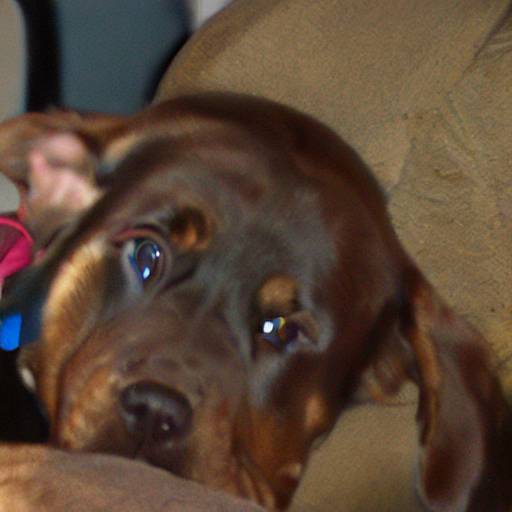}&%
        \includegraphics[width=0.095\linewidth]{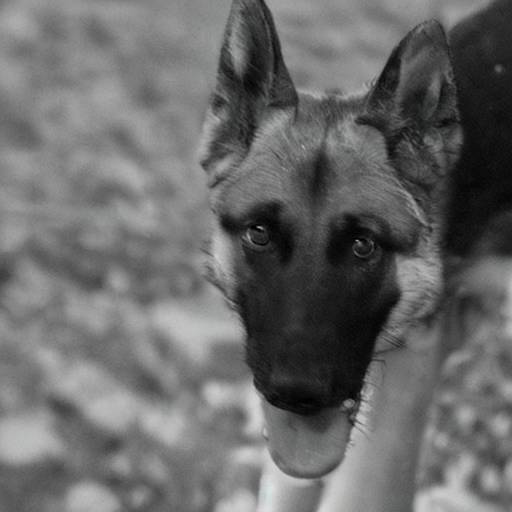}&%
        \includegraphics[width=0.095\linewidth]{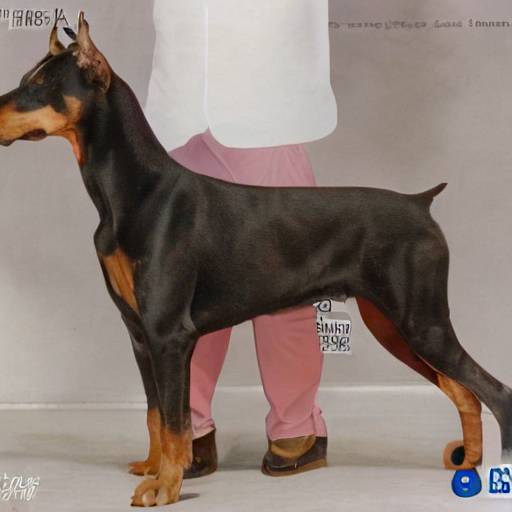}&%
        \includegraphics[width=0.095\linewidth]{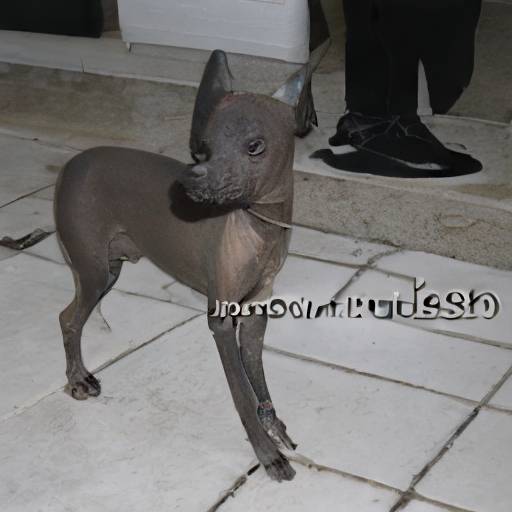}\\%
    \end{tabular}
    \caption{Nearest-neighbour grid comparison for ImageNet$_{\text{dogs}}$. Each row corresponds to a different initialization strategy; columns correspond to the queried training classes.}
    \label{fig:photos:dogs}
\end{figure}
\vfill

\clearpage
\section*{NeurIPS Paper Checklist}

\begin{enumerate}

\item {\bf Claims}
    \item[] Question: Do the main claims made in the abstract and introduction accurately reflect the paper's contributions and scope?
    \item[] Answer: \answerYes{}
    \item[] Justification: The abstract and introduction state three contributions: (i) a KL convergence result for VE-SGMs without any assumption on the data distribution, (ii) a non-asymptotic excess-risk bound for normalizing flows under mild conditions on the flow class, and (iii) a unified two-stage pipeline validated on synthetic heavy-tailed targets. Each claim corresponds to a specific theorem or experimental section, and the scope (synthetic targets in low dimension, with preliminary image experiments in the appendix) is stated explicitly.
    \item[] Guidelines:
    \begin{itemize}
        \item The answer \answerNA{} means that the abstract and introduction do not include the claims made in the paper.
        \item The abstract and/or introduction should clearly state the claims made, including the contributions made in the paper and important assumptions and limitations. A \answerNo{} or \answerNA{} answer to this question will not be perceived well by the reviewers.
    \end{itemize}

\item {\bf Limitations}
    \item[] Question: Does the paper discuss the limitations of the work performed by the authors?
    \item[] Answer: \answerYes{}
    \item[] Justification: Section 5 explicitly identifies three open directions: (i) learnability of the denoiser under heavy-tailed targets beyond the short-horizon regime, (ii) design of normalizing flows tailored to smoothed heavy-tailed distributions, and (iii) theoretical understanding of flow training on convolved distributions. The synthetic, low-dimensional nature of the experimental validation is also acknowledged.
    \item[] Guidelines:
    \begin{itemize}
        \item The answer \answerNA{} means that the paper has no limitation while the answer \answerNo{} means that the paper has limitations, but those are not discussed in the paper.
        \item The authors are encouraged to create a separate "Limitations" section in their paper.
        \item The paper should point out any strong assumptions and how robust the results are to violations of these assumptions.
        \item The authors should reflect on the scope of the claims made, e.g., if the approach was only tested on a few datasets or with a few runs.
        \item The authors should reflect on the factors that influence the performance of the approach.
        \item If applicable, the authors should discuss possible limitations of their approach to address problems of privacy and fairness.
    \end{itemize}

\item {\bf Theory assumptions and proofs}
    \item[] Question: For each theoretical result, does the paper provide the full set of assumptions and a complete (and correct) proof?
    \item[] Answer: \answerYes{}
    \item[] Justification: Theorem 2.1 is stated under H1--H2 and proved in Section 6.3, with auxiliary lemmas (6.2--6.7) establishing forward/backward regularity, a generalized Fokker-Planck equation, and the martingale property of the score norm. Theorem 3.1 is stated under H3--H4 and proved in Section 7, with Lemmas 7.2--7.3 controlling the second moment of the loss and a uniform concentration inequality, and Lemmas 7.5--7.7 supporting the covering-number argument behind Theorem 7.4. All assumptions are stated explicitly in the main text.
    \item[] Guidelines:
    \begin{itemize}
        \item All the theorems, formulas, and proofs in the paper should be numbered and cross-referenced.
        \item All assumptions should be clearly stated or referenced in the statement of any theorems.
        \item The proofs can either appear in the main paper or the supplemental material, but if they appear in the supplemental material, the authors are encouraged to provide a short proof sketch to provide intuition.
        \item Inversely, any informal proof provided in the core of the paper should be complemented by formal proofs provided in appendix or supplemental material.
        \item Theorems and Lemmas that the proof relies upon should be properly referenced.
    \end{itemize}

\item {\bf Experimental result reproducibility}
    \item[] Question: Does the paper fully disclose all the information needed to reproduce the main experimental results?
    \item[] Answer: \answerYes{}
    \item[] Justification: Appendix Section 8 reports the full target specification (mixture components, covariance structure, weights), the EDM scheduler parameters ($\sigma_{\max}$, $\sigma_{\min}$, $\rho$, number of steps), the MCMC configuration selected via the diagnostic in Figure 5, the MLP denoiser architecture (Table 2), and the coupling-flow architecture and training scheme (Table 3). The TarFlow appendix experiments report all hyperparameters in Table 6. Source code is released.
    \item[] Guidelines:
    \begin{itemize}
        \item The answer \answerNA{} means that the paper does not include experiments.
        \item If the paper includes experiments, a \answerNo{} answer to this question will not be perceived well by the reviewers: making the paper reproducible is important, regardless of whether the code and data are provided or not.
        \item If the contribution is a dataset and/or model, the authors should describe the steps taken to make their dataset or model reproducible (e.g., with respect to the data, code, and model).
        \item Depending on the contribution, reproducibility can be accomplished in various ways. For example, if the contribution is a novel architecture, describing the architecture fully might suffice, or if the contribution is a specific model and empirical evaluation, it may be necessary to either make it possible for others to replicate the model with the same dataset, or provide access to the model. In general, releasing code and data is often one good way to accomplish this, but reproducibility can also be provided via detailed instructions for how to replicate the results, access to a hosted model (e.g., in the case of a large language model), releasing of a model checkpoint, or other means that are appropriate to the research performed.
    \end{itemize}

\item {\bf Open access to data and code}
    \item[] Question: Does the paper provide open access to data and code?
    \item[] Answer: \answerYes{}
    \item[] Justification: The training and evaluation code is publicly released with reproduction instructions. All datasets used (FFHQ-64 and the two ImageNet-512 subsets in the appendix) are publicly available and are cited in the appendix; the synthetic targets are fully specified in Section 4 and Section 8.1.
    \item[] Guidelines:
    \begin{itemize}
        \item The answer \answerNA{} means that paper does not include experiments requiring code.
        \item Please see the NeurIPS code and data submission guidelines (\url{https://nips.cc/public/guides/CodeSubmissionPolicy}) for more details.
        \item While we encourage the release of code and data, we understand that this might not be possible, so "No" is an acceptable answer. Papers cannot be rejected simply for not including code, unless this is central to the contribution (e.g., for a new open-source benchmark).
        \item The instructions should contain the exact command and environment needed to run to reproduce the results. See the NeurIPS code and data submission guidelines (\url{https://nips.cc/public/guides/CodeSubmissionPolicy}) for more details.
        \item The authors should provide instructions on data access and preparation, including how to access the raw data, preprocessed data, intermediate data, and generated data, etc.
        \item The authors should provide scripts to reproduce all experimental results for the new proposed method and baselines. If only a subset of experiments are reproducible, they should state which ones are omitted from the script and why.
        \item At submission time, to preserve anonymity, the authors should release anonymized versions (if applicable).
        \item Providing as much information as possible in supplemental material (appended to the paper) is recommended, but including URLs to data and code is permitted.
    \end{itemize}

\item {\bf Experimental setting/details}
    \item[] Question: Are all training and test details specified?
    \item[] Answer: \answerYes{}
    \item[] Justification: Optimizer (Adam), learning rate, batch size, number of epochs, network width and depth, activation functions, normalization layers, and noise schedules are reported in Tables 2, 3, and 6 of the appendix. The evaluation protocol (number of generated samples, number of repetitions, projection slices for SWD/MSW, quantile levels for tail evaluation) is specified in Section 4 and Section 8.4.

\item {\bf Experiment statistical significance}
    \item[] Question: Are results accompanied by appropriate statistical information?
    \item[] Answer: \answerYes{}
    \item[] Justification: All MSW evaluations are repeated multiple times with independent reference draws, and results are reported as mean $\pm$ standard deviation. Tail quantile plots include reference bands at $\pm 2$ standard deviations.
    \item[] Guidelines:
    \begin{itemize}
        \item The answer \answerNA{} means that the paper does not include experiments.
        \item The authors should answer "Yes" if the results are accompanied by error bars, confidence intervals, or statistical significance tests, at least for the experiments that support the main claims of the paper.
        \item The factors of variability that the error bars are capturing should be clearly stated (for example, train/test split, initialization, random drawing of some parameter, or overall run with given experimental conditions).
        \item The method for calculating the error bars should be explained (closed form formula, call to a library function, bootstrap, etc.)
        \item The assumptions made should be given (e.g., Normally distributed errors).
        \item It should be clear whether the error bar is the standard deviation or the standard error of the mean.
        \item It is OK to report 1-sigma error bars, but one should state it. The authors should preferably report a 2-sigma error bar than state that they have a 96\% CI, if the hypothesis of Normality of errors is not verified.
        \item For asymmetric distributions, the authors should be careful not to show in tables or figures symmetric error bars that would yield results that are out of range (e.g. negative error rates).
        \item If error bars are reported in tables or plots, The authors should explain in the text how they were calculated and reference the corresponding figures or tables in the text.
    \end{itemize}

\item {\bf Experiments compute resources}
    \item[] Question: Are compute resources sufficiently described?
    \item[] Answer: \answerYes{}
    \item[] Answer: \answerYes{} Justification: We provide the following breakdown of compute resources used in the 
experiments. TarFlow training was performed on a single NVIDIA A100 GPU, 
with training times on the order of a few hours depending on the number of 
epochs. Sample generation and evaluation for TarFlow were performed on NVIDIA 
A100 and H100 GPUs: each generation run required between 10 and 20 hours, and 
each global evaluation required a comparable 10–20 hours. MCMC-based score 
estimation was performed on a single NVIDIA H100 GPU; generating $10^6$ samples 
with NUTS in dimension $d=2$ required approximately 10 hours per configuration. 
All remaining experiments, including the toy experiments on the GMM and HT targets 
reported in Section 4, were run on a single NVIDIA RTX A5000 GPU, with negligible 
runtimes per configuration (globally they can take some hours).
\begin{itemize}
        \item The answer \answerNA{} means that the paper does not include experiments.
        \item The paper should indicate the type of compute workers CPU or GPU, internal cluster, or cloud provider, including relevant memory and storage.
        \item The paper should provide the amount of compute required for each of the individual experimental runs as well as estimate the total compute.
        \item The paper should disclose whether the full research project required more compute than the experiments reported in the paper (e.g., preliminary or failed experiments that didn't make it into the paper).
    \end{itemize}

\item {\bf Code of ethics}
    \item[] Question: Does the research conform to the NeurIPS Code of Ethics?
    \item[] Answer: \answerYes{}
    \item[] Justification: The work is theoretical and methodological, uses publicly available datasets (FFHQ, ImageNet) consistent with their standard research use, and raises no specific ethical concerns. The authors have reviewed the NeurIPS Code of Ethics and confirm compliance.
    \item[] Guidelines:
    \begin{itemize}
        \item The answer \answerNA{} means that the authors have not reviewed the NeurIPS Code of Ethics.
        \item If the authors answer \answerNo{}, they should explain the special circumstances that require a deviation from the Code of Ethics.
        \item The authors should make sure to preserve anonymity (e.g., if there is a special consideration due to laws or regulations in their jurisdiction).
    \end{itemize}

\item {\bf Broader impacts}
    \item[] Question: Does the paper discuss societal impacts?
    \item[] Answer: \answerYes{}
    \item[] Justification: The introduction motivates heavy-tailed generative modeling through applications in climate science, insurance, and finance, where extreme-event modeling has direct societal relevance. The work is foundational and theoretical, so concrete downstream impacts are deferred to future application-specific studies.
    \item[] Guidelines:
    \begin{itemize}
        \item The answer \answerNA{} means that there is no societal impact of the work performed.
        \item If the authors answer \answerNA{} or \answerNo{}, they should explain why their work has no societal impact or why the paper does not address societal impact.
        \item Examples of negative societal impacts include potential malicious or unintended uses (e.g., disinformation, generating fake profiles, surveillance), fairness considerations (e.g., deployment of technologies that could make decisions that unfairly impact specific groups), privacy considerations, and security considerations.
        \item The conference expects that many papers will be foundational research and not tied to particular applications, let alone deployments. However, if there is a direct path to any negative applications, the authors should point it out.
        \item The authors should consider possible harms that could arise when the technology is being used as intended and functioning correctly, harms that could arise when the technology is being used as intended but gives incorrect results, and harms following from (intentional or unintentional) misuse of the technology.
        \item If there are negative societal impacts, the authors could also discuss possible mitigation strategies (e.g., gated release of models, providing defenses in addition to attacks, mechanisms for monitoring misuse, mechanisms to monitor how a system learns from feedback over time, improving the efficiency and accessibility of ML).
    \end{itemize}

\item {\bf Safeguards}
    \item[] Question: Are safeguards described for high-risk releases?
    \item[] Answer: \answerNA{}
    \item[] Justification: The paper does not release pretrained generative models or datasets that pose a misuse risk.

\item {\bf Licenses for existing assets}
    \item[] Question: Are external assets properly credited and licensed?
    \item[] Answer: \answerYes{}
    \item[] Justification: All external assets used --- the EDM denoisers of Karras et al. (2022, 2024), the Stable VAE encoder, the TarFlow architecture of Zhai et al. (2025), and the FFHQ and ImageNet datasets --- are cited in the main text or appendix, and used in accordance with their respective licenses for non-commercial research purposes.

\item {\bf New assets}
    \item[] Question: Are new assets documented?
    \item[] Answer: \answerNA{}
    \item[] Justification: The paper does not introduce new datasets, pretrained models, or benchmark assets intended for public release beyond the reproduction code.
    \item[] Guidelines:
    \begin{itemize}
        \item The answer \answerNA{} means that the paper does not release new assets.
        \item Researchers should communicate the details of the dataset/code/model as part of their submissions via structured templates. This includes details about training, license, limitations, etc.
        \item The paper should discuss whether and how consent was obtained from people whose asset is used.
        \item At submission time, remember to anonymize your assets (if applicable). You can either create an anonymized URL or include an anonymized zip file.
    \end{itemize}

\item {\bf Crowdsourcing and human subjects}
    \item[] Question: Does the paper include human subject experiments?
    \item[] Answer: \answerNA{}
    \item[] Justification: The paper does not involve crowdsourcing or human subject research.

\item {\bf IRB approvals}
    \item[] Question: Are IRB approvals required and described?
    \item[] Answer: \answerNA{}
    \item[] Justification: The paper does not involve human subjects research and no IRB approval is required.

\item {\bf Declaration of LLM usage}
    \item[] Question: Is LLM usage declared when relevant?
    \item[] Answer: \answerYes{}
    \item[] Justification: LLMs were used as auxiliary tools for code scaffolding, debugging, and language editing of the manuscript. 
    \item[] Guidelines:
    \begin{itemize}
        \item The answer \answerNA{} means that the core method development in this research does not involve LLMs as any important, original, or non-standard components.
        \item Please refer to our LLM policy (\url{https://neurips.cc/Conferences/2026/LLM}) for what should or should not be described.
    \end{itemize}

\end{enumerate}

\end{document}